\newcommand*{\ICML}{}

\newcommand*{\CAMREADY}{}

\ifdefined\ARXIV
	\documentclass[10pt]{article}
	\usepackage{libertine} 
	\usepackage{fullpage}
	
	\usepackage{natbib}
	
	\renewcommand{\cite}[1]{\citep{#1}}
	\setcitestyle{authoryear,round,citesep={;},aysep={,},yysep={;}}
	
	\usepackage[utf8]{inputenc} 
	\usepackage[T1]{fontenc}    
	\usepackage{booktabs}       
	\usepackage{multirow}
	\usepackage{microtype}      
	\usepackage{xcolor}         
	
	\usepackage{tabularx}
	\usepackage{thmtools}
	\usepackage{thm-restate}
	\usepackage[left=1.25in,right=1.25in,top=1in]{geometry}
	
	\usepackage{comment}
	\usepackage{hyperref}
	\definecolor{mydarkblue}{rgb}{0,0.08,0.5}
	\hypersetup{ %
		pdftitle={},
		pdfauthor={},
		pdfsubject={},
		pdfkeywords={},
		colorlinks=true,
		linkcolor=mydarkblue,
		citecolor=mydarkblue,
		filecolor=mydarkblue,
		urlcolor=mydarkblue
	}
	
	\skip\footins 9pt plus 4pt minus 2pt
	\def\footnoterule{\kern-3pt \hrule width 12pc \kern 2.6pt }
	\leftmargini\leftmargin \leftmarginii 2em
	\renewenvironment{abstract}%
	{%
		\vskip 0in%
		\centerline%
		{\large\bf Abstract}%
		\vspace{-1ex}%
		\begin{quote}%
		}
		{
			\par%
		\end{quote}%
		\vskip 0ex%
	}
	
	\title{\vskip -4ex \bf{Implicit Bias of Policy Gradient in Linear Quadratic Control: Extrapolation to Unseen Initial States} \vskip 0ex}
	
	\author{
		Author 1,\, Author 2,\, Author 3\\[1.5mm]
		{
			\fontsize{11}{11}\selectfont
			\textit{Tel Aviv University}
		}
	}
	\date{}
\fi

\ifdefined\NEURIPS
	\PassOptionsToPackage{numbers}{natbib}
	\documentclass{article}
	\ifdefined\CAMREADY
		\usepackage[final]{neurips_2020}
	\else
		\usepackage{neurips_2020}
	\fi
	\usepackage{footmisc}
	\usepackage[utf8]{inputenc} 
	\usepackage[T1]{fontenc}    
	\usepackage{hyperref}       	
	\usepackage{url}            	
	\usepackage{booktabs}    
	\usepackage{multirow}
	\usepackage{amsfonts}       
	\usepackage{nicefrac}       	
	\usepackage{microtype}      
\fi

\ifdefined\CVPR
	\documentclass[10pt,twocolumn,letterpaper]{article}
	\usepackage{footmisc}
	\usepackage{cvpr}
	\usepackage{times}
	\usepackage{epsfig}
	\usepackage{graphicx}
	\usepackage{amsmath}
	\usepackage{amssymb}
	\usepackage[square,comma,numbers]{natbib}
\fi
\ifdefined\AISTATS
	\documentclass[twoside]{article}
	\ifdefined\CAMREADY
		\usepackage[accepted]{aistats2015}
	\else
		\usepackage{aistats2015}
	\fi
	\usepackage{url}
	\usepackage[round,authoryear]{natbib}
\fi
\ifdefined\ICML
	\documentclass{article}
	\usepackage{footmisc}
	\usepackage{microtype}
	\usepackage{graphicx}
	\usepackage{booktabs}
	\usepackage{multirow}
	\usepackage{hyperref}
	
	\ifdefined\CAMREADY
		\usepackage[accepted]{icml2024_arxiv}
	\else
		\usepackage{icml2024}
	\fi
\fi
\ifdefined\ICLR
	\documentclass{article}
	\usepackage{iclr2019_conference,times}
	\usepackage{footmisc}
	\usepackage{hyperref}
	\usepackage{url}

	\ifdefined\CAMREADY
		\iclrfinalcopy
	\fi
\fi
\ifdefined\COLT
	\ifdefined\CAMREADY
		\documentclass{colt2017}
	\else
		\documentclass[anon,12pt]{colt2017}
	\fi
\	usepackage{times}
\fi

\usepackage{color}
\usepackage{amsfonts}
\usepackage{algorithm}
\usepackage{algorithmic}
\usepackage{graphicx}
\usepackage{nicefrac}
\usepackage{bbm}
\usepackage{wrapfig}
\usepackage{tikz}
\usepackage[titletoc,title]{appendix}
\usepackage{stmaryrd}
\usepackage{comment}
\usepackage{endnotes}
\usepackage{hyperendnotes}
\usepackage{enumerate}
\usepackage{enumitem}
\usepackage[normalem]{ulem}

\usepackage{titlesec}
\ifdefined\COLT
\else
	\usepackage{amsmath,amsthm,amssymb,mathtools}
\fi
\usepackage[small]{caption}
\usepackage[caption = false]{subfig}

\usepackage[extdef=true]{delimset}
\usepackage[capitalize,noabbrev]{cleveref}
\usepackage[most,listings]{tcolorbox}

\ifdefined\COLT
	\newtheorem{claim}[theorem]{Claim}
	\newtheorem{fact}[theorem]{Fact}
	\newtheorem{procedure}{Procedure}
	\newtheorem{conjecture}{Conjecture}	
	\newtheorem{hypothesis}{Hypothesis}	
	\newcommand{\qed}{\hfill\ensuremath{\blacksquare}}
\else
	\newtheorem{lemma}{Lemma}
	\newtheorem{corollary}{Corollary}
	\newtheorem{theorem}{Theorem}
	\newtheorem*{theorem*}{Theorem}
	\newtheorem{proposition}{Proposition}
	\newtheorem*{proposition*}{Proposition}

	\newtheorem{remark}{Remark}

	\theoremstyle{definition}
	\newtheorem{definition}{Definition}
\fi

\definecolor{green}{rgb}{0.0, 0.5, 0.0}

\definecolor{xcolor-gray}{gray}{0.95}

\def\be{\begin{equation}}
	\def\ee{\end{equation}}
\def\beas{\begin{eqnarray*}}
	\def\eeas{\end{eqnarray*}}
\def\bea{\begin{eqnarray}}
	\def\eea{\end{eqnarray}}

\newcommand{\xbf}{{\mathbf x}}
\newcommand{\ybf}{{\mathbf y}}
\newcommand{\zbf}{{\mathbf z}}
\newcommand{\ubf}{{\mathbf u}}
\newcommand{\vbf}{{\mathbf v}}

\newcommand{\wbf}{{\mathbf w}}

\newcommand{\ebf}{{\mathbf e}}

\newcommand{\1}{{\mathbf 1}}
\newcommand{\0}{{\mathbf 0}}
\newcommand{\Abf}{{\mathbf A}}
\newcommand{\Bbf}{{\mathbf B}}
\newcommand{\Cbf}{{\mathbf C}}

\newcommand{\Ibf}{{\mathbf I}}

\newcommand{\Pbf}{{\mathbf P}}

\newcommand{\Xbf}{{\mathbf X}}
\newcommand{\Ybf}{{\mathbf Y}}
\newcommand{\Qbf}{{\mathbf Q}}
\newcommand{\Rbf}{{\mathbf R}}
\newcommand{\Ubf}{{\mathbf U}}
\newcommand{\Vbf}{{\mathbf V}}
\newcommand{\Kbf}{{\mathbf K}}

\newcommand{\Mbf}{{\mathbf M}}
\newcommand{\Sigmabf}{{\mathbf \Sigma}}
\newcommand{\Zbf}{{\mathbf Z}}

\newcommand{\B}{{\mathcal B}}

\newcommand{\E}{{\mathcal E}}

\newcommand{\G}{{\mathcal G}}

\newcommand{\KK}{{\mathcal K}}
\newcommand{\M}{{\mathcal M}}
\newcommand{\NN}{{\mathcal N}}

\newcommand{\RR}{{\mathcal R}}
\renewcommand{\S}{{\mathcal S}}

\newcommand{\U}{{\mathcal U}}

\newcommand{\X}{{\mathcal X}}

\newcommand{\I}{{\mathcal I}}

\newcommand{\EE}{\mathop{\mathbb E}} 
\newcommand{\R}{{\mathbb R}}
\newcommand{\N}{{\mathbb N}}

\newcommand{\inprod}[2]  {\left\langle{#1},{#2}\right\rangle}
\newcommand{\inprodbig}[2]  {\big \langle{#1},{#2} \big \rangle}
\newcommand{\inprodBig}[2]  {\Big \langle{#1},{#2} \Big \rangle}

\DeclareMathOperator*{\argmin}{argmin}

\newcommand{\MT}{\mathcal{T}}

\DeclareMathOperator*{\Var}{Var}
\DeclareMathOperator*{\Cov}{Cov}
\DeclareMathOperator{\Tr}{Tr}
\newcommand{\Trnormalized}{\widebar{\mathrm{Tr}}}

\newcommand{\cost}{J}
\newcommand{\costhorizonone}{J_1}
\newcommand{\excesscost}{\E_{\mathrm{cost}}}
\newcommand{\optmes}{\E_{\mathrm{opt}}}
\newcommand{\Kpg}{\Kbf_{\mathrm{pg}}}
\newcommand{\Koptall}{\Kbf_{\mathrm{ext}}}
\newcommand{\Kminnorm}{\Kbf_{\mathrm{no\text{-}ext}}}

\newcommand{\trainstates}{\S}
\newcommand{\orthstates}{\U}
\newcommand{\pgstates}{\X_{\mathrm{{pg}}}}
\newcommand{\lspan}{\mathrm{span}}
\newcommand{\Ashift}{\Abf_{\mathrm{shift}}}

\DeclareFontFamily{U}{mathx}{\hyphenchar\font45}
\DeclareFontShape{U}{mathx}{m}{n}{<-> mathx10}{}
\DeclareSymbolFont{mathx}{U}{mathx}{m}{n}
\DeclareMathAccent{\widebar}{0}{mathx}{"73}

\definecolor{darkspringgreen}{rgb}{0.09, 0.45, 0.27}
\usetikzlibrary{decorations.pathreplacing,calligraphy}
\usetikzlibrary{patterns}

\ifdefined\ENABLEENDNOTES
	\renewcommand{\theendnote}{\alph{endnote}} 
\else
	\renewcommand{\endnote}[1]{\null} 
\fi

\ifdefined\CVPR
	\usepackage[breaklinks=true,bookmarks=false]{hyperref}
	\ifdefined\CAMREADY
		\cvprfinalcopy 
	\fi
	
\fi

\ifdefined\ARXIV
	\newcommand*{\ABBR}{}
\fi
\ifdefined\ICML
	\newcommand*{\ABBR}{}
\fi
\ifdefined\NEURIPS
	\newcommand*{\ABBR}{}
\fi
\ifdefined\ICLR
	\newcommand*{\ABBR}{}
\fi
\ifdefined\COLT
	\newcommand*{\ABBR}{}
\fi
\ifdefined\ABBR
	\newcommand{\eg}{{\it e.g.}}
	\newcommand{\ie}{{\it i.e.}}
	\newcommand{\cf}{{\it cf.}}

\fi

\begin{document}
	
	
	\ifdefined\ARXIV
		\maketitle
	\fi
	\ifdefined\NEURIPS
	\title{Paper Title}
		\author{
			Author 1 \\
			Author 1 Institution \\	
			\texttt{author1@email} \\
			\And
			Author 1 \\
			Author 1 Institution \\	
			\texttt{author1@email} \\
		}
		\maketitle
	\fi
	\ifdefined\CVPR
		\title{Paper Title}
		\author{
			Author 1 \\
			Author 1 Institution \\	
			\texttt{author1@email} \\
			\and
			Author 2 \\
			Author 2 Institution \\
			\texttt{author2@email} \\	
			\and
			Author 3 \\
			Author 3 Institution \\
			\texttt{author3@email} \\
		}
		\maketitle
	\fi
	\ifdefined\AISTATS
		\twocolumn[
		\aistatstitle{Paper Title}
		\ifdefined\CAMREADY
			\aistatsauthor{Author 1 \And Author 2 \And Author 3}
			\aistatsaddress{Author 1 Institution \And Author 2 Institution \And Author 3 Institution}
		\else
			\aistatsauthor{Anonymous Author 1 \And Anonymous Author 2 \And Anonymous Author 3}
			\aistatsaddress{Unknown Institution 1 \And Unknown Institution 2 \And Unknown Institution 3}
		\fi
		]	
	\fi
	\ifdefined\ICML
		\icmltitlerunning{Implicit Bias of Policy Gradient in Linear Quadratic Control: Extrapolation to Unseen Initial States}
		\twocolumn[
		\icmltitle{Implicit Bias of Policy Gradient in Linear Quadratic Control: \\ Extrapolation to Unseen Initial States} 
		\icmlsetsymbol{equal}{*}
		\begin{icmlauthorlist}
			\icmlauthor{Noam Razin}{tau,equal} 
			\icmlauthor{Yotam Alexander}{tau,equal}
			\icmlauthor{Edo Cohen-Karlik}{tau}
			\icmlauthor{Raja Giryes}{tau}
			\icmlauthor{Amir Globerson}{tau,google}
			\icmlauthor{Nadav Cohen}{tau}
		\end{icmlauthorlist}
		\icmlaffiliation{tau}{Tel Aviv University}
		\icmlaffiliation{google}{Google}
		\icmlcorrespondingauthor{Noam Razin}{noamrazin@mail.tau.ac.il}
		\icmlcorrespondingauthor{Yotam Alexander}{yotama@mail.tau.ac.il}
		\icmlkeywords{Implicit Bias, Policy Gradient, Extrapolation, Control Theory, Linear Quadratic Regulator}
		\vskip 0.3in
		]
		\printAffiliationsAndNotice{\icmlEqualContribution} 
	\fi
	\ifdefined\ICLR
		\title{Paper Title}
		\author{
			Author 1 \\
			Author 1 Institution \\
			\texttt{author1@email}
			\And
			Author 2 \\
			Author 2 Institution \\
			\texttt{author2@email}
			\And
			Author 3 \\ 
			Author 3 Institution \\
			\texttt{author3@email}
		}
		\maketitle
	\fi
	\ifdefined\COLT
		\title{Paper Title}
		\coltauthor{
			\Name{Author 1} \Email{author1@email} \\
			\addr Author 1 Institution
			\And
			\Name{Author 2} \Email{author2@email} \\
			\addr Author 2 Institution
			\And
			\Name{Author 3} \Email{author3@email} \\
			\addr Author 3 Institution}
		\maketitle
	\fi

	\begin{abstract}
In modern machine learning, models can often fit training data in numerous ways, some of which perform well on unseen (test) data, while others do not.
Remarkably, in such cases gradient descent frequently exhibits an \emph{implicit bias} that leads to excellent performance on unseen data.
This implicit bias was extensively studied in supervised learning, but is far less understood in optimal control (reinforcement learning).
There, learning a controller applied to a system via gradient descent is known as \emph{policy gradient}, and a question of prime importance is the extent to which a learned controller \emph{extrapolates to unseen initial states}.
This paper theoretically studies the implicit bias of policy gradient in terms of extrapolation to unseen initial states.
Focusing on the fundamental \emph{Linear Quadratic Regulator} (\emph{LQR}) problem, we establish that the extent of extrapolation depends on the degree of exploration induced by the system when commencing from initial states included in training.
Experiments corroborate our theory, and demonstrate its conclusions on problems beyond LQR, where systems are non-linear and controllers are neural networks.
We hypothesize that real-world optimal control may be greatly improved by developing methods for informed selection of initial states to train on.
\end{abstract}

	\ifdefined\COLT
		\medskip
	\fi

	
	\section{Introduction}
\label{sec:intro}

The ability to generalize from training data to unseen test data is a core aspect of machine learning.
Broadly speaking, there are two types of generalization one may hope for:
\emph{(i)}~in-distribution generalization, where test data is drawn from the same distribution as training data;
and
\emph{(ii)}~out-of-distribution generalization, also known as \emph{extrapolation}, where test data is drawn from a different distribution than that of the training data.
In modern regimes the training objective is often \emph{underdetermined}~---~\ie~it admits multiple solutions (parameter assignments fitting training data) that differ in their performance on test data~---~and the extent to which a learned solution generalizes is determined by an \emph{implicit bias} of the training algorithm~\cite{neyshabur2017implicit,vardi2023implicit}.
Remarkably, variants of gradient descent frequently converge to solutions with excellent in-distribution generalization~\cite{zhang2017understanding}, which in some cases extends to out-of-distribution generalization~\cite{miller2021accuracy}.
The implicit bias of gradient descent has accordingly attracted vast theoretical interest, with existing analyses focusing primarily on the basic framework of supervised learning (see, \eg,~\citet{neyshabur2014search,gunasekar2017implicit,soudry2018implicit,arora2019implicit,ji2019gradient,ji2019implicit,woodworth2020kernel,razin2020implicit,lyu2020gradient,lyu2021gradient,pesme2021implicit,razin2021implicit,razin2022implicit,frei2023double,frei2023benign,andriushchenko2023sgd,abbe2023generalization}).

As opposed to supervised learning, little is known about the implicit bias of gradient descent in the challenging framework of \emph{optimal control} (see overview of related work in \cref{sec:related}).
In optimal control~---~which in its broadest form is equivalent to reinforcement learning~---~the goal is to learn a \emph{controller} (also known as policy) that will steer a given \emph{system} (also known as environment) such that a given \emph{cost} is minimized (or equivalently, a given reward is maximized)~\cite{sontag2013mathematical}.
Algorithms that learn a controller by directly parameterizing it and setting its parameters through gradient descent are known as \emph{policy gradient} methods.
For implementing such methods, gradients with respect to controller parameters are either estimated via sampling~\cite{williams1992simple}, or, in cases where differentiable forms for the system and cost are at hand, the gradients may be computed through analytic differentiation (see, \eg,~\citet{hu2019chainqueen,qiao2020scalable,clavera2020model,mora2021pods,gillen2022leveraging,howell2022dojo,xu2022accelerated,wiedemann2023training}).

An issue of prime importance in optimal control (and reinforcement learning) is the extent to which a learned controller extrapolates to initial states unseen in training.
Indeed, in real-world settings training is often limited to few initial states, and a deployed controller is likely to encounter initial states that go well beyond what it has seen in training~\cite{zhu2020ingredients,dulac2021challenges}.
The ability of the controller to handle such initial states is imperative, particularly in safety-critical applications (\eg~robotics, industrial manufacturing, or autonomous driving).

The current paper seeks to take first steps towards theoretically addressing the following question.
\begin{center}
	\textbf{\textit{To what extent does the implicit bias of policy gradient lead to extrapolation to initial states unseen in training?}}
\end{center}
As a testbed for theoretical study, we consider the fundamental \emph{Linear Quadratic Regulator} (\emph{LQR}) problem~\cite{anderson2007optimal}.
There, systems are linear, costs are quadratic, and it is known that optimal controllers are linear~\cite{anderson2007optimal}.
Learning linear controllers in LQR via policy gradient has been the subject of various theoretical analyses (\eg,~\citet{fazel2018global,malik2019derivative,bhandari2019global,mohammadi2019global,mohammadi2021convergence,bu2019lqr,bu2020policy,jin2020analysis,gravell2020learning,hambly2021policy,hu2023toward}).
However, these analyses do not treat underdetermined training objectives, thus leave open the question of how implicit bias affects extrapolation.
To facilitate its study, we focus on LQR training objectives that are underdetermined.

Our theoretical analysis reveals that in underdetermined LQR problems, the extent to which linear controllers learned via policy gradient extrapolate to initial states unseen in training, depends on the interplay between the system and the initial states that were seen in training.
In particular, it depends on the degree of \emph{exploration} induced by the system when commencing from initial states seen in training.
We prove that if this exploration is insufficient, extrapolation does not take place.
On the other hand, we construct a setting that encourages exploration, and show that under it, extrapolation can be perfect.
We then consider a typical setting, \ie~one in which systems are generated randomly and initial states seen in training are arbitrary.
In this setting, we prove that the degree of exploration suffices for there to be non-trivial extrapolation~---~in expectation, and with high probability if the state space dimension is sufficiently large.

Two attributes of our analysis may be of independent interest.
First, are advanced tools we employ from the intersection of random matrix theory and topology.
Second, is a result by which the implicit bias of policy gradient over a linear controller does not minimize the Euclidean norm, in stark contrast to the implicit bias of gradient descent over linear predictors in supervised learning (\cf~\citet{zhang2017understanding}).

We corroborate our theory through experiments, demonstrating that the interplay between a linear system and initial states seen in training can lead a linear controller (learned via policy gradient) to extrapolate to initial states unseen in training.
Moreover, we show empirically that the phenomenon extends to non-linear systems and (non-linear) neural network controllers.

In real-world optimal control (and reinforcement learning), contemporary learning algorithms often extrapolate poorly to initial states unseen in training~\cite{rajeswaran2017towards,zhang2018study,zhang2019dissection,fujimoto2019off,witty2021measuring}.
Our results lead us to believe that this extrapolation may be greatly improved by developing methods for informed selection of initial states to train on.
We hope that our work will encourage research along this line. 

\textbf{Paper organization.}
The remainder of the paper is organized as follows.
\cref{sec:related} reviews related work.
\cref{sec:prelim} establishes preliminaries.
\cref{sec:analysis} delivers our theoretical analysis~---~a characterization of extrapolation to unseen initial states for linear controllers trained via policy gradient in underdetermined LQR problems.
\cref{sec:experiments} presents experiments with the analyzed LQR problems, as well as with non-linear systems and (non-linear) neural network controllers.
Lastly, \cref{sec:conclusion}~concludes.

	\section{Related Work}
\label{sec:related}

Most theoretical analyses of the implicit bias of gradient descent (or variants thereof) focus on the basic framework of supervised learning.
Such analyses traditionally aim to establish in-distribution generalization, or to characterize solutions found in training without explicit reference to test data (see, \eg,~\citet{neyshabur2014search,gunasekar2017implicit,soudry2018implicit,arora2019implicit,ji2019gradient,ji2019implicit,woodworth2020kernel,razin2020implicit,lyu2020gradient,lyu2021gradient,azulay2021implicit,pesme2021implicit,razin2021implicit,razin2022implicit,andriushchenko2023sgd,frei2023double,frei2023benign,marcotte2023abide,chou2023more,chou2024gradient}).
Among recent analyses are also ones centering on out-of-distribution generalization, \ie~on extrapolation (see~\citet{xu2021neural,abbe2022learning,abbe2023generalization,cohen2022implicit,cohen2023learning,zhou2023algorithms}).
These are motivated by the fact that in many real-world scenarios, training and test data are drawn from different distributions~\cite{shen2021towards}.
Our work is similar in that it also centers on extrapolation, and is also motivated by real-world scenarios (see \cref{sec:intro}).
It differs~in that it studies the challenging framework of optimal~control.

In optimal control, theoretical analyses of the implicit bias of gradient descent are relatively scarce.
For reinforcement learning, which in a broad sense is equivalent to optimal control:
\citet{hu2022actor} characterized a tendency towards high-entropy solutions with softmax parameterized policies;
and
\citet{kumar2021implicit,kumar2022dr3} revealed detrimental effects of implicit bias with value-based methods.
More relevant to our work are~\citet{zhang2020policy,zhang2021derivative,zhao2023data}, which for different LQR problems, establish that policy gradient implicitly enforces certain constraints on the parameters of a controller throughout optimization.
These analyses, however, do not treat underdetermined training objectives (they pertain to settings where there is a unique controller minimizing the training objective), thus leave open our question on the effect of implicit bias on extrapolation.
To the best of our knowledge, the current paper provides the first analysis of the implicit bias of policy gradient for underdetermined LQR problems.
Moreover, for optimal control problems in general, it provides the first analysis of the extent to which the implicit bias of policy gradient leads to extrapolation to initial states unseen in training.

Aside from \citet{zhang2020policy,zhang2021derivative,zhao2023data}, existing theoretical analyses of policy gradient for LQR problems largely fall into two categories.
First, are those proving convergence rates to the minimal cost, typically under assumptions ensuring a unique solution \cite{fazel2018global,malik2019derivative,bhandari2019global,mohammadi2019global,mohammadi2021convergence,bu2019lqr,bu2020policy,jin2020analysis,gravell2020learning,hambly2021policy,hu2023toward}.
Second, are those establishing sub-linear regret in online learning \cite{cohen2018online,agarwal2019online,agarwal2019logarithmic,cassel2021online,hazan2022introduction,chen2023regret}.
Both lines of work do not address the topic of implicit bias, which we focus on.

Finally, an empirical observation related to our work is that in real-world optimal control (and reinforcement learning), contemporary learning algorithms often extrapolate poorly to initial states unseen in training~\cite{rajeswaran2017towards,zhang2018study,zhang2019dissection,fujimoto2019off,witty2021measuring}.
This observation motivated our work, and our results suggest approaches for alleviating the limitation it reveals (see \cref{sec:conclusion}).

	\section{Preliminaries}
\label{sec:prelim}

\textbf{Notation.}
We use $\norm{\cdot}$ to denote the Euclidean norm of a vector or matrix,
$[N]$ to denote the set $\{ 1, \ldots, N \}$, where $N \in \N$, and $\%$ to denote the modulo operator.
We let $\ebf_1, \ldots, \ebf_D \in \R^D$ be the standard basis vectors.
Lastly, the subspace orthogonal to $\X \subset \R^D$ is denoted by $\X^\perp$, \ie~$\X^\perp := \brk[c]{ \vbf \in \R^D : \vbf \perp \xbf ~ , ~ \forall \xbf \in \X }$.

\subsection{Policy Gradient in Linear Quadratic Control}
\label{sec:prelim:policy_grad_lqr}

We consider a linear system, in which an initial state $\xbf_0 \in \R^D$ evolves according to:
\[
\xbf_{h} = \Abf \xbf_{h - 1} + \Bbf \ubf_{h - 1} ~~,~ \forall h \in \N \text{\,,}
\]
where $\Abf \in \R^{D \times D}$ and~$\Bbf^{D \times M}$ are matrices that define the system, and $\ubf_{h - 1} \in \R^M$ is the control at time ${h  - 1}$.
An LQR problem of horizon $H \in \N \cup \{ \infty\}$ over this system amounts to searching for controls $\ubf_0 , \ldots , \ubf_H$ that minimize the following quadratic cost:
\[
\sum\nolimits_{h = 0}^H \xbf_h^\top \Qbf \xbf_h + \ubf_h^\top \Rbf \ubf_h
\text{\,,}
\]
where $\Qbf \in \R^{D \times D}$ and $\Rbf \in \R^{M \times M}$ are positive semidefinite matrices that define the cost.
We focus on the practical case where the horizon $H$ is finite (extending our analysis to the asymptotic case $H = \infty$ is left for future work).
It is known that in the LQR problem, optimal controls are attained by a (state-feedback) linear controller~\cite{anderson2007optimal}, \ie~by setting each control $\ubf_h$ to be a certain linear function of the corresponding state~$\xbf_h$.
Accordingly, and in line with prior work (\eg,~\citet{fazel2018global,bu2019lqr,malik2019derivative}), we consider learning a linear controller parameterized by $\Kbf \in \R^{M \times D}$, which at time $h \in \{ 0 \} \cup [H]$ assigns the control $\ubf_h = \Kbf \xbf_h$.\footnote{
Since the horizon~$H$ is finite, in general, attaining optimal controls may require the linear controller to be time-varying, \ie~to implement different linear mappings at different times (\cf~\citet{anderson2007optimal}).
However, as detailed in~\cref{sec:prelim:underdetermined}, our analysis will consider settings in which a time-invariant linear controller suffices.
}
The cost attained by~$\Kbf$ with respect to a finite set $\X \subset \R^D$ of initial states is:
\be
\cost (\Kbf; \X) \! := \!
\frac{1}{\abs{ \X} } \! \sum\nolimits_{\xbf_0 \in \X} \sum\nolimits_{h = 0}^H \xbf_h^\top \Qbf \xbf_h + \xbf_h^\top \Kbf^\top \Rbf \Kbf \xbf_h
\text{\,,}
\label{eq:general_cost}
\ee
where, for each $\xbf_0 \in \X$, the states $\xbf_1 , \ldots , \xbf_H$ satisfy:\footnote{
As customary, we omit from the notation of $\xbf_1, \ldots, \xbf_H$ the dependence on $\xbf_0$ and $\Kbf$.
\label{foot:state_evol_notation}
}
\be
\xbf_{h} = \Abf \xbf_{h - 1} + \Bbf \Kbf \xbf_{h - 1} = (\Abf + \Bbf \Kbf)^h \xbf_0
\text{\,.}
\label{eq:state_evol}
\ee

Given a (finite) set $\trainstates \subset \R^D$ of initial states seen in training, the controller~$\Kbf$ is learned by minimizing the \emph{training cost}~$\cost (\cdot\, ; \trainstates)$.
Learning via policy gradient amounts to iteratively updating the controller as follows:
\be
\Kbf^{(t + 1)} = \Kbf^{(t)} - \eta \cdot \nabla \cost \brk{ \Kbf^{(t)} ; \trainstates } ~~ , ~ \forall t \in \N
\text{\,,}
\label{eq:policy_grad}
\ee
where $\eta > 0$ is a predetermined learning rate, and we assume throughout that $\Kbf^{(1)} = \0$.

\subsection{Underdetermined Linear Quadratic Control}
\label{sec:prelim:underdetermined}

Existing analyses of policy gradient for learning linear controllers in LQR (\eg,~\citet{fazel2018global,malik2019derivative,bu2019lqr,bu2020policy,mohammadi2019global,mohammadi2021convergence,bhandari2019global,hambly2021policy}) typically assume that $\Rbf$ is positive definite, meaning that controls are regularized, and that the set~$\S$ of initial states seen in training spans~$\R^D$ (or similarly, when training over a distribution of initial states, that the support of the distribution spans~$\R^D$).
Under these assumptions, the training cost $\cost ( \cdot\, ; \trainstates)$ is not underdetermined~---~it entails a single global minimizer, which produces optimal controls from any initial state.
Thus, our question on the effect of implicit bias on extrapolation (see \cref{sec:intro}) is not applicable.

To facilitate a study of the foregoing question, we focus on underdetermined problems (ones in which the training cost entails multiple global minimizers), obtained through the following assumptions:
\emph{(i)}~$\Rbf = \0$, meaning that controls are unregularized;\footnote{
This assumption is necessary, in the sense that without it, even if assumptions \emph{(ii)} and \emph{(iii)} hold, the training cost may not be underdetermined, \ie~it may entail a single global minimizer.
See \cref{app:determined_R_neq_zero} for details.
}
\emph{(ii)}~$M = D$ and $\Bbf \in \R^{D \times D}$ has full rank, implying that the controller's ability to affect the state is not limited;
and 
\emph{(iii)}~the set~$\S$ of initial states seen in training does not span~$\R^D$ (note that, except for the trivial case of $\S = \{ \0\}$,  this implies $D \geq 2$).
For conciseness, in the main text we fix $\Qbf$ to be an identity matrix, and assume that $\Bbf$ is an orthogonal matrix.
Extensions of our results to more general $\Qbf$ and $\Bbf$ are discussed throughout.

In our setting of interest, the cost attained by a controller $\Kbf$ with respect to an arbitrary (finite) set $\X \subset \R^D$ of initial states (\cref{eq:general_cost}) simplifies to:
\be
\cost (\Kbf; \X) =  \frac{1}{\abs{\X}} \hspace{-0.5mm}
\sum\nolimits_{\xbf_0 \in \X} \hspace{-0.25mm} \sum\nolimits_{h = 0}^H \norm*{ (\Abf + \Bbf \Kbf)^h \xbf_0 }^2
\text{,}
\label{eq:cost}
\ee
with the global minimum of this cost being:
\[
\cost^* (\X) := \min\nolimits_{\Kbf \in \R^{D \times D}} \cost (\Kbf; \X) = \frac{1}{ \abs{\X} } \sum\nolimits_{\xbf_0 \in \X} \norm{ \xbf_0 }^2
\text{\,.}
\]
A controller $\Kbf \in \R^{D \times D}$ attains this global minimum if and only if $\Kbf \xbf_0 = - \Bbf^{-1} \Abf \xbf_0$ for all $\xbf_0 \in \X$, or equivalently:
\be
\norm*{ (\Abf + \Bbf \Kbf) \xbf_0 }^2  = 0
~~ , ~ \forall \xbf_0 \in \X
\text{\,,}
\label{eq:optimal_cond}
\ee
\ie~every $\xbf_0 \in \X$ is mapped to zero by the state dynamics that $\Kbf$ induces (see \cref{app:proofs:cost_min_underdetermined} for step-by-step derivations of $\cost^* (\X)$ and the optimality condition in \cref{eq:optimal_cond}).

To see that in our setting the training cost $\cost ( \cdot\, ; \trainstates)$ is indeed underdetermined, notice that, since the set~$\trainstates$ of initial states seen in training does not span~$\R^D$, there exist infinitely many controllers~$\Kbf$ satisfying:
\be
\Kbf \xbf_0 = - \Bbf^{-1} \Abf \xbf_0 ~~,~ \forall \xbf_0 \in \trainstates
\text{\,.}
\label{eq:seen_controls_opt}
\ee
That is, there are infinitely many controllers minimizing the training cost.
We denote by~$\KK_\trainstates$ the (infinite) set comprising these controllers,~\ie:
\be
\KK_\trainstates := \brk[c]*{ \Kbf \in \R^{D \times D} : \cost ( \Kbf ; \trainstates) = \cost^* ( \trainstates ) }
\text{\,.}
\label{eq:train_cost_minimizers}
\ee

\textbf{Significance of underdetermined LQR.}
The main purpose of our underdetermined LQR setting is to serve as a testbed for theoretical study of implicit bias in optimal control, analogously to how underdetermined linear prediction serves as an important testbed for theoretical study of implicit bias in supervised learning (\eg,~\citet{soudry2018implicit,bartlett2020benign,shamir2022implicit}).  
We note however that our setting is also practically motivated.
See \cref{app:underdetermined_lqr_significance} for details.

\subsection{Quantifying Extrapolation}
\label{sec:prelim:extrapolation}

Let $\orthstates$ be an (arbitrary) orthonormal basis for~$\trainstates^\perp$ (subspace orthogonal to the initial states seen in training).
A controller~$\Kbf$ is fully determined by the controls it assigns to states in $\trainstates$ and~$\orthstates$.
The controllers in~$\KK_\trainstates$ (\cref{eq:train_cost_minimizers}), \ie~the controllers minimizing the training cost, all satisfy \cref{eq:seen_controls_opt}, and in particular agree on the controls they assign to states in~$\trainstates$.
However, they differ arbitrarily in the controls they assign to states in~$\orthstates$.
The performance of a controller~$\Kbf$ on states in~$\orthstates$ will quantify extrapolation of~$\Kbf$ to initial states unseen in training.
Two measures will facilitate this quantification.
The first, referred to as the \emph{optimality measure}, is based on the optimality condition in \cref{eq:optimal_cond}.
Namely, it measures extrapolation by how close $\norm{ (\Abf + \Bbf \Kbf) \xbf_0 }^2$ is to zero for every $\xbf_0 \in \orthstates$.
\begin{definition}
\label{def:opt_measure}
The \emph{optimality measure} of extrapolation for a controller~$\Kbf^{D \times D}$ is:
\[
\optmes ( \Kbf ) := \frac{1}{ \abs{ \orthstates } } \sum\nolimits_{\xbf_0 \in \orthstates} \norm*{ (\Abf + \Bbf \Kbf) \xbf_0 }^2
\text{\,.}
\]
\end{definition}
The second measure of extrapolation, referred to as the \emph{cost measure}, is the suboptimality of the cost attained by $\Kbf$ with respect to~$\orthstates$.
\begin{definition}
\label{def:cost_measure}
The \emph{cost measure} of extrapolation for a controller~$\Kbf^{D \times D}$ is:
\[
\excesscost ( \Kbf ) := \cost ( \Kbf ; \orthstates ) - \cost^* (\orthstates)
\text{\,.}
\]
\end{definition}
The optimality and cost measures are complementary: the former disentangles the impact of a controller on initial states from its impact on subsequent states, whereas the latter considers the impact on both initial states and subsequent states in a trajectory.
Both measures are non-negative, with lower values indicating better extrapolation.
Their minimal value is zero, and the unique member of~$\KK_\trainstates$ attaining this value is the perfectly extrapolating controller $\Koptall \in \KK_\trainstates$ defined by:
\be
\Koptall \xbf_0 = - \Bbf^{-1} \Abf \xbf_0 ~~,~ \forall \xbf_0 \in \orthstates
\text{\,.}
\label{eq:unseen_controls_opt}
\ee
More generally, an arbitrary controller has zero optimality measure if and only if it has zero cost measure.
We note that, as shown in \cref{app:extrapolation_measures_invariant}, the optimality and cost measures are both invariant to the choice of~$\orthstates$, hence we do not include it in their notation.

Throughout our analysis, we shall consider as a baseline the controller $\Kminnorm \in \KK_\trainstates$ defined by:
\be
\Kminnorm \xbf_0 = \0 ~~,~ \forall \xbf_0 \in \orthstates
\text{\,,}
\label{eq:unseen_controls_zero}
\ee
\ie~the controller which minimizes the training cost while assigning null controls to states in~$\orthstates$.\footnote{
	If one allows for a baseline that does not minimize the training cost, then the initial controller $\Kbf^{(1)} = \0$ is also a sensible choice.
	Our theoretical results (\cref{sec:analysis}) hold as stated when replacing $\Kminnorm$ with $\Kbf^{(1)}$ as the baseline.
}
Aside from degenerate cases, the optimality and cost measures of~$\Kminnorm$ are both positive.\footnote{%
	The optimality measure of~$\Kminnorm$ is zero if and only if $\Abf \xbf_0 = \0$ for all $\xbf_0 \in \U$, \ie~if and only if the zero controller minimizes the measure as well.
	An identical statement holds for the cost measure.
}
When quantifying extrapolation for a controller $\Kpg$ learned via policy gradient (\cref{eq:policy_grad}), we will compare its optimality and cost measures to those~of~$\Kminnorm$.
Namely, we will examine the ratios $\optmes (\Kpg) / \optmes (\Kminnorm)$ and $\excesscost (\Kpg) / \excesscost (\Kminnorm)$, where a value of one corresponds to trivial (no) extrapolation and a value of zero corresponds to perfect extrapolation.

	\begin{figure*}[t]
	\vspace{1.5mm}
	\begin{center}
		\includegraphics[width=1\textwidth]{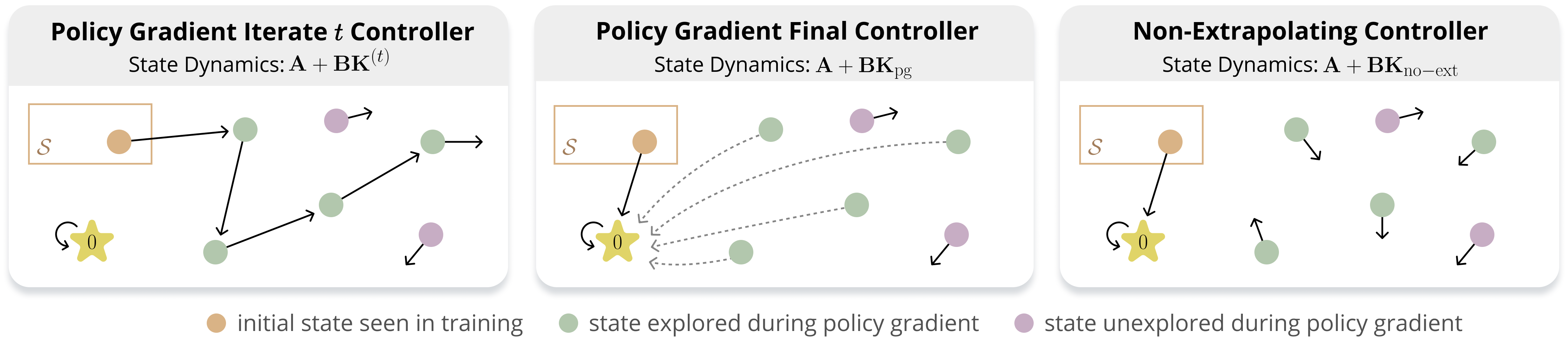}
	\end{center}
	\vspace{-1.5mm}
	\caption{
		Intuition behind our theoretical analysis: in underdetermined LQR problems (\cref{sec:prelim:underdetermined}), the extent to which a controller learned via policy gradient extrapolates to initial states unseen in training, depends on the degree of exploration induced by the system when commencing from initial states that were seen in training.
		Illustrated are the state dynamics induced by the $t$'th iterate of policy gradient $\Kbf^{(t)}$ (left), by the final policy gradient controller $\Kpg$ (middle), and by the non-extrapolating controller $\Kminnorm$ defined in \cref{sec:prelim:extrapolation} (right).
		An arbitrary controller $\Kbf$ extrapolates to initial states unseen in training if $\norm{ (\Abf + \Bbf \Kbf) \xbf}^2$ is small for $\xbf \in \trainstates^\perp$, \ie~if the dynamics induced by $\Kbf$ send towards zero states that are orthogonal to the set~$\trainstates$ of initial states seen in training (see \cref{sec:prelim:extrapolation}).
		Due to the structure of training cost gradients, the dynamics induced by the final policy gradient controller $\Kpg$ send towards zero every state encountered in training.
		Accordingly, the extent to which $\Kpg$ extrapolates depends on the degree of exploration~---~the overlap of states encountered in training with directions orthogonal to~$\trainstates$.
		On the other hand, the controller $\Kminnorm$ ensures that states in $\trainstates$ are sent to zero (thereby minimizing the training cost), but does not handle states in~$\trainstates^\perp$.
		It thus does not extrapolate.
	}
	\label{fig:intuition}
\end{figure*}

\section{Analysis of Implicit Bias}
\label{sec:analysis}

This section theoretically analyzes the extent to which policy gradient leads linear controllers in underdetermined LQR problems to extrapolate to initial states unseen in training.

\subsection{Intuition: Extrapolation Depends on Exploration}
\label{sec:analysis:intuition}

Our analysis will reveal that the extent of extrapolation to initial states unseen in training depends on the degree of exploration induced by the system when commencing from initial states that were seen in training.
Before going into the formal results, we provide intuition behind this dependence.

Per \cref{eq:cost,eq:state_evol}, the training cost attained by a controller~$\Kbf$ can be written as follows:
\[
	\cost (\Kbf ; \trainstates ) = c + \frac{1}{\abs{\trainstates}} \sum\nolimits_{\xbf_0 \in \trainstates} \sum\nolimits_{h = 1}^H \norm*{ (\Abf + \Bbf \Kbf) \xbf_{h - 1} }^2
	\text{\,,}
\]
where $c > 0$ does not depend on~$\Kbf$, and for each $\xbf_0 \in \trainstates$, the states $\xbf_1 , \ldots , \xbf_{H - 1}$ are produced by the system when commencing from~$\xbf_0$ and steered by~$\Kbf$.\textsuperscript{\ref{foot:state_evol_notation}}
Thus, minimizing the training cost amounts to finding a controller $\Kbf$ such that $\norm{(\Abf + \Bbf \Kbf) \xbf}^2 = 0$ for every $\xbf$ belonging to a trajectory emanating from $\trainstates$.
Notice that it is possible to do so by simply ensuring that $\norm{(\Abf + \Bbf \Kbf) \xbf}^2 = 0$ for every $\xbf \in \trainstates$.
This is because, if $(\Abf + \Bbf \Kbf) \xbf = \0$ for some $\xbf \in \trainstates$, then all subsequent states in a trajectory emanating from~$\xbf$ are~zero.

As discussed in \cref{sec:prelim:extrapolation}, for~$\Kbf$ to extrapolate to initial states unseen in training, we would like $\norm*{ (\Abf + \Bbf \Kbf) \xbf }^2$ to be small for every $\xbf \in \orthstates$, where $\orthstates$ is an orthonormal basis for~$\trainstates^\perp$ (subspace orthogonal to the set~$\trainstates$ of initial states seen in training).
We have seen in \cref{sec:prelim:extrapolation} that merely minimizing $\norm*{ (\Abf + \Bbf \Kbf) \xbf }^2$ for every $\xbf \in \trainstates$ implies nothing about the magnitude of this term for $\xbf \in \orthstates$.
In other words, it implies nothing about extrapolation.

Fortunately, it can be shown that the structure of $\nabla \cost ( \cdot \, ; \trainstates)$ is such that at every iteration~$t \in \N$ of policy gradient (\cref{eq:policy_grad}), the iterates $\Kbf^{( t )}$ and~$\Kbf^{( t + 1 )}$ tend to satisfy $\norm{(\Abf + \Bbf \Kbf^{( t + 1 )}) \xbf  \big\|^2 < \big\| (\Abf + \Bbf \Kbf^{( t )}) \xbf }^2$
for every state $\xbf$ along trajectories emanating from~$\trainstates$ that were encountered in the iteration, \ie~that have been steered by~$\Kbf^{( t )}$.
Consequently, with $\Kpg$ being a controller trained by policy gradient, \smash{$\norm*{ (\Abf + \Bbf \Kpg) \xbf }^2$} is relatively small for every state~$\xbf$ encountered in training (\ie~for every $\xbf$ belonging to a trajectory that was produced during training).
The extent to which $\Kpg$ extrapolates therefore depends on the degree of exploration~---~the overlap of states encountered in training with~$\orthstates$, \ie~with directions orthogonal to~$\trainstates$.

The above intuition is illustrated in \cref{fig:intuition}.
The remainder of the section is devoted to its formalization.

\subsection{Extrapolation Requires Exploration}
\label{sec:analysis:no_gen}

The current subsection proves that in the absence of sufficient exploration, extrapolation to initial states unseen in training does not take place.

Recall from \cref{sec:prelim} that we consider a linear system defined by matrices $\Abf$ and~$\Bbf$, a set~$\trainstates$ of initial states seen in training, and a linear controller learned via policy gradient, whose iterates are denoted by $\Kbf^{(t)}$ for $t \in \N$.
Let~$\pgstates$ be the set of states encountered in training.
More precisely, $\pgstates$~is the union over $\xbf_0 \in \trainstates$ and $t \in \N$, of the states in the length~$H$ trajectory emanating from~$\xbf_0$ and steered by~$\Kbf^{(t)}$:
\begin{align}
& \pgstates := \label{eq:pgstates} \\
& \brk[c]2{ \! (\Abf + \Bbf \Kbf^{(t)})^h \xbf_0 :  \xbf_0 \in \trainstates , h \in \{0\} \cup  [H - 1] , t \in \N } \nonumber
\text{.}
\end{align}

\cref{prop:no_exploration_no_extrapolation} below establishes that the learned controller can only extrapolate to initial states spanned by~$\pgstates$.
More precisely, for any $t \in \N$ and $\xbf \in \pgstates^\perp$, the controller $\Kbf^{(t)}$ assigns to~$\xbf$ a trivial control of zero.
This implies that if $\pgstates \subseteq \lspan (\trainstates)$, meaning no state outside $\lspan (\trainstates)$ is encountered in training, then the optimality extrapolation measure of~$\Kbf^{(t)}$ is trivial, \ie~equal to that of~$\Kminnorm$ (see \cref{sec:prelim:extrapolation}).
\cref{prop:no_exploration_no_extrapolation} further shows that such non-exploratory settings exist~---~there exist systems (matrices $\Abf$ and $\Bbf$) with which, for any choice of~$\trainstates$, it holds that $\pgstates \subseteq \lspan (\trainstates)$.
In the exemplified settings, similarly to the optimality extrapolation measure, the cost extrapolation measure of~$\Kbf^{(t)}$ is trivial (and greater than zero).

\begin{proposition}
\label{prop:no_exploration_no_extrapolation}
For any iteration $t \in \N$ of policy gradient, the following hold.
\begin{itemize}[leftmargin=1.5em]
	\vspace{-1.5mm}
	\item (Exploration is necessary for extrapolation) For any $\xbf \in \pgstates^\perp$ it holds that $\Kbf^{(t)} \xbf = \0$.
	Consequently, if $\pgstates \subseteq \lspan (\trainstates)$ then $\optmes (\Kbf^{(t)}) = \optmes ( \Kminnorm )$.
	
	\item (Existence of non-exploratory settings) There exist system matrices $\Abf$ and $\Bbf$ such that, for any set $\trainstates$ of initial states seen in training: $\pgstates \subseteq \lspan (\trainstates)$; and:
	\[
		\begin{split}
			\optmes (\Kbf^{(t)}) & = \optmes ( \Kminnorm ) = 1 \text{\,,} \\[0.3em]
			\excesscost (\Kbf^{(t)}) & = \excesscost ( \Kminnorm ) = H
			\text{\,,}
		\end{split}
	\]
	where we recall that $H$ is the horizon.
\end{itemize}
\end{proposition}

\begin{proof}[Proof sketch (proof in~\cref{app:proofs:no_exploration_no_extrapolation})]
We establish that for any $\Kbf \in \R^{D \times D}$, the rows of $\nabla \cost (\Kbf; \trainstates)$ are spanned by states in the trajectories that emanate from~$\trainstates$ and are steered by~$\Kbf$.
Since $\Kbf^{(1)} = \0$, it follows that for any $t \in \N$, the rows of~$\Kbf^{(t)}$ are spanned by~$\pgstates$, and so $\Kbf^{(t)} \xbf = \0$ for any $\xbf \in \pgstates^\perp$.
If $\pgstates \subseteq \lspan (\trainstates)$ then this immediately implies that the optimality measure attained by $\Kbf^{(t)}$ is equal to that of $\Kminnorm$.

As for existence of non-exploratory settings, suppose that $\Abf = \Bbf = \Ibf$, where $\Ibf$ is the identity matrix.
With arbitrary~$\trainstates$, we prove that $\pgstates \subseteq \lspan (\trainstates)$ by showing that the state dynamics induced by $\Abf, \Bbf,$ and $\Kbf^{(t)}$ are invariant to $\lspan (\trainstates)$, \ie~$(\Abf + \Bbf \Kbf^{(t)}) \xbf \in \lspan (\trainstates)$ if $\xbf \in \lspan (\trainstates)$.
Then, by the first part of the proposition, $\pgstates \subseteq \lspan (\trainstates)$ implies that $\Kbf^{(t)} \xbf = \0$ for any $\xbf \in \orthstates$.
The same is true for $\Kminnorm$.
Hence, $\Abf + \Bbf \Kbf^{(t)}$ and $\Abf + \Bbf \Kminnorm$ both map any $\xbf \in \orthstates$ back to itself.
Using this observation, the optimality and cost measures of extrapolation attained by $\Kbf^{(t)}$ and $\Kminnorm$ are readily computed.
\end{proof}

\begin{remark}
The first part of \cref{prop:no_exploration_no_extrapolation} (exploration is necessary for extrapolation) extends to the setting where $\Bbf \in \R^{D \times D}$ is arbitrary and $\Qbf \in \R^{D \times D}$ is any positive semidefinite matrix.
The proof in~\cref{app:proofs:no_exploration_no_extrapolation} accounts for this more general setting.
\end{remark}

\subsection{Extrapolation in Exploration-Inducing Setting}
\label{sec:analysis:shift}

\cref{sec:analysis:no_gen} proved that, in the absence of sufficient exploration, extrapolation to initial states unseen in training does not take place.
We now show that with sufficient exploration, extrapolation can take place.
Namely, we construct a system that encourages exploration when commencing from a given initial state, and show that with this system and initial state, training via policy gradient leads to extrapolation, which~---~depending on characteristics of the cost~---~varies between partial and perfect.

Suppose that we are given an initial state seen in training, which, without loss of generality, is the standard basis vector $\ebf_1 \in \R^D$.\footnote{
If the initial state seen in training is some non-zero vector~$\xbf_0$ that differs from~$\ebf_1$, then the system we will construct is to be modified by replacing~$\Abf$ with $\Pbf^{-1} \Abf \Pbf$, where $\Pbf \in \R^{D \times D}$ is some invertible matrix that maps $\xbf_0$ to~$\ebf_1$.
}
Assume for simplicity that the horizon~$H$ is divisible by the state space dimension~$D$.\footnote{%
Extension of the analysis in this subsection to arbitrary $H \geq 2$ is straightforward, but results in less concise expressions.
}
When commencing from~$\ebf_1$ and steered by the first iterate of policy gradient, \ie~by $\Kbf^{(1)} = \0$, the system produces the length~$H$ trajectory $\MT := ( \ebf_1, \Abf \ebf_1, \ldots, \Abf^{H - 1} \ebf_1 )$.
In light of \cref{sec:analysis:no_gen}, for encouraging exploration we would like the states in~$\MT$ to span the entire state space.
A simple choice that ensures this is $\Abf = \Ashift := \sum\nolimits_{d = 1}^{D} \ebf_{ d \% D + 1} \ebf_d^\top$.
Under this choice, $\MT$ cyclically traverses through the standard basis vectors, \ie~through $\ebf_1 , \ldots , \ebf_D$.

\cref{prop:shift} below establishes that, in the setting under consideration, the implicit bias of policy gradient leads to extrapolation.
Specifically, the learned controller attains optimality and cost measures of extrapolation that are substantially less than those of~$\Kminnorm$ (see \cref{sec:prelim:extrapolation}).
This phenomenon is more potent the longer the horizon $H$ is, with perfect extrapolation attained in the limit $H \to \infty$.

\begin{proposition}
\label{prop:shift}
Assume that $\trainstates = \brk[c]{ \ebf_1 }$, $\Abf = \Ashift$, and $H$ is divisible by~$D$.
Then, policy gradient with learning rate \smash{$\eta = \brk1{ H^2 / D + H }^{-1}$} converges to a controller~$\Kpg$ that:
\emph{(i)}~minimizes the training cost, \ie~$\cost (\Kpg ; \trainstates) = \cost^* (\trainstates)$;
and
\emph{(ii)}~satisfies:
\[
\begin{split}
\frac{ \optmes \brk1{ \Kpg } }{  \optmes \brk*{ \Kminnorm } } & \leq \frac{ 4(D - 1)^2 }{ (H + D)^2 }
\text{\,,} \\[0.5em]
\frac{ \excesscost (\Kpg) }{ \excesscost (\Kminnorm) } & \leq \frac{ 4(D - 1)^2 }{ (H + D)^2 }
\text{\,.}
\end{split}
\]
\end{proposition}

\begin{proof}[Proof sketch (proof in~\cref{app:proofs:shift})]
The analysis follows from first principles, building on a particularly lucid form that $\nabla \cost ( \Kbf^{(1)} ; \trainstates)$ takes.
Specifically, we derive an explicit expression for $\Kbf^{(2)}$, and show that it minimizes the training cost via the optimality condition of \cref{eq:optimal_cond}.
This implies that policy gradient converges to $\Kpg = \Kbf^{(2)}$.
Extrapolation in terms of the optimality and cost measures then follows from the derived expression for $\Kbf^{(2)}$.
\end{proof}

\begin{remark}
\cref{app:extension_q} generalizes \cref{prop:shift} to the setting where $\Qbf$ is any diagonal positive semidefinite matrix.
The generalized analysis sheds light on how $\Qbf$ impacts extrapolation. 
In particular, it shows that for certain values of $\Qbf$, extrapolation can be perfect even with a finite horizon~$H$.
\end{remark}

\subsubsection{Implicit Bias in Optimal Control $\neq$ Euclidean Norm Minimization}
\label{sec:analysis:shift:not_norm}

A widely known fact is that in supervised learning, when labels are continuous (regression) and the training objective is underdetermined, gradient descent over linear predictors implicitly minimizes the Euclidean norm.
That is, among all predictors minimizing the training objective, gradient descent converges to the one whose Euclidean norm is minimal (\cf~\citet{zhang2017understanding}).
A perhaps surprising implication of \cref{prop:shift}, formalized by \cref{lem:min_norm,cor:no_euc_norm_min} below, is that an analogous phenomenon does \emph{not} take place in optimal control.
In fact, among the controllers minimizing the training cost, the (unique) controller with minimal Euclidean norm is the non-extrapolating $\Kminnorm$.
Thus, the extrapolation guarantee of \cref{prop:shift} implies that policy gradient over a linear controller does not implicitly minimize the Euclidean norm.
This finding highlights that conventional wisdom regarding implicit bias in supervised learning cannot be blindly applied to optimal control.
We hope it will encourage further research dedicated to implicit bias in optimal control.\footnote{
	In particular, we do not exclude the possibility of policy gradient implicitly minimizing some complexity measure different from the Euclidean norm.
	We regard investigation of this prospect as an interesting avenue for future work.
}

\begin{lemma}
\label{lem:min_norm}
Of all controllers minimizing the training cost, \ie~all $\Kbf \in \KK_\trainstates$ (\cref{eq:train_cost_minimizers}), the non-extrapolating $\Kminnorm$ is the unique one with minimal Euclidean norm.
 \end{lemma}

\begin{proof}[Proof sketch (proof in~\cref{app:proofs:min_norm})]
Through the method of Lagrange multipliers, we show that if the rows of some $\Kbf \in \KK_\trainstates$ are in $\lspan (\trainstates)$, then $\Kbf$ is the unique member of~$\KK_\trainstates$ whose Euclidean norm is minimal.
We then show that the rows of~$\Kminnorm$ necessarily reside in $\lspan (\trainstates)$.
\end{proof}

\begin{corollary}
\label{cor:no_euc_norm_min}
In the setting of~\cref{prop:shift}, $\Kpg$~---~the controller to which policy gradient converges, and which minimizes the training cost~---~satisfies:
\[
\norm*{ \Kpg }^2 - \min_{\Kbf \in \KK_\trainstates} \norm{ \Kbf }^2 = \sum_{d = 2}^D \brk*{ 1 - \frac{ 2(d - 1) }{ H + D } }^2 = \Omega (D)
\text{\,.}
\]
\end{corollary}

\begin{proof}[Proof sketch (proof in~\cref{app:proofs:no_euc_norm_min})]
We derive an expression for $\Kminnorm$ to compute its squared Euclidean norm, which by \cref{lem:min_norm} is equal to $\min\nolimits_{ \Kbf \in \KK_\trainstates } \norm{ \Kbf }^2$.
Then, an expression for $\Kpg$, established as a lemma in the proof of \cref{prop:shift}, yields the desired result.
\end{proof}

\subsection{Extrapolation in Typical Setting}
\label{sec:analysis:general}

\cref{sec:analysis:no_gen,sec:analysis:shift} presented two ends of a spectrum.
On one end, \cref{sec:analysis:no_gen} proved that, in the absence of sufficient exploration, extrapolation to initial states unseen in training does not take place.
On the other end, \cref{sec:analysis:shift} constructed an exploration-inducing setting (namely, a system for a given initial state seen in training), and showed that it leads to extrapolation, which~---~depending on characteristics of the cost~---~varies between partial and perfect.
A natural question is what extrapolation may be expected in a typical setting.

We address the foregoing question by considering an arbitrary (non-zero) initial state seen in training $\xbf_0 \in \R^D$~---~which without loss of generality is assumed to have unit norm\footnote{
If~$\xbf_0$ does not have unit norm then the results we will establish are to be modified by introducing a multiplicative factor of~$\norm{ \xbf_0}^2$.
}
~---~and a randomly generated system matrix~$\Abf$.
For the randomness of~$\Abf$, we draw entries independently from a Gaussian distribution with mean zero and standard deviation~$1 / \sqrt{D}$.
This choice of standard deviation is common in the literature on random matrix theory~\cite{anderson2010introduction}, and ensures that with high probability, the spectral norm of~$\Abf$ is roughly constant, \ie~independent of the state space dimension~$D$ (\cf~Theorem 4.4.5 in \citet{vershynin2020high}).
When commencing from~$\xbf_0$ and steered by the first iterate of policy gradient, \ie~by $\Kbf^{(1)} = \0$, the system produces the length~$D$ trajectory $( \xbf_0 , \Abf \xbf_0 , \ldots, \Abf^{D - 1} \xbf_0 )$.
Since $\xbf_0$ is a cyclic vector of~$\Abf$ almost surely (see \cref{app:cyc_vector_generic} for a proof of this fact), the latter trajectory spans the entire state space almost surely.
The necessary condition for extrapolation put forth in \cref{sec:analysis:no_gen} is thus supported, implying that extrapolation could take place.

\cref{thm:typical_system} below establishes that a single iteration of policy gradient already leads~---~in expectation, and with high probability if the state space dimension is large~---~to non-trivial extrapolation, as quantified by the optimality measure.
The theorem overcomes considerable technical challenges (arising from the complexity of random systems) via advanced tools from the intersection of random matrix theory and topology.
These tools may be of independent interest.

\begin{theorem}
\label{thm:typical_system}
Let $\xbf_0 \in \R^D$ be an arbitrary unit vector.
Assume that the set~$\trainstates$ of initial states seen in training consists of~$\xbf_0$ (\ie~$\trainstates = \brk[c]{ \xbf_0 }$), that the entries of~$\Abf$ are drawn independently from a Gaussian distribution with mean zero and standard deviation $1 / \sqrt{D}$, and that the horizon $H$ is greater than one.
Then, with learning rate $\eta \leq \frac{1}{ 4 D H (H-1) (4H - 1)!!}$, where $N!! := N (N - 2) (N - 4) \cdots 3$ is the double factorial of an odd $N \in \N$, the second iterate of policy gradient, \ie~$\Kbf^{(2)}$, satisfies:
\[
\frac{ \EE_{\Abf} \brk[s]*{ \optmes \brk*{ \Kbf^{(2)} }  } }{ \EE_{\Abf} \brk[s]*{ \optmes \brk*{ \Kminnorm }  } } \leq 1 - \eta \cdot \frac{H (H - 1)}{D}
\text{\,,}
\]
where $\Kminnorm$ is the non-extrapolating controller defined in~\cref{sec:prelim:extrapolation}.
Moreover, for any $\delta \in (0, 1)$, if $D \geq \abs{\trainstates} + \frac{ 6 \abs{ \trainstates} H (H - 1) (4H - 1)!! }{ \delta }$ and $\eta \leq \frac{ 1 }{ 8 D^{2} H (H-1)  (4H - 1)!! }$, then with probability at least $1 - \delta$ over the choice of $\Abf$:
\[
\frac{ \optmes \brk*{ \Kbf^{(2)} }  }{ \optmes \brk*{ \Kminnorm} } \leq 1 - \eta \cdot \frac{ H (H-1) }{ 4D }
\text{\,.}
\]
Lastly, the above results hold even if we replace $\trainstates$ by an arbitrary set of orthonormal vectors.
\end{theorem}

\begin{proof}[Proof sketch (proof in~\cref{app:proofs:typical_system})]
The intuition behind the proof (valid for $H \geq D$) is as follows.
As stated in the discussion regarding exploration at the opening of this subsection, almost surely, the length~$D$ trajectory steered by the first iterate of policy gradient, \ie~by~$\Kbf^{(1)}$, spans the entire state space.
Therefore, almost surely, states encountered in training overlap with~$\orthstates$, \ie~with directions orthogonal to~$\trainstates$ (see \cref{sec:prelim:extrapolation}).
The intuitive arguments in \cref{sec:analysis:intuition} thus suggest that extrapolation will take place.

Converting the above intuition into a formal proof entails considerable technical challenges.
We address these challenges by employing advanced tools from the intersection of random matrix theory and topology.
Specifically, we employ a method from~\citet{redelmeier2014real} for computing expectations of traces of random matrix products, through the topological concept of \emph{genus expansion}.
For the convenience of the reader, a detailed outline of the proof is provided in \cref{app:proofs:typical_system}.
\end{proof}

\textbf{Limitations.}
Despite overcoming considerable technical challenges, \cref{thm:typical_system} remains limited in several ways:
\emph{(i)}~the requirements from the learning rate~$\eta$, and the requirement from the state space dimension~$D$ in the second (high probability) result, depend on $(4H - 1)!!$, which grows super exponentially with the horizon~$H$; 
\emph{(ii)}~extrapolation guarantees are provided only for the first iteration of policy gradient; \emph{(iii)}~in contrast to the analyses of~\cref{sec:analysis:no_gen,sec:analysis:shift}, extrapolation results apply only to the optimality measure, not to the cost measure; and \emph{(iv)} only Gaussian transition matrices are considered.
Experiments reported in \cref{sec:experiments:lqr} suggest that the limitations above can be alleviated.
Doing so is regarded as a valuable direction for future~work.

	\begin{figure*}[t]
	\vspace{1.5mm}
	\begin{center}
		\includegraphics[width=1\textwidth]{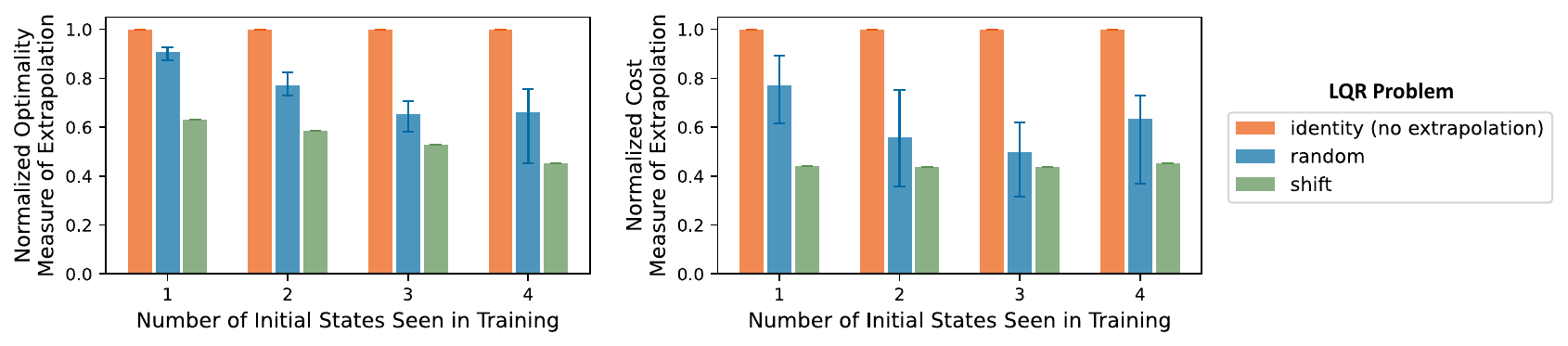}
	\end{center}
	\vspace{-2mm}
	\caption{
		In underdetermined LQR problems (\cref{sec:prelim:underdetermined}), the extent to which linear controllers learned via policy gradient extrapolate to initial states unseen in training, depends on the degree of exploration that the system induces from initial states that were seen in training.
		We evaluated LQR problems with state space dimension $D = 5$, horizon $H = 5$ (further experiments with larger $D$ and $H$ are reported in \cref{app:experiments:lqr}), and three different linear systems: \emph{(i)} an “identity'' system with $\Abf = \Ibf \in \R^{D \times D}$ (analyzed in \cref{sec:analysis:no_gen}); \emph{(ii)} a “shift'' system with \smash{$\Abf = \sum\nolimits_{d = 1}^D \ebf_{d \% D + 1} \ebf_d^\top$} (analyzed in \cref{sec:analysis:shift}); and \emph{(iii)} a random system, where the entries of $\Abf$ are sampled independently from a zero-mean Gaussian with standard deviation \smash{$1 / \sqrt{D}$} (analyzed in \cref{sec:analysis:general}).
		Reported are the optimality (\cref{def:opt_measure}) and cost (\cref{def:cost_measure}) measures of extrapolation, normalized by the respective quantities attained by the non-extrapolating controller $\Kminnorm$ (see \cref{sec:prelim:extrapolation}).
		A value of one corresponds to trivial (no) extrapolation and a value of zero corresponds to perfect extrapolation. 
		Bar heights stand for median values over $20$ runs differing in random seed, and error bars span the interquartile range ($25$'th to $75$'th percentiles).
		\textbf{Results:}
		In agreement with our theory: \emph{(i)} no extrapolation takes place under the “identity'' system, which does not induce exploration from initial states seen in training; while \emph{(ii)} substantial extrapolation is achieved under the “shift'' and random systems, which induce exploration.
		The extrapolation under “shift'' and random systems is not perfect, and this is also in agreement with our theory.  
		Note that our theory does not explain why random systems often (but not always) lead to less extrapolation than the “shift'' system.  
		Refining our analysis to explain this intricacy is an interesting direction for future work.
	}
	\label{fig:lqr_experiments_main}
\end{figure*}

\section{Experiments}
\label{sec:experiments}

In this section, we corroborate our theory (\cref{sec:analysis}) via experiments, demonstrating how the interplay between a system and initial states seen in training affects the extent to which a controller learned via policy gradient extrapolates to initial states unseen in training.
\cref{sec:experiments:lqr} presents experiments with the analyzed underdetermined LQR problems.
\cref{sec:experiments:nonlinear} considers non-linear systems and (non-linear) neural network controllers.
For conciseness, we defer some experiments and implementation details to \cref{app:experiments}.
\ifdefined\CAMREADY
Code for reproducing our experiments is available at \url{https://github.com/noamrazin/imp_bias_control}.
\fi

\subsection{Linear Quadratic Control}
\label{sec:experiments:lqr}

Our theoretical analysis considered underdetermined LQR problems in three settings, respectively comprising: \emph{(i)} systems that do not induce exploration from any initial state (\cref{sec:analysis:no_gen}); \emph{(ii)} systems with a “shift'' transition matrix~$\Abf$, which encourage exploration from certain initial states (\cref{sec:analysis:shift}); and \emph{(iii)} systems with a randomly generated transition matrix~$\Abf$, which admit exploration from any initial state (\cref{sec:analysis:general}).
According to our analysis, with systems that do not induce exploration from initial states seen in training, controllers trained via policy gradient do not extrapolate.
On the other hand, non-trivial extrapolation occurs under “shift'' and random systems.
\cref{fig:lqr_experiments_main} demonstrates these findings empirically, showcasing the relation between the system and extrapolation to initial states unseen in training.
\cref{fig:lqr_experiments_h8,fig:lqr_experiments_d40,fig:lqr_experiments_rnd_B_Q} in \cref{app:experiments:lqr} provide additional experiments in settings with, respectively: \emph{(i)} a longer time horizon; \emph{(ii)} a larger state space dimension; and \emph{(iii)} random $\Bbf$ and $\Qbf$ matrices.

\begin{figure*}[t!]
	\vspace{2mm}
	\begin{center}
		\includegraphics[width=0.98\textwidth]{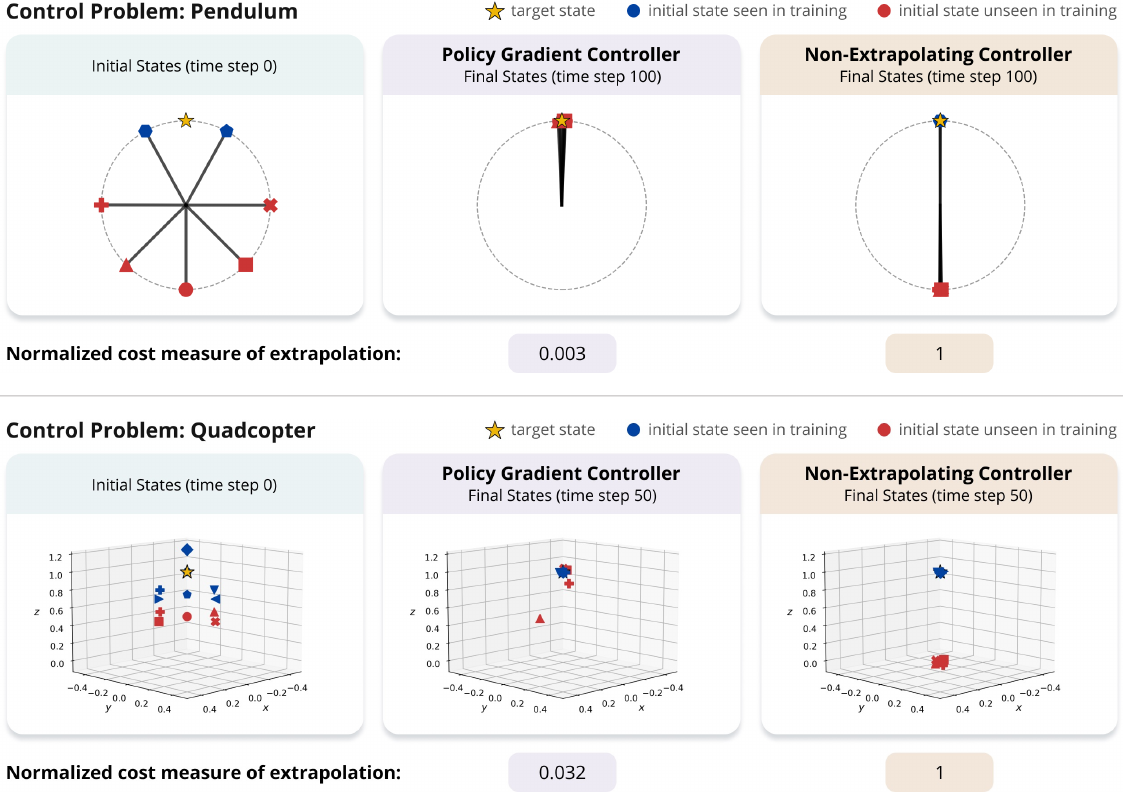}
	\end{center}
	\vspace{-1mm}
	\caption{
		In the pendulum and quadcopter control problems (see \cref{sec:experiments:nonlinear}), training a (non-linear) neural network controller via policy gradient often leads to a solution that extrapolates to initial states unseen in training, despite the existence of non-extrapolating solutions.
		\textbf{Left:} Initial states seen in training (blue) and initial states unseen in training that are used for evaluating extrapolation (red).
		\textbf{Middle:}
		Final states of trajectories emanating from initial states on the left, where the trajectories are steered by a (state-feedback) controller learned via policy gradient.
		The controller is parameterized as a fully-connected neural network with ReLU activation.
		\textbf{Right:}
		Final states of trajectories emanating from initial states on the left, where the trajectories are steered by a non-extrapolating controller, \ie~a controller that minimizes the cost for initial states seen in training while performing poorly on initial states unseen in training.
		We obtained such a controller by modifying the training objective to encourage steering unseen initial states to a state different than the target state.
		\textbf{Results:}
		Since an uncontrolled pendulum or quadcopter falls downwards from a given initial state, the systems qualitatively induce exploration of states with lower height.
		Complying with our theory for LQR problems (\cref{sec:analysis}), policy gradient yields near-perfect extrapolation to unseen initial states lower that those used for training.
		In particular, the cost measure of extrapolation, normalized by that attained by the non-extrapolating controller, is near the minimal value of zero (a value of one stands for no extrapolation).
		\textbf{Further details in \cref{app:experiments}:} 
		\cref{table:pend_experiments_states,table:quad_below_experiments_states} fully specify the initial and final states depicted above, and \cref{fig:pend_experiments_states_through_time,fig:quad_below_experiments_states_through_time} present the evolution of states through time under the policy gradient and non-extrapolating controllers.
	}
	\label{fig:pend_and_quad_experiments_main}
\end{figure*}

\subsection{Non-Linear Systems and Neural Network Controllers}
\label{sec:experiments:nonlinear}

The LQR problem is of central theoretical and practical importance in optimal control~\cite{anderson2007optimal}.
For example, it supports controlling \emph{non-linear systems} via iterative linearizations~\cite{li2004iterative}.
An alternative approach to controlling non-linear systems is to train (non-linear) \emph{neural network controllers} via policy gradient.
This approach is largely motivated by the success of neural networks in supervised learning, and has gained significant interest in recent years (see, \eg,~\citet{hu2019chainqueen,qiao2020scalable,clavera2020model,mora2021pods,gillen2022leveraging,howell2022dojo,xu2022accelerated,wiedemann2023training}).

Our analysis of underdetermined LQR problems (\cref{sec:analysis}) implies that, when a linear system induces exploration from initial states seen in training, a linear controller trained via policy gradient typically extrapolates to initial states unseen in training.
The current subsection empirically demonstrates that this phenomenon extends to non-linear systems and neural network controllers.
Experiments include two non-linear control problems, in which the goal is to steer either a pendulum or quadcopter towards a target state.

\textbf{The pendulum control problem.}
A classic non-linear control problem is that of stabilizing a (simulated) pendulum at an upright position (\cf~\citet{hazan2022introduction}).
At time step $h$, the two-dimensional state of the system is described by the vertical angle of the pendulum $\theta_h \in \R$ and its angular velocity $\dot{\theta}_h \in \R$.
The controller applies a torque $u_h \in \R$, with the goal of making the pendulum reach and stay at the target state $(\pi, 0)$.
Accordingly, the cost at each time step is the squared Euclidean distance between the current and target states.
See \cref{app:experiments:details:nn:pend} for explicit equations defining the state dynamics and cost.

\textbf{The quadcopter control problem.}
Another common non-linear control problem is that of controlling a (simulated) quadcopter (\cf~\citet{panerati2021learning}).
At time step $h$, the state of the system $\xbf_h \in \R^{12}$ comprises the quadcopter's position $(x_h, y_h, z_h) \in \R^3$, tilt angles $(\phi_h, \theta_h, \psi_h) \in \R^3$ (\ie~roll, pitch, and yaw), and their respective velocities.
The controller determines the revolutions per minute (RPM) for each of four motors by choosing $\ubf_h \in [0, \text{MAX\_RPM}]^4$, where $\text{MAX\_RPM}$ stands for the maximal supported RPM.
We consider the goal of making the quadcopter reach and stay at the target state $\xbf^* = (0, 0, 1, 0, \ldots, 0)$.
This is expressed by taking the cost at each time step to be a weighted squared Euclidean distance between the current and target states.
See \cref{app:experiments:details:nn:quad} for explicit equations defining the state dynamics and cost.

\textbf{Results.}
For both the pendulum and quadcopter control problems, we train via policy gradient a (state-feedback) controller, parameterized as a fully-connected neural network with ReLU activation.
The controls produced by a randomly initialized neural network are usually near zero.
Hence, during the first iterations of policy gradient, both the pendulum and quadcopter fall downwards from their respective initial states, qualitatively leading to exploration.
\cref{fig:pend_and_quad_experiments_main} shows that, in accordance with our theory for LQR problems, in both the pendulum and quadcopter problems, the controller can extrapolate near-perfectly to unseen initial states whose heights are lower than those of the seen initial states.
The extrapolation is observed qualitatively, in the sense of trajectories stabilizing near the target state, and quantitatively, as evaluated by the cost extrapolation measure in comparison to a non-extrapolating controller.\footnote{
	The cost measure of extrapolation (\cref{def:cost_measure}) is adapted to non-linear control problems by taking $\U$ to be a predetermined set of initial states unseen in training (see \cref{app:experiments:details} for further details).
	We do not evaluate the optimality measure (\cref{def:opt_measure}) since it is not directly applicable to non-linear control problems.
}
For the quadcopter control problem, \cref{fig:quad_experiments_add_unseen_initial_states,fig:quad_experiments_dist} in \cref{app:experiments:nn} demonstrate that, respectively: \emph{(i)} the extent of extrapolation to unseen initial states varies depending on their distance from the seen initial states; and \emph{(ii)} extrapolation also applies to unseen initial states with horizontal distance from the seen initial states.

	\section{Conclusion}
\label{sec:conclusion}

The implicit bias of gradient descent is a cornerstone of modern machine learning.
While extensively studied in supervised learning, it is far less understood in optimal control (reinforcement learning).
There, learning a controller applied to a system via gradient descent is known as policy gradient, and a question of prime importance (particularly for safety-critical applications, \eg~robotics, industrial manufacturing, or autonomous driving) is the extent to which a learned controller extrapolates to initial states unseen in training.
In this paper we theoretically studied the implicit bias of policy gradient in terms of extrapolation to initial states unseen in training.
Focusing on the fundamental LQR problem, we established that the extent of extrapolation depends on the degree of exploration induced by the system when commencing from initial states included in training.
Experiments corroborated our theory, and demonstrated its conclusions on problems beyond LQR, where systems are non-linear and controllers are neural networks.

Future work includes extending our theory in three ways.
First, is to alleviate the technical limitations specified in \cref{sec:analysis:general}.
Second, is to account for non-linear systems and neural network controllers such as those evaluated in our experiments.
Third, is to address settings where systems are unknown or non-differentiable, and gradients with respect to controller parameters are estimated via sampling.

An additional direction for future research, which we hope our work will inspire, is the development of practical methods for detecting initial states whose inclusion in training enhances extrapolation to initial states unseen in training.
In real-world optimal control (and reinforcement learning), with contemporary learning algorithms, extrapolation to initial states unseen in training is often poor \cite{rajeswaran2017towards,zhang2018study,zhang2019dissection,fujimoto2019off,witty2021measuring}.
We believe methods as described bear potential to greatly improve it.

	
	\ifdefined\NEURIPS
		\begin{ack}
			We thank Yonathan Efroni, Emily Redelmeier, Yuval Peled, and Alexander Hock for illuminating discussions, and Eshbal Hezroni for aid in preparing illustrative figures.
This work was supported by a Google Research Scholar Award, a Google Research Gift, Meta, the Yandex Initiative in Machine Learning, the Israel Science Foundation (grant 1780/21), the European Research Council (ERC) under the European Unions Horizon 2020 research and innovation programme (grant ERC HOLI 819080), the Tel Aviv University Center for AI and Data Science, the Adelis Research Fund for Artificial Intelligence, Len Blavatnik and the Blavatnik Family Foundation, and Amnon and Anat Shashua.
NR is supported by the Apple Scholars in AI/ML PhD fellowship.
		\end{ack}
	\else
		\newcommand{\ack}{We thank Yonathan Efroni, Emily Redelmeier, Yuval Peled, and Alexander Hock for illuminating discussions, and Eshbal Hezroni for aid in preparing illustrative figures.
This work was supported by a Google Research Scholar Award, a Google Research Gift, Meta, the Yandex Initiative in Machine Learning, the Israel Science Foundation (grant 1780/21), the European Research Council (ERC) under the European Unions Horizon 2020 research and innovation programme (grant ERC HOLI 819080), the Tel Aviv University Center for AI and Data Science, the Adelis Research Fund for Artificial Intelligence, Len Blavatnik and the Blavatnik Family Foundation, and Amnon and Anat Shashua.
NR is supported by the Apple Scholars in AI/ML PhD fellowship.}
	\fi
	\ifdefined\ARXIV
		\section*{Acknowledgements}
		We thank Yonathan Efroni, Emily Redelmeier, Yuval Peled, and Alexander Hock for illuminating discussions, and Eshbal Hezroni for aid in preparing illustrative figures.
This work was supported by a Google Research Scholar Award, a Google Research Gift, Meta, the Yandex Initiative in Machine Learning, the Israel Science Foundation (grant 1780/21), the European Research Council (ERC) under the European Unions Horizon 2020 research and innovation programme (grant ERC HOLI 819080), the Tel Aviv University Center for AI and Data Science, the Adelis Research Fund for Artificial Intelligence, Len Blavatnik and the Blavatnik Family Foundation, and Amnon and Anat Shashua.
NR is supported by the Apple Scholars in AI/ML PhD fellowship.
	\else
		\ifdefined\COLT
			\acks{We thank Yonathan Efroni, Emily Redelmeier, Yuval Peled, and Alexander Hock for illuminating discussions, and Eshbal Hezroni for aid in preparing illustrative figures.
This work was supported by a Google Research Scholar Award, a Google Research Gift, Meta, the Yandex Initiative in Machine Learning, the Israel Science Foundation (grant 1780/21), the European Research Council (ERC) under the European Unions Horizon 2020 research and innovation programme (grant ERC HOLI 819080), the Tel Aviv University Center for AI and Data Science, the Adelis Research Fund for Artificial Intelligence, Len Blavatnik and the Blavatnik Family Foundation, and Amnon and Anat Shashua.
NR is supported by the Apple Scholars in AI/ML PhD fellowship.}
		\else
			\ifdefined\CAMREADY
				\ifdefined\ICLR
					\newcommand*{\subsuback}{}
				\fi
				\ifdefined\NEURIPS
				\else
					\section*{Acknowledgements}
					We thank Yonathan Efroni, Emily Redelmeier, Yuval Peled, and Alexander Hock for illuminating discussions, and Eshbal Hezroni for aid in preparing illustrative figures.
This work was supported by a Google Research Scholar Award, a Google Research Gift, Meta, the Yandex Initiative in Machine Learning, the Israel Science Foundation (grant 1780/21), the European Research Council (ERC) under the European Unions Horizon 2020 research and innovation programme (grant ERC HOLI 819080), the Tel Aviv University Center for AI and Data Science, the Adelis Research Fund for Artificial Intelligence, Len Blavatnik and the Blavatnik Family Foundation, and Amnon and Anat Shashua.
NR is supported by the Apple Scholars in AI/ML PhD fellowship.
				\fi
			\fi
		\fi
	\fi

	\section*{References}
	{\small
		\ifdefined\ICML
			\bibliographystyle{icml2024}
		\else
			\bibliographystyle{plainnat}
		\fi
		\bibliography{refs}

\begin{thebibliography}{93}
\providecommand{\natexlab}[1]{#1}
\providecommand{\url}[1]{\texttt{#1}}
\expandafter\ifx\csname urlstyle\endcsname\relax
  \providecommand{\doi}[1]{doi: #1}\else
  \providecommand{\doi}{doi: \begingroup \urlstyle{rm}\Url}\fi

\bibitem[Abbe et~al.(2022)Abbe, Bengio, Cornacchia, Kleinberg, Lotfi, Raghu,
  and Zhang]{abbe2022learning}
Abbe, E., Bengio, S., Cornacchia, E., Kleinberg, J., Lotfi, A., Raghu, M., and
  Zhang, C.
\newblock Learning to reason with neural networks: Generalization, unseen data
  and boolean measures.
\newblock \emph{Advances in Neural Information Processing Systems}, 35, 2022.

\bibitem[Abbe et~al.(2023)Abbe, Bengio, Lotfi, and
  Rizk]{abbe2023generalization}
Abbe, E., Bengio, S., Lotfi, A., and Rizk, K.
\newblock Generalization on the unseen, logic reasoning and degree curriculum.
\newblock In \emph{International conference on machine learning}. PMLR, 2023.

\bibitem[Agarwal et~al.(2019{\natexlab{a}})Agarwal, Bullins, Hazan, Kakade, and
  Singh]{agarwal2019online}
Agarwal, N., Bullins, B., Hazan, E., Kakade, S., and Singh, K.
\newblock Online control with adversarial disturbances.
\newblock In \emph{International Conference on Machine Learning}. PMLR,
  2019{\natexlab{a}}.

\bibitem[Agarwal et~al.(2019{\natexlab{b}})Agarwal, Hazan, and
  Singh]{agarwal2019logarithmic}
Agarwal, N., Hazan, E., and Singh, K.
\newblock Logarithmic regret for online control.
\newblock \emph{Advances in Neural Information Processing Systems}, 32,
  2019{\natexlab{b}}.

\bibitem[Anderson \& Moore(2007)Anderson and Moore]{anderson2007optimal}
Anderson, B.~D. and Moore, J.~B.
\newblock \emph{Optimal control: linear quadratic methods}.
\newblock Courier Corporation, 2007.

\bibitem[Anderson et~al.(2010)Anderson, Guionnet, and
  Zeitouni]{anderson2010introduction}
Anderson, G.~W., Guionnet, A., and Zeitouni, O.
\newblock \emph{An introduction to random matrices}.
\newblock Number 118. Cambridge university press, 2010.

\bibitem[Andriushchenko et~al.(2023)Andriushchenko, Varre, Pillaud-Vivien, and
  Flammarion]{andriushchenko2023sgd}
Andriushchenko, M., Varre, A.~V., Pillaud-Vivien, L., and Flammarion, N.
\newblock Sgd with large step sizes learns sparse features.
\newblock In \emph{International Conference on Machine Learning}. PMLR, 2023.

\bibitem[Arora et~al.(2019)Arora, Cohen, Hu, and Luo]{arora2019implicit}
Arora, S., Cohen, N., Hu, W., and Luo, Y.
\newblock Implicit regularization in deep matrix factorization.
\newblock In \emph{Advances in Neural Information Processing Systems}, 2019.

\bibitem[Azulay et~al.(2021)Azulay, Moroshko, Nacson, Woodworth, Srebro,
  Globerson, and Soudry]{azulay2021implicit}
Azulay, S., Moroshko, E., Nacson, M.~S., Woodworth, B.~E., Srebro, N.,
  Globerson, A., and Soudry, D.
\newblock On the implicit bias of initialization shape: Beyond infinitesimal
  mirror descent.
\newblock In \emph{International Conference on Machine Learning}, 2021.

\bibitem[Bartlett et~al.(2020)Bartlett, Long, Lugosi, and
  Tsigler]{bartlett2020benign}
Bartlett, P.~L., Long, P.~M., Lugosi, G., and Tsigler, A.
\newblock Benign overfitting in linear regression.
\newblock \emph{Proceedings of the National Academy of Sciences}, 117\penalty0
  (48):\penalty0 30063--30070, 2020.

\bibitem[Bhandari \& Russo(2019)Bhandari and Russo]{bhandari2019global}
Bhandari, J. and Russo, D.
\newblock Global optimality guarantees for policy gradient methods.
\newblock \emph{arXiv preprint arXiv:1906.01786}, 2019.

\bibitem[Bhounsule et~al.(2016)Bhounsule, Ameperosa, Miller, Seay, and
  Ulep]{bhounsule2016dead}
Bhounsule, P.~A., Ameperosa, E., Miller, S., Seay, K., and Ulep, R.
\newblock Dead-beat control of walking for a torso-actuated rimless wheel using
  an event-based, discrete, linear controller.
\newblock In \emph{International Design Engineering Technical Conferences and
  Computers and Information in Engineering Conference}, volume 50152, pp.\
  V05AT07A042. American Society of Mechanical Engineers, 2016.

\bibitem[Boursier et~al.(2022)Boursier, Pillaud-Vivien, and
  Flammarion]{boursier2022gradient}
Boursier, E., Pillaud-Vivien, L., and Flammarion, N.
\newblock Gradient flow dynamics of shallow relu networks for square loss and
  orthogonal inputs.
\newblock \emph{Advances in Neural Information Processing Systems},
  35:\penalty0 20105--20118, 2022.

\bibitem[Bu et~al.(2019)Bu, Mesbahi, Fazel, and Mesbahi]{bu2019lqr}
Bu, J., Mesbahi, A., Fazel, M., and Mesbahi, M.
\newblock Lqr through the lens of first order methods: Discrete-time case.
\newblock \emph{arXiv preprint arXiv:1907.08921}, 2019.

\bibitem[Bu et~al.(2020)Bu, Mesbahi, and Mesbahi]{bu2020policy}
Bu, J., Mesbahi, A., and Mesbahi, M.
\newblock Policy gradient-based algorithms for continuous-time linear quadratic
  control.
\newblock \emph{arXiv preprint arXiv:2006.09178}, 2020.

\bibitem[Caron \& Traynor(2005)Caron and Traynor]{caron2005zero}
Caron, R. and Traynor, T.
\newblock The zero set of a polynomial.
\newblock \emph{WSMR Report}, pp.\  05--02, 2005.

\bibitem[Cassel \& Koren(2021)Cassel and Koren]{cassel2021online}
Cassel, A.~B. and Koren, T.
\newblock Online policy gradient for model free learning of linear quadratic
  regulators with $\sqrt{T}$ regret.
\newblock In \emph{International Conference on Machine Learning}. PMLR, 2021.

\bibitem[Chen et~al.(2023)Chen, Minasyan, Lee, and Hazan]{chen2023regret}
Chen, X., Minasyan, E., Lee, J.~D., and Hazan, E.
\newblock Regret guarantees for online deep control.
\newblock In \emph{Learning for Dynamics and Control Conference}. PMLR, 2023.

\bibitem[Chou et~al.(2023)Chou, Maly, and Rauhut]{chou2023more}
Chou, H.-H., Maly, J., and Rauhut, H.
\newblock More is less: inducing sparsity via overparameterization.
\newblock \emph{Information and Inference: A Journal of the IMA}, 12\penalty0
  (3), 2023.

\bibitem[Chou et~al.(2024)Chou, Gieshoff, Maly, and Rauhut]{chou2024gradient}
Chou, H.-H., Gieshoff, C., Maly, J., and Rauhut, H.
\newblock Gradient descent for deep matrix factorization: Dynamics and implicit
  bias towards low rank.
\newblock \emph{Applied and Computational Harmonic Analysis}, 68:\penalty0
  101595, 2024.

\bibitem[Clavera et~al.(2020)Clavera, Fu, and Abbeel]{clavera2020model}
Clavera, I., Fu, V., and Abbeel, P.
\newblock Model-augmented actor-critic: Backpropagating through paths.
\newblock \emph{International Conference on Learning Representations}, 2020.

\bibitem[Cohen et~al.(2018)Cohen, Hasidim, Koren, Lazic, Mansour, and
  Talwar]{cohen2018online}
Cohen, A., Hasidim, A., Koren, T., Lazic, N., Mansour, Y., and Talwar, K.
\newblock Online linear quadratic control.
\newblock In \emph{International Conference on Machine Learning}. PMLR, 2018.

\bibitem[Cohen-Karlik et~al.(2022)Cohen-Karlik, David, Cohen, and
  Globerson]{cohen2022implicit}
Cohen-Karlik, E., David, A.~B., Cohen, N., and Globerson, A.
\newblock On the implicit bias of gradient descent for temporal extrapolation.
\newblock In \emph{International Conference on Artificial Intelligence and
  Statistics}. PMLR, 2022.

\bibitem[Cohen-Karlik et~al.(2023)Cohen-Karlik, Menuhin-Gruman, Cohen, Giryes,
  and Globerson]{cohen2023learning}
Cohen-Karlik, E., Menuhin-Gruman, I., Cohen, N., Giryes, R., and Globerson, A.
\newblock Learning low dimensional state spaces with overparameterized
  recurrent neural network.
\newblock In \emph{International Conference on Learning Representations}, 2023.

\bibitem[De~Don{\'a} \& Goodwin(1999)De~Don{\'a} and
  Goodwin]{de1999disturbance}
De~Don{\'a}, J.~A. and Goodwin, G.~C.
\newblock Disturbance sensitivity issues in predictive control.
\newblock \emph{International Journal of Adaptive Control and Signal
  Processing}, 13\penalty0 (6):\penalty0 507--519, 1999.

\bibitem[Dulac-Arnold et~al.(2021)Dulac-Arnold, Levine, Mankowitz, Li,
  Paduraru, Gowal, and Hester]{dulac2021challenges}
Dulac-Arnold, G., Levine, N., Mankowitz, D.~J., Li, J., Paduraru, C., Gowal,
  S., and Hester, T.
\newblock Challenges of real-world reinforcement learning: definitions,
  benchmarks and analysis.
\newblock \emph{Machine Learning}, 110\penalty0 (9):\penalty0 2419--2468, 2021.

\bibitem[Emami-Naeini \& Franklin(1982)Emami-Naeini and
  Franklin]{emami1982deadbeat}
Emami-Naeini, A. and Franklin, G.
\newblock Deadbeat control and tracking of discrete-time systems.
\newblock \emph{IEEE Transactions on Automatic Control}, 27\penalty0
  (1):\penalty0 176--181, 1982.

\bibitem[Fazel et~al.(2018)Fazel, Ge, Kakade, and Mesbahi]{fazel2018global}
Fazel, M., Ge, R., Kakade, S., and Mesbahi, M.
\newblock Global convergence of policy gradient methods for the linear
  quadratic regulator.
\newblock In \emph{International conference on machine learning}. PMLR, 2018.

\bibitem[Frei et~al.(2023{\natexlab{a}})Frei, Vardi, Bartlett, and
  Srebro]{frei2023benign}
Frei, S., Vardi, G., Bartlett, P., and Srebro, N.
\newblock Benign overfitting in linear classifiers and leaky relu networks from
  kkt conditions for margin maximization.
\newblock In \emph{The Thirty Sixth Annual Conference on Learning Theory}.
  PMLR, 2023{\natexlab{a}}.

\bibitem[Frei et~al.(2023{\natexlab{b}})Frei, Vardi, Bartlett, and
  Srebro]{frei2023double}
Frei, S., Vardi, G., Bartlett, P.~L., and Srebro, N.
\newblock The double-edged sword of implicit bias: Generalization vs.
  robustness in relu networks.
\newblock \emph{Advances in Neural Information Processing Systems}, 36,
  2023{\natexlab{b}}.

\bibitem[Fujimoto et~al.(2019)Fujimoto, Meger, and Precup]{fujimoto2019off}
Fujimoto, S., Meger, D., and Precup, D.
\newblock Off-policy deep reinforcement learning without exploration.
\newblock In \emph{International conference on machine learning}. PMLR, 2019.

\bibitem[Gillen \& Byl(2022)Gillen and Byl]{gillen2022leveraging}
Gillen, S. and Byl, K.
\newblock Leveraging reward gradients for reinforcement learning in
  differentiable physics simulations.
\newblock \emph{arXiv preprint arXiv:2203.02857}, 2022.

\bibitem[Gravell et~al.(2020)Gravell, Esfahani, and
  Summers]{gravell2020learning}
Gravell, B., Esfahani, P.~M., and Summers, T.
\newblock Learning optimal controllers for linear systems with multiplicative
  noise via policy gradient.
\newblock \emph{IEEE Transactions on Automatic Control}, 66\penalty0 (11),
  2020.

\bibitem[Gunasekar et~al.(2017)Gunasekar, Woodworth, Bhojanapalli, Neyshabur,
  and Srebro]{gunasekar2017implicit}
Gunasekar, S., Woodworth, B.~E., Bhojanapalli, S., Neyshabur, B., and Srebro,
  N.
\newblock Implicit regularization in matrix factorization.
\newblock In \emph{Advances in Neural Information Processing Systems}, 2017.

\bibitem[Hambly et~al.(2021)Hambly, Xu, and Yang]{hambly2021policy}
Hambly, B., Xu, R., and Yang, H.
\newblock Policy gradient methods for the noisy linear quadratic regulator over
  a finite horizon.
\newblock \emph{SIAM Journal on Control and Optimization}, 59\penalty0 (5),
  2021.

\bibitem[Hautus \& Silverman(1983)Hautus and Silverman]{hautus1983system}
Hautus, M.~L. and Silverman, L.~M.
\newblock System structure and singular control.
\newblock \emph{Linear algebra and its applications}, 50:\penalty0 369--402,
  1983.

\bibitem[Hazan \& Singh(2022)Hazan and Singh]{hazan2022introduction}
Hazan, E. and Singh, K.
\newblock Introduction to online nonstochastic control.
\newblock \emph{arXiv preprint arXiv:2211.09619}, 2022.

\bibitem[Howell et~al.(2022)Howell, Cleac'h, Br{\"u}digam, Kolter, Schwager,
  and Manchester]{howell2022dojo}
Howell, T.~A., Cleac'h, S.~L., Br{\"u}digam, J., Kolter, J.~Z., Schwager, M.,
  and Manchester, Z.
\newblock Dojo: A differentiable physics engine for robotics.
\newblock \emph{arXiv preprint arXiv:2203.00806}, 2022.

\bibitem[Hu et~al.(2023)Hu, Zhang, Li, Mesbahi, Fazel, and
  Ba{\c{s}}ar]{hu2023toward}
Hu, B., Zhang, K., Li, N., Mesbahi, M., Fazel, M., and Ba{\c{s}}ar, T.
\newblock Toward a theoretical foundation of policy optimization for learning
  control policies.
\newblock \emph{Annual Review of Control, Robotics, and Autonomous Systems}, 6,
  2023.

\bibitem[Hu et~al.(2019)Hu, Liu, Spielberg, Tenenbaum, Freeman, Wu, Rus, and
  Matusik]{hu2019chainqueen}
Hu, Y., Liu, J., Spielberg, A., Tenenbaum, J.~B., Freeman, W.~T., Wu, J., Rus,
  D., and Matusik, W.
\newblock Chainqueen: A real-time differentiable physical simulator for soft
  robotics.
\newblock In \emph{2019 International conference on robotics and automation
  (ICRA)}, pp.\  6265--6271. IEEE, 2019.

\bibitem[Hu et~al.(2022)Hu, Ji, and Telgarsky]{hu2022actor}
Hu, Y., Ji, Z., and Telgarsky, M.
\newblock Actor-critic is implicitly biased towards high entropy optimal
  policies.
\newblock In \emph{International Conference on Learning Representations}, 2022.

\bibitem[Ji \& Telgarsky(2019{\natexlab{a}})Ji and Telgarsky]{ji2019gradient}
Ji, Z. and Telgarsky, M.
\newblock Gradient descent aligns the layers of deep linear networks.
\newblock \emph{International Conference on Learning Representations},
  2019{\natexlab{a}}.

\bibitem[Ji \& Telgarsky(2019{\natexlab{b}})Ji and Telgarsky]{ji2019implicit}
Ji, Z. and Telgarsky, M.
\newblock The implicit bias of gradient descent on nonseparable data.
\newblock In \emph{Conference on Learning Theory}, 2019{\natexlab{b}}.

\bibitem[Jin et~al.(2020)Jin, Schmitt, and Wen]{jin2020analysis}
Jin, Z., Schmitt, J.~M., and Wen, Z.
\newblock On the analysis of model-free methods for the linear quadratic
  regulator.
\newblock \emph{arXiv preprint arXiv:2007.03861}, 2020.

\bibitem[Kemp(2013)]{kemp2013math}
Kemp, T.
\newblock Math 247a: Introduction to random matrix theory.
\newblock \emph{Lecture notes}, 2013.

\bibitem[Kingma \& Ba(2015)Kingma and Ba]{kingma2015adam}
Kingma, D.~P. and Ba, J.
\newblock Adam: A method for stochastic optimization.
\newblock In \emph{International Conference on Learning Representations}, 2015.

\bibitem[Ku{\v{c}}era(1998)]{kuvcera1998deadbeat}
Ku{\v{c}}era, V.
\newblock Deadbeat control, pole placement, and lq regulation.
\newblock In \emph{Theory and Practice of Control and Systems}, pp.\  5--10.
  World Scientific, 1998.

\bibitem[Kumar et~al.(2021)Kumar, Agarwal, Ghosh, and
  Levine]{kumar2021implicit}
Kumar, A., Agarwal, R., Ghosh, D., and Levine, S.
\newblock Implicit under-parameterization inhibits data-efficient deep
  reinforcement learning.
\newblock In \emph{International Conference on Learning Representations}, 2021.

\bibitem[Kumar et~al.(2022)Kumar, Agarwal, Ma, Courville, Tucker, and
  Levine]{kumar2022dr3}
Kumar, A., Agarwal, R., Ma, T., Courville, A., Tucker, G., and Levine, S.
\newblock Dr3: Value-based deep reinforcement learning requires explicit
  regularization.
\newblock In \emph{International Conference on Learning Representations}, 2022.

\bibitem[Li \& Todorov(2004)Li and Todorov]{li2004iterative}
Li, W. and Todorov, E.
\newblock Iterative linear quadratic regulator design for nonlinear biological
  movement systems.
\newblock In \emph{First International Conference on Informatics in Control,
  Automation and Robotics}, volume~2. SciTePress, 2004.

\bibitem[Lyu \& Li(2020)Lyu and Li]{lyu2020gradient}
Lyu, K. and Li, J.
\newblock Gradient descent maximizes the margin of homogeneous neural networks.
\newblock \emph{International Conference on Learning Representations}, 2020.

\bibitem[Lyu et~al.(2021)Lyu, Li, Wang, and Arora]{lyu2021gradient}
Lyu, K., Li, Z., Wang, R., and Arora, S.
\newblock Gradient descent on two-layer nets: Margin maximization and
  simplicity bias.
\newblock \emph{Advances in Neural Information Processing Systems}, 34, 2021.

\bibitem[Malik et~al.(2019)Malik, Pananjady, Bhatia, Khamaru, Bartlett, and
  Wainwright]{malik2019derivative}
Malik, D., Pananjady, A., Bhatia, K., Khamaru, K., Bartlett, P., and
  Wainwright, M.
\newblock Derivative-free methods for policy optimization: Guarantees for
  linear quadratic systems.
\newblock In \emph{The 22nd international conference on artificial intelligence
  and statistics}. PMLR, 2019.

\bibitem[Marcotte et~al.(2023)Marcotte, Gribonval, and
  Peyr{\'e}]{marcotte2023abide}
Marcotte, S., Gribonval, R., and Peyr{\'e}, G.
\newblock Abide by the law and follow the flow: Conservation laws for gradient
  flows.
\newblock \emph{Advances in neural information processing systems}, 2023.

\bibitem[Marro et~al.(2002)Marro, Prattichizzo, and
  Zattoni]{marro2002geometric}
Marro, G., Prattichizzo, D., and Zattoni, E.
\newblock Geometric insight into discrete-time cheap and singular linear
  quadratic riccati (lqr) problems.
\newblock \emph{IEEE Transactions on Automatic Control}, 47\penalty0
  (1):\penalty0 102--107, 2002.

\bibitem[Mattavelli(2005)]{mattavelli2005improved}
Mattavelli, P.
\newblock An improved deadbeat control for ups using disturbance observers.
\newblock \emph{IEEE Transactions on Industrial Electronics}, 52\penalty0
  (1):\penalty0 206--212, 2005.

\bibitem[Metz et~al.(2021)Metz, Freeman, Schoenholz, and
  Kachman]{metz2021gradients}
Metz, L., Freeman, C.~D., Schoenholz, S.~S., and Kachman, T.
\newblock Gradients are not all you need.
\newblock \emph{arXiv preprint arXiv:2111.05803}, 2021.

\bibitem[Miller et~al.(2021)Miller, Taori, Raghunathan, Sagawa, Koh, Shankar,
  Liang, Carmon, and Schmidt]{miller2021accuracy}
Miller, J.~P., Taori, R., Raghunathan, A., Sagawa, S., Koh, P.~W., Shankar, V.,
  Liang, P., Carmon, Y., and Schmidt, L.
\newblock Accuracy on the line: on the strong correlation between
  out-of-distribution and in-distribution generalization.
\newblock In \emph{International Conference on Machine Learning}. PMLR, 2021.

\bibitem[Mohammadi et~al.(2019)Mohammadi, Zare, Soltanolkotabi, and
  Jovanovi{\'c}]{mohammadi2019global}
Mohammadi, H., Zare, A., Soltanolkotabi, M., and Jovanovi{\'c}, M.~R.
\newblock Global exponential convergence of gradient methods over the nonconvex
  landscape of the linear quadratic regulator.
\newblock In \emph{2019 IEEE 58th Conference on Decision and Control (CDC)}.
  IEEE, 2019.

\bibitem[Mohammadi et~al.(2021)Mohammadi, Zare, Soltanolkotabi, and
  Jovanovi{\'c}]{mohammadi2021convergence}
Mohammadi, H., Zare, A., Soltanolkotabi, M., and Jovanovi{\'c}, M.~R.
\newblock Convergence and sample complexity of gradient methods for the
  model-free linear--quadratic regulator problem.
\newblock \emph{IEEE Transactions on Automatic Control}, 67\penalty0 (5), 2021.

\bibitem[Mora et~al.(2021)Mora, Peychev, Ha, Vechev, and Coros]{mora2021pods}
Mora, M. A.~Z., Peychev, M., Ha, S., Vechev, M., and Coros, S.
\newblock Pods: Policy optimization via differentiable simulation.
\newblock In \emph{International Conference on Machine Learning}. PMLR, 2021.

\bibitem[Munkres(2018)]{munkres2018elements}
Munkres, J.~R.
\newblock \emph{Elements of algebraic topology}.
\newblock CRC press, 2018.

\bibitem[Neyshabur(2017)]{neyshabur2017implicit}
Neyshabur, B.
\newblock Implicit regularization in deep learning.
\newblock \emph{arXiv preprint arXiv:1709.01953}, 2017.

\bibitem[Neyshabur et~al.(2014)Neyshabur, Tomioka, and
  Srebro]{neyshabur2014search}
Neyshabur, B., Tomioka, R., and Srebro, N.
\newblock In search of the real inductive bias: On the role of implicit
  regularization in deep learning.
\newblock \emph{arXiv preprint arXiv:1412.6614}, 2014.

\bibitem[Panerati et~al.(2021)Panerati, Zheng, Zhou, Xu, Prorok, and
  Schoellig]{panerati2021learning}
Panerati, J., Zheng, H., Zhou, S., Xu, J., Prorok, A., and Schoellig, A.~P.
\newblock Learning to fly—a gym environment with pybullet physics for
  reinforcement learning of multi-agent quadcopter control.
\newblock In \emph{2021 IEEE/RSJ International Conference on Intelligent Robots
  and Systems (IROS)}. IEEE, 2021.

\bibitem[Paszke et~al.(2019)Paszke, Gross, Massa, Lerer, Bradbury, Chanan,
  Killeen, Lin, Gimelshein, Antiga, et~al.]{paszke2019pytorch}
Paszke, A., Gross, S., Massa, F., Lerer, A., Bradbury, J., Chanan, G., Killeen,
  T., Lin, Z., Gimelshein, N., Antiga, L., et~al.
\newblock Pytorch: An imperative style, high-performance deep learning library.
\newblock \emph{Advances in neural information processing systems}, 32, 2019.

\bibitem[Pesme et~al.(2021)Pesme, Pillaud-Vivien, and
  Flammarion]{pesme2021implicit}
Pesme, S., Pillaud-Vivien, L., and Flammarion, N.
\newblock Implicit bias of sgd for diagonal linear networks: a provable benefit
  of stochasticity.
\newblock \emph{Advances in Neural Information Processing Systems}, 34, 2021.

\bibitem[Qiao et~al.(2020)Qiao, Liang, Koltun, and Lin]{qiao2020scalable}
Qiao, Y.-L., Liang, J., Koltun, V., and Lin, M.~C.
\newblock Scalable differentiable physics for learning and control.
\newblock In \emph{International Conference on Machine Learning}. PMLR, 2020.

\bibitem[Rajeswaran et~al.(2017)Rajeswaran, Lowrey, Todorov, and
  Kakade]{rajeswaran2017towards}
Rajeswaran, A., Lowrey, K., Todorov, E.~V., and Kakade, S.~M.
\newblock Towards generalization and simplicity in continuous control.
\newblock \emph{Advances in Neural Information Processing Systems}, 30, 2017.

\bibitem[Razin \& Cohen(2020)Razin and Cohen]{razin2020implicit}
Razin, N. and Cohen, N.
\newblock Implicit regularization in deep learning may not be explainable by
  norms.
\newblock In \emph{Advances in Neural Information Processing Systems}, 2020.

\bibitem[Razin et~al.(2021)Razin, Maman, and Cohen]{razin2021implicit}
Razin, N., Maman, A., and Cohen, N.
\newblock Implicit regularization in tensor factorization.
\newblock \emph{International Conference on Machine Learning}, 2021.

\bibitem[Razin et~al.(2022)Razin, Maman, and Cohen]{razin2022implicit}
Razin, N., Maman, A., and Cohen, N.
\newblock Implicit regularization in hierarchical tensor factorization and deep
  convolutional neural networks.
\newblock \emph{International Conference on Machine Learning}, 2022.

\bibitem[Redelmeier(2014)]{redelmeier2014real}
Redelmeier, C. E.~I.
\newblock Real second-order freeness and the asymptotic real second-order
  freeness of several real matrix models.
\newblock \emph{International Mathematics Research Notices}, 2014\penalty0
  (12):\penalty0 3353--3395, 2014.

\bibitem[Shamir(2022)]{shamir2022implicit}
Shamir, O.
\newblock The implicit bias of benign overfitting.
\newblock In \emph{Conference on Learning Theory}. PMLR, 2022.

\bibitem[Shen et~al.(2021)Shen, Liu, He, Zhang, Xu, Yu, and
  Cui]{shen2021towards}
Shen, Z., Liu, J., He, Y., Zhang, X., Xu, R., Yu, H., and Cui, P.
\newblock Towards out-of-distribution generalization: A survey.
\newblock \emph{arXiv preprint arXiv:2108.13624}, 2021.

\bibitem[Sontag(2013)]{sontag2013mathematical}
Sontag, E.~D.
\newblock \emph{Mathematical control theory: deterministic finite dimensional
  systems}, volume~6.
\newblock Springer Science \& Business Media, 2013.

\bibitem[Soudry et~al.(2018)Soudry, Hoffer, Nacson, Gunasekar, and
  Srebro]{soudry2018implicit}
Soudry, D., Hoffer, E., Nacson, M.~S., Gunasekar, S., and Srebro, N.
\newblock The implicit bias of gradient descent on separable data.
\newblock \emph{The Journal of Machine Learning Research}, 19\penalty0 (1),
  2018.

\bibitem[Vardi(2023)]{vardi2023implicit}
Vardi, G.
\newblock On the implicit bias in deep-learning algorithms.
\newblock \emph{Communications of the ACM}, 66\penalty0 (6), 2023.

\bibitem[Vershynin(2020)]{vershynin2020high}
Vershynin, R.
\newblock High-dimensional probability.
\newblock \emph{University of California, Irvine}, 2020.

\bibitem[Wiedemann et~al.(2023)Wiedemann, W{\"u}est, Loquercio, M{\"u}ller,
  Floreano, and Scaramuzza]{wiedemann2023training}
Wiedemann, N., W{\"u}est, V., Loquercio, A., M{\"u}ller, M., Floreano, D., and
  Scaramuzza, D.
\newblock Training efficient controllers via analytic policy gradient.
\newblock In \emph{2023 IEEE International Conference on Robotics and
  Automation (ICRA)}. IEEE, 2023.

\bibitem[Williams(1992)]{williams1992simple}
Williams, R.~J.
\newblock Simple statistical gradient-following algorithms for connectionist
  reinforcement learning.
\newblock \emph{Machine learning}, 8:\penalty0 229--256, 1992.

\bibitem[Witty et~al.(2021)Witty, Lee, Tosch, Atrey, Clary, Littman, and
  Jensen]{witty2021measuring}
Witty, S., Lee, J.~K., Tosch, E., Atrey, A., Clary, K., Littman, M.~L., and
  Jensen, D.
\newblock Measuring and characterizing generalization in deep reinforcement
  learning.
\newblock \emph{Applied AI Letters}, 2\penalty0 (4):\penalty0 e45, 2021.

\bibitem[Woodworth et~al.(2020)Woodworth, Gunasekar, Lee, Moroshko, Savarese,
  Golan, Soudry, and Srebro]{woodworth2020kernel}
Woodworth, B., Gunasekar, S., Lee, J.~D., Moroshko, E., Savarese, P., Golan,
  I., Soudry, D., and Srebro, N.
\newblock Kernel and rich regimes in overparametrized models.
\newblock In \emph{Conference on Learning Theory}, 2020.

\bibitem[Xu et~al.(2022)Xu, Makoviychuk, Narang, Ramos, Matusik, Garg, and
  Macklin]{xu2022accelerated}
Xu, J., Makoviychuk, V., Narang, Y., Ramos, F., Matusik, W., Garg, A., and
  Macklin, M.
\newblock Accelerated policy learning with parallel differentiable simulation.
\newblock \emph{International Conference on Learning Representations}, 2022.

\bibitem[Xu et~al.(2021)Xu, Zhang, Li, Du, Kawarabayashi, and
  Jegelka]{xu2021neural}
Xu, K., Zhang, M., Li, J., Du, S.~S., Kawarabayashi, K.-i., and Jegelka, S.
\newblock How neural networks extrapolate: From feedforward to graph neural
  networks.
\newblock In \emph{International Conference on Learning Representations}, 2021.

\bibitem[Zhang et~al.(2019)Zhang, Ballas, and Pineau]{zhang2019dissection}
Zhang, A., Ballas, N., and Pineau, J.
\newblock A dissection of overfitting and generalization in continuous
  reinforcement learning.
\newblock In \emph{International conference on machine learning}, 2019.

\bibitem[Zhang et~al.(2017)Zhang, Bengio, Hardt, Recht, and
  Vinyals]{zhang2017understanding}
Zhang, C., Bengio, S., Hardt, M., Recht, B., and Vinyals, O.
\newblock Understanding deep learning requires rethinking generalization.
\newblock In \emph{International Conference on Learning Representations}, 2017.

\bibitem[Zhang et~al.(2018)Zhang, Vinyals, Munos, and Bengio]{zhang2018study}
Zhang, C., Vinyals, O., Munos, R., and Bengio, S.
\newblock A study on overfitting in deep reinforcement learning.
\newblock \emph{arXiv preprint arXiv:1804.06893}, 2018.

\bibitem[Zhang et~al.(2020)Zhang, Hu, and Basar]{zhang2020policy}
Zhang, K., Hu, B., and Basar, T.
\newblock Policy optimization for $\mathcal{H}_2$ linear control with
  $\mathcal{H}_\infty$ robustness guarantee: Implicit regularization and global
  convergence.
\newblock In \emph{Learning for Dynamics and Control}. PMLR, 2020.

\bibitem[Zhang et~al.(2021)Zhang, Zhang, Hu, and Basar]{zhang2021derivative}
Zhang, K., Zhang, X., Hu, B., and Basar, T.
\newblock Derivative-free policy optimization for linear risk-sensitive and
  robust control design: Implicit regularization and sample complexity.
\newblock \emph{Advances in Neural Information Processing Systems}, 34, 2021.

\bibitem[Zhao et~al.(2023)Zhao, D{\"o}rfler, and You]{zhao2023data}
Zhao, F., D{\"o}rfler, F., and You, K.
\newblock Data-enabled policy optimization for the linear quadratic regulator.
\newblock \emph{arXiv preprint arXiv:2303.17958}, 2023.

\bibitem[Zhou et~al.(2023)Zhou, Bradley, Littwin, Razin, Saremi, Susskind,
  Bengio, and Nakkiran]{zhou2023algorithms}
Zhou, H., Bradley, A., Littwin, E., Razin, N., Saremi, O., Susskind, J.,
  Bengio, S., and Nakkiran, P.
\newblock What algorithms can transformers learn? a study in length
  generalization.
\newblock \emph{arXiv preprint arXiv:2310.16028}, 2023.

\bibitem[Zhu et~al.(2020)Zhu, Yu, Gupta, Shah, Hartikainen, Singh, Kumar, and
  Levine]{zhu2020ingredients}
Zhu, H., Yu, J., Gupta, A., Shah, D., Hartikainen, K., Singh, A., Kumar, V.,
  and Levine, S.
\newblock The ingredients of real-world robotic reinforcement learning.
\newblock In \emph{International Conference on Learning Representations}, 2020.

\end{thebibliography}
	}

	\clearpage
	\appendix
	
	\onecolumn
	
	\ifdefined\ENABLEENDNOTES
		\theendnotes
	\fi
	

	
	\section{Significance of Underdetermined Linear Quadratic Control}
\label{app:underdetermined_lqr_significance}

As mentioned in~\cref{sec:prelim:underdetermined}, the main purpose of our underdetermined LQR setting is to serve as a testbed for theoretical study of implicit bias in optimal control, analogously to how underdetermined linear prediction serves as an important testbed for theoretical study of implicit bias in supervised learning (\eg,~\citet{soudry2018implicit,bartlett2020benign,shamir2022implicit}).
Insights derived from the analysis of implicit bias in underdetermined linear prediction have later led to formal guarantees for more complex settings with non-linear neural networks (\eg,~\citet{lyu2020gradient,boursier2022gradient,frei2023benign}). 
We believe that our analysis of underdetermined LQR will play an analogous role, laying foundations for analyzing implicit bias for non-linear neural networks in optimal control. 
The neural network experiments we present in \cref{sec:experiments:nonlinear} support this prospect.
	
In addition to the aforementioned theoretical motivation, our underdetermined LQR setting is also practically motivated.
Specifically, the assumption that $\trainstates$~---~the set of initial states seen in training~---~does not span the state space is motivated by the importance of extrapolation to initial states unseen in training (\cf~\citet{zhu2020ingredients,dulac2021challenges}).
The assumption $\Rbf= \0$, \ie~that controls are unregularized, is motivated by the following: \emph{(i)} it leads to what is known as a \emph{deadbeat control} problem, where the goal is to drive an initial state to zero in as few steps as possible~\citep{emami1982deadbeat,kuvcera1998deadbeat,mattavelli2005improved,bhounsule2016dead}; \emph{(ii)} it has been used in the context of model predictive control~\citep{de1999disturbance}; and \emph{(iii)} it falls under the category of singular LQR problems, and has been evaluated under that context~\citep{hautus1983system,marro2002geometric}.
Lastly, the assumption that $\Bbf$ has full rank is designed to ensure controllability (for any transition matrix $\Abf$), a characteristic of many practical systems (\cf~\citet{hazan2022introduction}).
	
	\section{Training Cost May Have a Single Global Minimizer When $\Rbf \neq \0$}
\label{app:determined_R_neq_zero}

Our analysis considers LQR problems in which the cost matrix $\Rbf$ is zero~---~see \cref{sec:prelim:underdetermined}.
Along with the assumptions that $\Bbf$ is full rank and that the set $\trainstates$ of initial states seen in training does not span the state space, assuming that $\Rbf = \0$ ensures the training cost is underdetermined, \ie~multiple controllers attain its global minimum.

As \cref{lem:determined_shift_infinite_horizon} below shows, the assumption of $\Rbf = \0$ is necessary, in the sense that there exist settings where $\Rbf \neq 0$ and the training cost has a single global minimizer, despite $\Bbf$ being full rank and $\trainstates$ not spanning the state space.
On the other hand, there also exist settings where $\Rbf \neq 0$ and the training cost is underdetermined~---~see \cref{lem:underdetermined_R_nonzero} below.
Since it is non-trivial to completely characterize the conditions under which the training cost is underdetermined with $\Rbf \neq \0$, we regard further analyzing the case of $\Rbf \neq 0$ as suitable for future work. 

Note that the example given in \cref{lem:determined_shift_infinite_horizon} is of an infinite horizon LQR problem.
Preliminary indications lead us to believe that there also exist finite horizon problems with $\Rbf \neq \0$ whose training cost has a single global minimizer.
In particular, for every controller $\Kbf$ that stabilizes the system (\ie~the largest singular value of $\Abf + \Bbf \Kbf$ is less than one), the training cost with finite horizon $H$ converges exponentially fast to the training cost with infinite horizon as $H$ grows.
Thus, we expect that when the training cost with infinite horizon has a single global minimizer, the training cost with finite horizon $H$ will effectively have a single global minimizer as well, so long as $H$ is not especially small.
Meaning, even if there exist multiple controllers minimizing the finite horizon training cost, they should all be extremely close, and accordingly produce near identical controls.
Empirical evidence supports this prospect.
Namely, for a finite horizon variant of the LQR problem considered in \cref{lem:determined_shift_infinite_horizon}, we ran policy gradient from different initial controllers, whose entries were sampled independently from a Gaussian distribution with a relatively large standard deviation of $0.1$.
The maximal distance between any two controllers that policy gradient reached was extremely small~---~$0.000027$ ($0.00002$ when normalizing by the median norm of the controllers).
Furthermore, this maximal distance kept decaying as optimization progressed.

\begin{lemma}
\label{lem:determined_shift_infinite_horizon}
Consider the exploration-inducing setting of \cref{prop:shift}, \ie~$\trainstates = \brk[c]{ \ebf_1 }$ and $\Abf = \Ashift$.
Furthermore, suppose that $\Bbf = \Rbf = \Ibf$, where $\Ibf \in \R^{D \times D}$ stands for the identity matrix, and that the time horizon is infinite.
In this case the training cost of a controller $\Kbf \in \R^{D \times D}$ is given by:
\[
\cost (\Kbf; \trainstates) = \sum\nolimits_{h = 0}^\infty \brk[s]*{ \norm*{ (\Ashift + \Kbf)^h \ebf_1 }^2 + \norm*{ \Kbf (\Ashift + \Kbf)^h \ebf_1  }^2  }
\text{\,.}
\]
Then, the training cost $\cost (\cdot \, ; \trainstates)$ has a single global minimizer.
\end{lemma}

\begin{proof}
For a given state $\xbf \in \R^D$, the unique optimal control is given by (\cf~Chapter 2.4 in~\citet{anderson2007optimal}):
\[
\ubf^* = - (\Pbf + \Ibf)^{-1} \Pbf \Ashift \xbf
\text{\,,}
\]
where $\Pbf$ is the unique positive definite solution of the following discrete algebraic Riccati equation:
\[
\Pbf = \Ashift^\top \Pbf \Ashift + \Ibf - \Ashift^\top \Pbf ( \Pbf + \Ibf )^{-1} \Pbf \Ashift
\text{\,.}
\]
The control $\ubf^*$ is optimal in the sense that choosing any other control at state $\xbf$ leads to a suboptimal cost along the trajectory.
It can be straightforwardly verified that $\Pbf = \frac{1 + \sqrt{5}}{2} \cdot \Ibf$, and so $\ubf^* = - c \cdot \Ashift \xbf$ for $c = \frac{1 + \sqrt{5} }{ 3 + \sqrt{5} } \in (0, 1)$.
This implies that the controller $\Kbf^* = -c \cdot \Ashift$ minimizes the training cost.
Note that when applying the optimal control at state $\xbf$, the next state in the trajectory is $(1 - c) \cdot \Ashift \xbf$.
In particular, the optimally controlled trajectory emanating from the initial state seen in training $\ebf_1$ is $\ebf_1, (1-c) \cdot \ebf_2, (1-c)^2 \cdot \ebf_3, \ldots$.
In order for a controller to minimize the training cost, it must produce the unique optimal controls for states in this trajectory.
Since the trajectory spans the state space $\R^D$, these controls uniquely determine the controller, \ie~$\Kbf^*$ is the unique global minimizer of the training cost.
\end{proof}

\begin{lemma}
\label{lem:underdetermined_R_nonzero}
Assume that $\trainstates = \brk[c]{ \ebf_1 }$, $\Abf = \Bbf = \Rbf = \Ibf$, where $\Ibf \in \R^{D \times D}$ stands for the identity matrix, and that the time horizon is infinite.
In this case the training cost of a controller $\Kbf \in \R^{D \times D}$ is given by:
\[
\cost (\Kbf; \trainstates) = \sum\nolimits_{h = 0}^\infty \brk[s]*{ \norm*{ (\Ibf + \Kbf)^h \ebf_1 }^2 + \norm*{ \Kbf (\Ibf + \Kbf)^h \ebf_1  }^2  }
\text{\,.}
\]
Then, the training cost $\cost (\cdot \, ; \trainstates)$ is underdetermined, \ie~multiple controllers attain its global minimum.
\end{lemma}

\begin{proof}
For a given state $\xbf \in \R^D$, the unique optimal control is given by (\cf~Chapter 2.4 in~\citet{anderson2007optimal}):
\[
\ubf^* = - (\Pbf + \Ibf)^{-1} \Pbf \xbf
\text{\,,}
\]
where $\Pbf$ is the unique positive definite solution of the following discrete algebraic Riccati equation:
\[
\Pbf = \Pbf + \Ibf - \Pbf ( \Pbf + \Ibf )^{-1} \Pbf
\text{\,.}
\]
The control $\ubf^*$ is optimal in the sense that choosing any other control at state $\xbf$ leads to a suboptimal cost along the trajectory.
It can be straightforwardly verified that $\Pbf = \frac{1 + \sqrt{5}}{2} \cdot \Ibf$, and so $\ubf^* = - c \cdot \xbf$ for $c = \frac{1 + \sqrt{5} }{ 3 + \sqrt{5} } \in (0, 1)$.
This implies that the controller $\Kbf^* = -c \cdot \Ibf$ minimizes the training cost.
Note that when applying the optimal control at state $\xbf$, the next state in the trajectory is $(1 - c) \cdot \xbf$.
In particular, the optimally controlled trajectory emanating from the initial state seen in training $\ebf_1$ is $\ebf_1, (1-c) \cdot \ebf_1, (1-c)^2 \cdot \ebf_1, \ldots$.
Since every state in this trajectory is spanned by $\ebf_1$, for any $\Kbf' \in \R^{D \times D}$ with rows orthogonal to $\ebf_1$, the controller $\Kbf^* + \Kbf'$ produces the same (optimal) controls as $\Kbf^*$ when commencing from $\ebf_1$.
Hence, there exist infinitely many controllers that minimize the training cost.
\end{proof}

	\section{Extrapolation Measures Are Invariant to the Choice of Orthonormal Basis}
\label{app:extrapolation_measures_invariant}

This appendix establishes that the optimality and cost measures of extrapolation (\cref{def:opt_measure,def:cost_measure} in \cref{sec:prelim:extrapolation}, respectively) are invariant to the choice of orthonormal basis $\orthstates$ for $\trainstates^\perp$, where $\trainstates^\perp$ is the subspace orthogonal to the set $\trainstates$ of initial states seen in training.
That is, for any two such orthonormal bases $\orthstates$ and $\orthstates'$, the respective values of the optimality and cost measures are the same.

\begin{lemma}
\label{lem:extrapolation_measures_invariant}
For any controller $\Kbf \in \R^{D \times D}$, the optimality and cost measures of extrapolation are invariant to the choice of orthonormal basis $\orthstates$ for $\trainstates^\perp$.
\end{lemma}

\begin{proof}
Let $\orthstates$ be an orthonormal basis of $\trainstates^\perp$, and denote by $\Vbf \in \R^{D \times \abs{\orthstates}}$ the matrix whose columns are the initial states in $\orthstates$.
Furthermore, let $\Zbf \in \R^{D \times \dim (\lspan (\trainstates)) }$ be a matrix whose columns form an orthonormal basis for $\lspan (\trainstates)$.
Notice that the concatenated matrix $[\Vbf, \Zbf] \in \R^{D \times D}$ is an orthogonal matrix.

Now, for $\Kbf \in \R^{D \times D}$, the optimality measure of extrapolation can be written in a matricized form as follows:
\[
\optmes (\Kbf) = \frac{1}{ \abs{\orthstates} } \norm*{ (\Abf + \Bbf \Kbf) \Vbf}^2 = \frac{1}{ \abs{\orthstates} } \brk*{ \norm*{ (\Abf + \Bbf \Kbf) [\Vbf, \Zbf]}^2 - \norm*{ (\Abf + \Bbf \Kbf) \Zbf }^2 }
\text{\,.}
\]
Since the Euclidean norm is orthogonally invariant, we get that:
\[
\optmes (\Kbf) = \frac{1}{ \abs{\orthstates} } \brk*{ \norm*{ \Abf + \Bbf \Kbf }^2 - \norm*{ (\Abf + \Bbf \Kbf) \Zbf }^2 }
\text{\,.}
\]
As can be seen in the expression above, the optimality measure of extrapolation does not depend on the choice of $\orthstates$.

Similarly, for $\Kbf \in \R^{D \times D}$, the cost measure of extrapolation can be written in a matricized form as follows: 
\[
\begin{split}
\excesscost (\Kbf) & = \cost (\Kbf; \orthstates) - \cost^* ( \orthstates) \\
&  = \frac{1}{\abs{\orthstates}} \sum\nolimits_{h = 0}^H \norm*{ (\Abf + \Bbf \Kbf)^h \Vbf }^2 - 1 \\
&  = \frac{1}{\abs{\orthstates}} \sum\nolimits_{h = 1}^H \norm*{ (\Abf + \Bbf \Kbf)^h \Vbf }^2 \\
& = \frac{1}{\abs{\orthstates}} \sum\nolimits_{h = 1}^H \brk*{ \norm*{ (\Abf + \Bbf \Kbf)^h [\Vbf, \Zbf] }^2 - \norm*{ (\Abf + \Bbf \Kbf)^h \Zbf }^2 }
\text{\,,}
\end{split}
\]
where we used the fact that $\cost^* (\X) = 1$ for any finite set of unit norm initial states $\X \subset \R^D$.
Again, since the Euclidean norm is orthogonally invariant, we get an expression for $\excesscost (\Kbf)$ that does not depend on the choice of $\orthstates$:
\[
\excesscost (\Kbf) = \frac{1}{\abs{\orthstates}} \sum\nolimits_{h = 1}^H \brk*{ \norm*{ (\Abf + \Bbf \Kbf)^h }^2 - \norm*{ (\Abf + \Bbf \Kbf)^h \Zbf }^2 }
\text{\,.}
\]
\end{proof}
	
	\section{Extension of Analysis for Exploration-Inducing Setting to Diagonal $\Qbf$}
\label{app:extension_q}

In this appendix, we generalize the analysis of the exploration-inducing setting from \cref{sec:analysis:shift} to the case where $\Qbf$ is a general diagonal positive semidefinite matrix (not necessarily the identity matrix $\Ibf$).
The generalized analysis sheds light on how $\Qbf$ impacts extrapolation. 
In particular, it shows that for certain values of $\Qbf$ extrapolation can be perfect even for a finite horizon $H$ (recall that, as shown in \cref{sec:analysis:shift}, when $\Qbf = \Ibf$ perfect extrapolation in the setting considered therein is attained only when $H \to \infty$).

Let $\Qbf \in \R^{D \times D}$ be a diagonal positive semidefinite matrix with diagonal entries $q_1, \ldots, q_D \geq 0$, and assume that $q_j > 0$ for at least some $j \in [D]$ (otherwise, the problem is trivial~---~the cost for any controller and initial state is zero).
For such $\Qbf$, the cost in an underdetermined LQR problem (\cref{eq:cost}), attained by a controller $\Kbf \in \R^{D \times D}$ over a finite set $\X \subset \R^D$ of initial states, can be written as:
\be
\cost (\Kbf; \X) = \frac{ 1 }{ \abs{\X} } \sum\nolimits_{\xbf_0 \X} \sum\nolimits_{h = 0}^H \norm*{ (\Abf + \Bbf \Kbf)^h \xbf_0 }_\Qbf^2
\text{\,,}
\label{eq:cost_Q}
\ee
where $\norm{ \vbf }_{\Qbf} := \sqrt{\vbf^\top \Qbf \vbf}$ for $\vbf \in \R^D$.
The global minimum of this cost is:
\[
\cost^* (\X) := \min\nolimits_{\Kbf \in \R^{D \times D}} \cost (\Kbf; \X) = \frac{1}{ \abs{\X} } \sum\nolimits_{\xbf_0 \in \X} \norm{ \xbf_0 }_\Qbf^2
\text{\,,}
\]
and any controller $\Kbf$ that attains this global minimum satisfies:
\be
\norm*{ (\Abf + \Bbf \Kbf) \xbf_0 }_\Qbf^2  = 0
~~ , ~ \forall \xbf_0 \in \X
\text{\,.}
\label{eq:optimal_cond_general_Q}
\ee

Let $\trainstates \subset \R^D$ be a finite set of initial states seen in training and $\orthstates$ be an (arbitrary) orthonormal basis for $\trainstates^\perp$.
In our analysis of underdetermined LQR problems with $\Qbf = \Ibf$, we quantified extrapolation to initial states unseen in training via the optimality and cost measures over $\orthstates$ (\cref{def:opt_measure,def:cost_measure}, respectively).
\cref{def:opt_measure_general_q,def:cost_measure_general_q} extend the optimality and cost measures to the case of a non-identity $\Qbf$ matrix.
As shown in the subsequent \cref{lem:extrapolation_measures_invariant_general_q}, similarly to the the case of $\Qbf = \Ibf$, the generalized measures are invariant to the choice of $\orthstates$.

\begin{definition}
	\label{def:opt_measure_general_q}
	Let $\Qbf \in \R^{D \times D}$ be a positive semidefinite matrix.
	The \emph{$\Qbf$-optimality measure} of extrapolation for a controller $\Kbf \in \R^{D \times D}$ is:
	\[
	\optmes^\Qbf (\Kbf) := \frac{1}{ \abs{\orthstates} } \sum\nolimits_{\xbf_0 \in \orthstates } \norm*{ (\Abf + \Bbf \Kbf) \xbf_0 }_\Qbf^2 
	\text{\,.}
	\]
\end{definition}

\begin{definition}
\label{def:cost_measure_general_q}
Let $\Qbf \in \R^{D \times D}$ be a positive semidefinite matrix.
The \emph{$\Qbf$-cost measure} of extrapolation for a controller $\Kbf \in \R^{D \times D}$ is:
\[
\excesscost^\Qbf ( \Kbf ) := \cost ( \Kbf ; \orthstates) - \cost^* (\orthstates)
\text{\,,}
\]
where $\cost (\cdot \,  ; \orthstates)$ is as defined in \cref{eq:cost_Q}.
\end{definition}

\begin{lemma}
\label{lem:extrapolation_measures_invariant_general_q}
For any controller $\Kbf \in \R^{D \times D}$, the $\Qbf$-optimality and $\Qbf$-cost measures of extrapolation are invariant to the choice of orthonormal basis $\orthstates$ for $\trainstates^\perp$.
\end{lemma}

\begin{proof}[Proof sketch (proof in \cref{app:proofs:extrapolation_measures_invariant_general_q})]
The proof follows by arguments similar to those used for proving \cref{lem:extrapolation_measures_invariant}.
\end{proof}

With the generalized measures of extrapolation in hand, \cref{prop:shift_diag_q} below generalizes \cref{prop:shift} from \cref{sec:analysis:shift}.
Namely, for $\Abf  = \Ashift := \sum\nolimits_{d = 1}^D \ebf_{d \% D + 1} \ebf_d^\top$ and set $\trainstates = \{ \ebf_1 \}$ of initial states seen in training, \cref{prop:shift_diag_q} characterizes how the extent to which policy gradient extrapolates depends on the entries of $\Qbf$.
As was the case for $\Qbf = \Ibf$ (\cf~\cref{sec:analysis:shift}), the learned controller attains $\Qbf$-optimality and $\Qbf$-cost measures of extrapolation that are substantially less than those attained by $\Kminnorm$ (\cref{sec:prelim:extrapolation}).
This phenomenon is more potent the longer the horizon $H$ is, with perfect extrapolation attained in the limit $H \to \infty$.

An interesting consequence of considering a diagonal $\Qbf$, not necessarily equal to the identity matrix, is that it brings about another setting under which perfect extrapolation is achieved.
Specifically, if $q_2 = \cdots = q_D = 0$ and $q_1 > 0$, then for any $H$ divisible by $D$, the learned controller achieves zero $\Qbf$-optimality and cost measures.
The fact that such $\Qbf$ matrices lead to perfect extrapolation can be intuitively attributed to a “credit assignment'' mechanism of a policy gradient iteration.
Namely, due to the structure of $\Abf$, the trajectory of states induced by $\Kbf^{(1)} = \0$ when commencing from $\ebf_1$ consists of $H / D$ repetitions of the cycle $\ebf_1, \ebf_2, \ldots, \ebf_D, \ebf_1$.
A cost is incurred along this trajectory only at at the start of each cycle.
Thus, the components of $\nabla \cost (\Kbf^{(1)} ; \trainstates)$, which exactly align with those of $\Abf$, will be of the same magnitude, \ie~$\nabla \cost (\Kbf^{(1)}; \trainstates) = \sum\nolimits_{d = 1}^D \beta \cdot \ebf_{d \% D + 1} \ebf_d^\top$ for some $\beta > 0$.
Reducing the cost for $\ebf_1$ via a policy gradient iteration will therefore reduce the cost for initial states in $\orthstates$ by the same amount.
This is in contrast to the case of $\Qbf = \Ibf$, where the components of $\nabla \cost (\Kbf^{(1)} ; \trainstates)$ also aligned with those of $\Abf$, but have different magnitudes, thereby resulting in varying degrees of extrapolation to initial states in $\orthstates$.
	
\begin{proposition}
	\label{prop:shift_diag_q}
	Assume that $\trainstates = \{ \ebf_1\}$, $\Abf = \Ashift$, $H$ is divisible by $D$, and the cost matrix $\Qbf \in \R^{D \times D}$ has diagonal entries $q_1, \ldots, q_D \geq 0$, where $q_j > 0$ for at least some $j \in [D]$.
	Furthermore, let $\alpha_{d} := \brk1{ 2 \sum\nolimits_{j = 2}^{d} q_j } / \brk1{ \brk1{ \frac{H}{D} + 1 } \sum\nolimits_{j = 1}^D q_j } \in [0, 1)$ for $d \in [D]$.
	Then, policy gradient with learning rate $\eta = \brk1{ \frac{H}{D} \brk1{ \frac{H}{D} + 1 } \sum\nolimits_{j = 1}^D q_j }^{-1}$ converges to a controller $\Kpg$ that: \emph{(i)} minimizes the training cost, \ie~$\cost (\Kpg ; \trainstates) = \cost^* (\trainstates)$; and \emph{(ii)} satisfies:
	\[
		\begin{split}
		 \optmes^\Qbf \brk1{ \Kpg } &= \frac{ \sum\nolimits_{d = 2}^D q_{d \% D + 1} \cdot \alpha_d^2   }{ \sum\nolimits_{d = 2}^D q_{d \% D + 1} }	\cdot \optmes^\Qbf \brk1{ \Kminnorm }
		 \text{\,,} \\[0.5em]
		\excesscost^\Qbf (\Kpg) & = \frac{ \sum\nolimits_{d = 2}^D \sum\nolimits_{h = 1}^{D - d + 1} q_{(h + d - 1) \% D + 1} \cdot \prod\nolimits_{d' = d}^{h + d - 1} \alpha^2_{d'} }{ \sum\nolimits_{d = 2}^D \sum\nolimits_{h = 1}^{D - d + 1} q_{(h + d - 1) \% D + 1} } \cdot \excesscost (\Kminnorm )
		\text{\,,}
		\end{split}
	\]
	where by convention if $\sum\nolimits_{d = 2}^D q_{d \% D + 1} = 0$ then the right hand sides of both equations above are zero as well.
\end{proposition}

\begin{proof}[Proof sketch (full proof in~\cref{app:proofs:shift_diag_q})]
The proof follows a line identical to that of \cref{prop:shift}, generalizing it to account for a diagonal positive semidefinite $\Qbf$ (as opposed to $\Qbf = \Ibf$).
\end{proof}

	\section{Random Systems Generically Induce Exploration}
\label{app:cyc_vector_generic}

Below, we formally state and prove the claim made in \cref{sec:analysis:general} regarding random transition matrices generically inducing exploration.

\begin{lemma}
	\label{lem:cyc_vector_generic}
	Given a non-zero $\xbf \in \R^D$, suppose that $\Abf \in \R^{D \times D}$ is generated randomly from a continuous distribution whose support is $\R^{D \times D}$.
	Then, $\xbf, \Abf \xbf, \ldots \Abf^{D - 1} \xbf$ form a basis of $\R^D$ almost surely (\ie~$\xbf$ is a cyclic vector of $\Abf$ almost surely).
\end{lemma}

\begin{proof}
	Denote by $\Ybf \in \R^{D \times D}$ the matrix whose columns are $\xbf, \Abf \xbf, \ldots, \Abf^{D - 1} \xbf$.
	Note that $\xbf$ is a cyclic vector of $\Abf$ if and only if the determinant of $\Ybf$, which is polynomial in the entries of $\Abf$, is non-zero.
	The zero set of a polynomial is either the entire space or a set of Lebesgue measure zero~\cite{caron2005zero}.
	Hence, it suffices to show that there exists an $\Abf$ such that the determinant of $\Ybf$ is non-zero, since that implies the set of matrices for which $\xbf$ is not a cyclic vector has probability zero.
	To see that such $\Abf$ exists, let $\zbf_1, \ldots, \zbf_{D - 1} \in \R^D$ be vectors completing $\xbf$ into a basis of $\R^D$.
	We can take $\Abf$ to be a matrix satisfying $\zbf_1 = \Abf \xbf$ and $\zbf_{d + 1} = \Abf \zbf_{d}$ for $d \in [D - 2]$ (the way $\Abf$ transforms $\zbf_{D - 1}$ can be chosen arbitrarily).
	Under this choice of $\Abf$, the columns of $\Ybf$, \ie~$\xbf, \Abf \xbf, \ldots, \Abf^{D - 1} \xbf$, are respectively equal to $\xbf, \zbf_1, \ldots, \zbf_{D - 1}$.
	Thus, $\Ybf$ is full rank and its determinant is non-zero.
\end{proof}
	
	\section{Deferred Proofs}
\label{app:proofs}

In this appendix, we provide full proofs for our theoretical results.

\textbf{Additional notation.} 
Throughout the proofs, we use $\Tr (\Cbf)$ to denote the trace of a matrix $\Cbf$.

\subsection{Cost Minimizing Controllers in an Underdetermined LQR Problem}
\label{app:proofs:cost_min_underdetermined}

We restate and prove the claim made in \cref{sec:prelim:underdetermined} regarding controllers that minimize the cost in an underdetermined LQR problem, for a given set of initial states.

\begin{lemma}
\label{lem:cost_min_underdetermined}
Let $\X \subset \R^D$ be an arbitrary finite set of initial states.
The global minimum of the cost $\cost (\cdot \, ; \X)$ (defined in \cref{eq:cost}) is:
\[
\cost^* (\X) := \min\nolimits_{\Kbf \in \R^{D \times D}} \cost (\Kbf; \X) = \frac{1}{ \abs{\X} } \sum\nolimits_{\xbf_0 \in \X} \norm{ \xbf_0 }^2
\text{\,.}
\]
Furthermore, a controller $\Kbf \in \R^{D \times D}$ attains this global minimum if and only if $\Kbf \xbf_0 = -\Bbf^{-1} \Abf \xbf_0$ for all $\xbf_0 \in \X$.
\end{lemma}

\begin{proof}
Let $\Kbf^* := - \Bbf^{-1} \Abf$.
Notice that this controller attains the following cost over $\X$:
\[
\begin{split}
\cost (\Kbf^* ; \X)  & =
\frac{1}{\abs{\X}}  \sum\nolimits_{\xbf_0 \in \X} \sum\nolimits_{h = 0}^H \norm*{ (\Abf + \Bbf \Kbf^*)^h \xbf_0 }^2 \\
& = \frac{1}{\abs{\X}}  \sum\nolimits_{\xbf_0 \in \X} \underbrace{\norm*{ ( \Abf + \Bbf \Kbf^*)^0 \xbf_0 }^2}_{ = \norm{\xbf_0}^2 }  + \frac{1}{\abs{\X}}  \sum\nolimits_{\xbf_0 \in \X} \sum\nolimits_{h = 1}^H \underbrace{ \norm*{ (\Abf + \Bbf \Kbf^*)^h \xbf_0 }^2 }_{ = 0 } \\
& =  \frac{1}{\abs{\X}} \sum\nolimits_{\xbf_0 \in \X} \norm*{\xbf_0 }^2
\text{\,.}
\end{split}
\]
Hence, $\cost^* (\X) \leq \frac{1}{\abs{\X}}  \sum\nolimits_{\xbf_0 \in \X} \norm*{\xbf_0 }^2$.
On the other hand, for any $\Kbf \in \R^{D \times D}$:
\[
\cost (\Kbf ; \X)  =
\frac{1}{\abs{\X}} \sum\nolimits_{\xbf_0 \in \X} \sum\nolimits_{h = 0}^H \norm*{ (\Abf + \Bbf \Kbf)^h \xbf_0 }^2 \geq \frac{1}{\abs{\X}}  \sum\nolimits_{\xbf_0 \in \X} \norm*{\xbf_0 }^2
\text{\,,}
\]
since for each initial state $\xbf_0 \in \X$, the term in the cost corresponding to time step $h = 0$ does not depend on $\Kbf$ and is equal to $\norm{\xbf_0}^2$.
This implies that $ \cost^* (\X) \geq \frac{1}{\abs{\X}}  \sum\nolimits_{\xbf_0 \in \X} \norm*{\xbf_0 }^2$, and so $\cost^* (\X) = \frac{1}{\abs{\X}}  \sum\nolimits_{\xbf_0 \in \X} \norm*{\xbf_0 }^2$.

Now, we show that a controller $\Kbf \in \R^{D \times D}$ attains the global minimum $\cost^* (\X)$ if and only if $\Kbf \xbf_0 = -\Bbf^{-1} \Abf \xbf_0$ for all $\xbf_0 \in \X$.
In the first direction, any $\Kbf \in \R^{D \times D}$ satisfying $\Kbf \xbf_0 = -\Bbf^{-1} \Abf \xbf_0$ for all $\xbf_0 \in \X$ upholds:
\[
\begin{split}
\cost (\Kbf ; \X)  & =
\frac{1}{\abs{\X}}  \sum\nolimits_{\xbf_0 \in \X} \sum\nolimits_{h = 0}^H \norm*{ (\Abf + \Bbf \Kbf)^h \xbf_0 }^2 \\
& = \frac{1}{\abs{\X}}  \sum\nolimits_{\xbf_0 \in \X} \underbrace{\norm*{ ( \Abf + \Bbf \Kbf)^0 \xbf_0 }^2}_{ = \norm{\xbf_0}^2 }  + \frac{1}{\abs{\X}}  \sum\nolimits_{\xbf_0 \in \X} \sum\nolimits_{h = 1}^H \underbrace{ \norm*{ (\Abf + \Bbf \Kbf)^{h - 1} (\Abf \xbf_0 + \Bbf \Kbf \xbf_0) }^2 }_{ = 0 } \\
& =  \cost^* (\X)
\text{\,.}
\end{split}
\]
In the other direction, let $\Kbf \in \R^{D \times D}$ be a controller for which $\cost (\Kbf; \X) = \cost^* (\X) =  \frac{1}{\abs{\X}} \sum\nolimits_{\xbf_0 \in \X} \norm*{\xbf_0 }^2$.
The cost attained by $\Kbf$ can be decomposed as done above:
\[
\begin{split}
	\cost (\Kbf ; \X)  & =
	\frac{1}{\abs{\X}}  \sum\nolimits_{\xbf_0 \in \X} \sum\nolimits_{h = 0}^H \norm*{ (\Abf + \Bbf \Kbf)^h \xbf_0 }^2 \\
	& = \cost^* (\X)  + \frac{1}{\abs{\X}} \sum\nolimits_{\xbf_0 \in \X} \sum\nolimits_{h = 1}^H \norm*{ (\Abf + \Bbf \Kbf)^h \xbf_0 }^2
	\text{\,.}
\end{split}
\]
Each term in the cost is non-negative.
Since $\cost (\Kbf ; \X) = \cost^* (\X)$, this implies that $\norm{ (\Abf + \Bbf \Kbf)^h \xbf_0 }^2 = 0$ for every $\xbf_0 \in \X$ and $h \in [H]$.
Focusing on time step $h = 1$, we get that $\Kbf$ satisfies $\norm{ (\Abf + \Bbf \Kbf) \xbf_0 }^2 = 0$ for all $\xbf_0 \in \X$.
Consequently, for all $\xbf_0 \in \X$:
\[
\Abf \xbf_0 = - \Bbf \Kbf \xbf_0
\text{\,.}
\]
Recalling that $\Bbf$ is invertible, we conclude that $\Kbf \xbf_0 = - \Bbf^{-1} \Abf \xbf_0$ for all $\xbf_0 \in \X$.
\end{proof}

\subsection{Gradient of the Cost in an LQR Problem}
\label{app:proofs:lqr_gradient}

Throughout, we make use of the following expression for the gradient of the cost in an underdetermined LQR problem (\cref{sec:prelim:underdetermined}).

\begin{lemma}
\label{lem:lqr_gradient}
Consider an underdetermined LQR problem defined by $\Abf, \Bbf \in \R^{D \times D}$, and positive semidefinite $\Qbf \in \R^{D \times D}$.
For any finite set of initial states $\X \subset \R^D$, the gradient of the cost $\cost (\cdot \,; \X)$ (\cref{eq:cost}) at $\Kbf \in \R^{D \times D}$ is given by:
\[
\nabla \cost (\Kbf; \X) = 2 \Bbf^\top \sum\nolimits_{h = 0}^{H - 1} \brk2{ \sum\nolimits_{s = 1}^{H - h} \brk[s]{ ( \Abf + \Bbf \Kbf )^{s - 1} }^\top \Qbf (\Abf + \Bbf \Kbf)^s } \Sigmabf_{\X, h}
\text{\,,}
\]
with $\Sigmabf_{\X, h} := \frac{1}{ \abs{\X} } \sum\nolimits_{\xbf_0 \in \X} \xbf_h \xbf_h^\top = \frac{1}{ \abs{\X} } \sum\nolimits_{\xbf_0 \in \X} (\Abf + \Bbf \Kbf)^h \xbf_0 \brk[s]{ (\Abf + \Bbf \Kbf)^h \xbf_0 }^\top$ for $h \in \{0\} \cup [H - 1]$.
\end{lemma}

\begin{proof}
	Notice that $\cost (\Kbf; \X)$ can be written as:
	\[
	\cost (\Kbf; \X) = \inprodBig{ \sum\nolimits_{h = 0}^H [ (\Abf + \Bbf \Kbf)^h ]^\top \Qbf (\Abf + \Bbf \Kbf)^h }{  \Sigmabf_{\X, 0} }
	\text{\,.}
	\]
	A straightforward computation shows that for any $\Delta \in \R^{D \times D}$:
	\[
	\begin{split}
		\cost (\Kbf + \Delta; \X) & = \cost (\Kbf ; \X) \\
		& \hspace{3mm} + \underbrace{ \inprodBig{ \sum\nolimits_{h = 1}^H \sum\nolimits_{s = 0}^{h - 1} [ (\Abf + \Bbf \Kbf)^{s} \Bbf \Delta (\Abf + \Bbf \Kbf)^{h - s - 1} ]^\top \Qbf (\Abf + \Bbf \Kbf)^h }{  \Sigmabf_{\X, 0} } }_{ (I) } \\
		& \hspace{3mm} + \underbrace{ \inprodBig{ \sum\nolimits_{h = 1}^H \sum\nolimits_{s = 0}^{h - 1} [ (\Abf + \Bbf \Kbf)^h ]^\top \Qbf (\Abf + \Bbf \Kbf)^{h - s - 1} \Bbf \Delta (\Abf + \Bbf \Kbf)^{s} }{  \Sigmabf_{\X, 0} } }_{(II)} \\
		& \hspace{3mm} + o ( \norm{\Delta} )
		\text{\,.}
	\end{split}
	\]
	Then, the identity $\Tr (\Xbf^\top \Ybf) = \Tr (\Xbf \Ybf^\top) = \inprod{\Xbf}{\Ybf}$ for any matrices $\Xbf, \Ybf$ of the same dimensions, along with the cyclic property of the trace, leads to:
	\[
		(I) = (II) = \inprodBig{  \sum\nolimits_{h = 1}^H \sum\nolimits_{s = 0}^{h - 1} \Bbf^\top \brk[s]1{ (\Abf + \Bbf \Kbf)^{h - s - 1} }^\top \Qbf (\Abf + \Bbf \Kbf)^{h - s} \Sigmabf_{\X, s} }{  \Delta }
		\text{\,,}
	\]
	from which we get:
	\[
		\cost (\Kbf + \Delta; \X) = \cost (\Kbf; \X) + 2  \inprodBig{  \sum\nolimits_{h = 1}^H \sum\nolimits_{s = 0}^{h - 1} \Bbf^\top \brk[s]1{ (\Abf + \Bbf \Kbf)^{h - s - 1} }^\top \Qbf (\Abf + \Bbf \Kbf)^{h - s} \Sigmabf_{\X, s} }{  \Delta } + o ( \norm{ \Delta } )
		\text{\,.}
	\]
	Since $\nabla \cost (\Kbf; \X)$ is the unique linear approximation of $\cost (\cdot\, ; \X)$ at $\Kbf$, it follows that:
	\[
		\begin{split}
		\nabla \cost (\Kbf; \X) & = 2 \sum\nolimits_{h = 1}^H \sum\nolimits_{s = 0}^{h - 1} \Bbf^\top \brk[s]1{ (\Abf + \Bbf \Kbf)^{h - s - 1} }^\top \Qbf (\Abf + \Bbf \Kbf)^{h - s} \Sigmabf_{\X, s} \\
		& = 2 \Bbf^\top\sum\nolimits_{h = 1}^H \sum\nolimits_{s = 0}^{h - 1} \brk[s]1{ (\Abf + \Bbf \Kbf)^{h - s - 1} }^\top \Qbf (\Abf + \Bbf \Kbf)^{h - s} \Sigmabf_{\X, s}
		\text{\,.}
		\end{split}
	\]
	The proof concludes by grouping terms with $\Sigmabf_{\X, h}$, for each $h \in \{0\} \cup [H - 1]$.
\end{proof}

\subsection{Proof of~\cref{prop:no_exploration_no_extrapolation}}
\label{app:proofs:no_exploration_no_extrapolation}

\textbf{Exploration is necessary for extrapolation.}
From~\cref{lem:lqr_gradient}, the gradient of $\cost ( \cdot \, ; \trainstates )$ at any $\Kbf \in \R^{D \times D}$ takes on the following form:
\[
\nabla \cost (\Kbf; \trainstates) = 2 \Bbf^\top \sum\nolimits_{h = 0}^{H - 1} \brk2{ \sum\nolimits_{s = 1}^{H - h} \brk[s]{ ( \Abf + \Bbf \Kbf )^{s - 1} }^\top \Qbf (\Abf + \Bbf \Kbf)^s } \Sigmabf_{\trainstates, h}
\text{\,,}
\]
where $\Sigmabf_{\trainstates, h} := \frac{1}{ \abs{\trainstates} } \sum\nolimits_{\xbf_0 \in \trainstates} (\Abf + \Bbf \Kbf)^h \xbf_0 \brk[s]{ (\Abf + \Bbf \Kbf)^h \xbf_0 }^\top$ for $h \in \{0\} \cup [H - 1]$.
Thus, at every policy gradient iteration $t \in \N$, the rows of $\nabla \cost (\Kbf^{(t)} ; \trainstates)$ are in the span of $\brk[c]{ (\Abf + \Bbf \Kbf^{(t)})^h \xbf_0 : \xbf_0 \in \trainstates, h \in \{0\} \cup [H - 1] }$, \ie~in the span of the states encountered when starting from initial states in $\trainstates$ and using the controller $\Kbf^{(t)}$.
Since $\Kbf^{(1)} = \0$, at every iteration $t \in \N$, the rows of $\Kbf^{(t)} = - \eta \sum\nolimits_{i = 1}^{t - 1} \nabla \cost (\Kbf^{(i)} ; \trainstates)$ are in the span of $\pgstates$.
Consequently, for any initial state $\vbf_0 \in \pgstates^\perp$ and $t \in \N$ we have that $\Kbf^{(t)} \vbf_0 = \0$.
On the other hand, for the non-extrapolating controller $\Kminnorm$ (defined in \cref{eq:unseen_controls_zero}) it also holds that $\Kminnorm \vbf_0 = \0$ for any $\vbf_0 \in \pgstates^\perp$, as $\pgstates^\perp \subseteq \trainstates^\perp$.
Thus, if $\pgstates \subseteq \lspan (\trainstates)$, then $\Kbf^{(t)} \vbf_0 = \Kminnorm \vbf_0 = \0$ for any $\vbf_0 \in\orthstates$, and:
\[
\optmes (\Kbf^{(t)}) = \frac{1}{ \abs{\orthstates} } \sum\nolimits_{ \vbf_0 \in \orthstates } \norm*{ ( \Abf + \Bbf \Kbf^{(t)} ) \vbf_0 }^2 = \frac{1}{ \abs{\orthstates} } \sum\nolimits_{ \vbf_0 \in \orthstates } \norm*{ \Abf \vbf_0 }^2 = \optmes ( \Kminnorm ) 
\text{\,.}
\]

\textbf{Existence of non-exploratory systems.}
Let $\Abf = \Bbf = \Ibf \in \R^{D \times D}$, where $\Ibf$ is the identity matrix.

We first prove that $\pgstates \subseteq \lspan (\trainstates)$.
To do so, it suffices to prove that for all $t \in \N$ the rows and columns of  $\Kbf^{(t)}$ are spanned by the set $\trainstates$ of initial states seen in training.
Indeed, in such a case $\Abf + \Bbf \Kbf^{(t)} = \Ibf + \Kbf^{(t)}$ is invariant to $\lspan (\trainstates)$, \ie~for any $\xbf \in \lspan (\trainstates)$ it holds that $(\Ibf + \Kbf^{(t)}) \xbf = \xbf + \Kbf^{(t)} \xbf \in \lspan (\trainstates)$, from which it readily follows that $\pgstates = \brk[c]{ (\Ibf + \Kbf^{(t)})^h \xbf_0 :  \xbf_0 \in \trainstates , h \in \{0\} \cup [H] , t \in \N } \subseteq \lspan (\trainstates)$.

We prove that the rows and columns of $\Kbf^{(t)}$ are spanned by $\trainstates$ by induction over $t \in \N$.
The base case of $t = 1$ is trivial since $\Kbf^{(1)} = \0$.
Assuming that the inductive claim holds for $t - 1 \in \N$, we show that it holds for $t$ as well.
According to \cref{lem:lqr_gradient}:
\[
\nabla \cost (\Kbf^{(t - 1)}; \trainstates) = 2 \sum\nolimits_{h = 0}^{H - 1} \brk2{ \sum\nolimits_{s = 1}^{H - h} \brk[s]{ ( \Ibf + \Kbf^{(t - 1)} )^{s - 1} }^\top (\Ibf+ \Kbf^{(t - 1)} )^s } \Sigmabf^{(t - 1)}_{\trainstates, h}
\text{\,,}
\]
where $\Sigmabf^{(t - 1)}_{\trainstates, h} := \frac{1}{ \abs{\trainstates} } \sum\nolimits_{\xbf_0 \in \trainstates} (\Ibf + \Kbf^{(t - 1)})^h \xbf_0 \brk[s]{ (\Ibf + \Kbf^{(t - 1)} )^h \xbf_0 }^\top$ for $h \in \{0\} \cup [H - 1]$.
By the inductive assumption, the rows and columns of $\Kbf^{(t - 1)}$ are in $\lspan (\trainstates)$.
Hence, both $\Ibf + \Kbf^{(t - 1)}$ and $( \Ibf + \Kbf^{(t - 1)} )^\top$ are invariant to $\lspan (\trainstates)$.
This implies that $(\Ibf + \Kbf^{(t - 1)})^h \xbf_0 \in \lspan (\trainstates)$ and $[( \Ibf + \Kbf^{(t - 1)} )^{s - 1}]^\top (\Ibf + \Kbf^{(t - 1)})^{h + s} \xbf_0 \in \lspan (\trainstates)$ for all $\xbf_0 \in \trainstates, h \in \{0\} \cup [H - 1],$ and $s \in [H - h]$.
Consequently, $\nabla \cost (\Kbf^{(t - 1)} ; \trainstates)$ is a sum of outer products between vectors that reside in $\lspan (\trainstates)$, and so its rows and columns are in $\lspan (\trainstates)$.
Along with the inductive assumption, we thus conclude that the rows and columns of $\Kbf^{(t)} = \Kbf^{(t - 1)} - \eta \cdot \nabla \cost ( \Kbf^{(t - 1)} ; \trainstates)$ are in $\lspan (\trainstates)$ as well.

\medskip

We now turn to prove that:
\[
\begin{split}
	& \optmes (\Kbf^{(t)}) = \optmes ( \Kminnorm ) = 1
	\text{\,,} \\[0.3em]
	& \excesscost (\Kbf^{(t)} ) = \excesscost ( \Kminnorm ) = H 
	\text{\,.}
\end{split}	
\]
As shown above, $\pgstates \subseteq \lspan (\trainstates)$, and so, by the first part of the proof, $\Kbf^{(t)} \vbf_0 =  \Kminnorm \vbf_0 = \0$ for any $\vbf_0 \in \orthstates \subseteq \pgstates^\perp$.
This implies that $(\Ibf + \Kbf^{(t)}) \vbf_0 = \vbf_0$ for any $\vbf_0 \in \orthstates$, from which it follows that:
\[
\optmes (\Kbf^{(t)}) = \frac{1}{ \abs{\orthstates} } \sum\nolimits_{ \vbf_0 \in \orthstates } \norm*{ ( \Ibf + \Kbf^{(t)} ) \vbf_0 }^2 = \frac{1}{ \abs{\orthstates} } \sum\nolimits_{ \vbf_0 \in \orthstates } \norm*{ \vbf_0 }^2 = 1
\text{\,.}
\]
Noticing that $\cost^* (\orthstates) = 1$ (\eg, this minimal cost is attained by $\Koptall$, defined in \cref{eq:unseen_controls_opt}), we similarly get:
\[
\excesscost (\Kbf^{(t)}) =  \frac{1}{ \abs{\orthstates} } \sum\nolimits_{ \vbf_0 \in \orthstates } \sum\nolimits_{h = 0}^H \norm*{ (\Ibf + \Kbf^{(t - 1)})^h \vbf_0 }^2 - 1 =   \frac{1}{ \abs{\orthstates} } \sum\nolimits_{ \vbf_0 \in \orthstates }  \sum\nolimits_{h = 0}^H \norm*{ \vbf_0 }^2 - 1 = H
\text{\,.}
\]
Additionally, by the definition of $\Kminnorm$ (\cref{eq:unseen_controls_zero}), we know that $(\Ibf + \Kminnorm) \vbf_0 = \vbf_0$ for any $\vbf_0 \in \orthstates$.
By the same computation made above for $\Kbf^{(t)}$, we thus get that the optimality and cost measures of extrapolation attained by $\Kminnorm$ over $\orthstates$ are equal to those attained by $\Kbf^{(t)}$.
\qed

\subsection{Proof of~\cref{prop:shift}}
\label{app:proofs:shift}

We begin by deriving an explicit formula for $\Kbf^{(2)}$ in \cref{lem:shift_Kpg_K2}, from which it follows that policy gradient converges in a single iteration to $\Kpg = \Kbf^{(2)}$.

\begin{lemma}
\label{lem:shift_Kpg_K2}
Policy gradient converges in a single iteration to:
\[
\Kpg = \Kbf^{(2)} = - \Bbf^\top \sum\nolimits_{d = 1}^D \brk*{ 1 - \frac{2(d - 1)}{H + D} } \cdot \ebf_{d \% D + 1} \ebf_d^\top
\text{\,,}
\]
which minimizes the training cost, \ie~$\cost (\Kpg ; \trainstates ) = \cost (\Kbf^{(2)} ; \trainstates ) = \cost^* (\trainstates)$.
\end{lemma}

\begin{proof}
For $\trainstates = \{ \ebf_1 \}$, by~\cref{lem:lqr_gradient}, the gradient of the training cost at $\Kbf^{(1)} = \0$ is given by:
\[
\nabla \cost ( \0  ; \ebf_1) = 2 \Bbf^\top \sum\nolimits_{h = 0}^{H - 1} \brk2{ \sum\nolimits_{s = 1}^{H - h} \brk[s]{ \Ashift^{s - 1} }^\top \Ashift^s } \Sigmabf_{\ebf_1, h}
\text{\,,}
\]
where $\Ashift = \sum\nolimits_{d = 1}^D \ebf_{d \% D + 1} \ebf_d^\top$ and $\Sigmabf_{\ebf_1, h} := \Ashift^h \ebf_1 \brk[s]{ \Ashift^h \ebf_1 }^\top = \ebf_{ h \% D + 1 } \ebf_{ h \% D + 1 }^\top$ for $h \in \{0\} \cup [H - 1]$.
Notice that $\Ashift \Ashift^\top = \Ibf$, \ie~$\Ashift$ is an orthogonal matrix.
Hence, $\brk[s]{ \Ashift^{s - 1} }^\top \Ashift^s =\Ashift$ for all $s \in [H]$ and:
\[
\begin{split}
\nabla \cost ( \0 ; \ebf_ 1) & = 2 \Bbf^\top \sum\nolimits_{h = 0}^{H - 1} \sum\nolimits_{s = 1}^{H - h} \Ashift \ebf_{h \% D + 1} \ebf_{h \% D + 1}^\top \\
& = 2 \Bbf^\top \sum\nolimits_{h = 0}^{H - 1} \sum\nolimits_{s = 1}^{H - h} \ebf_{(h + 1) \% D + 1} \ebf_{h \% D + 1}^\top \\
& = 2 \Bbf^\top \sum\nolimits_{h = 0}^{H - 1} (H - h) \cdot \ebf_{(h + 1) \% D + 1} \ebf_{h \% D + 1}^\top
\text{\,.}
\end{split}
\]
Recalling that $H = D \cdot L$ for some $L \in \N$, there are exactly $L = \frac{H}{D}$ terms in the sum corresponding to $\ebf_{d \% D + 1} \ebf_d^\top$, for each $d \in [D]$.
Focusing on elements $h \in \{0, D, 2D, \ldots, H - D\}$ in the sum, which satisfy $h \% D + 1 = 1$, the sum of coefficients for $\ebf_2 \ebf_1^\top$ is given by $H + (H - D) + \cdots + D = (\frac{H^2}{D} + H) \cdot 2^{-1}$.
More generally, for $d \in [D]$, the sum of coefficients for $\ebf_{d \% D + 1} \ebf_d^\top$ is $(H - d + 1) + (H - D - d + 1) + \cdots + (D - d + 1) = (\frac{H^2}{D} + H) \cdot 2^{-1}  - (d - 1) \frac{H}{D}$.
Thus, we may write:
\[
\nabla \cost ( \0 ; \ebf_ 1)  = \Bbf^\top \sum\nolimits_{d = 1}^{D} \brk*{ \brk*{ \frac{H^2}{D} + H } - \frac{2 (d - 1) H}{D} } \cdot \ebf_{d \% D + 1} \ebf_{d}^\top
\text{\,,}
\]
which, combined with $\Kbf^{(1)} = \0$ and $\eta = (H^2 / D + H)^{-1}$, leads to the sought-after expression for $\Kbf^{(2)}$:
\[
\Kbf^{(2)} = \Kbf^{(1)} - \eta \cdot \nabla \cost ( \Kbf^{(1)} ; \ebf_1 ) = - \Bbf^\top \sum\nolimits_{d = 1}^{D} \brk*{1 - \frac{ 2 (d - 1) }{H + D} } \cdot \ebf_{d \% D + 1} \ebf_{d}^\top
\text{\,.}
\]
To see that $\Kbf^{(2)}$ minimizes the training cost, notice that:
\[
(\Ashift + \Bbf \Kbf^{(2)}) \ebf_1 = \Ashift \ebf_1 - \Bbf \Bbf^\top \sum\nolimits_{d = 1}^{D} \brk*{1 - \frac{ 2(d - 1) }{H + D} } \cdot \ebf_{d \% D + 1} \ebf_{d}^\top \ebf_1 = \ebf_2 - \ebf_2 = \0
\text{\,,}
\]
where the second equality is due to $\Bbf \Bbf^\top = \Ibf$ and $\ebf_d^\top \ebf_1 = 0$ for $d \in \{2, \ldots, D\}$.
Consequently, $\cost (\Kbf^{(2)} ; \trainstates) = \sum\nolimits_{h = 0}^H \norm{ (\Ashift + \Bbf \Kbf^{(2)} )^h \ebf_1 }^2 = \norm{ \ebf_1 }^2 = 1$, which is the minimal training cost $\cost^* (\trainstates)$ since for any $\Kbf \in \R^{D \times D}$ the cost is a sum of $H + 1$ non-negative terms, with the one corresponding to $h = 0$ being equal to $\norm{\ebf_1}^2 = 1$.
\end{proof}

\medskip

\textbf{Extrapolation in terms of the optimality measure.} Next, we characterize the extent to which $\Kpg = \Kbf^{(2)}$ extrapolates, as measured by the optimality measure.
As shown by \cref{lem:extrapolation_measures_invariant} in \cref{app:extrapolation_measures_invariant}, the optimality measure is invariant to the choice of orthonormal basis $\orthstates$ for $\trainstates^\perp$.
Thus, because $\trainstates = \{\ebf_1\}$ we may assume without loss of generality that $\orthstates = \{ \ebf_2, \ldots, \ebf_D \}$.

For any $\ebf_d \in \orthstates$, by the definition of $\Kminnorm$ (\cref{eq:unseen_controls_zero}) we have that $(\Ashift + \Bbf \Kminnorm) \ebf_d = \Ashift \ebf_d = \ebf_{d \% D + 1}$.
Hence:
\be
\optmes  (\Kminnorm) = \frac{1}{D - 1} \sum\nolimits_{d = 2}^D \norm*{ (\Ashift + \Bbf \Kminnorm) \ebf_d }^2 = \frac{1}{D - 1} \sum\nolimits_{d = 2}^D \norm*{  \ebf_{d \% D + 1} }^2 = 1
\text{\,.}
\label{eq:shift_proof_Kminnorm_sf}
\ee
On the other hand, by \cref{lem:shift_Kpg_K2} for any $\ebf_d \in \orthstates$:
\[
\begin{split}
(\Ashift + \Bbf \Kpg) \ebf_d & = \ebf_{d \% D + 1} - \Bbf \Bbf^\top \sum\nolimits_{d' = 1}^{D} \brk*{1 - \frac{ 2 (d' - 1) }{H + D} } \cdot \ebf_{d' \% D + 1} \ebf_{d'}^\top \ebf_d \\
& =  \ebf_{d \% D + 1} - \brk*{1 - \frac{ 2 (d - 1) }{H + D} } \cdot \ebf_{d \% D + 1} \\
& = \frac{2 (d - 1) }{H + D} \cdot \ebf_{d \% D + 1}
\text{\,,}
\end{split}
\]
and so:
\be
 \optmes \brk*{ \Kpg } = \frac{1}{D - 1} \sum\nolimits_{d = 2}^D \norm*{ (\Ashift + \Bbf \Kpg) \ebf_d }^2 =  \frac{4  \sum\nolimits_{d = 2}^D (d - 1)^2 }{ (D - 1) (H + D)^2 }
\label{eq:shift_proof_K2_transition}
\ee
The desired guarantee on extrapolation in terms of the optimality measure follows from~\cref{eq:shift_proof_Kminnorm_sf,eq:shift_proof_K2_transition}:
\[
\frac{ \optmes \brk1{ \Kpg } }{  \optmes \brk*{ \Kminnorm } } = \frac{ 4 \sum\nolimits_{d = 2}^D (d - 1)^2 }{ (D - 1) (H + D)^2 } \leq \frac{ 4(D - 1)^2 }{ (H + D)^2 }
\text{\,.}
\]

\medskip

\textbf{Extrapolation in terms of the cost measure.}
Lastly, we characterize the extent to which $\Kpg = \Kbf^{(2)}$ extrapolates, as quantified by the cost measure.
As done above for proving extrapolation in terms of the optimality measure, by \cref{lem:extrapolation_measures_invariant} in \cref{app:extrapolation_measures_invariant} we may assume without loss of generality that $\orthstates = \{ \ebf_2, \ldots, \ebf_D\}$.

Fix some $\ebf_d \in \orthstates$.
We use the fact that $\Kpg = \Kbf^{(2)} = - \Bbf^\top \sum\nolimits_{d = 1}^D  (1 - \frac{d - 1}{H + D}) \cdot \ebf_{d \% D + 1} \ebf_d^\top$ (\cref{lem:shift_Kpg_K2}) to straightforwardly compute $\excesscost (\Kpg )$.
Specifically, recalling that $\Bbf \Bbf^\top = \Ibf$, we have that $\Ashift + \Bbf \Kpg = \sum\nolimits_{d' = 2}^D \frac{ 2(d' - 1) }{ H + D} \cdot \ebf_{d' \% D + 1} \ebf_{d'}^\top$.
Now, for any $h \in [H]$:
\[
\begin{split}
	(\Ashift + \Bbf \Kpg)^h \ebf_d & = (\Ashift + \Bbf \Kpg)^{h - 1}  \sum\nolimits_{d' = 2}^D \frac{ 2(d' - 1) }{ H + D} \cdot \ebf_{d' \% D + 1} \ebf_{d'}^\top \ebf_d \\
	& =  \frac{ 2 (d - 1) }{ H + D } \cdot (\Ashift + \Bbf \Kpg )^{h - 1} \ebf_{d \% D + 1}
	\text{\,.}
\end{split}
\]
If $h \leq D - d + 1$, unraveling the recursion from $h - 1$ to $0$ leads to:
\[
(\Ashift + \Bbf \Kpg )^h \ebf_d  = \brk*{ \prod\nolimits_{d' = d}^{h + d - 1} \frac{ 2 (d' - 1) }{ H + D } } \cdot \ebf_{(h + d - 1 )\% D + 1}
\text{\,.}
\]
On the other hand, if $h > D - d + 1$, then $(\Ashift + \Bbf \Kpg )^h \ebf_d = \0$ since:
\[
\begin{split}
	(\Ashift + \Bbf \Kpg )^h \ebf_d & = (\Ashift + \Bbf \Kpg )^{ h - (D - d + 1)} (\Ashift + \Bbf\Kpg )^{ D - d + 1} \ebf_d \\
	& = \brk*{ \prod\nolimits_{d' = d}^{D} \frac{ 2 (d' - 1) }{ H + D } } \cdot ( \Ashift + \Bbf \Kpg )^{ h - (D - d + 1)}  \ebf_1
	\text{\,,}
\end{split}
\]
and $( \Ashift + \Bbf \Kpg ) \ebf_{1} =  \sum\nolimits_{d' = 2}^D \frac{ 2(d' - 1) }{ H + D} \cdot \ebf_{d' \% D + 1} \ebf_{d'}^\top \ebf_1 = \0$.
Altogether, we get:
\[
\cost ( \Kpg ; \{ \ebf_d \} ) = \sum\nolimits_{h = 0}^H \norm*{ ( \Ashift + \Bbf \Kpg )^h \ebf_d   }^2 = \sum\nolimits_{h = 0}^{D - d + 1} \prod\nolimits_{d' = d}^{h + d - 1} \frac{ 4 (d' - 1)^2 }{ (H + D)^2 }
\text{\,,}
\]
and so:
\be
\cost ( \Kpg ; \orthstates ) = \frac{1}{D - 1} \sum\nolimits_{d = 2}^D \sum\nolimits_{h = 0}^{D - d + 1} \prod\nolimits_{d' = d}^{h + d - 1} \frac{ 4 (d' - 1)^2 }{ (H + D)^2 }
\text{\,.}
\label{eq:shift_cost_proof_pg_cost}
\ee

As for the cost attained by $\Kminnorm$, let $\ebf_d \in \orthstates$.
By the definition of $\Kminnorm$ (\cref{eq:unseen_controls_zero}), for $\ebf_{d'}\in \orthstates$ we have that $(\Ashift + \Bbf \Kminnorm) \ebf_{d'} = \Ashift \ebf_{d'} = \ebf_{d' \% D + 1}$ while $(\Ashift + \Bbf \Kminnorm ) \ebf_1 = \0$.
Thus, $( \Ashift + \Bbf \Kminnorm )^h \ebf_d = \ebf_{ (h + d - 1) \% D + 1 } $ for $h \leq D - d + 1$ and $(\Ashift + \Bbf \Kminnorm)^h \ebf_d = \0$ for $h > D - d + 1$.
This implies that:
\[
\cost ( \Kminnorm ; \{ \ebf_d \} ) = \sum\nolimits_{h = 0}^H \norm*{ ( \Ashift + \Bbf \Kminnorm )^h \ebf_d }^2 = D - d + 2
\text{\,,}
\]
and so:
\be
\cost ( \Kminnorm ; \orthstates ) = \frac{1}{D - 1} \sum\nolimits_{d = 2}^D \brk{ D - d + 2 }
\text{\,.}
\label{eq:shift_cost_proof_Kminnorm_cost}
\ee

Finally, noticing that $\cost^* (\orthstates) = 1$ (\eg, this minimal cost is attained by $\Koptall$, defined in \cref{eq:unseen_controls_opt}), by \cref{eq:shift_cost_proof_pg_cost,eq:shift_cost_proof_Kminnorm_cost} we get:
\[
\begin{split}
\frac{ \excesscost (\Kpg) }{ \excesscost ( \Kminnorm ) } & = \frac{ \cost (\Kpg ; \orthstates) - \cost^* (\orthstates) }{ \cost (\Kminnorm ; \orthstates) - \cost^* (\orthstates) } \\
& = \frac{ \frac{1}{D - 1} \sum\nolimits_{d = 2}^D \sum\nolimits_{h = 0}^{D - d + 1} \prod\nolimits_{d' = d}^{h + d - 1} \frac{ 4 (d' - 1)^2 }{ (H + D)^2 } - 1 }{  \frac{1}{D - 1} \sum\nolimits_{d = 2}^D \brk{ D - d + 2 } - 1 } \\
& = \frac{ \sum\nolimits_{d = 2}^D \sum\nolimits_{h = 1}^{D - d + 1} \prod\nolimits_{d' = d}^{h + d - 1} \frac{ 4 (d' - 1)^2 }{ (H + D)^2 } }{ \sum\nolimits_{d = 2}^D \brk{ D - d + 1 } }
\text{\,.}
\end{split}
\]
Since we can upper bound the nominator as follows:
\[
\sum\nolimits_{d = 2}^D \sum\nolimits_{h = 1}^{D - d + 1} \prod\nolimits_{d' = d}^{h + d - 1} \frac{ 4 (d' - 1)^2 }{ (H + D)^2 } \leq \frac{ 4(D - 1)^2 }{ (H + D)^2 } \sum\nolimits_{d = 2}^D \sum\nolimits_{h = 1}^{D - d + 1} 1 =  \frac{ 4(D - 1)^2 }{ (H + D)^2 } \sum\nolimits_{d = 2}^D (D - d + 1)
\text{\,,}
\]
we may conclude:
\[
\frac{ \excesscost (\Kpg) }{ \excesscost ( \Kminnorm ) } =  \frac{ \sum\nolimits_{d = 2}^D \sum\nolimits_{1 = 0}^{D - d + 1} \prod\nolimits_{d' = d}^{h + d - 1} \frac{ 4 (d' - 1)^2 }{ (H + D)^2 } }{ \sum\nolimits_{d = 2}^D \brk{ D - d + 1 } } \leq \frac{ 4(D - 1)^2 }{ (H + D)^2 }
\text{\,.}
\]
\qed

\subsection{Proof of~\cref{lem:min_norm}}
\label{app:proofs:min_norm}

Consider minimizing the squared Euclidean norm over the set of controllers with minimal training cost, \ie~over $\KK_\trainstates := \{ \Kbf \in \R^{D \times D} : \cost ( \Kbf ; \trainstates) = \cost^* ( \trainstates ) \}$:
\be
\min\nolimits_{\Kbf \in \KK_\trainstates} \norm{ \Kbf }^2
\text{\,.}
\label{eq:min_norm_proof_min_norm_objective_orig}
\ee
In an underdetermined LQR problem (\cref{sec:prelim:underdetermined}), the minimal training cost $\cost^* (\trainstates)$ is attained by a controller $\Kbf$ if and only if $\Kbf \xbf_0 = - \Bbf^{-1} \Abf \xbf_0$ for all initial states $\xbf_0 \in \trainstates$.
Let $\ubf_1, \ldots, \ubf_R \in \R^D$ be a basis of $\lspan (\trainstates)$, where $R \in [ \abs{\trainstates} ]$.
Requiring that $\Kbf \xbf_0 = - \Bbf^{-1} \Abf \xbf_0$ for all $\xbf_0 \in \trainstates$ is equivalent to requiring the equality holds for the basis $\ubf_1, \ldots, \ubf_R$. 
Thus, the objective in~\cref{eq:min_norm_proof_min_norm_objective_orig} is equivalent to:
\be
\min\nolimits_{\Kbf \in \R^{D \times D} } \, \norm{ \Kbf }^2 \text{ s.t. } \Kbf \ubf_r = - \Bbf^{-1} \Abf \ubf_r ~~,~ \forall r \in [R]
\text{\,,}
\label{eq:min_norm_proof_min_norm_objective_basis}
\ee
which entails minimizing a strongly convex function over a finite set of linear constraints.
Since the feasible set is non-empty, \eg, it contains $\Kminnorm$ (see its definition in~\cref{eq:unseen_controls_zero}), there exists a unique (optimal) solution, \ie~a unique controller that has minimal squared Euclidean norm among those minimizing the training cost.
We now prove that this unique solution is $\Kminnorm$.

Denote the $d$'th row of a matrix $\Cbf \in \R^{D \times D}$ by $\Cbf[d,:] \in \R^{D}$, for $d \in [D]$.
We can write the linear constraints in~\cref{eq:min_norm_proof_min_norm_objective_basis} as $R \cdot D$ constraints on the rows of $\Kbf$:
\[
\inprod{ \Kbf [d,:] }{ \ubf_r } = - \inprodbig{ \Bbf^{-1} [d, :] }{ \Abf \ubf_r} ~~,~ \forall d \in [D] , r \in [R]
\text{\,.}
\]
Since $\Kminnorm$ satisfies these constraints, by the method of Lagrange multipliers, to prove that $\Kminnorm$ is the unique solution of~\cref{eq:min_norm_proof_min_norm_objective_basis} we need only show that there exist $\{ \lambda_{d, r} \in \R \}_{d \in [D], r \in [R]}$ for which:
\[
\Kminnorm [d, :] = \sum\nolimits_{r = 1}^R \lambda_{d, r} \cdot \ubf_r ~~,~ \forall d \in [D]
\text{\,.}
\]
That is, it suffices to show that the rows of $\Kminnorm$ are in $\lspan (\ubf_1, \ldots, \ubf_R) = \lspan (\trainstates)$.
To see that this is indeed the case, recall that by the definition of $\Kminnorm$ (\cref{eq:unseen_controls_zero}) it satisfies $\Kminnorm \vbf_0 = \0$ for all $\vbf_0 \in \trainstates^\perp$.
This implies that the rows of $\Kminnorm$ necessarily reside in $\lspan (\trainstates)$, concluding the proof.
\qed

\subsection{Proof of~\cref{cor:no_euc_norm_min}}
\label{app:proofs:no_euc_norm_min}

By \cref{lem:min_norm}, $\Kminnorm = \argmin_{\Kbf \in \KK_\trainstates} \norm{ \Kbf }^2$.
We claim that in the considered setting $\Kminnorm = - \Bbf^{-1} \ebf_2 \ebf_1^\top$. 
Indeed, $(\Ashift + \Bbf (- \Bbf^{-1} \ebf_2 \ebf_1^\top)) \ebf_1 = \ebf_2 - \ebf_2 = \0$, meaning $- \Bbf^{-1} \ebf_2 \ebf_1^\top$ satisfies the optimality condition in \cref{eq:optimal_cond}.
Furthermore, for any $\vbf_0 \in \orthstates$ it holds that $- \Bbf^{-1} \ebf_2 \ebf_1^\top \vbf_0 = \0$ since $\vbf_0$ is orthogonal to $\ebf_1$, meaning $- \Bbf^{-1} \ebf_2 \ebf_1^\top$ satisfies \cref{eq:unseen_controls_zero}.
Thus, 
$\Kminnorm = - \Bbf^{-1} \ebf_2 \ebf_1^\top$ and \smash{$\norm{ \Kminnorm }^2 = \min_{\Kbf \in \KK_\trainstates} \norm{ \Kbf }^2 = 1$} (recall $\Bbf$ is orthogonal).
On the other hand, as established by \cref{lem:shift_Kpg_K2} in the proof of \cref{prop:shift}, $\Kpg = - \Bbf^\top \sum\nolimits_{d = 1}^D  (1 - \frac{2 (d - 1) }{H + D}) \cdot \ebf_{d \% D + 1} \ebf_d^\top$.
Consequently:
\[
\norm{ \Kpg }^2 = 1 + \sum\nolimits_{d = 2}^D \brk2{ 1 - \frac{ 2(d - 1) }{ H + D } }^2 = \argmin\nolimits_{\Kbf \in \KK_\trainstates} \norm{ \Kbf }^2 + \sum\nolimits_{d = 2}^D \brk2{ 1 - \frac{ 2(d - 1) }{ H + D } }^2
\text{\,.}
\]
Since $H \geq D \geq 2$ it holds that:
\[
\sum\nolimits_{d = 2}^D \brk2{ 1 - \frac{ 2(d - 1) }{ H + D } }^2  \geq  \sum\nolimits_{d = 2}^D \brk2{ 1 - \frac{ (d - 1) }{ D } }^2 \geq \sum\nolimits_{d = 2}^{ \lceil D / 2 \rceil } \frac{1}{4} = \frac{  \lceil D / 2 \rceil - 1}{ 4 }
\text{\,,}
\]
and so:
\[
\norm{ \Kpg }^2 - \argmin\nolimits_{\Kbf \in \KK_\trainstates} \norm{ \Kbf }^2 = \sum\nolimits_{d = 2}^D \brk2{ 1 - \frac{ 2(d - 1) }{ H + D } }^2 = \Omega (D)
\text{\,.}
\]
\qed

\subsection{Proof of~\cref{lem:extrapolation_measures_invariant_general_q}}
\label{app:proofs:extrapolation_measures_invariant_general_q}

Let $\orthstates$ be an orthonormal basis of $\trainstates^\perp$, and $\B$ be an orthonormal basis of $\lspan (\trainstates)$.

Now, for $\Kbf \in \R^{D \times D}$, the $\Qbf$-optimality measure of extrapolation can be written as follows:
\[
\begin{split}
\optmes^\Qbf (\Kbf) & = \frac{1}{ \abs{\orthstates} } \sum\nolimits_{\vbf_0 \in \orthstates} \norm*{ (\Abf + \Bbf \Kbf) \vbf_0 }_\Qbf^2 \\
& = \frac{1}{ \abs{\orthstates} } \sum\nolimits_{\vbf_0 \in \orthstates} \vbf_0^\top (\Abf + \Bbf \Kbf)^\top \Qbf (\Abf + \Bbf \Kbf) \vbf_0^\top \\
& = \frac{1}{ \abs{\orthstates} } \inprod{(\Abf + \Bbf \Kbf)^\top \Qbf (\Abf + \Bbf \Kbf) }{ \sum\nolimits_{\vbf_0 \in \orthstates} \vbf_0^\top  \vbf_0^\top }
\text{\,.}
\end{split}
\]
Adding and subtracting
\[
 \frac{1}{ \abs{\orthstates} }\inprod{(\Abf + \Bbf \Kbf)^\top \Qbf (\Abf + \Bbf \Kbf) }{ \sum\nolimits_{\vbf_0 \in \B} \vbf_0^\top  \vbf_0^\top }
\]
to the right hand side of the equation above, we have that:
\[
\optmes^\Qbf (\Kbf) =  \frac{1}{ \abs{\orthstates} } \inprod{(\Abf + \Bbf \Kbf)^\top \Qbf (\Abf + \Bbf \Kbf) }{ \sum\nolimits_{\vbf \in \orthstates \cup \B} \vbf_0 \vbf_0^\top } -  \frac{1}{ \abs{\orthstates} } \inprod{(\Abf + \Bbf \Kbf)^\top \Qbf (\Abf + \Bbf \Kbf) }{ \sum\nolimits_{\vbf \in \B} \vbf_0 \vbf_0^\top }
\text{\,.}
\]
Notice that $\sum\nolimits_{\vbf_0 \in \orthstates \cup \B} \vbf_0 \vbf_0^\top = \Ibf$, where $\Ibf$ stands for the identity matrix, since $\orthstates \cup \B$ is an orthonormal basis of $\R^D$.
Thus:
\[
\optmes^\Qbf (\Kbf) =  \frac{1}{ \abs{\orthstates} } \inprod{(\Abf + \Bbf \Kbf)^\top \Qbf (\Abf + \Bbf \Kbf) }{ \Ibf } -  \frac{1}{ \abs{\orthstates} } \inprod{(\Abf + \Bbf \Kbf)^\top \Qbf (\Abf + \Bbf \Kbf) }{ \sum\nolimits_{\vbf \in \B} \vbf_0 \vbf_0^\top }
\text{\,.}
\]
As can be seen in the expression above, the $\Qbf$-optimality measure of extrapolation does not depend on the choice of $\orthstates$.

Similarly, for $\Kbf \in \R^{D \times D}$, the $\Qbf$-cost measure of extrapolation can be written as follows: 
\[
\begin{split}
	\excesscost^\Qbf (\Kbf) & = \cost (\Kbf; \orthstates) - \cost^* ( \orthstates) \\
	&  = \frac{1}{\abs{\orthstates}} \sum\nolimits_{\vbf_0 \in \orthstates} \brk*{ \sum\nolimits_{h = 0}^H \norm*{ (\Abf + \Bbf \Kbf)^h \vbf_0}_\Qbf^2 - \norm{ \vbf_0 }_\Qbf^2 } \\
	& = \frac{1}{\abs{\orthstates}} \sum\nolimits_{\vbf_0 \in \orthstates} \sum\nolimits_{h = 1}^H \norm*{ (\Abf + \Bbf \Kbf)^h \vbf_0}_\Qbf^2 \\
	& = \frac{1}{\abs{\orthstates}} \sum\nolimits_{h = 1}^H \inprod{ \brk[s]*{ (\Abf + \Bbf \Kbf)^h }^\top \Qbf (\Abf + \Bbf \Kbf)^h }{ \sum\nolimits_{\vbf_0 \in \orthstates} \vbf_0 \vbf_0^\top }
	\text{\,,}
\end{split}
\]
where we used the fact that $\cost^* (\X) = \frac{1}{ \abs{\X} } \sum\nolimits_{\xbf_0 \in \X} \norm{ \xbf_0}_\Qbf^2$ for any finite set of initial states $\X \subset \R^D$.
Adding and subtracting for each summand $h \in [H]$ on the right hand side the term
\[
\inprod{ \brk[s]*{ (\Abf + \Bbf \Kbf)^h }^\top \Qbf (\Abf + \Bbf \Kbf)^h }{ \sum\nolimits_{\vbf_0 \in \B} \vbf_0 \vbf_0^\top }
\text{\,,}
\]
we have that:
\[
\begin{split}
	\excesscost^\Qbf (\Kbf) & = \frac{1}{\abs{\orthstates}} \sum\nolimits_{h = 1}^H \Big ( \inprod{ \brk[s]*{ (\Abf + \Bbf \Kbf)^h }^\top \Qbf (\Abf + \Bbf \Kbf)^h }{ \sum\nolimits_{\vbf_0 \in \orthstates \cup \B} \vbf_0 \vbf_0^\top } \\
	& \hspace{24mm} -  \inprod{ \brk[s]*{ (\Abf + \Bbf \Kbf)^h }^\top \Qbf (\Abf + \Bbf \Kbf)^h }{ \sum\nolimits_{\vbf_0 \in \B} \vbf_0 \vbf_0^\top } \Big ) \\
	&  = \frac{1}{\abs{\orthstates}} \sum\nolimits_{h = 1}^H \Big ( \inprod{ \brk[s]*{ (\Abf + \Bbf \Kbf)^h }^\top \Qbf (\Abf + \Bbf \Kbf)^h }{\Ibf } \\
	& \hspace{24mm} -  \inprod{ \brk[s]*{ (\Abf + \Bbf \Kbf)^h }^\top \Qbf (\Abf + \Bbf \Kbf)^h }{ \sum\nolimits_{\vbf_0 \in \B} \vbf_0 \vbf_0^\top } \Big )
	\text{\,,}
\end{split}
\]
where we again used the fact that $\sum\nolimits_{\vbf_0 \in \orthstates \cup \B} \vbf_0 \vbf_0^\top = \Ibf$ since $\orthstates \cup \B$ is an orthonormal basis of $\R^D$.
As can be seen from the expression above, the $\Qbf$-cost measure of extrapolation does not depend on the choice of $\orthstates$.
\qed

\subsection{Proof of~\cref{prop:shift_diag_q}}
\label{app:proofs:shift_diag_q}

The proof follows a line identical to that of \cref{prop:shift} (\cref{app:proofs:shift}), generalizing it to account for a diagonal $\Qbf$ with entries $q_1, \ldots, q_D \geq 0$ (as opposed to $\Qbf = \Ibf$), where $q_j > 0$ for at least some $j \in [D]$.

We first prove that $\Kpg = \Kbf^{(2)} = - \Bbf^\top \sum\nolimits_{d = 1}^D  (1 - \alpha_d) \cdot \ebf_{d \% D + 1} \ebf_d^\top$.
That is, policy gradient converges in a single iteration to the controller $- \Bbf^\top \sum\nolimits_{d = 1}^D  (1 - \alpha_d) \cdot \ebf_{d \% D + 1} \ebf_d^\top$, which minimizes the training cost.
For $\trainstates = \{ \ebf_1 \}$, by~\cref{lem:lqr_gradient} the gradient of the training cost at $\Kbf^{(1)} = \0$ is given by:
\[
\nabla \cost ( \0  ; \ebf_1) = 2 \Bbf^\top \sum\nolimits_{h = 0}^{H - 1} \brk2{ \sum\nolimits_{s = 1}^{H - h} \brk[s]{ \Ashift^{s - 1} }^\top \Qbf \Ashift^s } \Sigmabf_{\ebf_1, h}
\text{\,,}
\]
where $\Ashift = \sum\nolimits_{d = 1}^D \ebf_{d \% D + 1} \ebf_d^\top$ and $\Sigmabf_{\ebf_1, h} := \Ashift^h \ebf_1 \brk[s]{ \Ashift^h \ebf_1 }^\top = \ebf_{ h \% D + 1 } \ebf_{ h \% D + 1 }^\top$ for $h \in \{0\} \cup [H - 1]$.
Notice that $\Qbf \Ashift^s = \sum\nolimits_{d = 1}^D q_{(s + d - 1) \% D + 1} \cdot \ebf_{(s + d - 1) \% D + 1} \ebf_d^\top$ and $\brk[s]{ \Ashift^{s - 1} }^\top \Qbf \Ashift^s = \sum\nolimits_{d = 1}^D q_{(s + d - 1) \% D + 1} \cdot \ebf_{d \% D + 1} \ebf_d^\top$, for all $s \in [H]$.
Hence:
\[
\begin{split}
\nabla \cost ( \0 ; \ebf_ 1) & = 2 \Bbf^\top \sum\nolimits_{h = 0}^{H - 1} \sum\nolimits_{s = 1}^{H - h} \brk2{ \sum\nolimits_{d = 1}^D q_{(s + d - 1) \% D + 1} \cdot \ebf_{d \% D + 1} \ebf_d^\top } \ebf_{h \% D + 1} \ebf_{h \% D + 1}^\top \\
& = 2 \Bbf^\top \sum\nolimits_{h = 0}^{H - 1} \brk2{ \sum\nolimits_{s = 1}^{H - h} q_{ (h + s) \% D + 1 } } \cdot \ebf_{(h + 1) \% D + 1} \ebf_{h \% D + 1}^\top \\
& = 2 \Bbf^\top \sum\nolimits_{h = 0}^{H - 1} \brk2{ \sum\nolimits_{s = h + 1}^{H} q_{ s \% D + 1 } } \cdot \ebf_{(h + 1) \% D + 1} \ebf_{h \% D + 1}^\top 
\text{\,.}
\end{split}
\]
Recalling that $H = D \cdot L$ for some $L \in \N$, there are exactly $L = \frac{H}{D}$ terms in the sum corresponding to $\ebf_{d \% D + 1} \ebf_d^\top$, for each $d \in [D]$.
Focusing on elements $h \in \{0, D, 2D, \ldots, H - D\}$ in the sum, which satisfy $h \% D + 1 = 1$, the sum of coefficients for $\ebf_2 \ebf_1^\top$ is given by $\frac{H}{D} \sum\nolimits_{j = 1}^D q_j + (\frac{H}{D} - 1) \sum\nolimits_{j = 1}^D q_{j} + \cdots + \sum\nolimits_{j = 1}^D q_{j} = \frac{H}{2D} \brk{ \frac{H}{D} + 1 } \sum\nolimits_{j = 1}^D q_{j}$.
More generally, for $d \in [D]$, the relevant coefficients are those corresponding to $h \in \{ d - 1, D + d - 1, 2D + d - 1, \ldots, H - D + d - 1 \}$.
Since for every $l \in [\frac{H}{D}]$ it holds that $\sum\nolimits_{s = (l \cdot D + d - 1) + 1}^{H} q_{ s \% D + 1 } = \sum\nolimits_{s = l \cdot D + 1}^{H} q_{ s \% D + 1 } - \sum\nolimits_{j = 2}^d q_j$, the sum of coefficients for $\ebf_{d \% D + 1} \ebf_d^\top$ is obtained by subtracting $\frac{H}{D} \sum\nolimits_{j = 2}^d q_j$ from the sum of coefficients for $\ebf_2 \ebf_1^\top$, \ie~it is equal to $\frac{H}{2D} \brk{ \frac{H}{D} + 1 } \sum\nolimits_{j = 1}^D q_{j} - \frac{H}{D} \sum\nolimits_{j = 2}^d q_j$.
We may therefore write:
\[
\nabla \cost ( \0 ; \ebf_ 1)  = \Bbf^\top \sum\nolimits_{d = 1}^{D} \brk*{ \frac{H}{D} \brk*{ \frac{H}{D} + 1 } \sum\nolimits_{j = 1}^D q_{j} - 2 \frac{H}{D} \sum\nolimits_{j = 2}^d q_j } \cdot \ebf_{d \% D + 1} \ebf_{d}^\top
\text{\,,}
\]
which, combined with $\Kbf^{(1)} = \0$ and $\eta = \brk1{ \frac{H}{D} (\frac{H}{D} + 1) \sum\nolimits_{j = 1}^D q_{j} }^{-1}$, leads to the sought-after expression for $\Kbf^{(2)}$:
\[
\begin{split}
\Kbf^{(2)} & = \Kbf^{(1)} - \eta \cdot \nabla \cost ( \Kbf^{(1)} ; \ebf_1 ) \\
& = - \Bbf^\top \sum\nolimits_{d = 1}^{D} \brk3{ 1 - \frac{ 2 \sum\nolimits_{j = 2}^{d} q_j }{ \brk1{ \frac{H}{D} + 1} \sum\nolimits_{j = 1}^D q_j } } \cdot \ebf_{d \% D + 1} \ebf_{d}^\top \\
& = - \Bbf^\top \sum\nolimits_{d = 1}^{D} \brk*{ 1 - \alpha_d } \cdot \ebf_{d \% D + 1} \ebf_{d}^\top
\text{\,,}
\end{split}
\]
where $\alpha_{d} := \frac{ 2 \sum\nolimits_{j = 2}^{d} q_j }{ (\frac{H}{D} + 1) \sum\nolimits_{j = 1}^D q_j } \in [0, 1]$ for $d \in [D]$.
To see that $\Kbf^{(2)}$ minimizes the training cost, notice that:
\[
(\Ashift + \Bbf \Kbf^{(2)}) \ebf_1 = \Ashift \ebf_1 - \Bbf \Bbf^\top \sum\nolimits_{d = 1}^{D} \brk*{1 - \alpha_d } \cdot \ebf_{d \% D + 1} \ebf_{d}^\top \ebf_1 = \ebf_2 - \ebf_2 = \0
\text{\,,}
\]
where the second equality is by $\Bbf \Bbf^\top = \Ibf$, $\ebf_d^\top \ebf_1 = 0$ for $d \in \{2, \ldots, D\}$, and $\alpha_1 = 0$.
Consequently, $\cost (\Kbf^{(2)} ; \ebf_1) = \sum\nolimits_{h = 0}^H \norm{ (\Ashift + \Bbf \Kbf^{(2)} )^h \ebf_1 }_\Qbf^2 = \norm{ \ebf_1 }^2_\Qbf$, which is the minimal training cost $\cost^* (\ebf_1)$ since for any $\Kbf \in \R^{D \times D}$ the cost is a sum of $H + 1$ non-negative terms, with the one corresponding to $h = 0$ being equal to $\norm{\ebf_1}^2_\Qbf$.

\medskip

\textbf{Extrapolation in terms of the $\Qbf$-optimality measure.}
Next, we characterize the extent to which $\Kpg = \Kbf^{(2)}$ extrapolates, as measured by the $\Qbf$-optimality measure.
As shown by \cref{lem:extrapolation_measures_invariant_general_q} in \cref{app:extension_q}, the $\Qbf$-optimality measure is invariant to the choice of orthonormal basis $\orthstates$ for $\trainstates^\perp$.
Thus, because $\trainstates = \{ \ebf_1 \}$ we may assume without loss of generality that $\orthstates = \{ \ebf_2, \ldots, \ebf_D\}$.

For any $\ebf_d \in \orthstates$, by the definition of $\Kminnorm$ (\cref{eq:unseen_controls_zero}) we have that $(\Ashift + \Bbf \Kminnorm) \ebf_d = \Ashift \ebf_d = \ebf_{d \% D + 1}$.
Thus:
\be
\optmes^\Qbf \brk1{ \Kminnorm } = \frac{ 1 }{ D - 1 } \sum\nolimits_{d = 2}^D \norm*{ (\Ashift + \Bbf \Kminnorm) \ebf_d }_\Qbf^2 =  \frac{1}{D - 1} \sum\nolimits_{d = 2}^D \norm*{ \ebf_{d \% D + 1} }_\Qbf^2 = \frac{1}{D - 1} \sum\nolimits_{d = 2}^D q_{d \% D + 1}
\text{\,.}
\label{eq:mag_Q_shift_Kminnorm_deriv}
\ee
On the other hand:
\[
\begin{split}
	(\Ashift + \Bbf \Kpg) \ebf_d & = \ebf_{d \% D + 1} - \Bbf \Bbf^\top \sum\nolimits_{d' = 1}^{D} \brk*{1 - \alpha_{d'} } \cdot \ebf_{d' \% D + 1} \ebf_{d'}^\top \ebf_d \\
	& =  \ebf_{d \% D + 1} - \brk*{1 - \alpha_d } \cdot \ebf_{d \% D + 1} \\
	& = \alpha_d \cdot \ebf_{d \% D + 1}
	\text{\,,}
\end{split}
\]
and so:
\be
\begin{split}
	 \optmes^\Qbf \brk*{ \Kpg } & = \frac{ 1 }{ D - 1 } \sum\nolimits_{d = 2}^D \norm*{ (\Ashift + \Bbf \Kpg) \ebf_d }_\Qbf^2 \\
	 & = \frac{1}{D - 1} \sum\nolimits_{d = 2}^D \alpha_d^2 \cdot  \norm*{ \ebf_{d \% D + 1} }_\Qbf^2 \\
	 & = \frac{1}{ D - 1} \sum\nolimits_{d = 2}^D \alpha_d^2 \cdot q_{d \% D + 1}
\text{\,.}
\end{split}
\label{eq:mag_Q_shift_Kpg_deriv}
\ee
The desired guarantee on extrapolation in terms of the $\Qbf$-optimality measure follows from \cref{eq:mag_Q_shift_Kminnorm_deriv,eq:mag_Q_shift_Kpg_deriv}.

\medskip

\textbf{Extrapolation in terms of the $\Qbf$-cost measure.}
Lastly, we characterize the extent to which $\Kpg = \Kbf^{(2)}$ extrapolates, as quantified by the $\Qbf$-cost measure.
As done above for proving extrapolation in terms of the $\Qbf$-optimality measure, by \cref{lem:extrapolation_measures_invariant_general_q} in \cref{app:extension_q} we may assume without loss of generality that $\orthstates = \{ \ebf_2, \ldots, \ebf_D \}$.

Fix some $\ebf_d \in \orthstates$.
We use the fact that $\Kpg = \Kbf^{(2)} = - \Bbf^\top \sum\nolimits_{d = 1}^D  (1 - \alpha_d) \cdot \ebf_{d \% D + 1} \ebf_d^\top$, established in the beginning of the proof, to straightforwardly compute $\excesscost^\Qbf (\Kpg)$.
Specifically, recalling that $\Bbf \Bbf^\top = \Ibf$, we have that $\Ashift + \Bbf \Kpg = \sum\nolimits_{d' = 1}^D \alpha_{d'} \cdot \ebf_{d' \% D + 1} \ebf_{d'}^\top = \sum\nolimits_{d' = 2}^D \alpha_{d'} \cdot \ebf_{d' \% D + 1} \ebf_{d'}^\top$, where the second equality is by noticing that $\alpha_1 = 0$.
Now, for any $h \in [H]$:
\[
\begin{split}
	(\Ashift + \Bbf \Kpg)^h \ebf_d & = (\Ashift + \Bbf \Kpg)^{h - 1}  \sum\nolimits_{d' = 2}^D \alpha_{d'} \cdot \ebf_{d' \% D + 1} \ebf_{d'}^\top \ebf_d \\
	& =  \alpha_d \cdot (\Ashift + \Bbf \Kpg )^{h - 1} \ebf_{d \% D + 1}
	\text{\,.}
\end{split}
\]
If $h \leq D - d + 1$, unraveling the recursion from $h - 1$ to $0$ leads to:
\[
(\Ashift + \Bbf\Kpg )^h \ebf_d  = \brk2{ \prod\nolimits_{d' = d}^{h + d - 1} \alpha_{d'} } \cdot \ebf_{(h + d - 1 ) \% D + 1}
\text{\,.}
\]
On the other hand, if $h > D - d + 1$, then $(\Ashift + \Bbf \Kpg )^h \ebf_d = \0$ since:
\[
\begin{split}
	(\Ashift + \Bbf \Kbf^{(2)})^h \ebf_d & = (\Ashift + \Bbf \Kpg )^{ h - (D - d + 1)} (\Ashift + \Bbf \Kpg )^{ D - d + 1} \ebf_d \\
	& = \brk2{ \prod\nolimits_{d' = d}^{D} \alpha_{d'} } \cdot ( \Ashift + \Bbf \Kpg)^{ h - (D - d + 1)}  \ebf_1
	\text{\,,}
\end{split}
\]
and $( \Ashift + \Bbf \Kpg ) \ebf_{1} = \sum\nolimits_{d' = 2}^D \alpha_{d'} \cdot \ebf_{d' \% D + 1} \ebf_{d'}^\top \ebf_1 = \0$.
Altogether, we get:
\[
\begin{split}
	\cost (\Kpg ; \{ \ebf_d \} ) & = \sum\nolimits_{h = 0}^H \norm*{ ( \Ashift + \Bbf \Kpg )^h \ebf_d   }_\Qbf^2 \\
	& = \sum\nolimits_{h = 0}^{D - d + 1} \norm*{ \brk2{ \prod\nolimits_{d' = d}^{h + d - 1} \alpha_{d'} } \cdot \ebf_{(h + d - 1 ) \% D + 1} }_\Qbf^2 \\
	& = \sum\nolimits_{h = 0}^{D - d + 1} q_{ (h + d - 1) \% D + 1} \cdot \prod\nolimits_{d' = d}^{h + d - 1} \alpha_{d'}^2
	\text{\,,}
\end{split}
\]
and so:
\be
\cost (\Kpg ; \orthstates ) = \frac{1}{ D - 1} \sum\nolimits_{d = 2}^D \sum\nolimits_{h = 0}^{D - d + 1} q_{ (h + d - 1) \% D + 1} \cdot \prod\nolimits_{d' = d}^{h + d - 1} \alpha_{d'}^2
\text{\,.}
\label{eq:shiftQ_cost_proof_pg_cost}
\ee

As for the cost attained by $\Kminnorm$, let $\ebf_d \in \orthstates$.
By the definition of $\Kminnorm$ (\cref{eq:unseen_controls_zero}), for $\ebf_{d'} \in \orthstates$ we have that $(\Ashift + \Bbf \Kminnorm) \ebf_{d'} = \Ashift \ebf_{d'} = \ebf_{d' \% D + 1}$ while $(\Ashift + \Bbf \Kminnorm ) \ebf_1 = \0$.
Thus, $( \Ashift + \Bbf \Kminnorm )^h \ebf_d = \ebf_{ (h + d - 1) \% D + 1 } $ for $h \leq D - d + 1$ and $(\Ashift + \Bbf \Kminnorm)^h \ebf_d = \0$ for $h > D - d + 1$.
This implies that:
\[
\cost ( \Kminnorm ; \{ \ebf_d \} ) = \sum\nolimits_{h = 0}^H \norm*{ ( \Ashift + \Bbf \Kminnorm )^h \ebf_d }_\Qbf^2 = \sum\nolimits_{h = 0}^{D - d + 1} \norm*{ \ebf_{ (h + d - 1) \% D+ 1 } }_\Qbf^2 = \sum\nolimits_{h = 0}^{D - d + 1} q_{ (h + d - 1) \% D + 1 }
\text{\,,}
\]
and so:
\be
\cost (\Kminnorm ; \orthstates) = \frac{1}{ D - 1 } \sum\nolimits_{d = 2}^D  \sum\nolimits_{h = 0}^{D - d + 1} q_{ (h + d - 1) \% D + 1 }
\label{eq:shiftQ_cost_proof_Kminnorm_cost}
\ee

Finally, noticing that $\cost^* (\orthstates) = \frac{1}{D - 1} \sum\nolimits_{ d = 2}^D \norm{\ebf_d}_\Qbf^2 = \frac{1}{D - 1} \sum\nolimits_{d = 2}^D q_d$ (\eg,this minimal cost is attained by $\Koptall$, defined in \cref{eq:unseen_controls_opt}), by \cref{eq:shiftQ_cost_proof_pg_cost,eq:shiftQ_cost_proof_Kminnorm_cost} we get:
\[
\begin{split}
\excesscost^\Qbf (\Kpg)  & = \cost (\Kpg ; \orthstates) - \cost^* (\orthstates) \\
& =  \frac{1}{ D - 1} \sum\nolimits_{d = 2}^D \sum\nolimits_{h = 0}^{D - d + 1} q_{ (h + d - 1) \% D + 1} \cdot \prod\nolimits_{d' = d}^{h + d - 1} \alpha_{d'}^2 - \frac{1}{D - 1} \sum\nolimits_{d = 2}^D q_d \\
& = \frac{1}{ D - 1} \sum\nolimits_{d = 2}^D \sum\nolimits_{h = 1}^{D - d + 1} q_{ (h + d - 1) \% D + 1} \cdot \prod\nolimits_{d' = d}^{h + d - 1} \alpha_{d'}^2 
\text{\,,}
\end{split}
\]
and:
\[
\begin{split}
	\excesscost^\Qbf (\Kminnorm)  & = \cost (\Kminnorm ; \orthstates) - \cost^* (\orthstates) \\
	& =  \frac{1}{ D - 1} \sum\nolimits_{d = 2}^D \sum\nolimits_{h = 0}^{D - d + 1} q_{ (h + d - 1) \% D + 1} - \frac{1}{D - 1} \sum\nolimits_{d = 2}^D q_d \\
	& = \frac{1}{ D - 1} \sum\nolimits_{d = 2}^D \sum\nolimits_{h = 1}^{D - d + 1} q_{ (h + d - 1) \% D + 1}
	\text{\,.}
\end{split}
\]
The desired result readily follows from the expressions above for $\excesscost^\Qbf (\Kpg)$ and $\excesscost^\Qbf (\Kminnorm)$.
\qed

\subsection{Proof of~\cref{thm:typical_system}}
\label{app:proofs:typical_system}

In the proof below, we treat the more general case where $\trainstates$ is an arbitrary set of orthonormal initial states seen in training, which includes the special case of $\trainstates = \{ \xbf_0 \}$ for a unit norm $\xbf_0 \in \R^D$.
Furthermore, it will be useful to consider the optimality measure of extrapolation for individual states in $\orthstates$, as defined below.

\begin{definition}
\label{def:opt_measure_single}
The \emph{optimality measure} of extrapolation for a controller $\Kbf \in \R^{D \times D}$ and initial state $\xbf_0 \in \orthstates$ is:
\[
\optmes ( \Kbf; \xbf_0) := \norm{ (\Abf + \Bbf \Kbf) \xbf_0 }^2
\text{\,.}
\]
\end{definition}

\subsubsection{Proof Outline}
\label{app:proofs:typical_system:outline}

We begin with several preliminary lemmas in \cref{prep_lemmas}.
Then, towards establishing that an iteration of policy gradient leads to extrapolation in terms of the optimality measure, we examine $\inprodbig{ \nabla \optmes ( \Kbf^{(1)} ) }{  \nabla \cost (\Kbf^{(1)} ; \trainstates ) }$.
This inner product can be represented as a sum of matrix traces, where each matrix is a product of powers of $\Abf$ and matrices that depend only on $\vbf_0$ and initial states in $\trainstates$.
In \cref{app:proofs:typical_system:in_prod_lower_bound}, we show that $\EE_\Abf \brk[s]1{ \inprodbig{ \nabla \optmes ( \Kbf^{(1)} ) }{ \nabla \cost (\Kbf^{(1)} ; \trainstates ) } } \geq 2H(H-1) / D$ via basic properties of Gaussian random variables.

The remainder of the proof converts the lower bound on $\EE_\Abf \brk[s]1{ \inprodbig{ \nabla \optmes ( \Kbf^{(1)} ) }{ \nabla \cost (\Kbf^{(1)} ; \trainstates ) } }$ into guarantees on the optimality measure attained by $\Kbf^{(2)} = \Kbf^{(1)} - \eta \cdot \nabla \cost (\Kbf^{(1)}; \trainstates)$.
To do so, we employ tools lying at the intersection of random matrix theory and topology.
Namely, at the heart of our analysis lies a method from~\citet{redelmeier2014real} for computing the expectation for traces of random matrix products, based on the topological concept of \emph{genus expansion}. 
\cref{genus_expansion} provides a self-contained introduction to this method, for the interested reader.

In \cref{expectation_bound}, we employ the method of~\citet{redelmeier2014real} for establishing extrapolation in terms of expected optimality measure.
Specifically, the method facilitates upper bounding $\EE_\Abf \brk[s]1{ \norm{ \nabla \cost ( \Kbf^{(1)} ; \trainstates) }^2 }$. 
Along with the fact that $\optmes ( \cdot )$ is $2$-smooth and the lower bound on $\EE_\Abf \brk[s]1{ \inprodbig{ \nabla \optmes ( \Kbf^{(1)} ) }{ \nabla \cost (\Kbf^{(1)} ; \trainstates ) } }$, this guarantees a reduction in 	optimality measure compared to $\Kbf^{(1)}$ through an argument analogous to the fundamental descent lemma.
Noticing that the optimality measure attained by $\Kbf^{(1)}$ and $\Kminnorm$ are equal, concludes this part of the proof.

In \cref{high_prob_bound}, to establish extrapolation occurs with high probability for systems with sufficiently large state space dimension, we decompose $\inprodbig{ \nabla \optmes ( \Kbf^{(1)} ) }{ \nabla \cost (\Kbf^{(1)} ; \trainstates ) }$ into a sum of random variables, whose variances we upper bound by again employing the method of~\citet{redelmeier2014real}.
Chebyshev's inequality then implies that with high probability $\inprodbig{ \nabla \optmes ( \Kbf^{(1)} ) }{  \nabla \cost (\Kbf^{(1)} ; \trainstates ) } \geq  H (H-1) / D$.
Lastly, following arguments analogous to those used for establishing reduction of optimality measure in expectation leads to the high probability guarantee.

\subsubsection{Preliminary Lemmas}
\label{prep_lemmas}

\begin{lemma}
	\label{lem:norm_concentration}
	
	Let $Z_{1}, \ldots, Z_{K}$ be $D$-dimensional independent Gaussian random variables, such that $Z_{k} \sim \NN \brk1{ \0 ,\frac{1}{D}  \Ibf }$ for $k \in [K]$.
	Then: 
	\[
	\Pr \brk*{ \frac{\sum_{k = 1}^K \norm{ Z_{k} }^{2} }{ K } \geq  2 } \leq \frac{2}{K D}
	\text{\,.}
	\]
\end{lemma}

\begin{proof}
	For all $k \in [K]$, we have that $\EE \brk[s]1{ \norm{Z_{k}}^{2} } = 1$.
	Furthermore, let $z$ denote some entry of $Z_k$.
	Then:
	\[
	\Var \left ( \norm{ Z_{k} }^{2} \right ) = D \cdot \Var \left( z^{2} \right) = D\left( \EE \brk[s]*{ z^4 } - \EE \brk[s]*{ z^2 }^{2} \right) = D \left ( \frac{3}{D^{2}} - \frac{1}{D^{2}} \right ) = \frac{2}{D} 
	\text{\,,}
	\]
	where the third equality is by the fact that, for a univariate Gaussian random variable $y \sim \NN (0,1)$, we have $\EE \brk[s]{ y^4 } = 3$.
	Since $Z_1, \ldots, Z_K$ are independent: 
	\[
	\Var \left( \frac{ \sum_{k = 1}^K \norm{Z_k}^{2} }{K} \right) = \frac{2}{KD}
	\text{\,,}
	\]
	and so by Chebyshev's inequality we get:
	\[
	\Pr \brk*{ \frac{\sum_{k = 1}^K \norm{ Z_{k} }^{2} }{ K } \geq  2 } \leq \Pr \brk*{ \abs*{ \frac{ \sum\nolimits_{k = 1}^K \norm{ Z_k }^2 }{K} - 1 } \geq 1 } \leq \frac{2}{KD}
	\text{\,.}
	\]
\end{proof}

\begin{lemma}
	\label{lem:clm_cost_grad_inner_prod}
	For any $\vbf_0 \in \orthstates$ and $\xbf_0 \in \trainstates$ it holds that:
	\[
	\inprod{ \nabla \optmes ( \Kbf^{(1)} ; \vbf_0) }{ \nabla \cost( \Kbf^{(1)} ; \xbf_0 ) } = 4 \sum\nolimits_{n=  0}^{H - 1} \sum\nolimits_{k = 0}^{H - n - 1} \inprod{ \vbf_0 }{ \Abf^{n}\xbf_0 } \cdot \Tr  \brk*{ \vbf_0 (\Abf^{n}\xbf_0)^\top (\Abf^{k + 1})^{\top} \Abf^{k + 1}}
	\text{\,.}
	\]
\end{lemma}

\begin{proof}
	For a controller $\Kbf \in \R^{ D \times D }$, let $\costhorizonone ( \Kbf ; \{\vbf_0\} ) :=  \norm{\vbf_0}^{2}+\norm{ (\Abf+\Bbf \Kbf ) \vbf_0 }^{2}$ denote the cost (\cref{eq:cost}) that it attains over $\vbf_0$ for a time horizon $H = 1$. 
	Notice that $\optmes ( \Kbf; \vbf_0) = \costhorizonone ( \Kbf ; \{\vbf_0\}) - 1$, and so $\nabla \optmes ( \Kbf; \vbf_0) = \nabla \costhorizonone ( \Kbf ; \{\vbf_0\})$.
	Thus, applying the cost gradient formula of \cref{lem:lqr_gradient}, for both $\nabla \optmes ( \Kbf^{(1)} ; \vbf_0)$ and $\nabla \cost ( \Kbf^{(1)} ; \xbf_0)$, while recalling that $\Qbf = \Ibf$ and $\Kbf^{(1)} = \0$, we obtain:
	\[
	\inprod{ \nabla \optmes  (\Kbf^{(1)}; \vbf_0) }{ \nabla \cost(\Kbf^{(1)}; \xbf_0) } = \inprod{ 2 \Bbf^\top  \Abf  \vbf_0 \vbf_0^\top }{ 2 \Bbf^\top \sum\nolimits_{h = 0}^{H - 1} \brk2{ \sum\nolimits_{s = 1}^{H - h} \brk{ \Abf^{s - 1} }^\top \Abf^s } \Sigmabf_{\xbf_0, h} }
	\text{\,,}
	\]
	with $\Sigmabf_{\xbf_0, h} := \frac{1}{ \abs{\X} } \sum\nolimits_{\xbf_0 \in \X} \xbf_h \xbf_h^\top = \frac{1}{ \abs{\X} } \sum\nolimits_{\xbf_0 \in \X} \Abf^h \xbf_0 \brk[s]{ \Abf^h \xbf_0 }^\top$ for $h \in \{0\} \cup [H - 1]$.
	Since $\Bbf$ is an orthogonal matrix, by the identity $\Tr (\Xbf^\top \Ybf) = \Tr (\Xbf \Ybf^\top) = \inprod{\Xbf}{\Ybf}$ for matrices $\Xbf, \Ybf$ of the same dimensions, and the cyclic property of the trace, we get:
	\[
	\begin{split}
		\inprod{ \nabla \optmes  (\Kbf^{(1)}; \vbf_0) }{ \nabla \cost(\Kbf^{(1)}; \xbf_0) } & = 4 \Tr \brk*{ \vbf_0 \vbf_0^\top \Abf^\top \sum\nolimits_{h = 0}^{H - 1} \sum\nolimits_{s = 1}^{H - h} \brk{ \Abf^{s - 1} }^\top \Abf^s  \Sigmabf_{\xbf_0, h} } \\
		& = 4\sum\nolimits_{h = 0}^{H - 1} \sum\nolimits_{s = 1}^{H - h} \Tr \brk*{ \vbf_0 \vbf_0^\top \brk{ \Abf^{s} }^\top \Abf^s  \Sigmabf_{\xbf_0, h} } \\
		& = 4 \sum\nolimits_{h = 0}^{H - 1} \sum\nolimits_{s = 1}^{H - h} \Tr \brk*{  \Sigmabf_{\xbf_0, h}  \vbf_0 \vbf_0^\top \brk{ \Abf^{s} }^\top \Abf^s } \\
		& = 4 \sum\nolimits_{h = 0}^{H - 1} \sum\nolimits_{s = 1}^{H - h} \Tr \brk*{  \Abf^h \xbf_0 \xbf_0^\top (\Abf^h)^\top  \vbf_0 \vbf_0^\top \brk{ \Abf^{s} }^\top \Abf^s } \\
		& = 4 \sum\nolimits_{h = 0}^{H - 1} \sum\nolimits_{s = 1}^{H - h} \inprod{ \vbf_0 }{ \Abf^h \xbf_0 } \cdot \Tr \brk*{  \Abf^h \xbf_0 \vbf_0^\top \brk{ \Abf^{s} }^\top \Abf^s }
		\text{\,.}
	\end{split}
	\]
	The trace of a matrix and its transpose are equal.
	Hence, $\Tr \brk1{  \Abf^h \xbf_0 \vbf_0^\top \brk{ \Abf^{s} }^\top \Abf^s } = \Tr \brk1{   \brk{ \Abf^{s} }^\top  \Abf^s \vbf_0 \brk{ \Abf^h \xbf_0 }^\top }$.
	Applying the cyclic property of the trace once more, and introducing the indices $n = h$ and $k = s - 1$, concludes:
	\[
	\begin{split}
		\inprod{ \nabla \optmes  (\Kbf^{(1)}; \vbf_0) }{ \nabla \cost(\Kbf^{(1)}; \xbf_0) } & = 4 \sum\nolimits_{h = 0}^{H - 1} \sum\nolimits_{s = 1}^{H - h} \inprod{ \vbf_0 }{ \Abf^h \xbf_0 } \cdot \Tr \brk*{  \Abf^h \xbf_0 \vbf_0^\top \brk{ \Abf^{s} }^\top \Abf^s } \\
		& = 4 \sum\nolimits_{n = 0}^{H - 1} \sum\nolimits_{k = 0}^{H - n - 1} \inprod{ \vbf_0 }{ \Abf^n \xbf_0 } \cdot  \Tr \brk*{ \vbf_0 (\Abf^{n}\xbf_0)^\top (\Abf^{k + 1})^{\top} \Abf^{k + 1}}
		\text{\,.}
	\end{split}
	\]
\end{proof}

\begin{lemma}
	\label{smoothness_lemma}
	The function $\optmes (\cdot)$ is $2$-smooth.
	That is, for any $\Kbf, \Kbf' \in \R^{D \times D}$ it holds that $\norm{ \nabla \optmes (\Kbf ) - \nabla \optmes (\Kbf') } \leq 2 \norm{\Kbf - \Kbf'}$.
\end{lemma}

\begin{proof}
	For a controller $\Kbf \in \R^{ D \times D }$ and $\vbf_0 \in \orthstates$, let $\costhorizonone ( \Kbf ; \{ \vbf_0 \} ) :=  \norm{\vbf_0}^{2}+\norm{ (\Abf+\Bbf \Kbf ) \vbf_0 }^{2}$ denote the cost (\cref{eq:cost}) that it attains over $\vbf_0$ for a time horizon $H = 1$. 
	Notice that $\optmes ( \Kbf; \vbf_0) = \costhorizonone ( \Kbf ; \{\vbf_0\}) - 1$, and so $\nabla \optmes ( \Kbf; \vbf_0) = \nabla \costhorizonone ( \Kbf ; \{\vbf_0\})$.
	Thus, applying the cost gradient formula of \cref{lem:lqr_gradient}, we obtain:
	\[
	\nabla \optmes ( \Kbf ; \vbf_0 ) = 2 \Bbf^\top (\Abf + \Bbf \Kbf) \vbf_0 \vbf_0^\top
	\text{\,.}
	\]
	For any $\Kbf' \in \R^{D \times D}$ the above formula gives:
	\[
	\begin{split}
		\norm*{ \nabla \optmes (\Kbf ; \vbf_0) - \nabla \optmes (\Kbf' ; \vbf_0) } & = \norm*{ 2 \Bbf^\top \Bbf (\Kbf - \Kbf')  \vbf_0 \vbf_0^\top } \\
		& = \norm*{ 2 (\Kbf - \Kbf')  \vbf_0 \vbf_0^\top } 
		\text{\,,}
	\end{split}
	\]
	where the second equality is by recalling that $\Bbf^\top \Bbf = \Ibf$.
	By sub-multiplicativity of the matrix Euclidean norm, we get that:
	\[
	\norm*{ \nabla \optmes (\Kbf ; \vbf_0) - \nabla \optmes (\Kbf' ; \vbf_0) }  \leq 2 \norm*{ \vbf_0 \vbf_0^\top } \cdot \norm{\Kbf - \Kbf'} = 2 \norm*{ \Kbf - \Kbf' }
	\text{\,.}
	\]
	where the last equality is due to $\vbf_0$ being of unit norm. Finally, we have:
	\[
	\optmes ( \Kbf ) = \frac{1}{ \abs{\orthstates} } \sum\nolimits_{\vbf_0 \in \orthstates} \optmes ( \Kbf; \vbf_0)
	\text{\,,}
	\]
	and therefore $\optmes (\cdot)$ is $2$-smooth, being an average of $2$-smooth functions.
\end{proof}

\subsubsection{Lower Bound on $\EE_\Abf \brk[s]1{ \inprodbig{ \nabla \optmes ( \Kbf^{(1)} ) }{ \nabla \cost (\Kbf^{(1)} ; \trainstates ) } }$}
\label{app:proofs:typical_system:in_prod_lower_bound}

In this part of the proof, we establish that:
\[
	\EE\nolimits_\Abf \brk[s]*{ \inprodbig{ \nabla \optmes ( \Kbf^{(1)} ) }{ \nabla \cost (\Kbf^{(1)} ; \trainstates ) } } \geq \frac{2H(H-1)}{D}
	\text{\,.}
\]
To do so, it suffices to show that for all $\vbf_0 \in \orthstates$ it holds that $\EE_\Abf \brk[s]1{ \inprodbig{ \nabla \optmes ( \Kbf^{(1)} ; \vbf_0 ) }{ \nabla \cost (\Kbf^{(1)} ; \trainstates ) } } \geq \frac{2H(H-1)}{D}$.
Indeed, since $\optmes (\Kbf^{(1)}) = \frac{1}{ \abs{\orthstates} } \sum\nolimits_{\vbf_0 \in \orthstates} \optmes (\Kbf^{(1)} ; \vbf_0)$, linearity of the gradient and expectation then yield the desired lower bound.

We begin by proving that the expected inner product does not depend on the choice of orthonormal initial states in $\trainstates$ and $\orthstates$.

\begin{lemma} 
	\label{lem:in_prod_initial_states}
	For any $\vbf_0 \in \orthstates$, $\xbf_0 \in \trainstates$, and any two different standard basis vectors $\ebf_i, \ebf_j \in \R^D$:
	\[
	\EE\nolimits_{\Abf} \brk[s]*{ \inprod{ \nabla \optmes (\Kbf^{(1)} ; \vbf_0 ) }{\nabla \cost(\Kbf^{(1)}; \xbf_0 ) } } = \EE\nolimits_{\Abf} \brk[s]*{ \inprod{ \nabla \optmes (\Kbf^{(1)}; \ebf_{i}) }{ \nabla \cost(\Kbf^{(1)}; \ebf_{j}) } }
	\text{\,,}
	\]
	and, in particular, for any $n \in \{ 0\} \cup [H  - 1]$ and $k \in \{0\} \cup [H - n - 1]$:
	\[
		\EE\nolimits_{\Abf} \brk[s]*{ \inprod{ \vbf_0 }{ \Abf^{n}\xbf_0 } \cdot \Tr  \brk*{ \vbf_0 (\Abf^{n}\xbf_0)^\top (\Abf^{k + 1})^{\top} \Abf^{k + 1}} } = 
		\EE\nolimits_{\Abf} \brk[s]*{ \inprod{ \ebf_i }{ \Abf^{n} \ebf_j} \cdot \Tr  \brk*{ \ebf_i (\Abf^{n} \ebf_j)^\top (\Abf^{k + 1})^{\top} \Abf^{k + 1}} }
		\text{\,.}
	\]
\end{lemma}

\begin{proof}
	By \cref{lem:clm_cost_grad_inner_prod}:
	\[
	\inprod{ \nabla \optmes ( \Kbf^{(1)} ; \vbf_0) }{ \nabla \cost( \Kbf^{(1)} ; \xbf_0 ) } = 4 \sum\nolimits_{n =  0}^{H - 1} \sum\nolimits_{k = 0}^{H - n - 1} \inprod{ \vbf_0 }{ \Abf^{n}\xbf_0 } \cdot \Tr  \brk*{ \vbf_0 (\Abf^{n}\xbf_0)^\top (\Abf^{k + 1})^{\top} \Abf^{k + 1}}
	\text{\,.}
	\]
	It suffices to show that, for all $n \in \{0\} \cup [H - 1]$ and $k \in \{ 0\} \cup [H - n - 1]$, the random variable
	\[
	\inprod{ \vbf_0 }{ \Abf^{n}\xbf_0 } \cdot \Tr \brk*{ \vbf_0 (\Abf^n \xbf_0)^\top (\Abf^{k + 1})^\top \Abf^{k + 1} }
	\]
	is distributed identically as the random variable
	\[
	\inprod{ \ebf_i }{ \Abf^{n} \ebf_j } \cdot \Tr \brk*{ \ebf_i (\Abf^n \ebf_j)^\top (\Abf^{k + 1})^\top \Abf^{k + 1} }
	\text{\,.}
	\]
	Indeed, this implies that, for all $n \in \{0\} \cup [H - 1]$ and $k \in \{ 0\} \cup [H - n - 1]$:
	\[
	\EE\nolimits_{\Abf} \brk[s]*{ \inprod{ \vbf_0 }{ \Abf^{n}\xbf_0 } \cdot \Tr \brk*{ \vbf_0 (\Abf^n \xbf_0)^\top (\Abf^{k + 1})^\top \Abf^{k + 1} } } = \EE\nolimits_{\Abf} \brk[s]^{ \inprod{ \ebf_i }{ \Abf^{n} \ebf_j } \cdot \Tr \brk*{ \ebf_i (\Abf^n \ebf_j)^\top (\Abf^{k + 1})^\top \Abf^{k + 1} } }
	\text{\,,}
	\]
	and so from linearity of the expectation:
	\[
	\begin{split}
		\EE\nolimits_{\Abf} \brk[s]*{ \inprod{ \nabla \optmes (\Kbf^{(1)} ; \vbf_0 ) }{\nabla \cost(\Kbf^{(1)}; \xbf_0 ) } } & = \EE\nolimits_{\Abf} \brk[s]*{ 4 \sum\nolimits_{n =  0}^{H - 1} \sum\nolimits_{k = 0}^{H - n - 1} \inprod{ \vbf_0 }{ \Abf^{n}\xbf_0 } \cdot \Tr  \brk*{ \vbf_0 (\Abf^{n}\xbf_0)^\top (\Abf^{k + 1})^{\top} \Abf^{k + 1}} } \\
		& = \EE\nolimits_{\Abf} \brk[s]*{ 4 \sum\nolimits_{n =  0}^{H - 1} \sum\nolimits_{k = 0}^{H - n - 1}  \inprod{ \ebf_i }{ \Abf^{n} \ebf_j } \cdot \Tr \brk*{ \ebf_i (\Abf^n \ebf_j)^\top (\Abf^{k + 1})^\top \Abf^{k + 1} } } \\
		& =  \EE\nolimits_{\Abf} \brk[s]*{ \inprod{ \nabla \optmes (\Kbf^{(1)}; \ebf_{i}) }{ \nabla \cost(\Kbf^{(1)}; \ebf_{j}) } }
		\text{\,.}
	\end{split}
	\]
	
	Now, fix some $n \in \{0\} \cup [H - 1]$ and $k \in \{ 0\} \cup [H - n - 1]$.
	Let $\Ubf \in \R^{D \times D}$ be an orthogonal matrix satisfying $\Ubf \ebf_i = \vbf_0$ and $\Ubf \ebf_j = \xbf_0$, and let $\Mbf := \Ubf^\top \Abf \Ubf \in \R^{D \times D}$.
	Consider the random variable
	\[
	\inprod{ \ebf_i }{ \Mbf^{n} \ebf_j } \cdot \Tr \brk*{ \ebf_i (\Mbf^n \ebf_j)^\top (\Mbf^{k + 1})^\top \Mbf^{k + 1} }
	\text{\,.}
	\]
	By the definitions of $\Ubf$ and $\Mbf$ we have that $\inprod{ \ebf_i }{ \Mbf^{n} \ebf_j }  = \ebf_i^\top \Mbf^n \ebf_j =  \ebf_i^\top \Ubf^\top \Abf^n \Ubf \ebf_j = \vbf_0^\top \Abf^n \xbf_0$ and:
	\[
	\begin{split}
		\Tr \brk*{ \ebf_i (\Mbf^n \ebf_j)^\top (\Mbf^{k + 1})^\top \Mbf^{k + 1} } & = \Tr \brk*{ \Ubf^\top \vbf_0 (\Abf^n \xbf_0)^\top \Ubf (\Mbf^{k + 1})^\top \Mbf^{k + 1} } \\
		& = \Tr \brk*{ \vbf_0 (\Abf^n \xbf_0)^\top \Ubf (\Mbf^{k + 1})^\top \Mbf^{k + 1} \Ubf^\top} \\
		& = \Tr \brk*{ \vbf_0 (\Abf^n \xbf_0)^\top \Ubf \Ubf^\top (\Abf^{k + 1})^\top \Ubf\Ubf^\top \Abf^{k + 1} \Ubf \Ubf^\top} \\
		& = \Tr \brk*{ \vbf_0 (\Abf^n \xbf_0)^\top (\Abf^{k + 1})^\top \Abf^{k + 1} } 
		\text{\,,}
	\end{split}
	\]
	where the second equality is by the cyclic property of the trace.
	Thus:
	\[
	\inprod{ \ebf_i }{ \Mbf^{n} \ebf_j } \cdot \Tr \brk*{ \ebf_i (\Mbf^n \ebf_j)^\top (\Mbf^{k + 1})^\top \Mbf^{k + 1} } = \inprod{ \vbf_0 }{ \Abf^{n} \xbf_0 } \cdot \Tr \brk*{ \vbf_0 (\Abf^n \xbf_0)^\top (\Abf^{k + 1})^\top \Abf^{k + 1} } 
	\text{\,.}
	\]
	Notice that the orthogonality of $\Ubf$ implies that the entries of $\Mbf$ are independent Gaussian random variables with mean zero and standard deviation $1 / \sqrt{D}$.
	That is, the entries of $\Mbf$ and $\Abf$ are identically distributed.
	Combined with the equality above, we conclude that $\inprod{ \vbf_0 }{ \Abf^{n} \xbf_0 } \cdot \Tr \brk*{ \vbf_0 (\Abf^n \xbf_0)^\top (\Abf^{k + 1})^\top \Abf^{k + 1} } $ and $\inprod{ \ebf_i }{ \Abf^{n} \ebf_j } \cdot \Tr \brk*{ \ebf_i (\Abf^n \ebf_j)^\top (\Abf^{k + 1})^\top \Abf^{k + 1} }$ are identically distributed.
\end{proof}

\medskip

With \cref{lem:in_prod_initial_states} in place, we now lower bound the expected inner product between $\nabla \optmes ( \Kbf^{(1)} ; \vbf_0 )$ and $\nabla \cost (\Kbf^{(1)} ; \trainstates )$, for any $\vbf_0 \in \orthstates$ as necessary.

	Let $\vbf_0 \in \orthstates$.
	By \cref{lem:clm_cost_grad_inner_prod}:
	\[
	\inprod{ \nabla \optmes ( \Kbf^{(1)} ; \vbf_0) }{ \nabla \cost( 	\Kbf^{(1)} ; \trainstates ) } = \frac{ 4 }{ \abs{\trainstates } } \sum_{\xbf_0 \in \trainstates} \sum_{n =  0}^{H - 1} \sum_{k = 0}^{H - n - 1} \inprod{ \vbf_0 }{ \Abf^{n}\xbf_0 } \cdot \Tr  \brk*{ \vbf_0 (\Abf^{n}\xbf_0)^\top (\Abf^{k + 1})^{\top} \Abf^{k + 1}} 
	\text{\,.}
	\]
	Taking the expectation with respect to $\Abf$, we get:
	\be
		\EE\nolimits_{\Abf}  \brk[s]*{ \inprod{ \nabla \optmes (\Kbf^{(1)} ; \vbf_0) }{ \nabla \cost(\Kbf^{(1)};\trainstates) } } = \frac{ 4 }{ \abs{\trainstates } } \sum_{\xbf_0 \in \trainstates} \sum_{n =  0}^{H - 1} \sum_{k = 0}^{H - n - 1}  \EE\nolimits_{\Abf}  \brk[s]*{ \inprod{ \vbf_0 }{ \Abf^{n}\xbf_0 } \cdot \Tr  \brk*{ \vbf_0 (\Abf^{n}\xbf_0)^\top (\Abf^{k + 1})^{\top} \Abf^{k + 1}} }
		\text{\,.}
	\label{eq:in_prod_lower_bound_opt_train_cost_expec_grad_form}
	\ee
	The sought-after result will readily follow from the lemma below.
	
	\begin{lemma}
		\label{lem:expec_in_prod_term_lower_bound}
		For any $\xbf_0 \in \trainstates, \vbf_0 \in \orthstates, n \in [H - 1]$, and $k \in \{ 0 \} \cup [H - n - 1]$ it holds that:
		\be
		\EE\nolimits_{\Abf}  \brk[s]*{ \inprod{ \vbf_0 }{ \Abf^{n}\xbf_0 } \cdot \Tr  \brk*{ \vbf_0 (\Abf^{n}\xbf_0)^\top (\Abf^{k + 1})^{\top} \Abf^{k + 1}} }  \geq \frac{1}{D}
		\text{\,.}
		\label{eq:grad_inner_prod_one_term_expectation_lower_bound}
		\ee
	\end{lemma}
	\begin{proof}
		According to \cref{lem:in_prod_initial_states}, we may replace $\xbf_0$ and $\vbf_0$ with any two different standard basis vectors $\ebf_j \in \R^D$ and $\ebf_i \in \R^D$, respectively, since:
		\[
			\EE\nolimits_{\Abf}  \brk[s]*{ \inprod{ \vbf_0 }{ \Abf^{n}\xbf_0 } \cdot \Tr  \brk*{ \vbf_0 (\Abf^{n}\xbf_0)^\top (\Abf^{k + 1})^{\top} \Abf^{k + 1}} }
			= \EE\nolimits_{\Abf}  \brk[s]*{ \inprod{ \ebf_i }{ \Abf^{n} \ebf_j } \cdot \Tr  \brk*{ \ebf_i (\Abf^{n} \ebf_j)^\top (\Abf^{k + 1})^{\top} \Abf^{k + 1}} }
			\text{\,.}
		\]
		Thus, in what follows we show that:
		\[
			\EE\nolimits_{\Abf}  \brk[s]*{ \inprod{ \ebf_i }{ \Abf^{n} 	\ebf_j } \cdot \Tr  \brk*{ \ebf_i (\Abf^{n} \ebf_j)^\top (\Abf^{k + 1})^{\top} \Abf^{k + 1}} } \geq \frac{1}{D}
			\text{\,.}
		\]
		
		First, let us consider the case of $k = 0$.
		Note that in this case, by the cyclic property of the trace:
		\[
		\inprod{ \ebf_i }{ \Abf^{n} \ebf_j } \cdot \Tr  \brk*{ \ebf_i (\Abf^{n} \ebf_j)^\top (\Abf^{k + 1})^{\top} \Abf^{k + 1} } =  \inprod{ \ebf_i }{ \Abf^{n} \ebf_j } \cdot \inprod{  \Abf \ebf_i }{  \Abf^{n + 1} \ebf_j }
		\text{\,.}
		\]
		Denoting the $(w, z)$'th entry of $\Abf$ by $a_{w, z} \in \R$, for $w, z \in [D]$, the inner products on the right hand side can be written as:
		\[
		\inprod{ \ebf_i }{ \Abf^{n} \ebf_j } = \sum_{t_1, \ldots, t_{n - 1} = 1}^D a_{i, t_i} a_{t_1, t_2} \cdot \cdots \cdot a_{t_{n - 1}, j}
		\text{\,,}
		\]
		and:
		\[
		\inprod{  \Abf \ebf_i }{  \Abf^{n + 1} \ebf_j } = \sum_{l, r_1, \ldots, r_n = 1}^D a_{l, i} \cdot a_{l, r_1} a_{r_1, r_2} \cdot \cdots \cdot a_{r_n, j}
		\text{\,.}
		\]
		Combining both equations above leads to:
		\be
		\inprod{ \ebf_i }{ \Abf^{n} \ebf_j } \cdot \inprod{  \Abf \ebf_i }{  \Abf^{n + 1} \ebf_j } = \sum_{t_1, \ldots, t_{n - 1} = 1}^D \sum_{l, r_1, \ldots, r_n = 1}^D a_{i, t_i} a_{t_1, t_2} \cdot \cdots \cdot a_{t_{n - 1}, j} \cdot a_{l, i} \cdot a_{l, r_1} a_{r_1, r_2} \cdot \cdots \cdot a_{r_n, j}
		\text{\,.}
		\label{eq:A_entries_sum_k_zero}
		\ee
		Since the entries of $\Abf$ are independently distributed according to a zero-mean Gaussian with standard deviation $1 / \sqrt{D}$, basic properties of Gaussian random variables imply that, for any $w, z \in [D]$ and $p \in \N$:
		\be
		D^{ \frac{p}{2} } \cdot \EE\nolimits_{\Abf} \brk[s]*{ a_{w, z}^p } = \begin{cases}
			(p - 1)!! := (p - 1) (p - 3) \cdots 3 & , \text{ if $p$ is even} \\
			0 & , \text{ otherwise}
		\end{cases}
		\text{\,.}
		\label{eq:gaussian_power_expectation}
		\ee
		According to the above, the expectation of each summand on the right hand side of \cref{eq:A_entries_sum_k_zero} is non-negative.
		Moreover, for the expectation of a summand to be positive, every entry of $\Abf$ in it needs to have an even power (otherwise, the expectation is zero).
		We now describe a subset of indices for which this occurs.
		Consider indices $l, t_1, \ldots, t_{n - 1}, r_1, \ldots, r_n$ satisfying:
		\[
		t_1 = r_2, t_2 = r_3, \ldots, r_{n - 1} = r_n
		\text{\,.}
		\]
		This implies that:
		\[
		a_{t_1, t_2} = a_{r_2, r_3} , \ldots, a_{t_{n - 1}, i} = a_{r_n, i}
		\text{\,,}
		\]
		so these terms are paired up.
		We are left with $a_{i, r_2}, a_{l, i}, a_{l, r_1}, a_{r_1, r_2}$.
		Requiring that $r_1 = i$ pairs up these remaining terms.
		For all possible index assignments satisfying the specified constraints, the entries of $\Abf$ in the corresponding summand have even powers.
		As a result, by \cref{eq:gaussian_power_expectation} for such choice of indices:
		\[
		\EE\nolimits_{\Abf} \brk[s]*{ a_{i, t_i} a_{t_1, t_2} \cdot \cdots \cdot a_{t_{n - 1}, j} \cdot a_{l, i} \cdot a_{l, r_1} a_{r_1, r_2} \cdot \cdots \cdot a_{r_n, j} } \geq \frac{1}{ D^{n + 1} }
		\text{\,,}
		\]
		as there are overall $2n + 2$ terms in the product.
		It remains to count the number of index assignments, for the sums in \cref{eq:A_entries_sum_k_zero}, that satisfy the specified constraints.
		We have $D$ options for each of the unconstrained indices $l, r_2, \ldots, r_n$, and so there are overall $D^n$ relevant index assignments.
		Thus:
		\[
		\EE\nolimits_{\Abf}  \brk[s]*{  \inprod{ \ebf_i }{ \Abf^{n} \ebf_j } \cdot \inprod{  \Abf \ebf_i }{  \Abf^{n + 1} \ebf_j } } \geq \frac{D^n}{D^{n + 1}} = \frac{1}{D}
		\text{\,,}
		\]
		\ie, we have established \cref{eq:grad_inner_prod_one_term_expectation_lower_bound} for the case of $k = 0$.
		
		Next, we show that \cref{eq:grad_inner_prod_one_term_expectation_lower_bound} holds for $k \in [H - n - 1]$ by reducing this case to the case of $k = 0$.
		Notice that the expression within the expectation from \cref{eq:grad_inner_prod_one_term_expectation_lower_bound} can be written as:
		\[
		\inprod{ \ebf_i }{ \Abf^{n} \ebf_j } \cdot \Tr  \brk*{ \ebf_i (\Abf^{n} \ebf_j)^\top (\Abf^{k + 1})^{\top} \Abf^{k + 1}} = \inprod{ \ebf_i }{ \Abf^{n} \ebf_j } \cdot \inprod{ (\Abf^k)^\top \Abf^k \Abf \ebf_i }{ \Abf^{n + 1} \ebf_j }
		\text{\,.}
		\]
		Denoting $\Cbf := (\Abf^k)^\top \Abf^k$ and the $(w, z)$'th entry of $\Cbf$ by $c_{w,z}$, we have that:
		\be
		\begin{split}
			& \inprod{ \ebf_i }{ \Abf^{n} \ebf_j } \cdot \Tr  \brk*{ \ebf_i (\Abf^{n} \ebf_j)^\top (\Abf^{k + 1})^{\top} \Abf^{k + 1}}  \\[0.4em]
			& \hspace{3mm} = \inprod{ \ebf_i }{ \Abf^{n} \ebf_j } \cdot \inprod{ \Cbf \Abf \ebf_i }{ \Abf^{n + 1} \ebf_j } \\
			& \hspace{3mm} = \sum\nolimits_{w, z = 1}^D \sum_{t_1, \ldots, t_{n - 1} = 1}^D \sum_{l, r_1, \ldots, r_n = 1}^D a_{i, t_i} a_{t_1, t_2} \cdot \cdots \cdot a_{t_{n - 1}, j} \cdot c_{w, z} \cdot  a_{z, i} \cdot a_{w, r_1} a_{r_1, r_2} \cdot \cdots \cdot a_{r_n, j}
			\text{\,.}
		\end{split}
		\label{eq:A_entries_sum_k}
		\ee
		Let us focus on the contribution of $\Cbf$ to the expression above.
		For $w, z \in [D]$, the $(w, z)$'th entry of $\Cbf$ is given by:
		\[
		c_{w, z} = \sum_{y_1, \ldots, y_{2k - 1} = 1}^D ( a_{y_1, w} a_{y_2, y_1} \cdot \cdots \cdot a_{y_k, y_{k - 1}} ) \cdot ( a_{y_k, y_{k + 1} } a_{ y_{k + 1}, y_{k + 2}} \cdot \cdots \cdot a_{y_{2k - 1}, z} )
		\text{\,.}
		\]
		As before, we would like to pair up indices to achieve a lower bound on the number of summands in \cref{eq:A_entries_sum_k} in which entries of $\Abf$ have even powers.
		We therefore require that:
		\[
		y_1 = y_{2k - 1}, y_2 = y_{2k - 2}, \ldots, y_{k - 1} = y_{k + 1}
		\text{\,.}
		\]
		This pairs up all but the leftmost and rightmost terms of $c_{w, z}$.
		Going back to \cref{eq:A_entries_sum_k} and imposing $w = z$ on the index assignment, we have matched the terms added due to $k$ being non-zero.
		For the remaining indices, we can apply the same matching scheme used for the $k = 0$ case.
		Due to $k$ being non-zero, there are $2k$ additional terms, which add to the power of $D$ when lower bounding the expectation of each summand, but they are compensated by a summation over $k$ additional unconstrained indices.
	\end{proof}
	
	\medskip
	
	Overall, going back to \cref{eq:in_prod_lower_bound_opt_train_cost_expec_grad_form} and applying \cref{lem:expec_in_prod_term_lower_bound} concludes the proof:
	\[
	\begin{split}
		\EE\nolimits_{\Abf}  \brk[s]*{ \inprod{ \nabla \optmes (\Kbf^{(1)} ; \vbf_0) }{ \nabla \cost(\Kbf^{(1)};\trainstates) } } & = \frac{ 4 }{ \abs{\trainstates } } \sum_{\xbf_0 \in \trainstates} \sum_{n =  0}^{H - 1} \sum_{k = 0}^{H - n - 1}  \EE\nolimits_{\Abf}  \brk[s]*{ \inprod{ \vbf_0 }{ \Abf^{n}\xbf_0 } \cdot \Tr  \brk*{ \vbf_0 (\Abf^{n}\xbf_0)^\top (\Abf^{k + 1})^{\top} \Abf^{k + 1}} } \\
		& = \frac{ 4 }{ \abs{\trainstates } } \sum_{\xbf_0 \in \trainstates} \sum_{n =  1}^{H - 1} \sum_{k = 0}^{H - n - 1}  \EE\nolimits_{\Abf}  \brk[s]*{ \inprod{ \vbf_0 }{ \Abf^{n}\xbf_0 } \cdot \Tr  \brk*{ \vbf_0 (\Abf^{n}\xbf_0)^\top (\Abf^{k + 1})^{\top} \Abf^{k + 1}} } \\
		& \geq \frac{ 4 }{ \abs{\trainstates } } \sum_{\xbf_0 \in \trainstates} \sum_{n =  1}^{H - 1} \sum_{k = 0}^{H - n - 1} \frac{1}{D} \\
		& = \frac{2 H (H - 1) }{D}
		\text{\,,}
	\end{split}
	\]
	where the second equality is by noticing that for $n = 0$ the expectation is zero since $\inprod{ \vbf_0 }{ \Abf^n \xbf_0} = \inprod{\vbf_0}{\xbf_0} = 0$.

\subsubsection{Optimality Measure Decreases in Expectation}
\label{expectation_bound}

In this part of the proof, we establish that for any step size \(\eta \leq \frac{1}{ 4 D H (H-1) (4H - 1)!!}\): 
\[
\frac{ \EE_\Abf \brk[s]*{ \optmes \brk*{ \Kbf^{(2)} }  } }{ \EE_\Abf \brk[s]*{ \optmes \brk*{ \Kminnorm }  } } \leq 1 - \eta \cdot \frac{H (H - 1)}{D}
\text{\,.}
\]

\medskip

By \cref{smoothness_lemma}, $\optmes (\cdot)$ is $2$-smooth.
Thus:
\be
\begin{split}
	\optmes (\Kbf^{(2)}) & \leq \optmes (\Kbf^{(1)}) + \inprodbig{ \nabla \optmes (\Kbf^{(1)}) }{ \Kbf^{(2)} - \Kbf^{(1)} } + \norm{ \Kbf^{(2)} - \Kbf^{(1)}}^2 \\
	& = \optmes (\Kbf^{(1)}) - \eta \cdot \inprodbig{ \nabla \optmes (\Kbf^{(1)} ) }{ \nabla \cost ( \Kbf^{(1)} ; \trainstates) } + \eta^2 \cdot \norm{ \nabla \cost (\Kbf^{(1)} ; \trainstates ) }^2
	\text{\,.}
\end{split}
\label{eq:clm_descent_lem_single}
\ee
As we proved in \cref{app:proofs:typical_system:in_prod_lower_bound}, the expected inner product between $\nabla \optmes (\Kbf^{(1)})$ and $\nabla \cost (\Kbf^{(1)} ; \trainstates)$ is lower bounded as follows:
\[
\EE\nolimits_{\Abf}  \brk[s]*{ \inprod{ \nabla \optmes (\Kbf^{(1)}) }{ \nabla \cost(\Kbf^{(1)};\trainstates) } } \geq \frac{2H(H-1)}{D}
\text{\,.}
\]
Taking an expectation with respect to $\Abf$ over both sides of \cref{eq:clm_descent_lem_single} thus leads to:
\[
\begin{split}
	\EE\nolimits_{\Abf} \brk[s]*{ \optmes (\Kbf^{(2)}) } & \leq \EE\nolimits_{\Abf} \brk[s]*{ \optmes (\Kbf^{(1)}) } - \eta \cdot \EE\nolimits_{\Abf} \brk[s]*{ \inprodbig{ \nabla \optmes (\Kbf^{(1)} ) }{ \nabla \cost ( \Kbf^{(1)} ; \trainstates) } }+ \eta^2 \cdot \EE\nolimits_{\Abf} \brk[s]*{ \norm{ \nabla \cost (\Kbf^{(1)} ; \trainstates ) }^2 }
	\text{\,.} \\
	& \leq \EE\nolimits_{\Abf} \brk[s]*{ \optmes (\Kbf^{(1)}) } - \eta \cdot \frac{2H(H-1)}{D} + \eta^2 \cdot \EE\nolimits_{\Abf} \brk[s]*{ \norm{ \nabla \cost (\Kbf^{(1)} ; \trainstates ) }^2 }
	\text{\,.}
\end{split}
\]
Now, in order to upper bound $\EE\nolimits_{\Abf} \brk[s]1{ \norm{ \nabla \cost (\Kbf^{(1)} ; \trainstates ) }^2 }$, we employ a method from~\citet{redelmeier2014real}, mentioned in the proof outline (\cref{app:proofs:typical_system:outline}) and introduced in \cref{genus_expansion}, which facilitates computing expected traces of random matrix products through the topological concept of genus expansion.
For ease of exposition, we defer the upper bound on $\EE\nolimits_{\Abf} \brk[s]1{ \norm{ \nabla \cost (\Kbf^{(1)} ; \trainstates ) }^2 }$ to \cref{expected_grad_norm} in Appendix~\ref{expectation_bound:upper_bound_grad_norm} below.
Specifically, \cref{expected_grad_norm} shows that:
\[
\EE\nolimits_{\Abf} \brk[s]*{ \norm1{ \nabla \cost( \Kbf^{(1)} ; \trainstates ) }^2  } \leq 4  H^{2}(H-1)^{2}(4H-1)!! 
\text{\,.}
\]
Plugging this into our upper bound on $\EE\nolimits_{\Abf} \brk[s]*{ \optmes (\Kbf^{(2)}) }$ gives:
\[
\EE\nolimits_{\Abf} \brk[s]*{ \optmes (\Kbf^{(2)}) } \leq \EE\nolimits_{\Abf} \brk[s]*{ \optmes (\Kbf^{(1)}) } - \eta \cdot \frac{2H(H-1)}{D} + 4 \eta^2 \cdot H^{2}(H-1)^{2}(4H-1)!! 
\text{\,.}
\]
The assumption that $\eta \leq \frac{ 1 }{ 4 D H (H - 1) (4H - 1)!! }$ implies:
\[
4 \eta^2 \cdot H^{2}(H-1)^{2}(4H-1)!!  \leq  \frac{1}{2} \eta \cdot \frac{2H(H-1)}{D}
\text{\,.}
\]
Hence:
\be
\EE\nolimits_{\Abf} \brk[s]*{ \optmes (\Kbf^{(2)}) } \leq \EE\nolimits_{\Abf} \brk[s]*{ \optmes (\Kbf^{(1)}) } - \eta \cdot \frac{H(H-1)}{D}
\text{\,.}
\label{eq:expec_clm_single_upper_bound}
\ee
Note that $\Kbf^{(1)} \vbf_0 = \0$ for any $\vbf_0 \in \orthstates$, since $\Kbf^{(1)} =\0$, and similarly $\Kminnorm \vbf_0 = \0$ by the definition of $\Kminnorm$ in \cref{eq:unseen_controls_zero}.
Consequently, the expected optimality measures that they attain satisfy:
\[
\EE\nolimits_{\Abf} \brk[s]1{ \optmes ( \Kminnorm) } = \EE\nolimits_{\Abf} \brk[s]1{ \optmes ( \Kbf^{(1)}) } = \frac{ 1}{ \abs{\orthstates} } \sum\nolimits_{\vbf_0 \in \orthstates} \EE\nolimits_{\Abf} \brk[s]1{ \norm{ \Abf \vbf_0 }^2 } = 1
\text{\,,}
\]
due to $\vbf_0 \in \orthstates$ being of unit norm and the entries of $\Abf$ being independent Gaussian random variables with mean zero and standard deviation $1 / \sqrt{D}$.
Going back to \cref{eq:expec_clm_single_upper_bound} we may therefore conclude:
\[
\frac{ \EE\nolimits_{\Abf} \brk[s]*{ \optmes (\Kbf^{(2)}) } }{  \EE\nolimits_{\Abf} \brk[s]*{ \optmes ( \Kminnorm ) }  } \leq 1 - \eta \cdot \frac{H(H-1)}{D}
\text{\,.}
\]

\paragraph{Upper Bound on $\EE\nolimits_{\Abf} \brk[s]1{ \norm{ \nabla \cost (\Kbf^{(1)} ; \trainstates ) }^2 }$}
\label{expectation_bound:upper_bound_grad_norm}

\begin{lemma}
	\label{expected_grad_norm}
	It holds that:
	\[
	\EE\nolimits_{\Abf} \brk[s]*{ \norm{ \nabla \cost (\Kbf^{(1)} ; \trainstates ) }^2 } \leq 4  H^2 (H - 1)^2 (4H - 1)!! 
	\text{\,.}
	\]
\end{lemma}

\begin{proof}
	By \cref{lem:lqr_gradient} and the identity $\inprod{ \Xbf }{ \Ybf } = \Tr (\Xbf^\top \Ybf)$, for matrices $\Xbf, \Ybf$ of suitable dimensions, we get:
	\[
	\begin{split}
		& \norm1{ \nabla \cost (\Kbf^{(1)} ; \trainstates ) }^2 \\
		& \hspace{3mm} = \inprodbig{ \nabla \cost (\Kbf^{(1)} ; \trainstates ) }{ \nabla \cost (\Kbf^{(1)} ; \trainstates )  } \\
		& \hspace{3mm}= \frac{4}{ \abs{ \trainstates }^2 }\inprod{ \Bbf^\top \sum_{n = 0}^{H - 1} \sum_{k = 1}^{H - h} \brk{ \Abf^{k - 1} }^\top \Abf^k \sum_{\xbf_0 \in \trainstates} \Abf^n \xbf_0 \brk{ \Abf^n \xbf_0 }^\top }{ \Bbf^\top \sum_{m = 0}^{H - 1} \sum_{l = 1}^{H - h} \brk{ \Abf^{l - 1} }^\top \Abf^l  \sum_{\xbf_0 \in \trainstates} \Abf^m \ybf_0 \brk{ \Abf^m \ybf_0 }^\top } \\
		& \hspace{3mm} = \frac{ 4 }{ \abs{ \trainstates }^2 } \sum_{\xbf_0, \ybf_0 \in \trainstates} \sum_{n = 0}^{H - 1} \sum_{k = 1}^{H - n} \sum_{m = 0}^{H - 1} \sum_{l = 1}^{H - m} \Tr \brk*{ \Abf^n \xbf_0 \brk{ \Abf^n \xbf_0 }^\top (\Abf^k)^\top \Abf^{k - 1} \brk{ \Abf^{l - 1} }^\top \Abf^l \Abf^m \ybf_0 \brk{ \Abf^m \ybf_0 }^\top  }
		\text{\,.}
	\end{split}
	\]
	where recall $\Kbf^{(1)} = \0$ and $\Bbf$ is orthogonal.
	For convenience, let us change the summation over $k$ to be from $0$ to $H - n -1$, as opposed to from $1$ to $H - n$, and similarly the summation over $l$ to be from $0$ to $H - m - 1$.
	Along with the cyclic property of the trace we may write:
	\be
	\begin{split}
		& \norm1{ \nabla \cost (\Kbf^{(1)} ; \trainstates ) }^2 \\
		& \hspace{3mm} = \frac{ 4 }{ \abs{ \trainstates }^2 } \sum_{\xbf_0, \ybf_0 \in \trainstates} \sum_{n = 0}^{H - 1} \sum_{k = 0}^{H - n - 1} \sum_{m = 0}^{H - 1} \sum_{l = 0}^{H - m - 1} \Tr \brk*{ \Abf^n \xbf_0 \brk{ \Abf^n \xbf_0 }^\top (\Abf^{k + 1})^\top \Abf^{k} \brk{ \Abf^{l} }^\top \Abf^{l + 1} \Abf^m \ybf_0 \brk{ \Abf^m \ybf_0 }^\top  } \\
		& \hspace{3mm} = \frac{ 4 }{ \abs{ \trainstates }^2 } \sum_{\xbf_0, \ybf_0 \in \trainstates} \sum_{n = 0}^{H - 1} \sum_{k = 0}^{H - n - 1} \sum_{m = 0}^{H - 1} \sum_{l = 0}^{H - m - 1} \Tr \brk*{ \brk{ \Abf^n \xbf_0 }^\top (\Abf^{k + 1})^\top \Abf^{k} \brk{ \Abf^{l} }^\top \Abf^{l + 1} \Abf^m \ybf_0 \brk{ \Abf^m \ybf_0 }^\top  \Abf^n \xbf_0 } \\
		& \hspace{3mm} =  \frac{ 4 }{ \abs{ \trainstates }^2 } \sum_{\xbf_0, \ybf_0 \in \trainstates} \sum_{n = 0}^{H - 1} \sum_{k = 0}^{H - n - 1} \sum_{m = 0}^{H - 1} \sum_{l = 0}^{H - m - 1} \ybf_0^\top (\Abf^m)^\top  \Abf^n \xbf_0 \cdot \Tr \brk*{ \brk{ \Abf^n \xbf_0 }^\top (\Abf^{k + 1})^\top \Abf^{k} \brk{ \Abf^{l} }^\top \Abf^{l + 1} \Abf^m \ybf_0 } \\
		& \hspace{3mm} =  \frac{ 4 }{ \abs{ \trainstates }^2 } \sum_{\xbf_0, \ybf_0 \in \trainstates} \sum_{n = 0}^{H - 1} \sum_{k = 0}^{H - n - 1} \sum_{m = 0}^{H - 1} \sum_{l = 0}^{H - m - 1} \Tr \brk*{ (\Abf^n)^\top  \Abf^m \ybf_0 \xbf_0^\top } \cdot \Tr \brk*{ \brk{ \Abf^n \xbf_0 }^\top (\Abf^{k + 1})^\top \Abf^{k} \brk{ \Abf^{l} }^\top \Abf^{l + 1} \Abf^m \ybf_0 }  \\
		& \hspace{3mm} = \frac{ 4 }{ \abs{ \trainstates }^2 } \sum_{\xbf_0, \ybf_0 \in \trainstates} \sum_{n = 0}^{H - 1} \sum_{k = 0}^{H - n - 1} \sum_{m = 0}^{H - 1} \sum_{l = 0}^{H - m - 1} \Tr \brk*{ (\Abf^n)^\top  \Abf^m \ybf_0 \xbf_0^\top } \cdot \Tr \brk*{(\Abf^{n + k + 1})^\top \Abf^{k} \brk{ \Abf^{l} }^\top \Abf^{m + l + 1} \ybf_0 \xbf_0^\top} 
		\text{\,.}
	\end{split}
	\label{eq:sq_grad_norm_prelim_deriv}
	\ee
	Now, for each $\xbf_0$, $\ybf_0$, $n$, $k$, $m$, $l$ we will show that:
	\be
	\EE\nolimits_{\Abf} \brk[s]*{  \Tr \brk*{ (\Abf^n)^\top  \Abf^m \ybf_0 \xbf_0^\top }  \cdot \Tr \brk*{(\Abf^{n + k + 1})^\top \Abf^{k} \brk{ \Abf^{l} }^\top \Abf^{m + l + 1} \ybf_0 \xbf_0^\top} } \leq  (p - 1)!! 
	\text{\,,}
	\label{eq:sq_grad_norm_expec_single_term_upper_bound}
	\ee
	where $p := 2( n + k + m + l + 1) \leq 4H$ and $(p - 1)!! := (p - 1)(p - 3) \cdots 3$.
	To that end, we employ the method from~\citet{redelmeier2014real}, which is based on the topological concept of genus expansion.
	For completeness, \cref{genus_expansion} provides a self-contained introduction to the method, and \cref{thm:genus_expansion_main_theorem} therein lays out the result which we will use.
	We assume below familiarity with the notation and concepts detailed in \cref{genus_expansion}.
	
	For invoking \cref{thm:genus_expansion_main_theorem}, let us define a permutation $\gamma$ over $[p]$ via the cycle decomposition:
	\[
	\gamma = (1, \ldots, m + n)(m + n + 1, \ldots, p)
	\text{\,,}
	\]
	and a mapping $\epsilon : [p] \to \{-1, 1\}$ by:
	\[
	\epsilon(1) = -1 , \ldots ,\epsilon(n) = -1 \text{\,,}
	\]
	\[
	\epsilon(n+1) = 1 , \ldots , \epsilon(n+m)=1 \text{\,,}
	\]
	\[
	\epsilon(n+m+1) = -1 , \ldots , \epsilon(2n+m+k+1) = -1 \text{\,,}
	\]
	\[
	\epsilon(2n+m+k+2) = 1 , \ldots,\epsilon(2n+m+1+2k) = 1 \text{\,,}
	\]
	\[
	\epsilon(2n+m+2k+2) = -1 , \ldots , \epsilon(2n+m+2k+l+1) = -1 \text{\,,}
	\]
	\[
	\epsilon(2n+m+2k+l+2) = 1 , \ldots , \epsilon(p) = 1 \text{\,.}
	\]
	Furthermore, define $\Cbf_1, \ldots, \Cbf_p \in \R^{D \times D}$ by:
	\[
	\Cbf_1 = \Ibf, \ldots, \Cbf_{m + n - 1} = \Ibf, \Cbf_{m + n} = \ybf_0 \xbf_0^\top \text{\,,}
	\]
	\[
	\Cbf_{m + n + 1} = \Ibf, \ldots, \Cbf_{p - 1} = \Ibf, \Cbf_{p} = \ybf_0 \xbf_0^\top \text{\,,}
	\]
	where $\Ibf$ is the identity matrix.
	For the above choice of $\gamma, \epsilon,$ and matrices $\Cbf_1, \ldots, \Cbf_p$ it holds that:
	\[
	\EE\nolimits_{\Abf} \brk[s]*{ \Tr_\gamma \brk*{ \Abf_{\epsilon (1)} \Cbf_1, \ldots, \Abf_{\epsilon (p)} \Cbf_p } } = \EE \nolimits_{\Abf} \brk[s]*{ \Tr \brk*{ (\Abf^n)^\top  \Abf^m \ybf_0 \xbf_0^\top } \cdot \Tr \brk*{(\Abf^{n + k + 1})^\top \Abf^{k} \brk{ \Abf^{l} }^\top \Abf^{m + l + 1} \ybf_0 \xbf_0^\top}  }
	\text{\,.}
	\]
	Invoking \cref{thm:genus_expansion_main_theorem}, we may write \cref{eq:genus_formula} (from \cref{thm:genus_expansion_main_theorem} of \cref{genus_expansion}) as:
	\be
	\EE\nolimits_{\Abf} \brk[s]*{ \Tr_\gamma \brk*{ \Abf_{\epsilon (1)} \Cbf_1, \ldots, \Abf_{\epsilon (p)} \Cbf_p } } = \sum_{ \pi \in \{ \rho \delta \rho : \rho \in \M_p \} } D^{ \chi (\gamma, \delta_\epsilon \pi \delta_\epsilon) - \abs{\gamma} } \cdot \Trnormalized_{ \frac{ \gamma_-^{-1} \delta_\epsilon \pi \delta_\epsilon \gamma_+ }{ 2 } } \brk*{ \Cbf_1, \ldots, \Cbf_p }
	\text{\,.}
	\label{eq:genus_expectation_proof_formula}
	\ee
	Notice that, due to our choice of $\Cbf_1, \ldots, \Cbf_p$, each summand on the right hand side of \cref{eq:genus_expectation_proof_formula} is non-negative.
	
	Now, suppose that $\xbf_0 \neq \ybf_0$.
	Then, $\xbf_0$ is orthogonal to $\ybf_0$ (recall $\trainstates$ is an orthonormal set of initial states), and so:
	\[
	\Trnormalized \brk*{ \ybf_0 \xbf_0^\top } = \Trnormalized \brk*{ \xbf_0 \ybf_0^\top } = \Trnormalized \brk*{ (\ybf_0 \xbf_0^\top)^2 } = \Trnormalized \brk*{ (\xbf_0 \ybf_0^\top)^2 } = 0
	\text{\,,}
	\]
	while:
	\[
	\Trnormalized \brk*{ \xbf_0 \ybf_0^\top \ybf_0 \xbf_0^\top } = \Trnormalized \brk*{ \ybf_0 \xbf_0^\top \xbf_0 \ybf_0^\top } = \frac{1}{D}
	\text{\,.}
	\]
	The only way a summand on the right hand side of \cref{eq:genus_expectation_proof_formula}, corresponding to $\pi = \rho \delta \rho$, can provide a non-zero contribution is if a cycle $\RR = (1, \ldots, R)$ of $\gamma_{+}^{-1} \delta_\epsilon \pi \delta_\epsilon \gamma_-  / 2$ contains either no non-identity matrices, in which case $\Trnormalized (\Cbf_1 \cdot \cdots \cdot \Cbf_R) = 1$, or if it contains two non-identity matrices appearing once transposed and once without transposition, in which case $\Trnormalized (\Cbf_1 \cdot \cdots \cdot \Cbf_R) = 1 / D$.
	Accordingly, the two non-identity matrices among $\Cbf_1, \ldots, \Cbf_p$ must appear in a single cycle of $\gamma_{+}^{-1} \delta_\epsilon \pi \delta_\epsilon \gamma_-  / 2$ for a summand to be non-zero.
	It follows that the surface $\G (\gamma, \epsilon, \rho)$ (see construction in \cref{genus_expansion}) must be connected.
	Thus, by \cref{top_theorem} in \cref{genus_expansion}, the Euler characteristic of such a surface satisfies $\chi (\G (\gamma , \epsilon, \rho)) \leq 2$.
	Finally, from \cref{prop:euler_char} we know that $\chi (\G (\gamma , \epsilon, \rho)) = \chi (\gamma, \delta_\epsilon \pi \delta_\epsilon)$.
	As a result, a non-zero summand contributes at most $D^{\chi ( \gamma, \delta_\epsilon \pi \delta_\epsilon) - \abs{\gamma} -1} = D^{2 - 2 - 1} = D^{-1}$.
	
	Now, suppose that $\xbf_0 = \ybf_0$.
	In this case:
	\[
		\Trnormalized \brk*{ \ybf_0 \xbf_0^\top } = \Trnormalized \brk*{ \xbf_0 \ybf_0^\top } = \Trnormalized \brk*{ (\ybf_0 \xbf_0^\top)^2 } = \Trnormalized \brk*{ (\xbf_0 \ybf_0^\top)^2 }  =
		\Trnormalized \brk*{ \xbf_0 \ybf_0^\top \ybf_0 \xbf_0^\top } = \Trnormalized \brk*{ \ybf_0 \xbf_0^\top \xbf_0 \ybf_0^\top } = \frac{ 1}{ D }
		\text{\,.}
	\]
	If $\Cbf_{m + n} = \Cbf_p = \ybf_0 \xbf_0^\top$ are in the same cycle of $\gamma_{+}^{-1} \delta_\epsilon \pi \delta_\epsilon \gamma_-  / 2$, then as in the $\xbf_0 = \ybf_0$ case, the surface $\G (\gamma, \epsilon, \rho)$ is connected.
	Thus, by \cref{top_theorem} $\chi (\G (\gamma, \epsilon, \rho)) \leq 2$ and the corresponding summand contributes a factor of $D^{2 - 2 - 1} = D^{-1}$.
	On the other hand, 	If $\Cbf_{m + n} = \Cbf_p = \ybf_0 \xbf_0^\top$ are not in the same cycle of $\gamma_{+}^{-1} \delta_\epsilon \pi \delta_\epsilon \gamma_-  / 2$, then the surface $\G (\gamma, \epsilon, \rho)$ can have two connected components (it cannot have more than two because $\abs{\gamma} = 2$), and so $\chi (\G (\gamma , \epsilon, \rho)) \leq 4$ by \cref{top_theorem}.
	We therefore obtain a factor of $1 / D^2$ from the trace along $\gamma_{+}^{-1} \delta_\epsilon \pi \delta_\epsilon \gamma_-  / 2$ of $\Cbf_1, \ldots, \Cbf_p$, and the corresponding summand contributes at most $D^{4-2-2} = D^{0} = 1$.
	
	Overall, the number of summands in \cref{eq:genus_expectation_proof_formula} is $\abs{\M_p} = (p - 1)!!$, \ie~the number of pairings of $[p]$, and we have seen that each summand contributes at most $1$ (for both the $\xbf_0 = \ybf_0$ and $\xbf_0 \neq \ybf_0$ cases).
	Hence, from \cref{eq:genus_expectation_proof_formula} we get \cref{eq:sq_grad_norm_expec_single_term_upper_bound}:
	\[
	\EE\nolimits_{\Abf} \brk[s]*{ \Tr \brk*{ (\Abf^n)^\top  \Abf^m \ybf_0 \xbf_0^\top } \cdot \Tr \brk*{(\Abf^{n + k + 1})^\top \Abf^{k} \brk{ \Abf^{l} }^\top \Abf^{m + l + 1} \ybf_0 \xbf_0^\top} } \leq  (p - 1)!!  \leq (4H - 1)!!
	\text{\,.}
	\]
	Going back to \cref{eq:sq_grad_norm_prelim_deriv} and taking an expectation with respect to $\Abf$ concludes:
	\[
	\begin{split}
		& \EE\nolimits_{\Abf} \brk[s]*{ \norm1{ \nabla \cost (\Kbf^{(1)} ; \trainstates ) }^2 } \\
		& \hspace{3mm} = \frac{ 4 }{ \abs{ \trainstates }^2 } \sum_{\xbf_0, \ybf_0 \in \trainstates} \sum_{n = 0}^{H - 1} \sum_{k = 0}^{H - n - 1} \sum_{m = 0}^{H - 1} \sum_{l = 0}^{H - m - 1} \EE\nolimits_{\Abf} \brk[s]*{ \Tr \brk*{ (\Abf^n)^\top  \Abf^m \ybf_0 \xbf_0^\top } \cdot \Tr \brk*{(\Abf^{n + k + 1})^\top \Abf^{k} \brk{ \Abf^{l} }^\top \Abf^{m + l + 1} \ybf_0 \xbf_0^\top} } \\
		& \hspace{3mm} \leq 4 H^2 (H - 1)^2 (4H - 1)!!
		\text{\,.}
	\end{split}
	\]
\end{proof}

\subsubsection{Optimality Measure Decreases With High Probability}	
\label{high_prob_bound}

In this part of the proof, we establish that for any $\delta \in (0, 1)$, if $D \geq \abs{\trainstates} + \frac{ 6 \abs{ \trainstates} H (H - 1) (4H - 1)!! }{ \delta }$ and $\eta \leq \frac{ 1 }{ 8 D^{2} H (H-1)  (4H - 1)!! }$, then with probability at least $1 - \delta$ over the choice of $\Abf$:
\[
\frac{ \optmes \brk*{ \Kbf^{(2)}}  }{  \optmes \brk*{ \Kminnorm } } \leq 1 - \eta \cdot \frac{ H (H-1) }{ 4D }
\text{\,.}
\]

\medskip

To that end, we begin by converting the lower bound on $\EE_\Abf \brk[s]1{ \inprodbig{ \nabla \optmes ( \Kbf^{(1)} ) }{ \nabla \cost (\Kbf^{(1)} ; \trainstates ) } }$ from \cref{app:proofs:typical_system:in_prod_lower_bound} into a bound that holds with high probability.
By \cref{lem:clm_cost_grad_inner_prod}:
\[
	\inprod{ \nabla \optmes ( \Kbf^{(1)} ) }{ \nabla \cost( \Kbf^{(1)} ; \trainstates ) } = \frac{ 4 }{ \abs{ \trainstates } \abs{ \orthstates } } \sum_{\xbf_0 \in \trainstates} \sum_{\vbf_0 \in \orthstates} \sum_{n =  0}^{H - 1} \sum_{k = 0}^{H - n - 1} \inprod{ \vbf_0 }{ \Abf^{n}\xbf_0 } \cdot \Tr  \brk*{ \vbf_0 (\Abf^{n}\xbf_0)^\top (\Abf^{k + 1})^{\top} \Abf^{k + 1}}
\text{\,.}
\]
For $\xbf_0 \in \trainstates, \vbf_0 \in \orthstates, n \in \{0\} \cup [H - 1], k \in \{0\} \cup [H - n - 1]$, introducing the random variables:
\be
\begin{split}
Z_{\vbf_0, \xbf_0, n, k} & := \inprod{ \vbf_0 }{ \Abf^{n}\xbf_0 } \cdot \Tr  \brk*{ \vbf_0 (\Abf^{n}\xbf_0)^\top (\Abf^{k + 1})^{\top} \Abf^{k + 1}} \text{\,,} \\
Y_{\xbf_0, n, k} & := \frac{ 1 }{ \abs{ \orthstates } } \sum\nolimits_{\vbf_0 \in \orthstates} Z_{\vbf_0, \xbf_0, n, k} 
\text{\,,}
\end{split}
\label{eq:Z_Y_def}
\ee
we may write:
\be
	\inprod{ \nabla \optmes ( \Kbf^{(1)} ) }{ \nabla \cost( \Kbf^{(1)} ; \trainstates ) } = \frac{ 4 }{ \abs{ \trainstates } } \sum_{\xbf_0 \in \trainstates} \sum_{n =  0}^{H - 1} \sum_{k = 0}^{H - n - 1} Y_{\xbf_0, n, k} = \frac{ 4 }{ \abs{ \trainstates } } \sum_{\xbf_0 \in \trainstates} \sum_{n =  1}^{H - 1} \sum_{k = 0}^{H - n - 1} Y_{\xbf_0, n, k}
	\text{\,,}
	\label{eq:in_prod_Y}
\ee
where the last transition is by noticing that $Y_{\xbf_0, n, k} = 0$ for $n = 0$ since $\inprod{\vbf_0}{\Abf^n \xbf_0} = \inprod{\vbf_0}{\xbf_0} = 0$.
\cref{lem:expec_in_prod_term_lower_bound} in \cref{app:proofs:typical_system:in_prod_lower_bound} has shown that $\EE \brk[s] { Z_{\vbf_0, \xbf_0, n, k} } \geq 1 / D$, and so $\EE \brk[s]{ Y_{\xbf_0, n, k} } \geq 1 / D$ as well, for all $\xbf_0 \in \trainstates, \vbf_0 \in \orthstates, n \in [H - 1],$ and $k \in \{0\} \cup [H - n - 1]$.

Now, fix some $\xbf_0 \in \trainstates, n \in [H - 1],$ and $k \in \{0\} \cup [H - n - 1]$.
For upper bounding $\Var (Z_{\vbf_0, \xbf_0, n, k})$ and $\Cov ( Z_{\vbf_0, \xbf_0, n, k}, Z_{\vbf_0', \xbf_0, n, k} )$, for $\vbf_0, \vbf_0' \in \orthstates$, we employ a method from~\citet{redelmeier2014real}, mentioned in the proof outline (\cref{app:proofs:typical_system:outline}) and introduced in \cref{genus_expansion}, which facilitates computing expected traces of random matrix products through the topological concept of genus expansion.
For ease of exposition, we defer these bounds, with which we upper bound $\Var (Y_{\xbf_0, n, k})$, to Appendices~\ref{high_prob_bound:var_bound} and \ref{high_prob_bound:cov_bound} below.
Specifically, \cref{var_prop,cov_prop} therein show that:
\[
\Var \brk*{ Z_{\vbf_0, \xbf_0, n, k} } \leq \frac{ (4H - 1)!! }{ D^2} ~~,~~ \Cov \brk*{  Z_{\vbf_0, \xbf_0, n, k} ,  Z_{\vbf'_0, \xbf_0, n, k} } \leq \frac{ (4H - 1)!! }{ D^3}
\text{\,,}
\]
for all $\vbf_0 \neq \vbf'_0 \in \orthstates$, where $N!! := N (N - 2) (N - 4) \cdots 3$ is the double factorial of an odd $N \in \N$.
The above imply:
\[
\begin{split}
\Var (Y_{\xbf_0, n, k}) & = \frac{ 1 }{ \abs{\orthstates}^2 } \brk*{ \sum\nolimits_{\vbf_0 \in \orthstates} \Var \brk*{ Z_{\vbf_0, \xbf_0, n, k} } + \sum\nolimits_{\vbf_0 \neq \vbf'_0 \in \orthstates} \Cov \brk*{ Z_{\vbf_0, \xbf_0, n, k} , Z_{ \vbf'_0, \xbf_0, n, k } } } \\
& \leq \frac{ 1 }{ \abs{\orthstates}^2 } \brk*{ \frac{ \abs{\orthstates} (4H - 1)!! }{ D^2 }  + \frac{ \abs{\orthstates}^2 (4H - 1)!! }{ D^3 } } \\
& \leq \frac{ 2 (4H - 1)!! }{ \abs{ \orthstates } D^2 }
\text{\,.}
\end{split}
\]
Thus, since $\EE \brk[s]{ Y_{\xbf_0, n, k} } \geq 1 / D$, Chebyshev's inequality gives:
\[
\Pr \brk*{ Y_{\xbf_0, n, k} \leq  \frac{1}{2D} } \leq \Pr \brk*{ \abs1{  Y_{\xbf_0, n, k} - \EE \brk[s]{ Y_{\xbf_0, n, k} } } \geq \frac{1}{2D} } \leq \frac{ 8 (4H - 1)!! }{ \abs{\orthstates} }
\text{\,.}
\]
Applying a union bound over all $\abs{\trainstates} H (H - 1) / 2$ possible options for $\xbf_0 \in \trainstates, n \in [H - 1], k \in \{ 0\} \cup [H - n - 1]$ we arrive at:
\[
\begin{split}
\Pr \brk*{ \exists \xbf_0 \in \trainstates, n \in [H - 1], k \in \{0\} \cup [H - n - 1] : ~ Y_{\xbf_0, n, k} \leq  \frac{1}{2D} } \leq \frac{ 4 \abs{ \trainstates} H (H - 1) (4H - 1)!! }{ \abs{\orthstates} }
\text{\,.}
\end{split}
\]
Since $\abs{\orthstates} = D - \abs{\trainstates}$, combined with \cref{eq:in_prod_Y} the above implies that with probability at least $1 - \frac{ 4 \abs{\trainstates} H (H - 1) (4H - 1)!! }{ D - \abs{\trainstates} }$:
\be
	\inprod{ \nabla \optmes ( \Kbf^{(1)} ) }{ \nabla \cost( \Kbf^{(1)} ; \trainstates ) } = \frac{ 4 }{ \abs{ \trainstates } } \sum_{\xbf_0 \in \trainstates} \sum_{n =  1}^{H - 1} \sum_{k = 0}^{H - n - 1} Y_{\xbf_0, n, k} \geq \frac{4}{ \abs{ \trainstates} } \cdot \frac{ \abs{ \trainstates} H (H - 1) }{ 2 } \cdot \frac{1}{2D} = \frac{H (H - 1) }{ D }
\text{\,.}
\label{eq:high_prob_in_prod_lower_bound}
\ee

\medskip

With the lower bound on $\inprod{ \nabla \optmes ( \Kbf^{(1)} ) }{ \nabla \cost( \Kbf^{(1)} ; \trainstates ) }$ in place, we turn our attention to establishing that, with high probability, a policy gradient iteration reduces the optimality extrapolation measure.
By \cref{smoothness_lemma}, $\optmes (\cdot)$ is $2$-smooth.
Thus:
\[
\begin{split}
	\optmes (\Kbf^{(2)}) & \leq \optmes (\Kbf^{(1)}) + \inprodbig{ \nabla \optmes (\Kbf^{(1)}) }{ \Kbf^{(2)} - \Kbf^{(1)} } + \norm{ \Kbf^{(2)} - \Kbf^{(1)}}^2 \\
	& = \optmes (\Kbf^{(1)}) - \eta \cdot \inprodbig{ \nabla \optmes (\Kbf^{(1)} ) }{ \nabla \cost ( \Kbf^{(1)} ; \trainstates) } + \eta^2 \cdot \norm{ \nabla \cost (\Kbf^{(1)} ; \trainstates ) }^2
	\text{\,.}
\end{split}
\]
As can be seen from the equation above, aside from the lower bound on $\inprod{ \nabla \optmes ( \Kbf^{(1)} ) }{ \nabla \cost( \Kbf^{(1)} ; \trainstates ) }$, to show that the optimality measure decreases it is necessary to upper bound $\norm{ \nabla \cost (\Kbf^{(1)} ; \trainstates ) }^2$.
To do so, we can use \cref{expected_grad_norm} and Markov's inequality:
\[
\Pr \brk*{ \norm{ \nabla \cost (\Kbf^{(1)} ; \trainstates ) }^2 \leq 4 D H^2 (H - 1)^2 (4H - 1)!! } \geq 1 - \frac{1}{D}
\text{\,.}
\]
Together with \cref{eq:high_prob_in_prod_lower_bound}, we have that with probability at least $1 - \frac{ 4 \abs{\trainstates} H (H - 1) (4H - 1)!! }{ D - \abs{\trainstates} } - \frac{1}{D}$:
\[
\optmes (\Kbf^{(2)}) \leq \optmes (\Kbf^{(1)}) - \eta \cdot \frac{ H (H - 1) }{D} + \eta^2 \cdot 4 D H^2 (H - 1)^2 (4H - 1)!!
\text{\,.}
\]
Since by assumption $\eta \leq \frac{ 1 }{ 8 D^{2} H (H-1)  (4H - 1)!! }$:
\[
\eta^2 \cdot 4 D H^2 (H - 1)^2 (4H - 1)!! \leq \frac{1}{2} \eta \cdot \frac{H (H - 1)}{D}
\text{\,,}
\]
from which it follows that, with probability at least $1 - \frac{ 4 \abs{\trainstates} H (H - 1) (4H - 1)!! }{ D - \abs{\trainstates} } - \frac{1}{D}$:
\be
\optmes (\Kbf^{(2)}) \leq \optmes (\Kbf^{(1)}) - \eta \cdot \frac{ H (H - 1) }{ 2D }
\text{\,.}
\label{eq:opt_upper_bound_K1_eta}
\ee
Now, recall that $\Kbf^{(1)} = \0$, and so $\optmes (\Kbf^{(1)}) = \optmes (\Kminnorm) = \frac{1}{ \abs{\orthstates} } \sum\nolimits_{ \vbf_0 \in \orthstates } \norm{\Abf \vbf_0 }^2$.
We claim that with high probability $\optmes (\Kminnorm) \leq 2$.
Indeed, since $\orthstates$ is an orthonormal set of vectors, $\{ \Abf \vbf_0  : \vbf_0 \in \orthstates\}$ is a set of independent random variables.
Furthermore, the entries of $\Abf\vbf_0$, for $\vbf_0 \in \orthstates$, are distributed independently according to a Gaussian distribution with mean zero and standard deviation $1 / \sqrt{D}$.
Hence, \cref{lem:norm_concentration} implies that with probability at least $1 - \frac{ 2}{ D ( D - \abs{\trainstates} ) }$:
\[
 \optmes (\Kbf^{(1)}) = \optmes (\Kminnorm) = \frac{1}{ \abs{\orthstates} } \sum\nolimits_{ \vbf_0 \in \orthstates } \norm{\Abf \vbf_0 }^2 \leq 2
\text{\,.}
\]
Dividing both sides of \cref{eq:opt_upper_bound_K1_eta} by $ \optmes ( \Kminnorm )$, and applying a union bound, we get that with probability at least $1 - \frac{ 4 \abs{\trainstates} H (H - 1) (4H - 1)!! }{ D - \abs{\trainstates} } - \frac{1}{D} - \frac{ 2 }{ D (D - \abs{ \trainstates} )}$:
\[
\frac{ \optmes (\Kbf^{(2)}) }{ \optmes ( \Kminnorm) } \leq 1- \eta \cdot \frac{ H (H - 1) }{ 4D }
\text{\,.}
\]
Finally, notice that:
\[
\begin{split}
\frac{1}{D}  \leq \frac{ 1 }{D - \abs{\trainstates}} \leq \frac{ \abs{\trainstates} H (H - 1) (4H - 1)!! }{ D - \abs{\trainstates} } 
\quad , \quad 
\frac{ 2 }{ D ( D - \abs{\trainstates} ) } \leq \frac{  \abs{\trainstates} H (H - 1) (4H - 1)!! }{ D - \abs{\trainstates} }
\text{\,,}
\end{split}
\]
and therefore the upper bound above holds with probability at least $1 - \frac{ 6 \abs{\trainstates} H (H - 1) (4H - 1)!! }{ D - \abs{\trainstates} }$.
Restating it in terms of a fixed failure probability $\delta \in (0, 1)$, we conclude that if $D \geq \abs{\trainstates} + \frac{ 6 \abs{ \trainstates} H (H - 1) (4H - 1)!! }{ \delta }$, then with probability of at least $1 - \delta$:
\[
\frac{ \optmes (\Kbf^{(2)}) }{ \optmes ( \Kminnorm) } \leq 1- \eta \cdot \frac{ H (H - 1) }{ 4D }
\text{\,.}
\]
\qed

\paragraph{Upper Bound on  $\Var \brk{ Z_{\vbf_0, \xbf_0,n,k} }$}
\label{high_prob_bound:var_bound}
\

\begin{proposition}
\label{var_prop}
For any $\xbf_0 \in \trainstates, \vbf_0 \in \orthstates, n \in \cup [H - 1], k \in \{0\} \cup [H - n - 1]$:
\[
	\Var \brk*{ Z_{\vbf_0, \xbf_0, n, k} } \leq \frac{ (4H - 1)!! }{ D^2 }
	\text{\,,}
\]
where $Z_{\vbf_0, \xbf_0, n, k}$ is as defined in \cref{eq:Z_Y_def}.
\end{proposition}
\begin{proof} 
Since $\Var \brk*{ Z_{\vbf_0, \xbf_0, n, k} } = \EE \brk[s]1{Z_{\vbf_0, \xbf_0, n, k}^2 } - \EE \brk[s]1{ Z_{\vbf_0, \xbf_0, n, k}}^2 \leq \EE \brk[s]1{ Z_{\vbf_0, \xbf_0, n, k}^2 }$, it suffices to upper bound the second moment $\EE \brk[s]1{ Z_{\vbf_0, \xbf_0, n, k}^2 }$, which upholds:
\[
\EE \brk[s]*{ Z_{\vbf_0, \xbf_0, n, k}^2 } = \EE\nolimits_{\Abf} \brk[s]*{ \Tr \brk*{ (\Abf^n)^\top \vbf_0 \xbf_0^\top  }^2 \cdot \Tr  \brk*{ (\Abf^{n + k + 1})^{\top} \Abf^{k + 1} \vbf_0 \xbf_0^\top }^2  }
\text{\,.}
\]
Let  $p := 2 (n + k + 1) \leq 2H$.
To show that $\EE \brk[s]1{ Z_{\vbf_0, \xbf_0, n, k}^2 }  \leq \frac{ (p - 1)!! }{ D^2 }$ we employ the method from \citet{redelmeier2014real}, which is based on the topological concept of genus expansion.
For completeness, \cref{genus_expansion} provides a self-contained introduction to the method, and \cref{thm:genus_expansion_main_theorem} therein lays out the result which we will use.
We assume below familiarity with the notation and concepts detailed in \cref{genus_expansion}.

For invoking \cref{thm:genus_expansion_main_theorem}, let us define a permutation $\gamma$ over $[p]$ via the cycle decomposition:
\[
\gamma = (1, \ldots, n)(n + 1, \ldots, p)(p +1, \ldots, p + n)(p + n + 1, \ldots, 2p)
\text{\,,}
\]
and a mapping $\epsilon: [2p] \to \{ -1 , 1\}$ by:
\[
\epsilon(1) = -1 , \ldots ,\epsilon(2n + k + 1) = -1 \text{\,,}
\]
\[
\epsilon (2n + k + 2) = 1, \ldots, \epsilon (p) = 1 \text{\,,}
\]
\[
\epsilon (p + 1) = -1, \ldots, \epsilon (p + 2n + k + 1) = -1 \text{\,,}
\]
\[
\epsilon (p + 2n + k + 2) = 1, \ldots , \epsilon (2p) = 1
\text{\,.}
\]
Additionally, define $\Cbf_1, \ldots, \Cbf_{2p} \in \R^{D \times D}$ as follows:
\[
	\Cbf_1 = \Ibf, \ldots, \Cbf_{n - 1} = \Ibf, \Cbf_n = \vbf_0 \xbf_0^\top \text{\,,}
\]
\[
	\Cbf_{n + 1} = \Ibf, \ldots, \Cbf_{p - 1} = \Ibf, \Cbf_p = \vbf_0 \xbf_0^\top \text{\,,}
\]
\[
\Cbf_{p+ 1} = \Ibf, \ldots, \Cbf_{p + n - 1} = \Ibf, \Cbf_{p + n} = \vbf_0 \xbf_0^\top \text{\,,}
\]
\[
\Cbf_{p + n + 1} = \Ibf, \ldots, \Cbf_{2p - 1} = \Ibf, \Cbf_{2p} = \vbf_0 \xbf_0^\top \text{\,,}
\]
where $\Ibf$ is the identity matrix.
For the above choice of $\gamma, \epsilon,$ and matrices $\Cbf_1, \ldots, \Cbf_{2p}$ it holds that:
\[
\EE\nolimits_{\Abf} \brk[s]*{ \Tr_\gamma \brk*{ \Abf_{\epsilon (1)} \Cbf_1, \ldots, \Abf_{\epsilon (2p)} \Cbf_{2p} } } = \EE\nolimits_{\Abf} \brk[s]*{ \Tr \brk*{ (\Abf^n)^\top \vbf_0 \xbf_0^\top  }^2 \cdot \Tr  \brk*{ (\Abf^{n + k + 1})^{\top} \Abf^{k + 1} \vbf_0 \xbf_0^\top }^2  } = \EE \brk[s]*{ Z_{\vbf_0, \xbf_0, n, k}^2 } 
\text{\,.}
\]
Invoking \cref{thm:genus_expansion_main_theorem}, we may write \cref{eq:genus_formula} (from \cref{thm:genus_expansion_main_theorem} of \cref{genus_expansion}) as:
\[
 \EE \brk[s]*{ Z_{\vbf_0, \xbf_0, n, k}^2 }  = \EE\nolimits_{\Abf} \brk[s]*{ \Tr_\gamma \brk*{ \Abf_{\epsilon (1)} \Cbf_1, \ldots, \Abf_{\epsilon (2p)} \Cbf_{2p} } } = \sum_{ \pi \in \{ \rho \delta \rho : \rho \in \M_{2p} \} } D^{ \chi (\gamma, \delta_\epsilon \pi \delta_\epsilon) - \abs{\gamma} } \cdot \Trnormalized_{ \frac{ \gamma_-^{-1} \delta_\epsilon \pi \delta_\epsilon \gamma_+ }{ 2 } } \brk*{ \Cbf_1, \ldots, \Cbf_{2p} }
\text{\,.}
\]
Notice that, due to the choice of $\Cbf_1, \ldots, \Cbf_{2p}$, each summand on the right hand side is non-negative.
We claim that for a non-zero summand corresponding to $\pi$ it necessarily holds that $\chi ( \gamma, \delta_\epsilon \pi \delta_\epsilon ) \leq 4$.
Meaning:
\[
 \EE \brk[s]*{ Z_{\vbf_0, \xbf_0, n, k}^2 } = \sum_{ \pi \in \{ \rho \delta \rho : \rho \in \M_{2p} \} , \chi (\gamma, \delta_\epsilon \pi \delta_\epsilon) \leq 4} D^{ \chi (\gamma, \delta_\epsilon \pi \delta_\epsilon) - \abs{\gamma} } \cdot \Trnormalized_{ \frac{ \gamma_-^{-1} \delta_\epsilon \pi \delta_\epsilon \gamma_+ }{ 2 } } \brk*{ \Cbf_1, \ldots, \Cbf_{2p} }
\text{\,.}
\]
To see why this is the case, note that $\abs{\gamma} = 4$.
It follows that the corresponding surfaces described in \cref{genus_expansion} are obtained by gluing four faces.
Furthermore, because:
\[
\Trnormalized \brk*{ \vbf_0 \xbf_0^\top } = \Trnormalized \brk*{ \xbf_0 \vbf_0^\top } = \Trnormalized \brk*{ \brk*{ \vbf_0 \xbf_0^\top }^2 } = \Trnormalized \brk*{ \brk*{ \xbf_0 \vbf_0^\top }^2 } = 0
\text{\,,}
\]
\[
\Trnormalized \brk*{ \xbf_0 \vbf_0^\top \vbf_0 \xbf_0^\top } = \Trnormalized \brk*{ \vbf_0 \xbf_0^\top \xbf_0 \vbf_0^\top } = \frac{1}{D}
\text{\,,}
\]
the only way a summand corresponding to $\pi = \rho \delta \rho$ can be non-zero is if a cycle $\RR = (1, \ldots, R)$ of $\gamma_{+}^{-1} \delta_\epsilon \pi \delta_\epsilon \gamma_-  / 2$  contains either no non-identity matrices, in which case $\Trnormalized (\Cbf_1, \ldots, \Cbf_R) = 1$, or if it contains two or four such matrices, with non-identity matrices appearing once transposed and once without transposition, in which case $\Trnormalized ( \Cbf_1, \ldots, \Cbf_R) = 1 / D$.
Thus, the four non-identity matrices among $\Cbf_1, \ldots, \Cbf_{2p}$ must appear in either one or two different cycles.
Accordingly, to get a non-zero contribution, $\delta_{\epsilon} \pi \delta_{\epsilon}$ must either connect all four faces or connect two pairs among them, \ie~the surface $\G (\gamma, \epsilon, \rho)$ must have either one or two connected components (see construction in \cref{genus_expansion}).
By \cref{top_theorem} in \cref{genus_expansion}, the Euler characteristic of such a surface satisfies $\chi (\G (\gamma , \epsilon, \rho)) \leq 4$.

Overall, we have established that:
\[
 \EE \brk[s]*{ Z_{\vbf_0, \xbf_0, n, k}^2 } = \sum_{ \pi \in \{ \rho \delta \rho : \rho \in \M_{2p} \} , \chi (\gamma, \delta_\epsilon \pi \delta_\epsilon) \leq 4} D^{ \chi (\gamma, \delta_\epsilon \pi \delta_\epsilon) - \abs{\gamma} } \cdot \Trnormalized_{ \frac{ \gamma_-^{-1} \delta_\epsilon \pi \delta_\epsilon \gamma_+ }{ 2 } } \brk*{ \Cbf_1, \ldots, \Cbf_{2p} }
 \text{\,.}
\]
To conclude the proof, we show that each summand on the right hand side contributes at most $1 / D^2$.
Let us examine all possible cases for $\pi = \rho \delta \rho$.
If all non-identity matrices are in a single cycle of $\gamma_{+}^{-1} \delta_\epsilon \pi \delta_\epsilon \gamma_-  / 2$, then:
\[
\Trnormalized_{ \frac{ \gamma_-^{-1} \delta_\epsilon \pi \delta_\epsilon \gamma_+ }{ 2 } } \brk*{ \Cbf_1, \ldots, \Cbf_{2p} } \leq \frac{1}{D}
\text{\,,}
\]
and the surface $\G (\gamma, \epsilon, \rho)$ is connected, so $\chi ( \G (\gamma, \epsilon, \rho)) \leq 2$ and the summand corresponding to $\pi$ is at most $1 / D^3$.
On the other hand, if there are two cycles containing non-identity matrices, then:
\[
\Trnormalized_{ \frac{ \gamma_-^{-1} \delta_\epsilon \pi \delta_\epsilon \gamma_+ }{ 2 } } \brk*{ \Cbf_1, \ldots, \Cbf_{2p} } \leq \frac{1}{D^2}
\text{\,,}
\]
and $\chi ( \G (\gamma, \epsilon, \rho)) \leq 4$, so the summand corresponding to $\pi$ is at most $1 / D^2$.
As we showed above, these are the only cases which give a non-zero contribution.
Hence:
\[
\begin{split}
 \EE \brk[s]*{ Z_{\vbf_0, \xbf_0, n, k}^2 } & = \sum_{ \pi \in \{ \rho \delta \rho : \rho \in \M_{2p} \} , \chi (\gamma, \delta_\epsilon \pi \delta_\epsilon) \leq 4} D^{ \chi (\gamma, \delta_\epsilon \pi \delta_\epsilon) - \abs{\gamma} } \cdot \Trnormalized_{ \frac{ \gamma_-^{-1} \delta_\epsilon \pi \delta_\epsilon \gamma_+ }{ 2 } } \brk*{ \Cbf_1, \ldots, \Cbf_{2p} } \\
 & \leq \sum_{ \pi \in \{ \rho \delta \rho : \rho \in \M_{2p} \} , \chi (\gamma, \delta_\epsilon \pi \delta_\epsilon) \leq 4}  \frac{1}{D^2} \\
 & \leq \frac{ (2p - 1)!! }{ D^2}
\text{\,,}
\end{split}
\]
where the last transition is by the number of pairings of $[2p]$ being equal to $\abs{\M_{2p}} = (2p - 1)!!$.
The proof concludes by noticing that $2p = 4( n + k + 1) \leq 4H$.
\end{proof}

\paragraph{Upper Bound on $\Cov \brk{ Z_{\vbf_0,\xbf_0,n,k}, Z_{\vbf_0',\xbf_0,n,k} }$}
\label{high_prob_bound:cov_bound}
\

\begin{proposition}
\label{cov_prop}
For any $\xbf_0 \in \trainstates, \vbf_0, \vbf'_0 \in \orthstates, n \in \cup [H - 1], k \in \{0\} \cup [H - n - 1]$ with $\vbf_0 \neq \vbf'_0$:
\[
\Cov \brk*{ Z_{\vbf_0, \xbf_0, n, k}, Z_{\vbf'_0, \xbf_0, n, k} } \leq \frac{ (4H - 1)!! }{ D^3 }
\text{\,,}
\]
where $Z_{\vbf_0, \xbf_0, n, k}$ is as defined in \cref{eq:Z_Y_def}.
\end{proposition}

\begin{proof}
Note that $\Cov \brk1{ Z_{\vbf_0, \xbf_0, n, k}, Z_{\vbf'_0, \xbf_0, n, k} } = \EE \brk[s]1{ Z_{\vbf_0, \xbf_0, n, k} Z_{\vbf'_0, \xbf_0, n, k}  } - \EE \brk[s]1{ Z_{\vbf_0, \xbf_0, n, k}} \EE \brk[s]1{ Z_{\vbf'_0, \xbf_0, n, k}  }$, where:
\[
\begin{split}
  & \EE \brk[s]1{ Z_{\vbf_0, \xbf_0, n, k}, Z_{\vbf'_0, \xbf_0, n, k}  } \\
  & \hspace{3mm} = \EE\nolimits_{\Abf} \brk[s]*{ \Tr \brk*{ (\Abf^n)^\top \vbf_0 \xbf_0^\top  } \Tr  \brk*{ (\Abf^{n + k + 1})^{\top} \Abf^{k + 1} \vbf_0 \xbf_0^\top } \Tr \brk*{ (\Abf^n)^\top \vbf'_0 \xbf_0^\top  } \Tr  \brk*{ (\Abf^{n + k + 1})^{\top} \Abf^{k + 1} \vbf'_0 \xbf_0^\top }  }  \text{\,,} \\[0.4em]
 & \EE \brk[s]1{ Z_{\vbf_0, \xbf_0, n, k} } \EE \brk[s]1{ Z_{\vbf'_0, \xbf_0, n, k}  } \\
 & \hspace{3mm} = \EE\nolimits_{\Abf} \brk[s]*{ \Tr \brk*{ (\Abf^n)^\top \vbf_0 \xbf_0^\top  } \Tr  \brk*{ (\Abf^{n + k + 1})^{\top} \Abf^{k + 1} \vbf_0 \xbf_0^\top } } \EE\nolimits_{\Abf} \brk[s]*{ \Tr \brk*{ (\Abf^n)^\top \vbf'_0 \xbf_0^\top  } \Tr  \brk*{ (\Abf^{n + k + 1})^{\top} \Abf^{k + 1} \vbf'_0 \xbf_0^\top }  } \text{\,.}
\end{split}
\]
Let  $p := 2 (n + k + 1) \leq 2H$.
We will show that both $\EE \brk[s]1{ Z_{\vbf_0, \xbf_0, n, k} Z_{\vbf'_0, \xbf_0, n, k}  }$ and $\EE \brk[s]1{ Z_{\vbf_0, \xbf_0, n, k}} \EE \brk[s]1{ Z_{\vbf'_0, \xbf_0, n, k}  }$ can be written as a sum, in which each summand is at most $1 / D^{2}$ and the coefficient corresponding to $1 / D^{2}$ is the same.
As a result, this will lead to an upper bound on the covariance that depends on $1 / D^{3}$.

We first examine $\EE \brk[s]1{ Z_{\vbf_0, \xbf_0, n, k} } $.
The analysis below applies equally to $\EE \brk[s]1{ Z_{\vbf'_0, \xbf_0, n, k} } $ as well (note that by \cref{lem:in_prod_initial_states} of \cref{app:proofs:typical_system:in_prod_lower_bound} we know that $\EE \brk[s]1{ Z_{\vbf_0, \xbf_0, n, k} } = \EE \brk[s]1{ Z_{\vbf'_0, \xbf_0, n, k} }$).
Below, we make use of the method from \citet{redelmeier2014real}, which is based on the topological concept of genus expansion.
For completeness, \cref{genus_expansion} provides a self-contained introduction to the method, and \cref{thm:genus_expansion_main_theorem} therein lays out the result which we will use.
We assume familiarity with the notation and concepts detailed in \cref{genus_expansion}.

For invoking \cref{thm:genus_expansion_main_theorem}, let us define a permutation $\gamma$ over $[p]$ via the cycle decomposition:
\[
\gamma = (1, \ldots, m + n)(m + n + 1, \ldots, p)
\text{\,,}
\]
and a mapping $\epsilon: [2p] \to \{ -1 , 1\}$ by:
\[
\epsilon(1) = -1 , \ldots ,\epsilon(2n + k + 1) = -1 \text{\,,}
\]
\[
\epsilon (2n + k + 2) = 1, \ldots, \epsilon (p) = 1 \text{\,.}
\]
Additionally, define $\Cbf_1, \ldots, \Cbf_{p} \in \R^{D \times D}$ as follows:
\[
\Cbf_1 = \Ibf, \ldots, \Cbf_{n - 1} = \Ibf, \Cbf_n = \vbf_0 \xbf_0^\top \text{\,,}
\]
\[
\Cbf_{n + 1} = \Ibf, \ldots, \Cbf_{p - 1} = \Ibf, \Cbf_p = \vbf_0 \xbf_0^\top \text{\,,}
\]
where $\Ibf$ is the identity matrix.
For the above choice of $\gamma, \epsilon,$ and matrices $\Cbf_1, \ldots, \Cbf_{p}$ it holds that:
\[
\EE\nolimits_{\Abf} \brk[s]*{ \Tr_\gamma \brk*{ \Abf_{\epsilon (1)} \Cbf_1, \ldots, \Abf_{\epsilon (p)} \Cbf_{p} } } =  \EE\nolimits_{\Abf} \brk[s]*{ \Tr \brk*{ (\Abf^n)^\top \vbf_0 \xbf_0^\top  } \Tr  \brk*{ (\Abf^{n + k + 1})^{\top} \Abf^{k + 1} \vbf_0 \xbf_0^\top } } = \EE \brk[s]*{ Z_{\vbf_0, \xbf_0, n, k} } 
\text{\,.}
\]
Invoking \cref{thm:genus_expansion_main_theorem}, we may write \cref{eq:genus_formula} (from \cref{thm:genus_expansion_main_theorem} of \cref{genus_expansion}) as:
\be
\EE \brk[s]*{ Z_{\vbf_0, \xbf_0, n, k} }  = \EE\nolimits_{\Abf} \brk[s]*{ \Tr_\gamma \brk*{ \Abf_{\epsilon (1)} \Cbf_1, \ldots, \Abf_{\epsilon (p)} \Cbf_{p} } } = \sum_{ \pi \in \{ \rho \delta \rho : \rho \in \M_{p} \} } D^{ \chi (\gamma, \delta_\epsilon \pi \delta_\epsilon) - \abs{\gamma} } \cdot \Trnormalized_{ \frac{ \gamma_-^{-1} \delta_\epsilon \pi \delta_\epsilon \gamma_+ }{ 2 } } \brk*{ \Cbf_1, \ldots, \Cbf_{p} }
\text{\,.}
\label{eq:genus_expectation_proof_formula_cov_bound}
\ee
We claim that, for any $\pi$ corresponding to a summand on the right hand side of the equation above either
\[
\Trnormalized_{ \frac{ \gamma_-^{-1} \delta_\epsilon \pi \delta_\epsilon \gamma_+ }{ 2 } } \brk*{ \Cbf_1, \ldots, \Cbf_{p} } = 0
\]
or 
\[
\Trnormalized_{ \frac{ \gamma_-^{-1} \delta_\epsilon \pi \delta_\epsilon \gamma_+ }{ 2 } } \brk*{ \Cbf_1, \ldots, \Cbf_{p} } = \frac{1}{D}
\text{\,.}
\]
To see it is so, notice that if $\gamma_{+}^{-1} \delta_\epsilon \pi \delta_\epsilon \gamma_-  / 2$ comprises a cycle containing the two non-identity matrices, appearing once transposed and once without transposition, then the normalized trace for that cycle is equal to $1 / D$ and the normalized trace for the remaining cycle is $1$.
Otherwise, one of the normalized traces for a cycle of $\gamma_{+}^{-1} \delta_\epsilon \pi \delta_\epsilon \gamma_-  / 2$ is equal to zero.
Hence, for each summand on the right hand side of \cref{eq:genus_expectation_proof_formula_cov_bound} corresponding to $\pi = \rho \delta \rho$, whose contribution is non-zero, the two faces of $\G(\gamma, \epsilon, \rho)$ (see construction in \cref{genus_expansion}) induced by $\gamma$ are connected.
By \cref{top_theorem}, for such $\pi$ we get $\chi (\gamma, \delta_\epsilon \pi \delta_\epsilon) - \abs{\gamma} = \chi (\gamma, \delta_\epsilon \pi \delta_\epsilon) - 2 \leq D^0 = 1$.

Overall, the above implies that each non-zero summand in the expression for $\EE \brk[s]{ Z_{\vbf_0, \xbf_0, n, k} }$, out of the $\abs{ \M_p } = (p - 1)!!$ summands, is upper bounded by $1 / D$.
Since, as mentioned above, the same holds for $\EE \brk[s]*{ Z_{\vbf'_0, \xbf_0, n, k} }$, we get that $\EE \brk[s]*{ Z_{\vbf_0, \xbf_0, n, k} } \EE \brk[s]*{ Z_{\vbf'_0, \xbf_0, n, k} }$ can be represented as a sum in which each term is upper bounded by $1 / D^3$ or is equal to $1 / D^2$.
What remains is to show that $\EE \brk[s]{ Z_{\vbf_0, \xbf_0, n, k} Z_{\vbf'_0, \xbf_0, n, k} }$ can be written as a sum of $(2p - 1)!!$ terms, each upper bounded by $1 / D^3$ or equal to $1 / D^2$.
Fortunately, we will see that terms equal to $1 / D^2$ cancel out with those of $\EE \brk[s]{ Z_{\vbf_0, \xbf_0, n, k} } \EE \brk[s]{ Z_{\vbf'_0, \xbf_0, n, k} }$, and as a result $\Cov \brk{ Z_{\vbf_0, \xbf_0, n, k}, Z_{\vbf'_0, \xbf_0, n, k}}$ is upper bounded by $(2p - 1)!! / D^3$.

For $i \in \mathbb{Z}$, let us denote by $c_i \brk1{ \EE \brk[s]{ Z_{\vbf_0, \xbf_0, n, k} Z_{\vbf'_0, \xbf_0, n, k} } }$ and $c_i \brk1{ \EE \brk[s]{ Z_{\vbf_0, \xbf_0, n, k} } \EE \brk[s]{ Z_{\vbf'_0, \xbf_0, n, k} } }$ the coefficients of $D^i$ in the respective expressions.
According to the discussion above, we need only show that:
\[
c_{-2} \brk1{ \EE \brk[s]{ Z_{\vbf_0, \xbf_0, n, k} Z_{\vbf'_0, \xbf_0, n, k} } } = c_{-2} \brk1{ \EE \brk[s]{ Z_{\vbf_0, \xbf_0, n, k} } \EE \brk[s]{ Z_{\vbf'_0, \xbf_0, n, k} } }
\text{\,,}
\]
and that for all $i \geq -1$:
\[
c_{i} \brk1{ \EE \brk[s]{ Z_{\vbf_0, \xbf_0, n, k} Z_{\vbf'_0, \xbf_0, n, k} } }  = 0	
\text{\,.}
\]
We apply again the method from \citet{redelmeier2014real}.
In particular, we invoke \cref{thm:genus_expansion_main_theorem} by defining the permutation $\gamma$ over $[2p]$ via the cycle decomposition:
\[
\gamma = (1, \ldots, n)(n + 1, \ldots, p)(p +1, \ldots, p + n)(p + n + 1, \ldots, 2p)
\text{\,,}
\]
and a mapping $\epsilon: [2p] \to \{ -1 , 1\}$ by:
\[
\epsilon(1) = -1 , \ldots ,\epsilon(2n + k + 1) = -1 \text{\,,}
\]
\[
\epsilon (2n + k + 2) = 1, \ldots, \epsilon (p) = 1 \text{\,,}
\]
\[
\epsilon (p + 1) = -1, \ldots, \epsilon (p + 2n + k + 1) = -1 \text{\,,}
\]
\[
\epsilon (p + 2n + k + 2) = 1, \ldots , \epsilon (2p) = 1
\text{\,.}
\]
Furthermore, define $\Cbf_1, \ldots, \Cbf_{2p} \in \R^{D \times D}$ as follows:
\[
\Cbf_1 = \Ibf, \ldots, \Cbf_{n - 1} = \Ibf, \Cbf_n = \vbf_0 \xbf_0^\top \text{\,,}
\]
\[
\Cbf_{n + 1} = \Ibf, \ldots, \Cbf_{p - 1} = \Ibf, \Cbf_p = \vbf_0 \xbf_0^\top \text{\,,}
\]
\[
\Cbf_{p+ 1} = \Ibf, \ldots, \Cbf_{p + n - 1} = \Ibf, \Cbf_{p + n} = \vbf'_0 \xbf_0^\top \text{\,,}
\]
\[
\Cbf_{p + n + 1} = \Ibf, \ldots, \Cbf_{2p - 1} = \Ibf, \Cbf_{2p} = \vbf'_0 \xbf_0^\top \text{\,,}
\]
where $\Ibf$ is the identity matrix.
For the above choice of $\gamma, \epsilon,$ and matrices $\Cbf_1, \ldots, \Cbf_{2p}$ it holds that:
\[
\EE\nolimits_{\Abf} \brk[s]*{ \Tr_\gamma \brk*{ \Abf_{\epsilon (1)} \Cbf_1, \ldots, \Abf_{\epsilon (2p)} \Cbf_{2p} } } = \EE \brk[s]*{ Z_{\vbf_0, \xbf_0, n, k} Z_{\vbf'_0, \xbf_0, n, k} } 
\text{\,.}
\]
Thus, invoking \cref{thm:genus_expansion_main_theorem} leads to:
\be
\EE \brk[s]*{ Z_{\vbf_0, \xbf_0, n, k} Z_{\vbf'_0, \xbf_0, n, k} } = \sum_{ \pi \in \{ \rho \delta \rho : \rho \in \M_{2p} \} } D^{ \chi (\gamma, \delta_\epsilon \pi \delta_\epsilon) - \abs{\gamma} } \cdot \Trnormalized_{ \frac{ \gamma_-^{-1} \delta_\epsilon \pi \delta_\epsilon \gamma_+ }{ 2 } } \brk*{ \Cbf_1, \ldots, \Cbf_{2p} }
\text{\,.}
\label{eq:genus_expectation_proof_formula_cov_bound_second}
\ee
Notice that, due to the choice of $\Cbf_1, \ldots, \Cbf_{2p}$, each summand on the right hand side is non-negative.
Specifically, the normalized traces for different combination of the non-identity matrices among $\Cbf_1, \ldots, \Cbf_{2p}$ satisfy:
\[
\Trnormalized \brk*{ \xbf_0 \vbf_0^\top } = \Trnormalized \brk*{ \vbf_0 \xbf_0^\top } = \Trnormalized \brk*{ \xbf_0 \vbf_0^{\prime\top} } = \brk*{  \vbf'_0 \xbf_0^\top } = 0
\text{\,,}
\]
\[
\Trnormalized \brk*{ \brk*{ \xbf_0 \vbf_0^\top }^2 } = \Trnormalized \brk*{ \brk*{ \vbf_0 \xbf_0^\top }^2 } = \Trnormalized \brk*{ \brk*{ \xbf_0 \vbf_0^{\prime \top} }^2 } = \Tr \brk*{ \brk*{ \vbf'_0 \xbf_0^\top }^2 } = \Trnormalized \brk*{ \vbf_0 \xbf_0^\top \xbf_0 \vbf_0^{\prime \top} } = \Tr \brk*{ \xbf_0 \vbf_0^\top \vbf'_0 \xbf_0^\top } = 0
\text{\,,}
\]
\[
\Trnormalized \brk*{ \vbf_0 \xbf_0^\top \xbf_0 \vbf_0^\top } = \Trnormalized \brk*{ \xbf_0 \vbf_0^\top \vbf_0 \xbf_0^\top } = \Trnormalized \brk*{ \vbf'_0 \xbf_0^\top \xbf_0 \vbf_0^{\prime \top} } = \Tr \brk*{ \xbf_0 \vbf_0^{\prime \top} \vbf'_0 \xbf_0^\top } = \frac{1}{D}
\text{\,.}
\]
Now, for $\pi = \rho \delta \rho$ corresponding to a summand on the right hand side of \cref{eq:genus_expectation_proof_formula_cov_bound_second}, let $F_1, F_2, F_3, F_4$ be the faces in $\G (\gamma, \epsilon, \rho)$ (see construction in \cref{genus_expansion}), ordered according to their appearance in $\gamma$.
It follows that for the summand to be non-zero there are only two options: either all four faces $F_1, F_2, F_3, F_4$ are connected, meaning all four non-identity matrices are in the same cycle of $\gamma_{+}^{-1} \delta_\epsilon \pi \delta_\epsilon \gamma_-  / 2$, or $F_1$ is connected to $F_2$ and $F_3$ to $F_4$.
Any other summand will contribute zero due to the trace identities above.
We claim that the first option gives a contribution of order $1 / D^3$.
Indeed, for such $\pi = \rho \delta \rho$  we have that:
\[
\Trnormalized_{ \frac{ \gamma_-^{-1} \delta_\epsilon \pi \delta_\epsilon \gamma_+ }{ 2 } } \brk*{ \Cbf_1, \ldots, \Cbf_{2p} } = \frac{1}{D} 
\text{\,,}
\]
and $ D^{ \chi (\gamma, \delta_\epsilon \pi \delta_\epsilon) - \abs{\gamma} } \leq D^{-2}$ by \cref{top_theorem}  (recall $\abs{\gamma} = 4$).
As for the second option, note that any $\sigma := \delta_\epsilon \pi \delta_\epsilon$ that connects $F_1$ to $F_2$ and $F_3$ to $F_4$ can be factorized into $\sigma = \sigma_1 \sigma_2$, where $\sigma_1$ and $\sigma_2$ are the restrictions of $\sigma$ to the elements of $F_1 \cup F_2$ and $F_3 \cup F_4$, respectively.
We may similarly factorize $\gamma$ as $\gamma = \gamma_1 \gamma_2$.
It follows that the contribution of this summand factorizes as:
\[
\begin{split}
& D^{ \chi (\gamma, \sigma) - \abs{\gamma} } \cdot \Trnormalized_{ \frac{ \gamma_-^{-1} \sigma\gamma_+ }{ 2 } } \brk*{ \Cbf_1, \ldots, \Cbf_{2p} } \\
& \hspace{3mm} = \brk*{ D^{ \chi (\gamma_1, \sigma_1) - \abs{\gamma_1} } \cdot \Trnormalized_{ \frac{ \gamma_{1,-}^{-1} \sigma_1 \gamma_{1,+} }{ 2 } } \brk*{ \Cbf_1, \ldots, \Cbf_{p} } } \brk*{ D^{ \chi (\gamma_2, \sigma_2) - \abs{\gamma_2} } \cdot \Trnormalized_{ \frac{ \gamma_{2,-}^{-1} \sigma_2 \gamma_{2,+} }{ 2 } } \brk*{ \Cbf_{p + 1}, \ldots, \Cbf_{2p} } }
\text{\,.}
\end{split}
\]
This factorization corresponds precisely to a term in the expansion of $\EE \brk[s]{ Z_{\vbf_0, \xbf_0, n, k} } \EE \brk[s]{ Z_{\vbf'_0, \xbf_0, n, k} }$, and vice-versa.
Because all summands which give a contribution of $1 / D^2$ have this form, we get that:
\[
c_{-2} \brk1{ \EE \brk[s]{ Z_{\vbf_0, \xbf_0, n, k} Z_{\vbf'_0, \xbf_0, n, k} } } = c_{-2} \brk1{ \EE \brk[s]{ Z_{\vbf_0, \xbf_0, n, k} } \EE \brk[s]{ Z_{\vbf'_0, \xbf_0, n, k} } }
\text{\,.}
\]

To conclude, we have shown that both $\EE \brk[s]{ Z_{\vbf_0, \xbf_0, n, k} Z_{\vbf'_0, \xbf_0, n, k} }$ and $\EE \brk[s]{ Z_{\vbf_0, \xbf_0, n, k} } \EE \brk[s]{ Z_{\vbf'_0, \xbf_0, n, k} }$ can be represented as a sum of non-negative terms, each upper bounded by $1 / D^3$ or equal to $1 / D^2$.
Furthermore, the terms equal to $1 / D^2$ are the same, for both $\EE \brk[s]{ Z_{\vbf_0, \xbf_0, n, k} Z_{\vbf'_0, \xbf_0, n, k} }$ and $\EE \brk[s]{ Z_{\vbf_0, \xbf_0, n, k} } \EE \brk[s]{ Z_{\vbf'_0, \xbf_0, n, k} }$, and so cancel out in $\Cov \brk1{ Z_{\vbf_0, \xbf_0, n, k}, Z_{\vbf'_0, \xbf_0, n, k} }$.
Consequently, the covariance can be upper bounded by a sum of at most $\abs{ \M_{2p} } = (2p - 1)!! \leq (4H - 1)!!$ terms, each upper bounded by $1 / D^3$.
Thus:
\[
\Cov \brk*{ Z_{\vbf_0, \xbf_0, n, k}, Z_{\vbf'_0, \xbf_0, n, k} } \leq \frac{ (4H - 1)!! }{ D^3 }
\text{\,.}
\]
\end{proof}

\subsubsection{Genus Expansion of Gaussian Matrices}
\label{genus_expansion}

In this appendix, we introduce the concept of a \emph{genus expansion}~---~a proof technique from random matrix theory, whereby one expresses the traces of random matrix products as a sum over topological spaces.
Specifically, we adapt a result from~\citet{redelmeier2014real} that is used for bounding certain quantities in \cref{expectation_bound,high_prob_bound}.

\textbf{Additional notation.}
We require the following notation, which is an adaptation of that used in \citet{redelmeier2014real}.
Given matrices $\Cbf_{1}, \ldots,\Cbf_{N} \in \R^{D \times D}$, we denote $\Cbf_{-n}:=\Cbf_{n}^{\top}$ for $n \in [N]$.
For $\Cbf \in \R^{D \times D}$, we let $\Trnormalized (\Cbf) = \frac{1}{D} \Tr (\Cbf)$ be its normalized trace.
We denote by $\M_N$ the set of all pairings of $[N]$, \ie~the set of all permutations which have $N / 2$ cycles of length $2$ (note that if $N$ is odd then this set is empty). 
For a permutation $\gamma : [N] \to [N]$, we denote by $\abs{ \gamma }$ the number of cycles in its cycle decomposition.
Lastly, we use $\delta: \{-N, \ldots, -1, 1, \ldots, N\} \to \{-N, \ldots, -1, 1, \ldots, N\}$ to denote the mapping satisfying $\delta (n) = -n$.

Towards adapting the result of~\citet{redelmeier2014real}, we lay out several preliminary definitions.

\begin{definition}
	For a subset $\I \subseteq [N]$, let $\gamma : \I \to \I$ be a permutation given by the following cycle decomposition: $\gamma = (z_1, \ldots,  z_{n_1}) (z_{n_1 + 1}, \ldots, z_{n_2}) \cdots (z_{n_{k - 1} + 1}, \ldots, z_{n_k} )$, where $z_1, \ldots, z_{n_k} \in \I$ denote the elements of $\I$.
	We define $\gamma_+$ to be the permutation on $\{ -N, \ldots, -1, 1, \ldots, N\}$ that extends $\gamma$ by acting as the identity for $i \notin \I$.
	Additionally, we define $\gamma_- := \delta \gamma_+ \delta$.
\end{definition}

Note that $\gamma_-$ is a permutation with cycle decomposition:
\[
\gamma_- = ( - z_1, \ldots, -z_{n_1}) (-z_{n_1 + 1}, \ldots, -z_{n_2}) \cdots ( -z_{n_{z - 1} + 1}, \ldots, -z_{n_k} )
\text{\,.}
\]

\begin{definition}
	For a set of non-zero integers $\I \subset \mathbb{Z}$, a permutation $\pi$ on $\I \cup -\I$, where $- \I := \{ -i : i \in \I\}$, is called a \emph{premap} if $\delta \pi \delta = \pi^{-1}$ and no cycle of $\pi$ contains both $i$ and $-i$, for any $i \in \I$. 
\end{definition}

\begin{definition}
	For a subset $\I \subseteq \{ -N, \ldots, -1, 1, \ldots, N\}$, let $\gamma$ be a premap given by the cycle decomposition $\gamma = (z_1, \ldots,  z_{n_1}) (z_{n_1 + 1}, \ldots, z_{n_2}) \cdots (z_{n_{k - 1} + 1}, \ldots, z_{n_k} )$.
	We define the permutation $\frac{\gamma}{2}$ over $\I$ as follows.
	For each cycle of $\gamma$, if its smallest element in absolute value is positive, then the cycle is left unchanged.
	Otherwise, the cycle is removed, \ie~$\frac{\gamma}{2}$ acts as the identity for the elements in the removed cycle.
\end{definition}

\begin{definition}
	For a subset $\I \subseteq \{-N, \ldots, -1, 1, \ldots, N\}$, let $\gamma$ be a premap given by the cycle decomposition $\gamma = (z_1, \ldots,  z_{n_1}) (z_{n_1 + 1}, \ldots, z_{n_2}) \cdots (z_{n_{k - 1} + 1}, \ldots, z_{n_k} )$ and $\Cbf_1, \ldots, \Cbf_N \in \R^{D \times D}$.
	We define the \emph{trace along $\gamma$} of $\Cbf_1, \ldots, \Cbf_N$ to be:
	\[
	\Tr_\gamma (\Cbf_1, \ldots, \Cbf_N) := \Tr ( \Cbf_{z_1} \cdot \cdots \cdot \Cbf_{z_{n_1}} ) \cdot \Tr ( \Cbf_{z_{n_1 + 1}} \cdot \cdots \cdot \Cbf_{z_{n_2}}) \cdot \cdots \cdot \Tr ( \Cbf_{z_{k - 1} + 1} \cdot \cdots \cdot \Cbf_{n_k} )
	\text{\,.}
	\]
	Analogously, we define the \emph{normalized trace along $\gamma$} to be:
	\[
	\Trnormalized_\gamma (\Cbf_1, \ldots, \Cbf_N) := \Trnormalized ( \Cbf_{z_1} \cdot \cdots \cdot \Cbf_{z_{n_1}} ) \cdot \Trnormalized ( \Cbf_{z_{n_1 + 1}} \cdot \cdots \cdot \Cbf_{z_{n_2}}) \cdot \cdots \cdot \Trnormalized ( \Cbf_{z_{k - 1} + 1} \cdot \cdots \cdot \Cbf_{n_k} )
	\text{\,.}
	\]
\end{definition}

\begin{definition}
	\label{def:euler_char_perm}
	Let $\I \subset \mathbb{Z}$ be a set of integers which does not contain both $i$ and $-i$, for any integer $i \in \mathbb{Z}$.
	Furthermore, let $\gamma$ be a permutation on $\I$ and $\pi$ a premap on $\I \cup -\I := \{ -i : i \in \I\}$.
	The \emph{Euler characteristic} of $(\gamma, \pi)$ is defined by:
	\[
	\chi (\gamma, \pi) := \abs3{ \frac{ \gamma_+^{-1} \gamma_- }{2} } + \abs3{ \frac{ \pi }{2} } + \abs3{ \frac{ \gamma_+^{-1} \pi^{-1} \gamma_- }{ 2 } } - \abs3{\I}
	\text{\,.}
	\]
\end{definition}

\medskip

With the definitions above in place, we are now in a position to import the result of~\citet{redelmeier2014real}, which, for random Gaussian matrices, provides a formula for the expectation of the normalized trace along a permutation $\gamma$.

\begin{theorem}[Adaptation of Lemma~3.8 in~\citet{redelmeier2014real}]
	\label{thm:genus_expansion_main_theorem}
	Let $\gamma$ be a permutation over $[N]$, and $\epsilon : [N] \to \{-1, 1\}$.
	Furthermore, suppose that the entries of $\Abf \in \R^{D \times D}$ are sampled independently from a Gaussian with mean zero and standard deviation $1 / \sqrt{D}$, and $\Cbf_1, \ldots, \Cbf_N \in \R^{D \times D}$ are some fixed (non-random) matrices.
	Then:
	\be
	\EE\nolimits_{\Abf} \brk[s]*{ \Trnormalized_\gamma \brk*{ \Abf_{\epsilon (1)} \Cbf_1, \ldots, \Abf_{\epsilon (N)} \Cbf_N } } = \sum_{ \pi \in \{ \rho \delta \rho : \rho \in \M_N \} } D^{ \chi (\gamma, \delta_\epsilon \pi \delta_\epsilon) - 2 \abs{\gamma} } \cdot \Trnormalized_{ \frac{ \gamma_-^{-1} \delta_\epsilon \pi \delta_\epsilon \gamma_+ }{ 2 } } \brk*{ \Cbf_1, \ldots, \Cbf_N }
	\text{\,,}
	\label{eq:genus_formula}
	\ee
	where $\delta_\epsilon$ is a mapping on $\{-N, \ldots, -1, 1, \ldots, N\}$ defined by $\delta_\epsilon : k \mapsto \epsilon (k) k$, and we extend $\epsilon$ to $\{-N, \ldots, -1\}$ symmetrically, \ie~by setting $\epsilon (k) = \epsilon (-k)$.
\end{theorem}

To obtain explicit bounds over expected traces along a permutation, based on \cref{thm:genus_expansion_main_theorem}, we need to bound the Euler characteristic $\chi$.
For that purpose, \citet{redelmeier2014real} makes use of a topological interpretation of $\chi$ via the concept of genus expansion.
Specifically, as we show below, each summand on the right hand side of \cref{eq:genus_formula} corresponds to a two-dimensional surface whose topological properties determine the size of the summand.
We first give some necessary background from topology.

\begin{definition} 
	\label{def:euler_char_surface}
	The \emph{Euler characteristic} of a surface $\G$ is defined by:
	\[
	\chi (\G) := V(\G) + F(\G) - E(\G)
	\text{\,,}
	\]
	where $V, F, E$ are the number of vertices, faces, and edges of $\G$, respectively.
	Strictly speaking, $\chi (\G)$ is calculated by constructing a CW complex which is homeomorphic to $\G$ and determining $V (\G), F (\G),$ and $E (\G)$ through it.
	A basic theorem in topology shows that $V (\G), F (\G),$ and $E (\G)$ are invariant under homotopy, and so the choice of CW complex does not matter (\cf~\citet{munkres2018elements}).
\end{definition}

The following proposition establishes basic properties of the Euler characteristic of surfaces.
\begin{proposition}
	\label{top_theorem}
	For a surface $\G$, the Euler charactersitic $\chi$ satisfies:
	\begin{itemize}
		\item if $\G$ has $m \in \N$ connected components $\G_1, \ldots, \G_m$, then $\chi(\G)=\chi(\G_1) + \cdots + \chi (\G_m)$;
		
		\item and a connected surface $\G$ satisfies $\chi(\G) \leq 2$, with equality holding if and only if $\G$ is homeomorphic to a sphere.
	\end{itemize}
\end{proposition}

\begin{proof}
	These are basic properties from the field of topology~---~see \citet{munkres2018elements}.
\end{proof}

Now, given a permutation $\gamma$ on $[N]$, a function $\epsilon : [N] \to \{-1, 1\}$, and a pairing $\rho \in \M_N$, we construct a surface $\G (\gamma, \epsilon, \rho)$, whose properties will then determine the corresponding term in the sum of \cref{eq:genus_formula}.

Let $\G(\gamma,\epsilon,\rho)$ be the following (perhaps disconnected) two dimensional surface.
Each cycle $(z_1, \ldots, z_m)$ of $\gamma$ is associated with the front of an $m$-gon. 
The back of this $m$-gon is associated with the corresponding cycle of \(\gamma_{-}\), \ie~with $(-z_1, \ldots, -z_m)$.
The $m$-gon will serve as one of the faces of $\G(\gamma,\epsilon,\rho)$. 
For orienting the edges of the face defined above, if $\Cbf_{\epsilon (z_j)}$ is transposed, \ie~$\epsilon(z_j) = -1$, the corresponding edge is oriented clockwise, and otherwise it is oriented counterclockwise. 
At each vertex of the face we place the matrix $\Cbf_{z_j}$.
We now connect faces defined by different cycles according to the following procedure.
Let $\sigma := \delta_{\epsilon} \rho \delta \rho\delta_{\epsilon}$, which is a pairing of $\{-N, \ldots, -1, 1, \ldots, N\}$. 
For every pair $(n , \sigma (n))$, where $n \in \{-N, \ldots, -1, 1, \ldots, N\}$, we glue edge $n$ to $\sigma(n)$ according to their respective orientations (where the signs of $(n, \sigma(n))$ determine whether we flip these orientations, \ie~glue the fronts or backs of each edge). 
Overall, we obtain a surface $\G(\gamma,\epsilon,\rho)$ from these glued faces.

Finally, \cref{prop:euler_char} establishes that the Euler characteristic of $\G (\gamma, \epsilon, \rho)$ (\cref{def:euler_char_surface}), constructed above, is equal to the Euler characteristic of $(\gamma , \delta_\epsilon \rho \delta \rho \delta_\epsilon)$ (\cref{def:euler_char_perm}).

\begin{proposition}
	\label{prop:euler_char}
	Given a permutation $\gamma$ on $[N]$, a function $\epsilon : [N] \to \{-1, 1\}$, and a pairing $\rho \in \M_N$, let $\sigma := \delta_\epsilon \rho \delta \rho \delta_\epsilon$.
	For the surface $\G (\gamma, \epsilon, \rho)$ constructed as specified above, it holds that $\chi (\gamma, \sigma) = \chi ( \G (\gamma, \epsilon, \rho) )$.
\end{proposition}

\begin{proof}
	Recall that $\chi (\gamma,\sigma)$ is given by (\cf~\cref{def:euler_char_perm}):
	\[
	\abs3{ \frac{ \gamma_+^{-1} \gamma_- }{2} } + \abs3{ \frac{\sigma}{2} } + \abs3{ \frac{ \gamma_+^{-1} \sigma \gamma_- }{ 2 } } - N
	\text{\,,}
	\]
	and the Euler characteristic of $\G(\gamma,\epsilon,\rho)$ is given by (\cf~\cref{def:euler_char_surface}):
	\[
	V( \G(\gamma,\epsilon,\rho) ) + F (\G (\gamma,\epsilon,\rho) ) - E( \G(\gamma,\epsilon,\rho) )
	\text{\,.}
	\]
	Thus it suffices to show that the following hold:
	\[
	\abs3{ \frac{\gamma_{+}^{-1}\gamma_{-}}{2} } = F( \G(\gamma,\epsilon,\rho)) \quad , \quad \abs3{ \frac{ \sigma }{2} } - N = - E ( \G (\gamma, \epsilon, \rho)) \quad , \quad \abs3{ \frac{ \gamma_+^{-1} \sigma \gamma_{-} }{2} } = V (\G (\gamma, \epsilon, \rho))
	\text{\,.}
	\]
	The first equality (left) follows immediately from the fact that $|\frac{\gamma_{+}^{-1}\gamma_{-}}{2}| = |\gamma|$, and the construction of \(\G(\gamma,\epsilon,\rho ) \). 
	As for the second equality (middle), since $\sigma$ is a premap, which is a pairing on a domain of size $2N$, we have that $|\frac{\sigma}{2}| = \frac{|\sigma|}{2}= \frac{N}{2}$.
	On the other hand, by construction $\G(\gamma,\epsilon,\rho )$ has $\frac{N}{2}$ edges. 
	The third equality (right) relies on a generalization of Lemma~13.5 from~\citet{kemp2013math} to account for non-orientable gluings.
	Specifically, the vertices of \(\G(\gamma,\epsilon,\rho)\) correspond to the cycles of \(\frac{  \gamma_{-1}^{-1} \sigma \gamma_{+} }{2}\), \ie~each vertex of \(\G(\gamma,\epsilon,\rho)\) corresponds to the gluing of the vertices in some cycle of \(\frac{\gamma_{-1}^{-1}\sigma \gamma_{+}}{2}\).
	Note that \(\sigma\),and therefore \(\gamma_{-1}^{-1}\sigma \gamma_{+}\), are premaps. 
	Thus the division by $2$ leaves us with a permutation that acts on a set containing exactly one of $\{-n, n \}$, for each $n \in [N]$. 
	This corresponds to the choice whether to glue each edge  of the polygons from the front or the back.	
\end{proof}

	\section{Further Experiments and Implementation Details}
\label{app:experiments}

\subsection{Further Experiments With Underdetermined LQR Problems}
\label{app:experiments:lqr}

\cref{fig:lqr_experiments_h8,fig:lqr_experiments_d40,fig:lqr_experiments_rnd_B_Q} supplement \cref{fig:lqr_experiments_main} (from \cref{sec:experiments:lqr}) by including analogous experiments with, respectively: \emph{(i)} a longer time horizon $H = 8$ (instead of $H = 5$); \emph{(ii)} a larger state space dimension $D = 40$ (instead of $D = 5$); and \emph{(iii)} random $\Bbf$ and positive semidefinite $\Qbf$ matrices (instead of $\Bbf = \Qbf = \Ibf$).

\subsection{Further Experiments With Neural Network Controllers in Non-Linear Systems}
\label{app:experiments:nn}

For the quadcopter control problem,  \cref{fig:quad_experiments_add_unseen_initial_states,fig:quad_experiments_dist} supplement \cref{fig:pend_and_quad_experiments_main} by demonstrating that, respectively: \emph{(i)} the extent of extrapolation varies depending on the distance from initial states seen in training; and \emph{(ii)} extrapolation occurs to initial states unseen in training at different horizontal distances from the initial states seen in training (in addition to unseen initial states below those seen in training).

\subsection{Further Implementation Details}
\label{app:experiments:details}

We provide implementation details omitted from our experimental reports (\cref{sec:experiments,app:experiments:lqr,app:experiments:nn}).
Source code for reproducing our results and figures, based on the PyTorch~\citep{paszke2019pytorch} framework,\ifdefined\CAMREADY
~can be found at \url{https://github.com/noamrazin/imp_bias_control}.
\else
~is attached as supplementary material and will be made publicly available.
\fi
The experiments with underdetermined LQR problems (\cref{sec:experiments:lqr,app:experiments:lqr}) were carried out on a standard laptop, whereas for experiments with neural network controllers in non-linear systems (\cref{sec:experiments:nonlinear,app:experiments:nn}) we used a single Nvidia RTX 2080 Ti GPU.

\subsubsection{Linear Quadratic Control (\cref{sec:experiments:lqr})}
\label{app:experiments:details:lqr}

\textbf{System.} In all experiments, except for those with the “random $\Abf, \Bbf, \Qbf$'' system (\cref{fig:lqr_experiments_rnd_B_Q}), we set $\Bbf = \Qbf = \Ibf \in \R^{D \times D}$.

\textbf{Initial states.} For experiments with $d \in [D]$ initial states seen in training, we trained on the first $d$ standard basis vectors, and used the remaining standard basis vectors for evaluating extrapolation.

\textbf{Optimization.}
We ran policy gradient over a linear controller for $10^5$ iterations using a learning rate of $10^{-3}$.
For the experiments of \cref{fig:lqr_experiments_d40}, to allow stable training with a larger state space dimension and longer horizon, we ran policy gradient for twice as many iterations using a smaller learning rate of $10^{-4}$.

In the experiments of \cref{fig:lqr_experiments_main}, for all system types, median training cost across random seeds was within $10^{-8}$ of the minimal possible training cost.
In the experiments of \cref{fig:lqr_experiments_h8}, for all system types with $H = 8$, median training cost was within $2 \cdot 10^{-5}$ of the minimal possible training cost.
In the experiments of \cref{fig:lqr_experiments_d40}, for all system types, median training cost was within $2 \cdot 10^{-3}$ of the minimal possible training cost.
Lastly, in the experiments of \cref{fig:lqr_experiments_rnd_B_Q}, for the “random $\Abf, \Bbf, \Qbf$'' system type, median training cost was within $0.02$ of the minimal possible training cost.

\subsubsection{The Pendulum Control Problem (\cref{sec:experiments:nonlinear})}
\label{app:experiments:details:nn:pend}

	\textbf{System.}
	The two-dimensional state of the system is described by the vertical angle of the pendulum $\theta \in \R$ and its angular velocity $\dot{\theta} \in \R$.
	At time step $h$, the controller applies a torque $u_h \in \R$, giving rise to the following non-linear dynamics for a unit length pendulum with a unit mass object mounted on top of it:
	\be
	\begin{split}
		\theta_{h} & = \theta_{h - 1} + \Delta \cdot \dot{\theta}_{h - 1} \\
		\dot{\theta}_{h} & = \dot{\theta}_{h - 1} + \Delta \cdot \brk1{ u_{h - 1} - g \cdot \sin ( \theta_{h - 1} ) }
	\end{split}
	~~,~ \forall h \in [H]
	\text{\,,}
	\label{eq:pend_state_dynamics}
	\ee
	where $\Delta = 0.05$ is a time discretization resolution and $g = 10$ is the gravitational acceleration constant.

	\textbf{Cost.}
	The goal of the controller is to make the pendulum reach and stay at the target state $(\pi, 0)$.
	Accordingly, the cost at each time step is the squared Euclidean distance from $(\pi, 0)$.
	Specifically, suppose that we are given a (finite) set of initial states $\X \subset \R^2$.
	For a (state-feedback) controller $\pi_\wbf : \R^2 \to \R$, parameterized by $\wbf \in \R^P$, the cost is defined by:
	\be
	\cost (\wbf ; \X) := \frac{ 1 }{ H \cdot \abs{ \X } } \sum\nolimits_{ (\theta_0, \dot{\theta}_0) \in \X } \sum\nolimits_{h = 0}^H \norm*{ (\theta_h, \dot{\theta}_h ) - (\pi, 0) }^2
	\text{\,,}
	\label{eq:pend_cost}
	\ee
	where $\theta_h$ and $\dot{\theta}_h$ evolve according to \cref{eq:pend_state_dynamics} with $u_{h - 1} = \pi_\wbf ( \theta_{h - 1}, \dot{\theta}_{h - 1} )$, for $h \in [H]$.
	In all experiments, the time horizon is set to $H = 100$.
	
	\textbf{Initial states.}
	For the experiments of \cref{fig:pend_and_quad_experiments_main}, \cref{table:pend_experiments_states} specifies the initial states used for training and those used for evaluating extrapolation to initial states unseen in training.
	
	\textbf{Controller parameterization.} 
	We parameterized the controller as a fully-connected neural network with ReLU activation.
	The network was of depth $4$ and width $50$.
	Parameters were randomly initialized according to the default PyTorch implementation.
	
	\textbf{Non-extrapolating controller.}
	To obtain a non-extrapolating controller for \cref{fig:pend_and_quad_experiments_main}, we trained the controller using a modified objective instead of the standard training cost.
	In addition to the cost over initial states seen in training, the modified objective includes an “adversarial'' cost term over initial states unseen in training, for which the target state is set to be either $(0, 0)$ or $(2\pi, 0)$ (as opposed to the original target state $(\pi, 0)$).
	Specifically, for a coefficient $\lambda = 0.1$, the modified objective is given by:
	\[
		\cost (\wbf ; \trainstates) + \lambda \cdot \frac{1}{ H \cdot \abs{ \U } } \sum\nolimits_{ (\theta_0, \dot{\theta}_0) \in \U } \sum\nolimits_{h = 0}^H \norm*{ (\theta_h, \dot{\theta}_h ) - \brk*{ \bar{\theta}_{\theta_0} , 0} }^2
		\text{\,,}
	\]
	where $\trainstates \subset \R^2$ is the set of initial states seen in training, $\U \subset \R^2 \setminus \trainstates$ is the set of initial states used for evaluating extrapolation to initial states unseen in training, $\cost (\cdot \,; \trainstates)$ is defined by \cref{eq:pend_cost}, $\theta_h$ and $\dot{\theta}_h$ evolve according to \cref{eq:pend_state_dynamics} with $u_{h - 1} = \pi_\wbf ( \theta_{h - 1}, \dot{\theta}_{h - 1} )$, for $h \in [H]$, and $\bar{\theta}_{\theta_0}  = 0$ if $\theta_0 \leq \pi$ and $\bar{\theta}_{\theta_0} = 2 \pi$ if $\theta_0 > \pi$.
	We trained five controllers with this modified objective, using different random seeds, and selected for \cref{fig:pend_and_quad_experiments_main} the one attaining the lowest training cost.
	
	\textbf{Optimization.}
	The training cost was minimized via policy gradient with learning rate $5 \cdot 10^{-4}$.
	For training the non-extrapolating controller over the modified objective (specified above), we found the Adam optimizer~\cite{kingma2015adam} to be substantially more effective.
	Hence, for that purpose, we used Adam with default $\beta_1, \beta_2$ coefficients and learning rate $3 \cdot 10^{-4}$.
	Optimization proceeded until the training objective did not improve by at least $10^{-5}$ over $5,\!000$ consecutive iterations or $75,\!000$ iterations elapsed.
	The final controller in each run was taken to be that which achieved the lowest training cost across the iterations.
	We carried out five training runs with different random seeds, over both the standard and modified objectives, and chose to display the policy gradient controller that attained the median cost measure of extrapolation, and as a baseline the non-extrapolating controller that attained the lowest training cost.
	
	\textbf{Computing the normalized cost measure of extrapolation.} 
	Let $\wbf_{\mathrm{no-ext}} \in \R^P$ be the parameters of the non-extrapolating controller.
	The normalized cost measure of extrapolation attained by $\wbf \in \R^P$ for a set of initial states unseen in training $\U \subset \R^2 \setminus \trainstates$ is computed as follows: $ ( \cost ( \wbf ; \U ) - \tilde{\cost}^* (\U) ) / ( \cost (\wbf_{\mathrm{no-ext}} ; \U ) - \tilde{\cost}^* (\U) )$,
	where $\cost (\cdot \,; \U )$ is defined by \cref{eq:pend_cost} and $\tilde{\cost}^* (\U)$ is an estimate of the minimal possible cost over $\U$.
	We obtained the estimate $\tilde{\cost}^* (\U)$ by training a neural network controller (of the same architecture specified above) for minimizing the cost only over $\U$, \ie~for minimizing $\cost (\cdot \,; \U)$.
	We carried out five such runs, differing in random seed, and took  $\tilde{\cost}^* (\U)$ to be the minimal cost attained across the runs.

\subsubsection{The Quadcopter Control Problem (\cref{sec:experiments:nonlinear})}
\label{app:experiments:details:nn:quad}

\textbf{System.}
The state of the system $\xbf = (x, y, z, \phi, \theta, \psi, \dot{x}, \dot{y}, \dot{z}, \dot{\phi}, \dot{\theta}, \dot{\psi}) \in \R^{12}$ comprises the quadcopter's position $(x, y, z) \in \R^3$, tilt angles $(\phi, \theta, \psi) \in \R^3$ (\ie~roll, pitch, and yaw), and their respective velocities.
At time step $h$, the controller chooses $\ubf_h \in [0, \text{MAX\_RPM}]^4$, which determines the revolutions per minute (RPM) for each of the four motors, where $ \text{MAX\_RPM} = 21713.71$ is the maximal supported RPM. 
Our implementation of the state dynamics is adapted from the \href{https://github.com/DiffEqML/torchcontrol}{\texttt{torchcontrol}} GitHub repository, which is based on the explicit dynamics given in~\citet{panerati2021learning}.
For completeness, we lay out explicitly the evolution at time step $h \in [H]$:
\be
\begin{split}
	\begin{pmatrix}
		x_h \\
		y_h \\
		z_h
	\end{pmatrix} & = \begin{pmatrix}
		x_{h - 1} \\
		y_{h - 1} \\
		z_{h - 1}
	\end{pmatrix}  + \Delta \cdot \begin{pmatrix}
		\dot{x}_{h - 1} \\
		\dot{y}_{h - 1} \\
		\dot{z}_{h - 1}
	\end{pmatrix} \text{\,,}
\end{split}
\label{eq:quad_state_dynamics_1}
\ee
\be
\begin{split}
	\begin{pmatrix}
		\phi_h \\
		\theta_h \\
		\psi_h
	\end{pmatrix} & = \begin{pmatrix}
		\phi_{h - 1} \\
		\theta_{h - 1} \\
		\psi_{h - 1}
	\end{pmatrix}  + \Delta \cdot\begin{pmatrix}
		\dot{\phi}_{h - 1} \\
		\dot{\theta}_{h - 1} \\
		\dot{\psi}_{h - 1}
	\end{pmatrix} \text{\,,}
\end{split}
\label{eq:quad_state_dynamics_2}
\ee
\be
\begin{split}
	\begin{pmatrix}
		\dot{x}_{h} \\
		\dot{y}_{h} \\
		\dot{z}_{h}
	\end{pmatrix} & = \begin{pmatrix}
		\dot{x}_{h - 1} \\
		\dot{y}_{h - 1} \\
		\dot{z}_{h - 1}
	\end{pmatrix}  + \frac{\Delta}{m} \cdot \brk*{ \Vbf_{h - 1} \begin{pmatrix} 0 \\ 0 \\ k_f \cdot \norm{ \ubf_{h - 1} }^2 \end{pmatrix} - \begin{pmatrix} 0 \\ 0 \\ g \end{pmatrix} } \text{\,,}
\end{split}
\label{eq:quad_state_dynamics_3}
\ee
\be
\begin{split}
	\begin{pmatrix}
		\dot{\phi}_{h} \\
		\dot{\theta}_{h} \\
		\dot{\psi}_{h}
	\end{pmatrix} & = \begin{pmatrix}
		\dot{\phi}_{h - 1} \\
		\dot{\theta}_{h - 1} \\
		\dot{\psi}_{h - 1}
	\end{pmatrix} + \Delta \cdot \Pbf^{-1}  \brk*{ 
		\begin{pmatrix} 
			\frac{ k_F \cdot l}{ \sqrt{2} } \cdot \brk{ \ubf^2_{h - 1} [1] + \ubf^2_{h - 1} [2] - \ubf^2_{h - 1} [3] - \ubf^2_{h - 1} [4]} \\
			\frac{ k_F \cdot l}{ \sqrt{2} } \cdot \brk{ - \ubf^2_{h - 1} [1] + \ubf^2_{h - 1} [2] +	 \ubf^2_{h - 1} [3] - \ubf^2_{h - 1} [4]} \\
			k_T \cdot \brk{ - \ubf^2_{h - 1} [1] + \ubf^2_{h - 1} [2] - \ubf^2_{h - 1} [3] + \ubf^2_{h - 1} [4]}
		\end{pmatrix} 
		-  
		\begin{pmatrix}
			\dot{\phi}_{h - 1} \\
			\dot{\theta}_{h - 1} \\
			\dot{\psi}_{h - 1}
		\end{pmatrix} \times \Pbf \begin{pmatrix}
			\dot{\phi}_{h - 1} \\
			\dot{\theta}_{h - 1} \\
			\dot{\psi}_{h - 1}
		\end{pmatrix}
	}
	\text{\,,}
\end{split}
\label{eq:quad_state_dynamics_4}
\ee
where $\times$ stands here for the cross product of two vectors, $\ubf_{h - 1} [1], \ldots, \ubf_{h - 1} [4]$ are the entries of $\ubf_{h - 1}$, $\Delta = 0.02$ is the time discretization resolution, $g = 9.81$ is the gravitational acceleration constant, $m = 0.027$ is the quadcopter's mass, $l = 0.0397$ is the quadcopter's arm length, $k_F = 3.16 \cdot 10^{-10}, K_T = 7.94 \cdot 10^{-12}$ describe physical constants related to the conversion of motor RPM to torque, and:
\[
\begin{split}
	\Vbf_{h - 1} & =  \begin{pmatrix}
		\cos (\theta) & \sin ( \theta ) \sin ( \phi ) & \sin (\theta ) \cos ( \phi ) \\ 
		\sin ( \theta ) \sin ( \psi ) & - \cos (\theta ) \sin ( \phi ) \sin ( \psi ) + \cos ( \phi ) \cos (\psi) & - \cos (\theta)  \cos (\phi) \sin (\psi) - \sin (\phi) \cos (\psi) \\
		- \sin (\theta) \cos (\psi) & \cos (\theta) \sin (\phi) \cos (\psi) + \cost (\phi) \sin (\psi) & \cos (\theta) \cos (\phi) \cos (\psi) - \sin (\phi) \sin (\psi)
	\end{pmatrix}
	\\[0.4em]
	\Pbf & = \begin{pmatrix}
		1.4 \cdot 10^{-5} & 0 & 0 \\
		0 & 1.4 \cdot 10^{-5} & 0 \\
		0 & 0 & 2.17 \cdot 10^{-5}
	\end{pmatrix}
\end{split}
\]
are rotation and inertial matrices, respectively.
For brevity of notation, we omitted the subscript $h -1$ from $\phi, \theta, \psi$ in the definition of $\Vbf_{h - 1}$.

\textbf{Cost.}
The goal of the controller is to make the quadcopter reach and stay at the target state $\xbf^* = (0, 0, 1, 0, \ldots, 0)$.
In accordance with the \href{https://github.com/DiffEqML/torchcontrol}{\texttt{torchcontrol}} implementation, the cost at each time step is a weighted squared Euclidean distance from $\xbf^*$.
Specifically, suppose that we are given a (finite) set of initial states $\X\subset \R^{12}$.
For a (state-feedback) controller $\pi_\wbf : \R^{12} \to [0, \text{MAX\_RPM}]^4$, parameterized by $\wbf \in \R^P$, the cost is defined by:
\be
\cost (\wbf ; \X) := \frac{ 1 }{ H \cdot  \abs{ \X } } \sum\nolimits_{ \xbf_0 \in \X } \sum\nolimits_{h = 0}^H \sum\nolimits_{d = 1}^{12} \alpha_d^2 \cdot \brk*{ \xbf_h [d] - \xbf^*[d] }^2
\text{\,,}
\label{eq:quad_cost}
\ee
where $\xbf_h \in \R^{12}$ evolves according to \cref{eq:quad_state_dynamics_1,eq:quad_state_dynamics_2,eq:quad_state_dynamics_3,eq:quad_state_dynamics_4} with $\ubf_{h - 1} = \pi_\wbf ( \xbf_{h - 1} )$, for $h \in [H]$, the cost weights are $\alpha_1 = \alpha_2 = \alpha_3 = 1$ and $\alpha_4 = \cdots = \alpha_{12} = 0.1$, and $\xbf_h [d], \xbf^* [d]$ denote the $d$'th entries of $\xbf_h, \xbf^*$, respectively.
In all experiments, the time horizon is set to $H = 50$.

\textbf{Initial states.}
For the experiments of \cref{fig:pend_and_quad_experiments_main,fig:quad_experiments_add_unseen_initial_states,fig:quad_experiments_dist}, \cref{table:quad_below_experiments_states,,table:quad_experiments_add_unseen_initial_states_states,table:quad_dist_experiments_states} specify the initial states used for training and those used for evaluating extrapolation to initial states unseen in training, respectively.

\textbf{Controller parameterization.} 
As in pendulum control experiments (\cf~\cref{app:experiments:details:nn:pend}), we parameterized the controller as a fully-connected neural network with ReLU activation.
The network was of depth $4$ and width $50$, and its parameters were randomly initialized according to the default PyTorch implementation.
To convert the network's outputs into values within $[0, \text{MAX\_RPM}]$, we applied the hyperbolic tangent activation and linearly scaled the result.
That is, denoting by $\zbf \in \R^4$ the output of the network for some state, the chosen control was $\ubf = ( \tanh ( \zbf ) + \1 ) \cdot \frac{\text{MAX\_RPM}}{2}$, where $\tanh$ is applied element-wise and $\1 \in \R^4$ is the vector whose entries are all equal to one.

\textbf{Non-extrapolating controller.}
Similarly to the pendulum control experiments (\cf~\cref{app:experiments:details:nn:pend}), to obtain a non-extrapolating controller baselines for \cref{fig:pend_and_quad_experiments_main,fig:quad_experiments_add_unseen_initial_states,fig:quad_experiments_dist}, we trained controllers using a modified objective instead of the standard training cost.
In addition to the cost over initial states seen in training, the modified objective includes an “adversarial'' cost term over initial states unseen in training, for which the target state is set to be $\bar{\xbf} = (0, 0, 0, 0, \ldots, 0)$ (as opposed to the original target state $\xbf^* = (0, 0, 1, 0, \ldots, 0)$).
Specifically, for a coefficient $\lambda = 0.1$, the modified objective is given by:
\[
\cost (\wbf ; \trainstates) + \lambda \cdot \frac{1}{ \abs{ H \cdot \U } } \sum\nolimits_{ \xbf_0 \in \U } \sum\nolimits_{h = 0}^H\sum\nolimits_{d = 1}^{12} \alpha_d^2 \cdot \brk*{ \xbf_h [d] - \bar{\xbf}[d] }^2
\text{\,,}
\]
where $\trainstates \subset \R^{12}$ is the set of initial states seen in training, $\U \subset \R^{12} \setminus \trainstates$ is the set of initial states used for evaluating extrapolation to initial states unseen in training, $\cost (\cdot \,; \trainstates)$ is defined by \cref{eq:quad_cost}, $\xbf_h \in \R^{12}$ evolves according to \cref{eq:quad_state_dynamics_1,eq:quad_state_dynamics_2,eq:quad_state_dynamics_3,eq:quad_state_dynamics_4} with $\ubf_{h - 1} = \pi_\wbf ( \xbf_{h - 1} )$, for $h \in [H]$, the cost weights are $\alpha_1 = \alpha_2 = \alpha_3 = 1$ and $\alpha_4 = \cdots = \alpha_{12} = 0.1$, and $\xbf_h [d], \bar{\xbf} [d]$ denote the $d$'th entries of $\xbf_h, \bar{\xbf}$, respectively.
For each of~\cref{fig:pend_and_quad_experiments_main,fig:quad_experiments_add_unseen_initial_states,fig:quad_experiments_dist}, we trained five controllers with this modified objective, using different random seeds, and selected the one attaining the lowest training cost.

\textbf{Optimization.}
In all experiments, the training cost was minimized via the Adam optimizer~\cite{kingma2015adam} with default $\beta_1, \beta_2$ coefficients and learning rate $3 \cdot 10^{-4}$.
Optimization proceeded until the training objective did not improve by at least $10^{-5}$ over $5,\!000$ consecutive iterations or $75,\!000$ iterations elapsed.
The final controller in each run was taken to be that which achieved the lowest training cost across the iterations.
We carried out five training runs with different random seeds, over both the standard and modified objectives, and chose to display the policy gradient controller that attained the median cost measure of extrapolation, and as a baseline the non-extrapolating controller that attained the lowest training cost.

\textbf{Computing the normalized cost measure of extrapolation.} 
The normalized cost measure of extrapolation was computed according to the process described in \cref{app:experiments:details:nn:pend} for the pendulum control experiments.

\clearpage

\begin{figure*}[h!]
	\vspace{5mm}
	\begin{center}
		\includegraphics[width=1\textwidth]{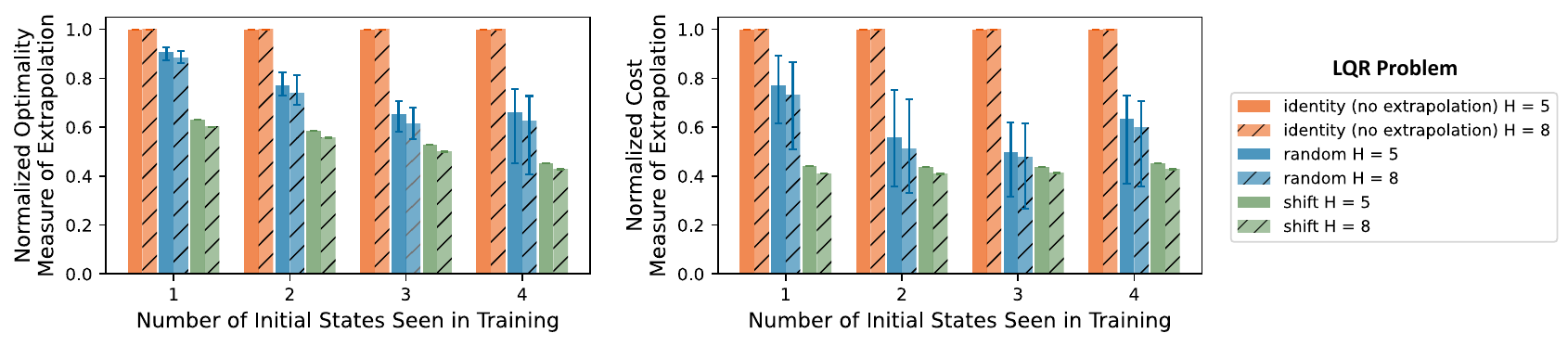}
	\end{center}
	\vspace{-2.5mm}
	\caption{
		In underdetermined LQR problems (\cref{sec:prelim:underdetermined}), the extent to which linear controllers learned via policy gradient extrapolate to initial states unseen in training, depends on the degree of exploration that the system induces from initial states that were seen in training.
		This figure supplements \cref{fig:lqr_experiments_main} by including results for analogous experiments over systems with a longer time horizon $H = 8$ (instead of $H = 5$).
		\textbf{Results:} The increase in time horizon improved extrapolation to unseen initial states, in accordance with the analysis of \cref{sec:analysis:shift}.
		A drawback of increasing the time horizon, however, is that it can lead to instabilities during training (\cf~\citet{metz2021gradients}).
		Indeed, for state space dimension $D = 5$, we were unable to consistently train controllers when the time horizon was substantially longer than $H = 8$.
		Thus, techniques enabling stable training with long time horizons may be a promising tool for improving extrapolation.
	}
	\label{fig:lqr_experiments_h8}
\end{figure*}

\begin{figure*}[h!]
	\vspace{0mm}
	\begin{center}
		\includegraphics[width=1\textwidth]{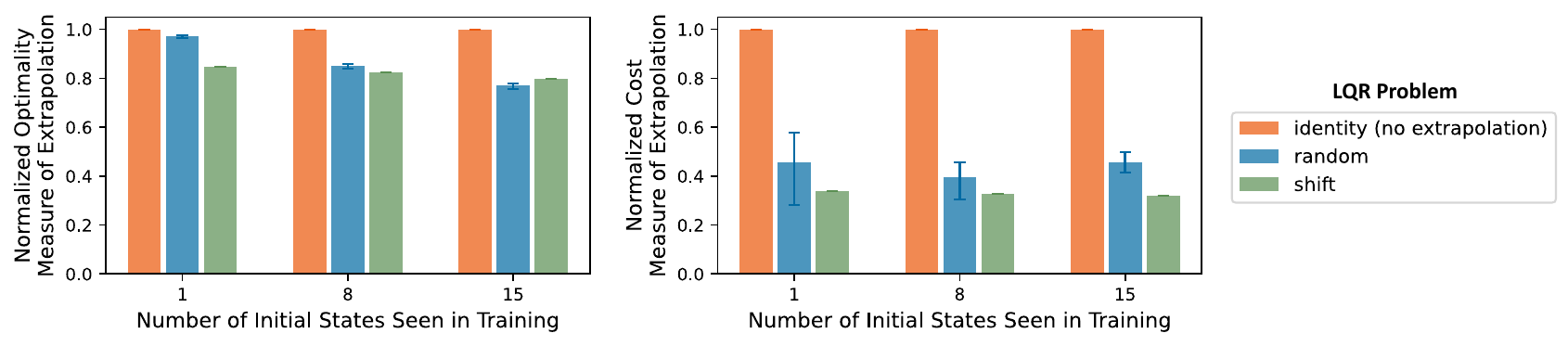}
	\end{center}
	\vspace{-2.5mm}
	\caption{
		In underdetermined LQR problems (\cref{sec:prelim:underdetermined}), the extent to which linear controllers learned via policy gradient extrapolate to initial states unseen in training, depends on the degree of exploration that the system induces from initial states that were seen in training.
		This figure supplements \cref{fig:lqr_experiments_main} by including results for analogous experiments over systems with a larger state space dimension $D = 40$ and horizon $H = 40$ (instead of $D = H = 5$).
		To reduce the cost of experiments with a larger state space dimension and longer horizon, we carried out $10$ (instead of $20$) runs per system type and number of initial states seen in training.
	}
	\label{fig:lqr_experiments_d40}
\end{figure*}

\begin{figure*}[h!]
	\vspace{0mm}
	\begin{center}
		\includegraphics[width=1\textwidth]{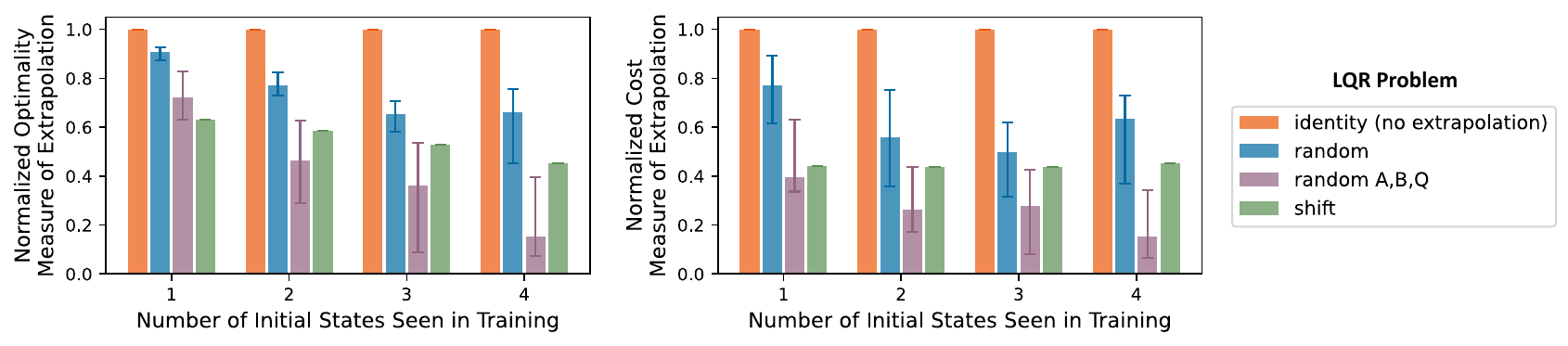}
	\end{center}
	\vspace{-2.5mm}
	\caption{
		In underdetermined LQR problems (\cref{sec:prelim:underdetermined}), the extent to which linear controllers learned via policy gradient extrapolate to initial states unseen in training, depends on the degree of exploration that the system induces from initial states that were seen in training.
		This figure supplements \cref{fig:lqr_experiments_main} by including results for analogous experiments over an LQR problem with random $\Abf \in \R^{D \times D}$, $\Bbf \in \R^{D \times D}$, and positive semidefinite $\Qbf \in \R^{D \times D}$ (instead of just a random $\Abf$).
		Specifically, in the “random $\Abf, \Bbf, \Qbf$'' system, the entries of $\Abf$ and $\Bbf$ were sampled independently from a zero-mean Gaussian with standard deviation \smash{$1 / \sqrt{D}$}.
		As for $\Qbf$, we first sampled the entries of a matrix $\Zbf \in \R^{D \times D}$ independently, again from a zero-mean Gaussian with standard deviation \smash{$1 / \sqrt{D}$}.
		Then, we set $\Qbf = \Zbf \Zbf^\top$.
		\textbf{Results:} Non-trivial extrapolation is achieved under the “random $\Abf, \Bbf, \Qbf$'' system, in accordance with the fact that random systems generically induce exploration (see discussion in \cref{sec:analysis:general}).
		The extent of extrapolation is significantly better compared to systems where just $\Abf$ is random (referred to as “random'' in the legend and analyzed in \cref{thm:typical_system}).
		Theoretical investigation of this phenomenon is left for future work.
	}
	\label{fig:lqr_experiments_rnd_B_Q}
\end{figure*}

\begin{figure*}[h!]
	\vspace{2mm}
	\begin{center}
		\includegraphics[width=\textwidth]{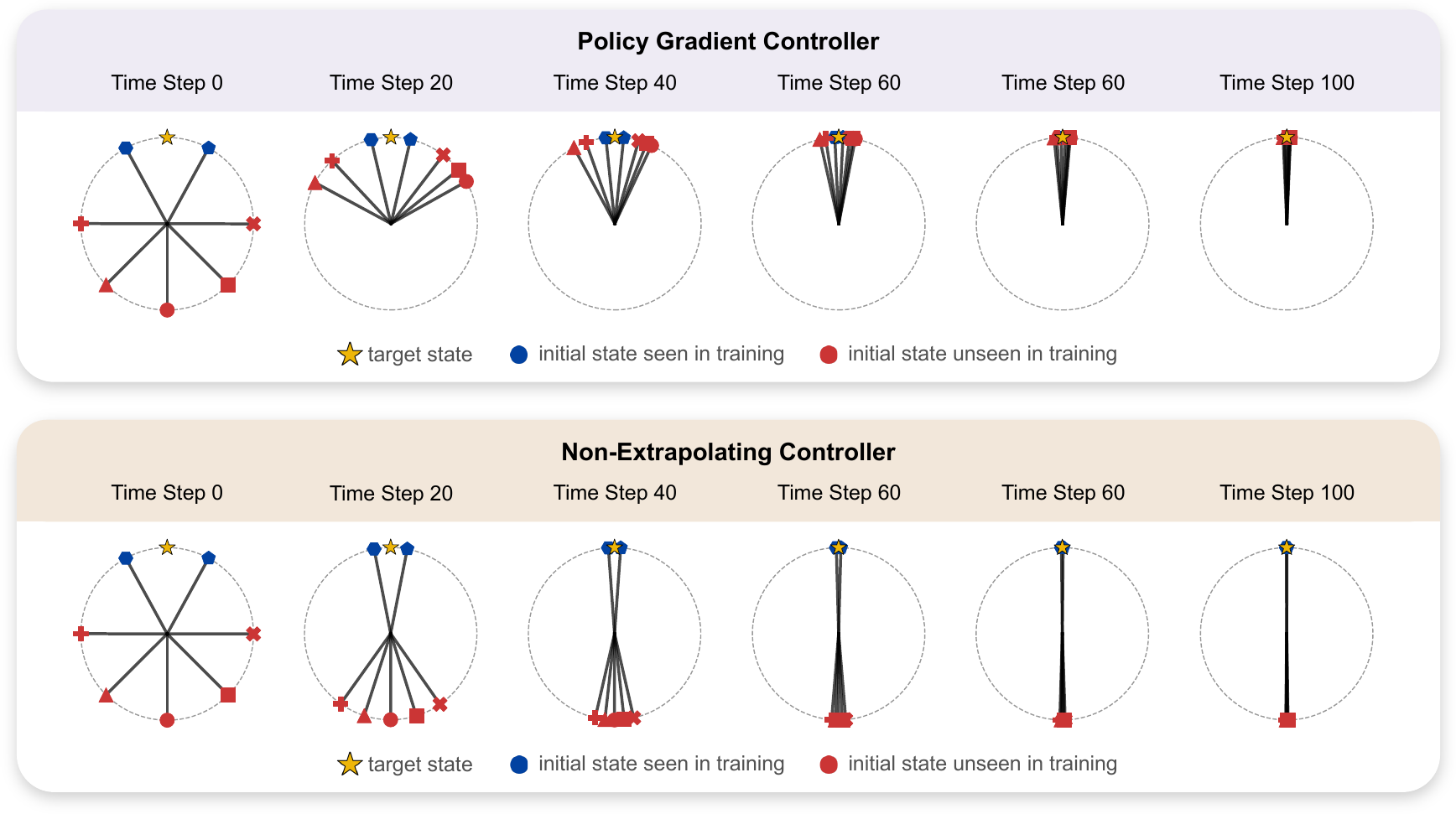}
	\end{center}
	\vspace{-2mm}
	\caption{
		For the pendulum control experiments in \cref{fig:pend_and_quad_experiments_main}, presented is the evolution of states through time under the policy gradient (top) and non-extrapolating (bottom) controllers. 
	}
	\label{fig:pend_experiments_states_through_time}
\end{figure*}

\begin{figure*}[h!]
	\vspace{2mm}
	\begin{center}
		\includegraphics[width=\textwidth]{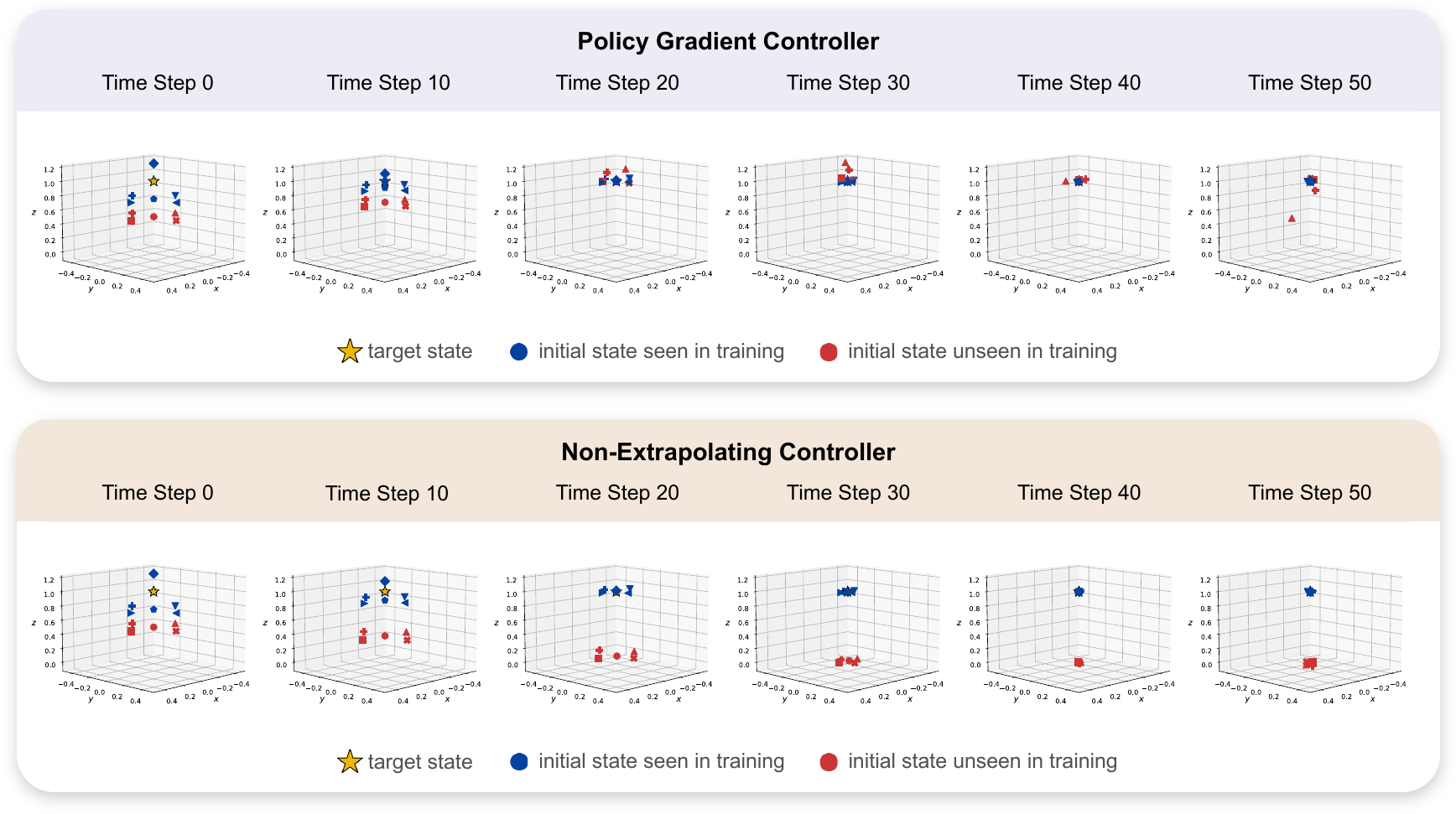}
	\end{center}
	\vspace{-2mm}
	\caption{
		For the quadcopter control experiments in \cref{fig:pend_and_quad_experiments_main}, presented is the evolution of states through time under the policy gradient (top) and non-extrapolating (bottom) controllers. 
	}
	\label{fig:quad_below_experiments_states_through_time}
\end{figure*}

\begin{figure*}[h!]
	\vspace{1mm}
	\begin{center}
		\includegraphics[width=1\textwidth]{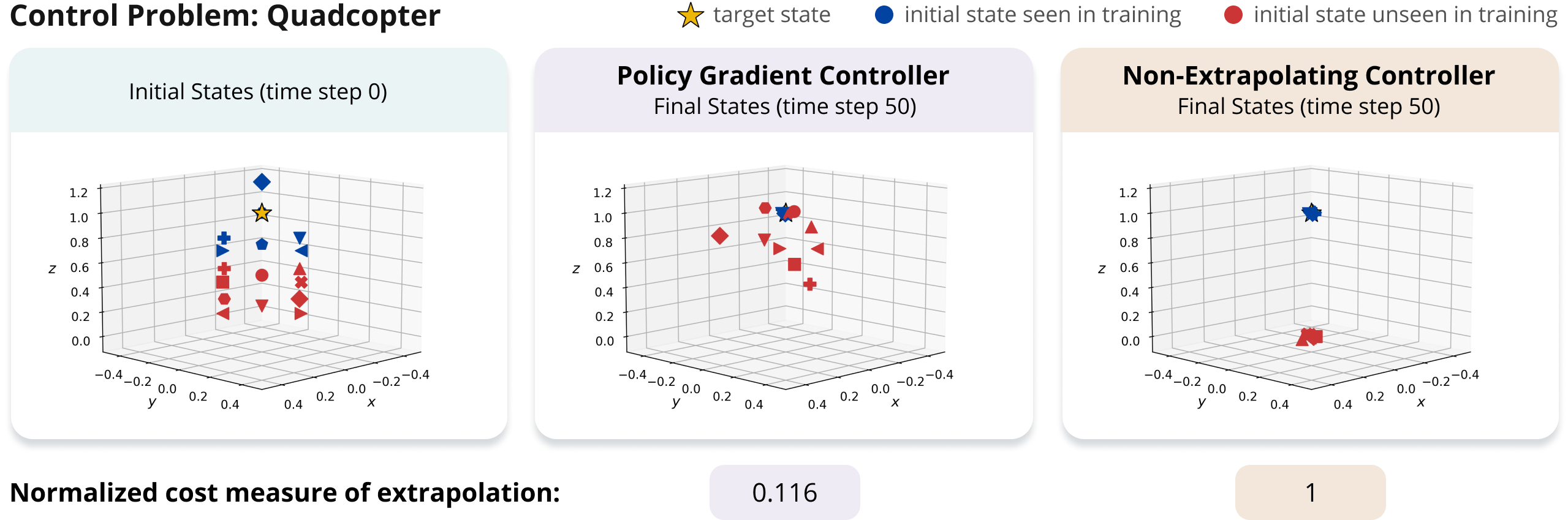}
	\end{center}
	\vspace{-1.5mm}
	\caption{
		In the quadcopter control problem (\cref{sec:experiments:nonlinear}), training a (non-linear) neural network controller via policy gradient often leads to a solution that extrapolates to initial states unseen in training, despite the existence of non-extrapolating solutions.
		This figure supplements \cref{fig:pend_and_quad_experiments_main} by including the results of an identical experiment, but with additional unseen initial states that are farther away from the initial states seen in training.
		See caption of \cref{fig:pend_and_quad_experiments_main} for details on the experiment.
		\textbf{Results:}
		As one might expect, while the extent of extrapolation is still highly non-trivial, it decays the farther away initial states unseen in training are from the initial states seen in training.
		\textbf{Further details in \cref{app:experiments}:} 
		\cref{table:quad_experiments_add_unseen_initial_states_states} fully specifies the initial and final states depicted above, and \cref{fig:quad_below_add_unseen_initial_states_experiments_states_through_time} presents the evolution of states through time under the policy gradient and non-extrapolating controllers.
	}
	\label{fig:quad_experiments_add_unseen_initial_states}
\end{figure*}

\begin{figure*}[h!]
	\vspace{0mm}
	\begin{center}
		\includegraphics[width=\textwidth]{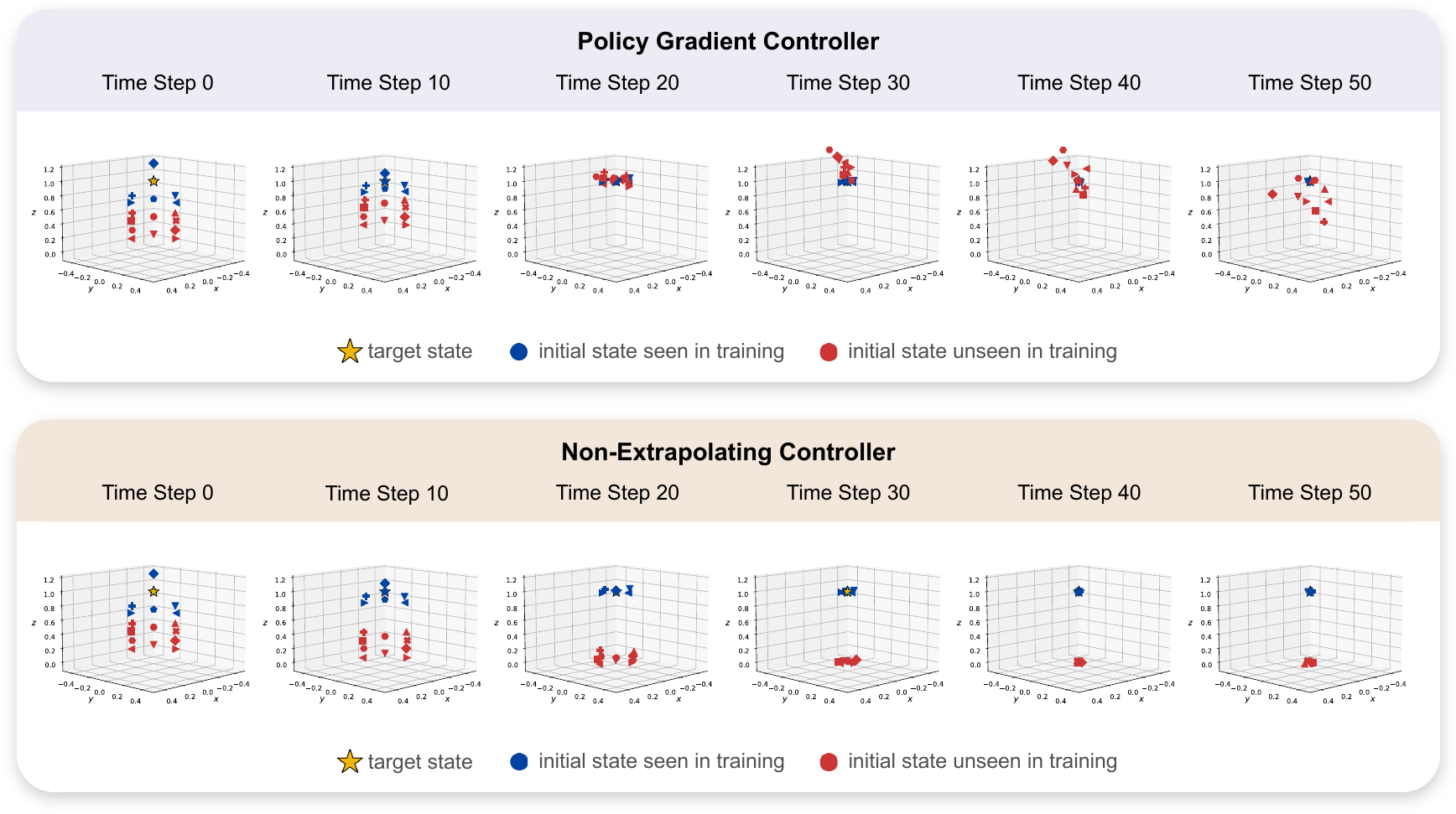}
	\end{center}
	\vspace{-2mm}
	\caption{
		For the policy gradient (top) and non-extrapolating (bottom) controllers from \cref{fig:quad_experiments_add_unseen_initial_states}, presented is the evolution of states through time.
	}
	\label{fig:quad_below_add_unseen_initial_states_experiments_states_through_time}
\end{figure*}

\begin{figure*}[h!]
	\vspace{1mm}
	\begin{center}
		\includegraphics[width=1\textwidth]{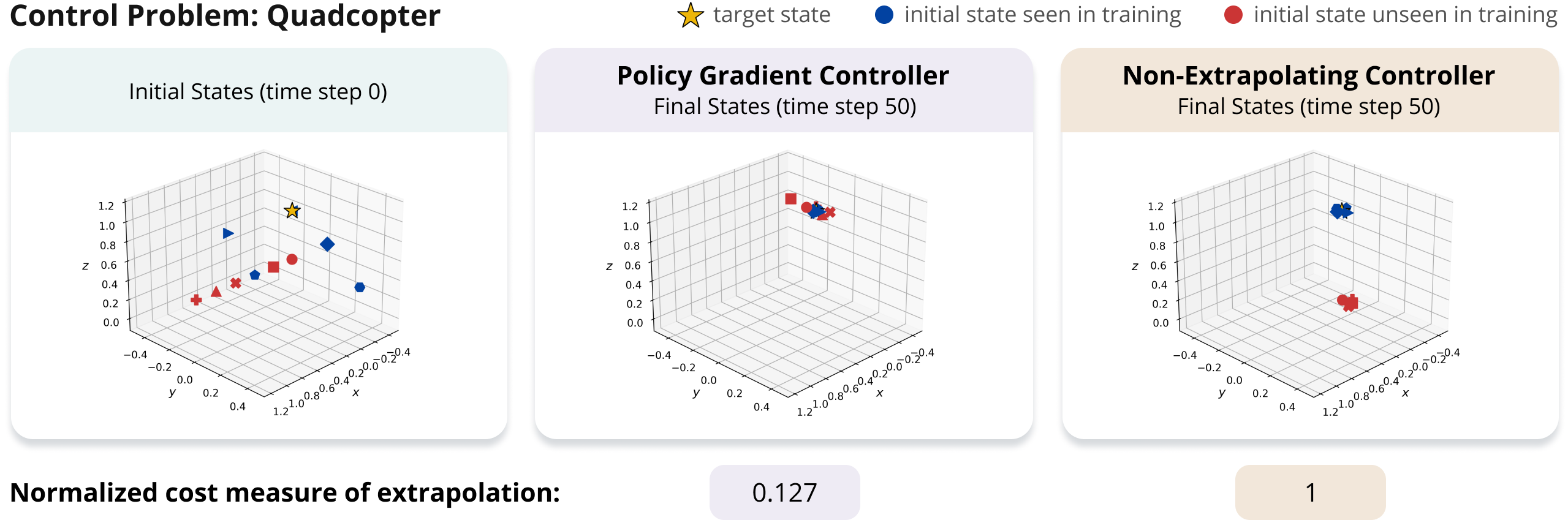}
	\end{center}
	\vspace{-1.5mm}
	\caption{
		In the quadcopter control problem (\cref{sec:experiments:nonlinear}), training a (non-linear) neural network controller via policy gradient often leads to a solution that extrapolates to initial states unseen in training, despite the existence of non-extrapolating solutions.
		This figure supplements \cref{fig:pend_and_quad_experiments_main} by including the results of an analogous experiment, where the unseen initial states are at different horizontal distances from the initial states seen in training (instead of being at a lower height).
		See caption of \cref{fig:pend_and_quad_experiments_main} for details on the experiment.
		\textbf{Results:}
		Remarkably, the controller trained via policy gradient extrapolates well to unseen initial states at various horizontal distances from the initial states seen in training.
		In contrast to unseen initial states below those used for training, for which extrapolation was observed in~\cref{fig:pend_and_quad_experiments_main}, an uncontrolled system does not induce exploration to states at different horizontal distances, in the naive sense of visiting the state along trajectories emanating from the initial states seen in training.
		Hence, the results of this experiment highlight the importance of finding a quantitative measure of exploration for non-linear systems, which may facilitate the theoretical study of extrapolation therein.
		\textbf{Further details in \cref{app:experiments}:} 
		\cref{table:quad_dist_experiments_states} fully specifies the initial and final states depicted above, and \cref{fig:quad_dist_experiments_states_through_time} presents the evolution of states through time under the policy gradient and non-extrapolating controllers.
	}
	\label{fig:quad_experiments_dist}
\end{figure*}

\begin{figure*}[h!]
	\vspace{0mm}
	\begin{center}
		\includegraphics[width=\textwidth]{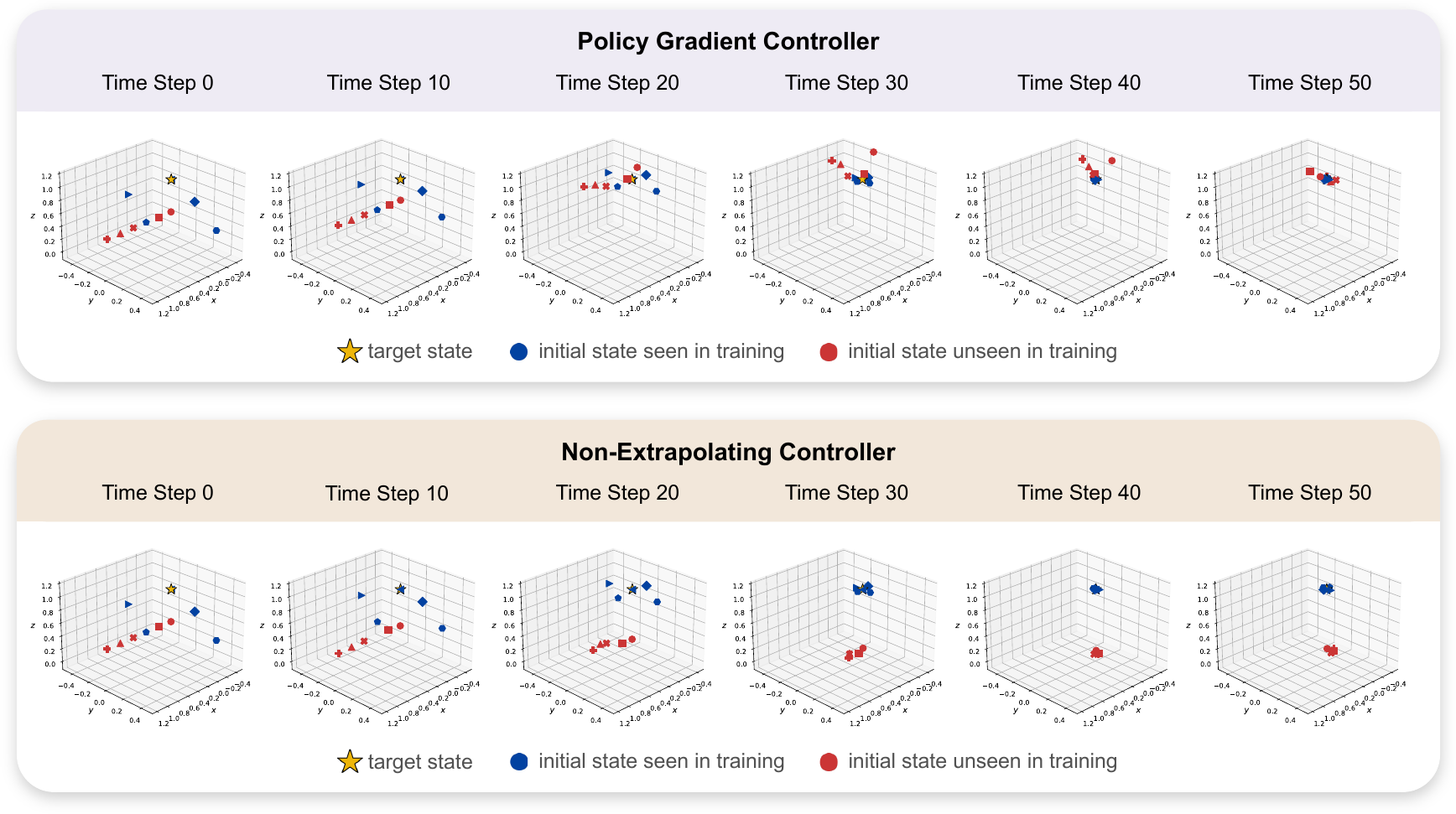}
	\end{center}
	\vspace{-2mm}
	\caption{
		For the policy gradient (top) and non-extrapolating (bottom) controllers from \cref{fig:quad_experiments_dist}, presented is the evolution of states through time.
	}
	\label{fig:quad_dist_experiments_states_through_time}
\end{figure*}

\begin{table*}[h!]
	\caption{
		Target, initial, and final states for the pendulum control experiments depicted in  \cref{fig:pend_and_quad_experiments_main}.
		Each state is described by the vertical angle of the pendulum $\theta \in \R$ and its angular velocity $\dot{\theta} \in \R$.
	}
	\vspace{-1mm}
	\begin{center}
		\fontsize{7.5}{9.5}\selectfont
		\begin{tabular}{c|lcc}
			\toprule
			& & $\theta$ & $\dot{\theta}$ \\
			\midrule
			& \includegraphics[scale=0.45]{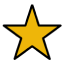} Target State & 3.14 & 0 \\
			\midrule
			\multirow{6}{*}{\rotatebox[origin=c]{90}{\includegraphics[scale=0.45]{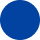} \, Training}} & \multirow{2}{*}{Initial States}  & 2.64 & 0.00 \\
			& & 3.64 & 0.00 \\
			\cmidrule{2-4}
			& \multirow{2}{*}{\shortstack[l]{Final States for \\ Policy Gradient Controller}}  & 3.13 & 0.01 \\
			& & 3.15 & -0.01 \\
			\cmidrule{2-4}
			& \multirow{2}{*}{\shortstack[l]{Final States for \\ Non-Extrapolating Controller}}  & 3.14 & 0.00 \\
			& & 3.15 & -0.00 \\
			\midrule
			\multirow{15}{*}{\rotatebox[origin=c]{90}{\includegraphics[scale=0.45]{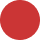} \, Unseen}} & \multirow{5}{*}{Initial States}  & 0.00 & 0.00 \\
			& & 0.79 & 0.00 \\
			& & 1.57 & 0.00 \\
			& & 4.71 & 0.00 \\
			& & 5.50 & 0.00 \\
			\cmidrule{2-4}
			& \multirow{5}{*}{\shortstack[l]{Final States for \\ Policy Gradient Controller}}  & 3.10 & 0.03 \\
			& & 3.10 & 0.03 \\
			& & 3.11 & 0.02 \\
			& & 3.17 & -0.03 \\		
			& & 3.18 & -0.04 \\
			\cmidrule{2-4}
			& \multirow{5}{*}{\shortstack[l]{Final States for \\ Non-Extrapolating Controller}}  & -0.00 & -0.00 \\
			& & 0.01 & -0.00 \\
			& & 0.02 & -0.01 \\
			& & 6.27 & 0.01 \\		
			& & 6.28 & 0.01 \\
			\bottomrule
		\end{tabular}
	\end{center}
	\label{table:pend_experiments_states}
\end{table*}

\begin{table*}[h!]
	\caption{
		Target, initial, and final states for the quadcopter control experiments depicted in  \cref{fig:pend_and_quad_experiments_main}.
		Each state of the system $\xbf = (x, y, z, \phi, \theta, \psi, \dot{x}, \dot{y}, \dot{z}, \dot{\phi}, \dot{\theta}, \dot{\psi}) \in \R^{12}$ comprises the quadcopter's position $(x, y, z)$, tilt angles $(\phi, \theta, \psi)$ (\ie~roll, pitch, and yaw), and their respective velocities.
	}
	\vspace{-1mm}
	\begin{center}
		\fontsize{7.5}{9.5}\selectfont
		\begin{tabular}{c|lcccccccccccc}
			\toprule
			& & $x$ & $y$ & $z$ & $\phi$ & $\theta$ & $\psi$ & $\dot{x}$ & $\dot{y}$ & $\dot{z}$  & $\dot{\phi}$  &  $\dot{\theta}$ &  $\dot{\psi}$ \\
			\midrule
			& \includegraphics[scale=0.45]{icons/target_state} Target State & 0 & 0 & 1 & 0 & 0 & 0 & 0 & 0 & 0 & 0 & 0 & 0 \\
			\midrule
			\multirow{21}{*}{\rotatebox[origin=c]{90}{\includegraphics[scale=0.45]{icons/training_state} \, Training}} & \multirow{7}{*}{Initial States}  & 0.00 & 0.00 & 0.75 & 0.00 & 0.00 & 0.00 & 0.00 & 0.00 & 0.00 & 0.00 & 0.00 & 0.00 \\
			& & 0.00 & 0.00 & 1.00 & 0.00 & 0.00 & 0.00 & 0.00 & 0.00 & 0.00 & 0.00 & 0.00 & 0.00 \\
			& & 0.00 & 0.00 & 1.25 & 0.00 & 0.00 & 0.00 & 0.00 & 0.00 & 0.00 & 0.00 & 0.00 & 0.00 \\
			& & 0.25 & 0.00 & 0.75 & 0.00 & 0.00 & 0.00 & 0.00 & 0.00 & 0.00 & 0.00 & 0.00 & 0.00 \\
			& & 0.00 & 0.25 & 0.75 & 0.00 & 0.00 & 0.00 & 0.00 & 0.00 & 0.00 & 0.00 & 0.00 & 0.00 \\
			& & -0.25 & 0.00 & 0.75 & 0.00 & 0.00 & 0.00 & 0.00 & 0.00 & 0.00 & 0.00 & 0.00 & 0.00 \\
			& & 0.00 & -0.25 & 0.75 & 0.00 & 0.00 & 0.00 & 0.00 & 0.00 & 0.00 & 0.00 & 0.00 & 0.00 \\
			\cmidrule{2-14}
			& \multirow{7}{*}{\shortstack[l]{Final States for \\ Policy Gradient Controller}} & 0.00 & -0.00 & 1.00 & -0.01 & -0.00 & 0.01 & -0.01 & 0.00 & 0.00 & 0.00 & 0.00 & 0.02 \\
			& & 0.00 & 0.00 & 1.00 & -0.01 & -0.00 & 0.01 & 0.00 & -0.00 & -0.01 & -0.01 & 0.00 & 0.01 \\
			& & -0.00 & 0.00 & 1.00 & -0.01 & 0.00 & 0.01 & 0.00 & -0.00 & -0.01 & -0.01 & -0.00 & 0.01 \\
			& & -0.03 & -0.00 & 1.00 & -0.01 & 0.07 & 0.01 & -0.10 & -0.00 & -0.01 & -0.02 & -0.11 & 0.01 \\
			& & -0.00 & -0.02 & 1.00 & -0.05 & -0.00 & -0.02 & 0.00 & 0.02 & 0.02 & 0.11 & 0.02 & 0.15 \\
			& & 0.03 & 0.00 & 1.00 & -0.01 & -0.07 & 0.01 & 0.10 & 0.01 & -0.00 & -0.00 & 0.09 & 0.02 \\
			& & -0.00 & 0.02 & 1.00 & 0.03 & -0.00 & 0.04 & 0.01 & -0.01 & -0.06 & -0.12 & -0.02 & -0.12 \\
			\cmidrule{2-14}
			& \multirow{7}{*}{\shortstack[l]{Final States for \\ Non-Extrapolating Controller}} & 0.00 & -0.00 & 1.00 & 0.02 & 0.00 & -0.02 & 0.00 & 0.01 & 0.01 & 0.02 & 0.00 & -0.01 \\
			& & 0.00 & -0.00 & 1.00 & 0.02 & 0.00 & -0.02 & 0.01 & 0.01 & 0.01 & 0.02 & -0.00 & -0.01 \\
			& & -0.00 & -0.00 & 1.00 & 0.01 & 0.00 & -0.01 & 0.01 & -0.01 & -0.00 & 0.00 & -0.01 & -0.03 \\
			& & -0.03 & 0.00 & 1.00 & 0.03 & 0.09 & -0.04 & -0.09 & 0.00 & -0.02 & 0.04 & -0.12 & -0.05 \\
			& & -0.00 & -0.02 & 1.00 & -0.01 & 0.00 & -0.06 & -0.02 & 0.05 & 0.00 & 0.20 & 0.05 & 0.15 \\
			& & 0.03 & 0.00 & 1.01 & 0.00 & -0.08 & 0.01 & 0.09 & 0.00 & 0.02 & -0.00 & 0.12 & 0.04 \\
			& & 0.00 & 0.02 & 1.01 & 0.05 & 0.00 & 0.02 & 0.02 & -0.04 & -0.00 & -0.17 & -0.06 & -0.18 \\
			\midrule
			\multirow{15}{*}{\rotatebox[origin=c]{90}{\includegraphics[scale=0.45]{icons/unseen_state} \, Unseen}} & \multirow{5}{*}{Initial States} & 0.00 & 0.00 & 0.50 & 0.00 & 0.00 & 0.00 & 0.00 & 0.00 & 0.00 & 0.00 & 0.00 & 0.00 \\
			& & 0.25 & 0.00 & 0.50 & 0.00 & 0.00 & 0.00 & 0.00 & 0.00 & 0.00 & 0.00 & 0.00 & 0.00 \\
			& & 0.00 & 0.25 & 0.50 & 0.00 & 0.00 & 0.00 & 0.00 & 0.00 & 0.00 & 0.00 & 0.00 & 0.00 \\
			& & -0.25 & 0.00 & 0.50 & 0.00 & 0.00 & 0.00 & 0.00 & 0.00 & 0.00 & 0.00 & 0.00 & 0.00 \\
			& & 0.00 & -0.25 & 0.50 & 0.00 & 0.00 & 0.00 & 0.00 & 0.00 & 0.00 & 0.00 & 0.00 & 0.00 \\
			\cmidrule{2-14}
			& \multirow{5}{*}{\shortstack[l]{Final States for \\ Policy Gradient Controller}} & -0.00 & -0.00 & 1.00 & -0.01 & 0.00 & 0.00 & 0.00 & 0.01 & 0.01 & -0.00 & -0.01 & 0.02 \\
			& & -0.03 & 0.01 & 1.02 & 0.01 & 0.07 & 0.01 & -0.08 & 0.02 & -0.03 & 0.04 & -0.12 & 0.05 \\
			& & -0.01 & -0.00 & 1.04 & -0.02 & -0.00 & 0.00 & 0.05 & 0.07 & -0.06 & 0.12 & -0.08 & 0.16 \\
			& & 0.16 & -0.05 & 0.50 & -0.07 & -0.71 & -0.19 & -0.12 & 0.06 & -3.19 & 0.70 & 4.80 & -0.25 \\
			& & -0.01 & 0.05 & 0.88 & 0.06 & -0.01 & 0.03 & -0.02 & -0.38 & -0.05 & -0.68 & 0.93 & -0.44 \\
			\cmidrule{2-14}
			& \multirow{5}{*}{\shortstack[l]{Final States for \\ Non-Extrapolating Controller}} & 0.02 & 0.01 & -0.01 & 0.09 & 0.02 & -0.11 & 0.14 & 0.05 & 0.03 & 0.17 & -0.03 & -0.15 \\
			& & -0.03 & 0.01 & 0.00 & 0.10 & 0.13 & -0.11 & -0.10 & 0.02 & -0.39 & 0.02 & 0.15 & 0.33 \\
			& & 0.02 & -0.02 & 0.00 & 0.03 & 0.02 & -0.16 & 0.08 & 0.05 & 0.04 & 0.38 & -0.03 & 0.06 \\
			& & 0.06 & 0.02 & -0.02 & 0.06 & -0.09 & -0.06 & 0.31 & 0.07 & -0.09 & 0.13 & -0.06 & -0.10 \\
			& & 0.02 & 0.04 & -0.04 & 0.15 & 0.02 & -0.02 & 0.11 & 0.10 & -0.57 & -0.17 & 0.04 & 0.08 \\
			\bottomrule
		\end{tabular}
	\end{center}
	\label{table:quad_below_experiments_states}
\end{table*}

\begin{table*}[h!]
	\caption{
		Target, initial, and final states for the quadcopter control experiments depicted in  \cref{fig:quad_experiments_add_unseen_initial_states}.
		Each state of the system $\xbf = (x, y, z, \phi, \theta, \psi, \dot{x}, \dot{y}, \dot{z}, \dot{\phi}, \dot{\theta}, \dot{\psi}) \in \R^{12}$ comprises the quadcopter's position $(x, y, z)$, tilt angles $(\phi, \theta, \psi)$ (\ie~roll, pitch, and yaw), and their respective velocities.
	}
	\vspace{-1mm}
	\begin{center}
		\fontsize{7.5}{9.5}\selectfont
		\begin{tabular}{c|lcccccccccccc}
			\toprule
			& & $x$ & $y$ & $z$ & $\phi$ & $\theta$ & $\psi$ & $\dot{x}$ & $\dot{y}$ & $\dot{z}$  & $\dot{\phi}$  &  $\dot{\theta}$ &  $\dot{\psi}$ \\
			\midrule
			& \includegraphics[scale=0.45]{icons/target_state} Target State & 0 & 0 & 1 & 0 & 0 & 0 & 0 & 0 & 0 & 0 & 0 & 0 \\
			\midrule
			\multirow{21}{*}{\rotatebox[origin=c]{90}{\includegraphics[scale=0.45]{icons/training_state} \, Training}} & \multirow{7}{*}{Initial States}  & 0.00 & 0.00 & 0.75 & 0.00 & 0.00 & 0.00 & 0.00 & 0.00 & 0.00 & 0.00 & 0.00 & 0.00 \\
			& & 0.00 & 0.00 & 1.00 & 0.00 & 0.00 & 0.00 & 0.00 & 0.00 & 0.00 & 0.00 & 0.00 & 0.00 \\
			& & 0.00 & 0.00 & 1.25 & 0.00 & 0.00 & 0.00 & 0.00 & 0.00 & 0.00 & 0.00 & 0.00 & 0.00 \\
			& & 0.25 & 0.00 & 0.75 & 0.00 & 0.00 & 0.00 & 0.00 & 0.00 & 0.00 & 0.00 & 0.00 & 0.00 \\
			& & 0.00 & 0.25 & 0.75 & 0.00 & 0.00 & 0.00 & 0.00 & 0.00 & 0.00 & 0.00 & 0.00 & 0.00 \\
			& & -0.25 & 0.00 & 0.75 & 0.00 & 0.00 & 0.00 & 0.00 & 0.00 & 0.00 & 0.00 & 0.00 & 0.00 \\
			& & 0.00 & -0.25 & 0.75 & 0.00 & 0.00 & 0.00 & 0.00 & 0.00 & 0.00 & 0.00 & 0.00 & 0.00 \\
			\cmidrule{2-14}
			& \multirow{7}{*}{\shortstack[l]{Final States for \\ Policy Gradient Controller}} & -0.00 & 0.00 & 1.00 & 0.01 & 0.00 & -0.00 & -0.01 & 0.01 & -0.00 & 0.02 & 0.02 & -0.00 \\
			& & 0.00 & 0.00 & 1.00 & 0.01 & -0.00 & -0.01 & -0.00 & -0.01 & 0.00 & 0.00 & 0.01 & -0.01 \\
			& & 0.00 & -0.00 & 1.00 & 0.01 & -0.00 & -0.02 & 0.00 & -0.01 & -0.00 & -0.00 & -0.00 & -0.02 \\
			& & -0.02 & 0.00 & 1.00 & 0.04 & 0.09 & -0.04 & -0.04 & 0.01 & -0.02 & 0.09 & -0.11 & 0.04 \\
			& & -0.00 & -0.02 & 1.00 & -0.01 & 0.00 & -0.07 & 0.01 & 0.03 & 0.01 & 0.21 & -0.01 & 0.16 \\
			& & 0.03 & -0.00 & 1.00 & 0.02 & -0.09 & -0.01 & 0.05 & -0.02 & 0.01 & -0.06 & 0.08 & -0.03 \\
			& & 0.00 & 0.02 & 1.00 & 0.03 & -0.00 & 0.04 & 0.01 & -0.05 & 0.02 & -0.17 & -0.02 & -0.18 \\
			\cmidrule{2-14}
			& \multirow{7}{*}{\shortstack[l]{Final States for \\ Non-Extrapolating Controller}} & -0.00 & 0.00 & 1.00 & 0.02 & -0.00 & -0.02 & 0.01 & 0.03 & 0.01 & 0.04 & -0.01 & -0.00 \\
			& & -0.00 & 0.00 & 1.00 & 0.02 & -0.00 & -0.02 & -0.00 & 0.02 & 0.00 & 0.03 & 0.00 & -0.01 \\
			& & -0.00 & 0.00 & 1.00 & 0.02 & 0.00 & -0.02 & -0.01 & 0.03 & 0.00 & 0.04 & 0.02 & -0.01 \\
			& & -0.03 & -0.00 & 1.00 & 0.02 & 0.10 & -0.01 & -0.03 & -0.02 & -0.03 & -0.02 & -0.16 & -0.01 \\
			& & -0.00 & -0.02 & 1.01 & -0.01 & -0.01 & -0.07 & -0.02 & 0.07 & 0.06 & 0.22 & 0.01 & 0.12 \\
			& & 0.03 & 0.01 & 1.00 & 0.02 & -0.10 & -0.04 & 0.05 & 0.05 & 0.04 & 0.03 & 0.08 & -0.03 \\
			& & -0.00 & 0.02 & 1.00 & 0.04 & 0.01 & 0.03 & 0.01 & -0.04 & -0.05 & -0.17 & -0.00 & -0.15 \\
			\midrule
			\multirow{30}{*}{\rotatebox[origin=c]{90}{\includegraphics[scale=0.45]{icons/unseen_state} \, Unseen}} & \multirow{10}{*}{Initial States} & 0.00 & 0.00 & 0.50 & 0.00 & 0.00 & 0.00 & 0.00 & 0.00 & 0.00 & 0.00 & 0.00 & 0.00 \\
			& & 0.25 & 0.00 & 0.50 & 0.00 & 0.00 & 0.00 & 0.00 & 0.00 & 0.00 & 0.00 & 0.00 & 0.00 \\
			& & 0.00 & 0.25 & 0.50 & 0.00 & 0.00 & 0.00 & 0.00 & 0.00 & 0.00 & 0.00 & 0.00 & 0.00 \\
			& & -0.25 & 0.00 & 0.50 & 0.00 & 0.00 & 0.00 & 0.00 & 0.00 & 0.00 & 0.00 & 0.00 & 0.00 \\
			& & 0.00 & -0.25 & 0.50 & 0.00 & 0.00 & 0.00 & 0.00 & 0.00 & 0.00 & 0.00 & 0.00 & 0.00 \\
			& & 0.00 & 0.00 & 0.25 & 0.00 & 0.00 & 0.00 & 0.00 & 0.00 & 0.00 & 0.00 & 0.00 & 0.00 \\
			& & 0.25 & 0.00 & 0.25 & 0.00 & 0.00 & 0.00 & 0.00 & 0.00 & 0.00 & 0.00 & 0.00 & 0.00 \\
			& & 0.00 & 0.25 & 0.25 & 0.00 & 0.00 & 0.00 & 0.00 & 0.00 & 0.00 & 0.00 & 0.00 & 0.00 \\
			& & -0.25 & 0.00 & 0.25 & 0.00 & 0.00 & 0.00 & 0.00 & 0.00 & 0.00 & 0.00 & 0.00 & 0.00 \\
			& & 0.00 & -0.25 & 0.25 & 0.00 & 0.00 & 0.00 & 0.00 & 0.00 & 0.00 & 0.00 & 0.00 & 0.00 \\
			\cmidrule{2-14}
			& \multirow{10}{*}{\shortstack[l]{Final States for \\ Policy Gradient Controller}} & -0.01 & 0.05 & 1.02 & 0.01 & 0.02 & 0.03 & -0.27 & 0.07 & 0.02 & 0.09 & 0.55 & 0.01 \\
			& & 0.09 & 0.14 & 0.64 & 0.07 & 0.46 & 0.23 & 1.51 & 0.47 & 0.00 & 1.23 & -2.99 & 0.81 \\
			& & 0.01 & 0.02 & 1.01 & -0.00 & -0.00 & -0.04 & -0.15 & 0.00 & 0.00 & 0.15 & 0.23 & 0.09 \\
			& & -0.04 & 0.13 & 0.91 & 0.02 & 0.16 & 0.15 & -0.77 & 0.21 & -0.03 & 0.10 & 2.32 & 0.15 \\
			& & 0.10 & 0.25 & 0.51 & -0.10 & 0.27 & 0.58 & 0.21 & 0.02 & -1.40 & 1.27 & 2.83 & 0.84 \\
			& & 0.24 & 0.11 & 0.85 & -0.57 & -0.54 & 0.28 & 0.17 & 0.35 & -2.93 & 1.37 & -0.08 & 2.25 \\
			& & -0.14 & 0.07 & 0.70 & -0.73 & 0.84 & 0.03 & -0.10 & 0.59 & -3.11 & 1.00 & 2.67 & 1.19 \\
			& & 0.09 & 0.05 & 0.74 & -0.33 & 0.11 & 0.06 & -0.15 & 0.66 & -1.42 & 1.01 & 3.72 & 0.83 \\
			& & 0.55 & 0.14 & 0.95 & -1.13 & -1.16 & 0.17 & 0.89 & 0.52 & -3.04 & 1.43 & -0.90 & 3.42 \\
			& & 0.35 & 0.22 & 1.14 & -0.74 & -0.48 & 0.34 & 0.41 & 0.74 & -2.68 & 0.18 & -1.60 & 2.29 \\
			\cmidrule{2-14}
			& \multirow{10}{*}{\shortstack[l]{Final States for \\ Non-Extrapolating Controller}} & -0.01 & -0.02 & 0.01 & -0.00 & 0.01 & 0.01 & -0.00 & -0.10 & 0.00 & -0.24 & 0.03 & 0.01 \\
			& & -0.04 & -0.01 & -0.01 & 0.04 & 0.17 & -0.02 & -0.01 & -0.05 & -0.09 & -0.01 & -0.39 & 0.01 \\
			& & 0.00 & -0.03 & 0.01 & -0.03 & -0.01 & -0.05 & -0.03 & 0.05 & -0.02 & 0.32 & -0.03 & 0.16 \\
			& & 0.04 & -0.02 & -0.02 & -0.02 & -0.16 & -0.00 & -0.09 & -0.06 & 0.08 & -0.03 & 0.67 & -0.16 \\
			& & 0.00 & 0.02 & -0.01 & 0.05 & 0.01 & 0.06 & -0.06 & -0.14 & -0.13 & -0.27 & 0.08 & -0.25 \\
			& & -0.01 & -0.03 & -0.00 & 0.00 & 0.01 & -0.01 & -0.01 & -0.04 & 0.02 & -0.13 & -0.00 & 0.03 \\
			& & -0.02 & -0.01 & -0.01 & 0.03 & 0.13 & -0.02 & 0.03 & -0.06 & -0.11 & -0.08 & -0.42 & 0.08 \\
			& & 0.00 & -0.02 & 0.00 & -0.03 & -0.02 & -0.05 & 0.02 & 0.02 & 0.04 & 0.19 & -0.17 & 0.24 \\
			& & -0.02 & -0.00 & -0.01 & -0.01 & -0.08 & 0.00 & -0.08 & -0.05 & 0.03 & 0.05 & 0.35 & -0.13 \\
			& & -0.00 & 0.01 & -0.01 & 0.05 & 0.01 & 0.05 & 0.01 & -0.10 & -0.16 & -0.17 & -0.03 & -0.23 \\
			\bottomrule
		\end{tabular}
	\end{center}
	\label{table:quad_experiments_add_unseen_initial_states_states}
\end{table*}

\begin{table*}[h!]
	\caption{
		Target, initial, and final states for the quadcopter control experiments depicted in  \cref{fig:quad_experiments_dist}.
		Each state of the system $\xbf = (x, y, z, \phi, \theta, \psi, \dot{x}, \dot{y}, \dot{z}, \dot{\phi}, \dot{\theta}, \dot{\psi}) \in \R^{12}$ comprises the quadcopter's position $(x, y, z)$, tilt angles $(\phi, \theta, \psi)$ (\ie~roll, pitch, and yaw), and their respective velocities.
	}
	\vspace{-1mm}
	\begin{center}
		\fontsize{7.5}{9.5}\selectfont
		\begin{tabular}{c|lcccccccccccc}
			\toprule
			& & $x$ & $y$ & $z$ & $\phi$ & $\theta$ & $\psi$ & $\dot{x}$ & $\dot{y}$ & $\dot{z}$  & $\dot{\phi}$  &  $\dot{\theta}$ &  $\dot{\psi}$ \\
			\midrule
			& \includegraphics[scale=0.45]{icons/target_state} Target State & 0 & 0 & 1 & 0 & 0 & 0 & 0 & 0 & 0 & 0 & 0 & 0 \\
			\midrule
			\multirow{15}{*}{\rotatebox[origin=c]{90}{\includegraphics[scale=0.45]{icons/training_state} \, Training}} & \multirow{5}{*}{Initial States}  & 0.50 & 0.00 & 0.50 & 0.00 & 0.00 & 0.00 & 0.00 & 0.00 & 0.00 & 0.00 & 0.00 & 0.00 \\
			& & 0.00 & 0.50 & 0.50 & 0.00 & 0.00 & 0.00 & 0.00 & 0.00 & 0.00 & 0.00 & 0.00 & 0.00 \\
			& & -0.50 & 0.00 & 0.50 & 0.00 & 0.00 & 0.00 & 0.00 & 0.00 & 0.00 & 0.00 & 0.00 & 0.00 \\
			& & 0.00 & -0.50 & 0.50 & 0.00 & 0.00 & 0.00 & 0.00 & 0.00 & 0.00 & 0.00 & 0.00 & 0.00 \\
			& & 0.00 & 0.00 & 1.00 & 0.00 & 0.00 & 0.00 & 0.00 & 0.00 & 0.00 & 0.00 & 0.00 & 0.00 \\
			\cmidrule{2-14}
			& \multirow{5}{*}{\shortstack[l]{Final States for \\ Policy Gradient Controller}} & -0.05 & -0.02 & 1.00 & -0.02 & 0.25 & 0.01 & 0.03 & -0.06 & -0.11 & -0.02 & -0.05 & 0.00 \\
			& & -0.01 & -0.03 & 0.99 & -0.14 & 0.01 & -0.13 & -0.04 & 0.27 & -0.07 & 0.14 & 0.06 & 0.10 \\
			& & 0.04 & 0.00 & 1.01 & 0.00 & -0.25 & -0.03 & -0.09 & 0.04 & -0.01 & 0.12 & 0.12 & 0.14 \\
			& & -0.01 & 0.03 & 1.01 & 0.17 & 0.01 & 0.13 & 0.00 & -0.32 & 0.06 & -0.09 & -0.04 & -0.19 \\
			& & 0.00 & -0.00 & 1.00 & 0.01 & 0.00 & -0.01 & 0.01 & 0.00 & -0.01 & 0.01 & -0.01 & -0.01 \\
			\cmidrule{2-14}
			& \multirow{5}{*}{\shortstack[l]{Final States for \\ Non-Extrapolating Controller}} & -0.06 & -0.00 & 1.01 & -0.04 & 0.20 & 0.02 & -0.03 & -0.03 & -0.06 & -0.06 & -0.25 & -0.07 \\
			& & 0.01 & -0.04 & 1.01 & -0.13 & 0.02 & -0.09 & 0.02 & 0.19 & -0.08 & 0.31 & 0.04 & 0.21 \\
			& & 0.06 & -0.00 & 1.01 & 0.04 & -0.20 & -0.02 & -0.02 & 0.02 & 0.15 & 0.09 & 0.42 & 0.08 \\
			& & -0.01 & 0.04 & 1.00 & 0.15 & -0.01 & 0.09 & -0.05 & -0.24 & -0.28 & -0.12 & 0.05 & -0.46 \\
			& & -0.00 & -0.00 & 1.00 & 0.03 & -0.00 & -0.02 & -0.02 & -0.02 & 0.04 & -0.08 & 0.01 & -0.07 \\
			\midrule
			\multirow{15}{*}{\rotatebox[origin=c]{90}{\includegraphics[scale=0.45]{icons/unseen_state} \, Unseen}} & \multirow{5}{*}{Initial States} & 0.00 & 0.00 & 0.50 & 0.00 & 0.00 & 0.00 & 0.00 & 0.00 & 0.00 & 0.00 & 0.00 & 0.00 \\
			& & 0.25 & 0.00 & 0.50 & 0.00 & 0.00 & 0.00 & 0.00 & 0.00 & 0.00 & 0.00 & 0.00 & 0.00 \\
			& & 0.75 & 0.00 & 0.50 & 0.00 & 0.00 & 0.00 & 0.00 & 0.00 & 0.00 & 0.00 & 0.00 & 0.00 \\
			& & 1.00 & 0.00 & 0.50 & 0.00 & 0.00 & 0.00 & 0.00 & 0.00 & 0.00 & 0.00 & 0.00 & 0.00 \\
			& & 1.25 & 0.00 & 0.50 & 0.00 & 0.00 & 0.00 & 0.00 & 0.00 & 0.00 & 0.00 & 0.00 & 0.00 \\
			\cmidrule{2-14}
			& \multirow{5}{*}{\shortstack[l]{Final States for \\ Policy Gradient Controller}} & 0.15 & 0.01 & 1.09 & 0.38 & 0.90 & 0.83 & 2.30 & -2.22 & -1.55 & 1.57 & 1.17 & 0.80 \\
			& & 0.10 & -0.13 & 1.09 & -0.11 & 0.08 & -0.00 & 0.71 & -0.55 & 0.01 & -0.95 & -0.64 & -0.77 \\
			& & -0.10 & 0.05 & 0.99 & -0.10 & 0.54 & -0.09 & 0.38 & 0.77 & 0.25 & 0.69 & 0.10 & -0.62 \\
			& & -0.03 & 0.03 & 0.97 & -0.32 & 1.02 & -0.23 & 1.72 & 1.46 & -0.64 & -0.12 & 0.90 & -1.17 \\
			& & -0.11 & -0.06 & 0.98 & -0.41 & 1.20 & -0.12 & 0.92 & 1.11 & -1.42 & -0.53 & 2.14 & -1.16 \\
			\cmidrule{2-14}
			& \multirow{5}{*}{\shortstack[l]{Final States for \\ Non-Extrapolating Controller}} & -0.09 & -0.04 & 0.02 & -0.01 & 0.04 & -0.06 & -0.04 & 0.03 & 0.21 & -0.12 & 0.20 & -0.24 \\
			& & -0.13 & 0.02 & 0.00 & 0.01 & 0.18 & 0.02 & -0.20 & -0.05 & 0.20 & -0.00 & -0.12 & -0.04 \\
			& & -0.07 & 0.01 & -0.01 & -0.01 & 0.23 & 0.03 & -0.34 & 0.02 & -0.14 & 0.18 & -0.51 & -0.24 \\
			& & -0.11 & 0.00 & 0.01 & -0.05 & 0.30 & -0.01 & -0.18 & 0.08 & -0.11 & 0.15 & -0.81 & -0.10 \\
			& & -0.16 & 0.00 & 0.02 & -0.11 & 0.40 & 0.01 & -0.14 & 0.07 & -0.05 & 0.21 & -0.97 & -0.29 \\
			\bottomrule
		\end{tabular}
	\end{center}
	\label{table:quad_dist_experiments_states}
\end{table*}

\end{document}